\documentclass[twoside,11pt]{article}

\usepackage[preprint]{jmlr2e}

\usepackage[american]{babel}
\usepackage{amsmath}
\usepackage{amssymb}
\usepackage{mathtools}
\usepackage{hyperref}
\usepackage{url}
\usepackage{duckuments}

\usepackage{graphicx}
\usepackage{subcaption}
\usepackage[utf8]{inputenc} 
\usepackage[T1]{fontenc}    

\usepackage{booktabs}       
\usepackage{amsfonts}       
\usepackage{nicefrac}       
\usepackage{microtype}      
\usepackage{color}
\usepackage{lipsum}
\usepackage{booktabs} 
\usepackage{tabularx,longtable,multirow}
\usepackage{makecell}
\usepackage{duckuments}

\newcommand{\indep}{\perp \!\!\! \perp}
\newcommand{\A}{\mathcal{A}}
\newcommand{\R}{\mathbb{R}}
\newcommand{\N}{\mathcal{N}}

\newcommand{\x}{\bm{x}}
\newcommand{\Bb}{\bm{b}}

\newcommand{\y}{\mathbf{y}}
\newcommand{\z}{\bm{z}}

\newcommand{\jointfamily}{\mathcal{P}^{T,M}_{\A,\B}}
\newcommand{\msmfamily}{\mathcal{M}^T(\Pi_{\A}^M,\mathcal{P}_{\B}^M)}

\newcommand{\B}{\mathcal{B}}

\newcommand{\s}{\bm{s}}

\newcommand{\mparam}{\bm{\theta}}	
\newcommand{\vparam}{\bm{\phi}}	

\def\1{\bm{1}}
\usepackage{bm}
\usepackage{enumitem}
\usepackage{wrapfig}
\usepackage{xcolor}
\usepackage{natbib}
\usepackage[capitalize,noabbrev]{cleveref}

\newlist{thmlist}{enumerate}{1}
\setlist[thmlist]{label=(\roman{thmlisti}),
                  ref=\thetheorem.(\roman{thmlisti}),
                  noitemsep}

\usepackage{lastpage}
\jmlrheading{XX}{2026}{1-\pageref{LastPage}}{11/25; Revised x/xx}{x/xx}{xx-xxxx}{Carles Balsells-Rodas, Toshiko Matsui, Pedro A.M. Mediano, Yixin Wang and Yingzhen Li}

\ShortHeadings{Identifiability of Regime-Switching Models}{Balsells-Rodas, Matsui, Mediano, Wang and Li}
\firstpageno{1}

\usepackage{physics}
\usepackage{amsmath}
\usepackage{mathdots}
\usepackage{yhmath}
\usepackage{cancel}
\usepackage{color}
\usepackage{siunitx}
\usepackage{array}
\usepackage{multirow}
\usepackage{amssymb}
\usepackage{gensymb}
\usepackage{tabularx}
\usepackage{extarrows}
\usepackage{booktabs}
\usepackage{tikz}
\usepackage{pgfplots}
\pgfplotsset{compat=1.18}
\usepgfplotslibrary{colormaps}
\usetikzlibrary{backgrounds,patterns,patterns.meta}

\usetikzlibrary{fadings}
\usetikzlibrary{shadows.blur,shapes,shapes.callouts,calc,arrows.meta,positioning,decorations.pathreplacing,fit}
\newcommand{\tick}{\textcolor{red!70!black}{\checkmark}}
\newcommand{\SoftBlob}[4][]{%
  \draw[
    line width=1.5pt,
    draw=#2,
    fill=#3
  ] plot[smooth cycle, tension=0.85] coordinates {#4};
}

\newcommand{\SoftBlobPath}[1]{%
  plot[smooth cycle, tension=0.85] coordinates {#1}
}

\begin{document}

\title{On the Identifiability of Regime-Switching Models \\ with Multi-Lag Dependencies}

\author{\name Carles Balsells-Rodas \email cb221@imperial.ac.uk \\
       \addr Department of Computing, Imperial College London, London, UK
       \AND
       \name Toshiko Matsui \email t.matsui19@imperial.ac.uk \\
       \addr Department of Computing, Imperial College London, London, UK
       \AND
       \name Pedro A.M. Mediano \email p.mediano@imperial.ac.uk \\
       \addr Department of Computing, Imperial College London, London, UK
       \AND
       \name Yixin Wang \email yixinw@umich.edu \\
       \addr Department of Statistics, University of Michigan, Ann Arbor, MI 48109, USA
       \AND
       \name Yingzhen Li \email yingzhen.li@imperial.ac.uk \\
       \addr Department of Computing, Imperial College London, London, UK}

\editor{My editor}

\maketitle

\begin{abstract}
Identifiability is central to the interpretability of deep latent variable models, ensuring parameterisations are uniquely determined by the data-generating distribution. However, it remains underexplored for deep regime-switching time series. We develop a general theoretical framework for multi-lag Regime-Switching Models (RSMs), encompassing Markov Switching Models (MSMs) and Switching Dynamical Systems (SDSs). For MSMs, we formulate the model as a temporally structured finite mixture and prove identifiability of both the number of regimes and the multi-lag transitions in a nonlinear-Gaussian setting. For SDSs, we establish identifiability of the latent variables up to permutation and scaling via temporal structure, which in turn yields conditions for identifiability of regime-dependent latent causal graphs (up to regime/node permutations). Our results hold in a fully unsupervised setting through architectural and noise assumptions that are directly enforceable via neural network design. We complement the theory with a flexible variational estimator that satisfies the assumptions and validate the results on synthetic benchmarks. Across real-world datasets from neuroscience, finance, and climate, identifiability leads to more trustworthy interpretability analysis, which is crucial for scientific discovery.
\end{abstract}


\section{Introduction}

Identifiability in deep latent variable models has been a long-standing challenge over the recent years. It formalises when parameters are uniquely determined by the data-generating distribution. In practice, exact parameter recovery is often impossible, and instead, one assumes equivalence classes which are determined by the target of interest. For example, Causal Representation Learning (CRL; \citep{scholkopf2021toward}), targets the recovery of the latent causal graph, typically up to permutation of nodes. In Independent Component Analysis (ICA; \citep{comon1994independent}), the equivalence class targets independent sources up to permutations and scaling, and often the functional form of the mixing function. 

General, nonlinear ICA is ill-posed without additional structure \citep{hyvarinen1999nonlinear}. Recent advances establish identifiability by combining deep generative models with auxiliary information \citep{khemakhem2020variational}, or finite mixture models with piecewise linear mixing \citep{kivva2022identifiability}. Other advances exploit temporal structure \citep{hyvarinen2017nonlinear,hyvarinen2019nonlinear,halva2020hidden,halva2021disentangling}, or latent causal processes \citep{yao2022temporally,yao2022learning,lippe2023causal,lippe2023biscuit}. 

This paper focuses on identifiability for \emph{Regime-Switching Models} (RSMs), including Markov Switching Models (MSMs) and Switching Dynamical Systems (SDSs), from the perspective of deep sequential generative models. RSMs describe time series whose behaviour depends on a discrete latent \emph{regime} (state, or switch) that changes over time. They form a subset of \emph{State-Space Models} (SSMs), a well-established class of sequence models \citep{lindgren1978markov,poritz1982linear, hamilton1989new}. With the increasing popularity of sequential latent variable models \citep{chung2015recurrent,Li2018DisentangledSA,babaeizadeh2018stochastic,Saxena2021ClockworkVA}, SSMs have been incorporated as latent dynamical priors \citep{johnson2016composing,linderman2016recurrent,linderman2017bayesian,fraccaro2017disentangled,dong2020collapsed,ansari2021deep,liu2023graph, geadah2024parsing}. However, these works mainly focus on designing flexible priors and stable training schemes \citep{dong2020collapsed,ansari2021deep,liu2023graph,zhao2023revisiting}. 

To address these gaps, we develop a unified theoretical framework that establishes identifiability for RSMs with multi-lag dependencies. First, we show that multi-lag MSMs can be framed as finite mixture models with temporally-structured components, and prove identifiability of the number of regimes and the multi-lag transition functions in a nonlinear Gaussian setting. The analysis extends classical finite-mixture identifiability \citep{yakowitz1968identifiability} to the multi-lag case, which is nontrivial due to lagged dependencies. Second, for SDSs, we prove identifiability of the latent variables up to permutation and scaling by extending \citet{kivva2022identifiability} to the temporal setting. This yields identifiability of the regime-dependent latent causal graphs up to permutations of regimes and nodes. Our results hold in a fully unsupervised setting by imposing restrictions on the latent noise distributions and the functional class of the generative model. In contrast, other works that handle regime-switching require auxiliary information \citep{yao2022temporally,yao2022learning,lippe2023biscuit} or stationarity \citep{hyvarinen2017nonlinear} in the temporal correlations \citep{halva2021disentangling}. Recently, similar results with flexible assumptions have been established \citep{song2023temporally, song2024causal,balsellsrodas2025causal}.

The rest of the paper is organised as follows.\footnote{This paper extends the conference version \citet{balsells-rodas2024on} significantly by: (i) providing a more intuitive view of Markov Switching Models as finite mixtures and extending identifiability to higher temporal dependencies (with substantial new proofs and intuitions in the main text), (ii) going beyond affine identifiability of the latent variables for SDSs to establish permutation/scaling identifiability and causal structure identifiability up to node/regime permutation, and (iii) adding new real-world applications.} Section \ref{sec:related_work} discusses related work and positions our contributions. Section \ref{sec:background} reviews finite mixtures and latent variable models. Section \ref{sec:theory}, presents our main results: (i) parametric MSMs with nonlinear Gaussian transitions (nonparametric analysis is deferred to Appendix \ref{app:proof_identifiability}), (ii) identifiability of SDS latents up to permutation and scaling, and (iii) identifiability of regime-dependent causal graphs up to regime/node permutation. Section \ref{sec:estimation} presents a flexible estimation approach which follows the assumptions for identifiability, and builds on \citet{dong2020collapsed,ansari2021deep}. Section \ref{sec:experiments_msm} validates our theoretical framework on synthetic data with baselines and ablations. Finally, Section \ref{sec:real_world}, presents results on brain activity, financial time series, and climate data, illustrating interpretability gains enabled by identifiability.

\paragraph{Why identifiability matters in practice.} Applications across neuroscience \citep{karniol-tambour2024modeling}, climate \citep{saggioro2020reconstructing}, human motion \citep{dong2020collapsed,liu2023graph}, and animal behaviour \citep{ansari2021deep} demonstrate the promise of deep regime-switching models. However, without identifiability, interpretations of the learned latents (e.g., regimes, causal structures) may be spurious, as distinct or even contradictory parametrisations can achieve the same likelihood, thus invalidating scientific discovery. More broadly, identifiability prevents arbitrary latent reparametrisations \citep{locatello2019challenging} and supports valid causal claims \citep{scholkopf2021toward}. Our theoretical framework provides flexible conditions under which learned regimes, latents, and graphs are uniquely determined up to acceptable equivalences, and therefore supporting trustworthy interpretations.
\section{Related Work}\label{sec:related_work}

MSMs (or autoregressive hidden Markov models) were first introduced for applications to speech analysis \citep{poritz1982linear,ephraim2005revisiting}. In econometrics, \citet{hamilton1989new} popularised MSMs for business cycles, with numerous extensions such as MS-GARCH \citep{haas2004new}. For overviews of asymptotic properties and additional applications, see \citet{fruhwirth2006finite}. 
SDSs extend MSMs with continuous latent states and regime-dependent dynamics. Classic developments include switching Kalman filters \citep{murphy1998switching} and variational learning with applications to respiration data \citep{ghahramani2000variational}. More recently, regime switching has been used as a prior within deep latent-variable models, including SVAEs \citep{johnson2016composing, zhao2023revisiting}, soft-switching Kalman Filters \citep{fraccaro2017disentangled}, and recurrent SDSs \citep{linderman2016recurrent,dong2020collapsed} with explicit duration models \citep{ansari2021deep} or graph-based interactions \citep{liu2023graph}. 

Finite-state HMMs enjoy strong identifiability under mild conditions \citep{allman2009identifiability,gassiat2016inference,gassiat2020identifiability}. However, these strategies do not translate directly to MSMs due to explicit autoregression in continuous variables. While results exist for certain discrete MSMs \citep{an2013identifiability}, a general nonparametric theory for multi-lag MSMs has been missing. Our MSM identifiability results address this gap by extending classical finite-mixture identifiability \citep{yakowitz1968identifiability} to the multi-lag setting.

Modern identifiability for deep generative models often builds on nonlinear ICA with auxiliary variables or temporal structure. iVAE proves identifiability with conditional priors from auxiliary information \citep{khemakhem2020variational}; \citet{kivva2022identifiability} remove this requirement by using mixture priors with piecewise linear mixing. Temporal information can also act as auxiliary information, yielding identifiability via contrastive learning \citep{hyvarinen2016unsupervised,hyvarinen2019nonlinear}. In CRL, temporal identifiability has been established under distribution shifts or latent temporal structure \citep{yao2022learning,yao2022temporally,lippe2023causal,lippe2023biscuit}.

Several works introduce latent regime structure to relax stationarity or the need for auxiliary information. HM-NICA replaces observed regimes with an identifiable stationary HMM prior \citep{gassiat2016inference}; IIA-HMM \citep{morioka2021independent} introduces a recurrent mixing with nonstationary innovations through a stationary HMM \citep{gassiat2016inference}. SNICA \citep{halva2021disentangling} considers linear SDSs with weak temporal dependencies (weak stationarity \citep{hyvarinen2017nonlinear}), but does not identify the number of switches or latent transitions. NCTRL \citep{song2023temporally} establishes results for SDSs. The strategy follows firs-order MSM results presented in \citet{balsells-rodas2024on}, with Gaussianity imposed at the observation level. CtrlNS \citep{song2024causal} achieves identifiability via mechanism sparsity \citep{lachapelle2022disentanglement} and direct observation-regime dependence, assuming a known number of regimes. Very recent works explicitly target regime-dependent causal discovery: CASTOR \citep{rahmani2025causal} assumes contemporaneous effects and SDCI \citep{balsellsrodas2025causal} allows flexible state dependencies; both rely on differing causal structures across regimes. Our results are complementary to these methodologies, as multi-lag MSMs can be viewed as a generalisation directly targeting regime-dependent transitions.

\paragraph{Positioning.} Compared with prior identifiability results for regime-switching and temporal CRL/ICA, our framework offers the following key advantages:
\begin{itemize}[leftmargin=1.5em]
  \setlength{\itemsep}{1pt}
  \setlength{\parskip}{0pt}
  \setlength{\parsep}{0pt}
    \item \textbf{Multi-Lag MSM theory.} To our knowledge, we provide the first general identifiability results for multi-lag MSMs. Our results also generalise beyond standard MSMs (Remark~\ref{remark:theory_scope}). 
    \item \textbf{Assumptions at the function level.} Related works on temporal CRL \citep{yao2022learning, yao2022temporally,song2023temporally,song2024causal} and ICA \citep{morioka2021independent,halva2021disentangling} often rely on distributional or tail assumptions that are difficult to verify when designing models. Instead, we directly impose functional constraints that can be enforced via the neural parametrisation, making model specification more controllable.
    \item \textbf{Unknown number of regimes.} Methods such as HMM-ICA \citep{halva2020hidden} and CtrlNS \citep{song2024causal} assume a known number of regimes. By extending classical finite-mixture theory \citep{yakowitz1968identifiability}, our results do not require knowing the number of regimes. This matters in practice, as we can determine them by model selection \citep{mclachlan2000finite}, which is not possible when the number of regimes is assumed known.
\end{itemize}
\section{Background}\label{sec:background}

\subsection{Finite Mixture Models}\label{sec:finite_mixture}

Our theoretical framework presented in Section \ref{sec:msm_theory} extends finite mixture model results from \citet{yakowitz1968identifiability}, which show identifiability through linear independence of the mixing components. Consider a distribution family with functions defined on $\z\in\R^m$,
\begin{equation}\label{eq:family_functions}
    \mathcal{F}_{\A} := \{F_{a}(\z) \mid a \in \A \},
\end{equation}
where $F_{a}(\z)$ is an $m$-dimensional CDF. The index set $\A$ is assumed to satisfy that $F_{a}(\z)$, as a function of $(\z,a)$, is measurable on $\R^m\times\A$. We introduce the notion of linear independence under finite mixtures of a family $\mathcal{F}_{\A}$.
\begin{definition}
\label{def:linear_independence_finite_mixture}
A family of functions $\mathcal{F}_{\A}$ (Eq.~\eqref{eq:family_functions}) is said to contain linearly independent functions under finite mixtures if for any $\A_0 \subset \A$ such that $|\A_0| < +\infty$, the functions in $\{F_a(\z) \mid a \in \A_0 \}$ are linearly independent.
\end{definition}
The above definition is a weaker requirement of linear independence on function families as it allows linear dependence from the linear combination of infinitely many other functions. Consider the following finite mixture distribution family:
\begin{equation}
    \mathcal{H}_{\A} := \left\{ H(\z) = \sum_{i=1}^K c_i F_{a_i}(\z) \mid c_i> 0,\  a_i \in \A,\  a_i \neq a_j, \forall i \neq j,\  \sum_{i=1}^K c_i = 1,\ K < +\infty \right\},
\end{equation}
which is defined from the \emph{convex hull} of $\mathcal{F}_\A$ ($\mathrm{conv}(\mathcal{F}_\A)$), or a linear combination of CDFs in $\mathcal{F}_{\A}$, with positive coefficients. Now we define the \emph{identifiable finite mixture family} following \citet{yakowitz1968identifiability}.
\begin{definition}
\label{def:identifiability_finite_mixture}
The finite mixture family $\mathcal{H}_{\A}$ is said to be identifiable up to permutations, when for any two finite mixtures $H(\z) = \sum_{i=1}^K c_i F_{a_i}(\z)$ and $\tilde{H}(\z) = \sum_{i=1}^{\tilde{K}} \tilde{c}_i F_{\tilde{a}_i}(\z)$, $H(\z) = \tilde{H}(\z)$ for all $\z \in \mathbb{R}^m$, if and only if $K = \tilde{K}$ and for each $1 \leq i \leq K$ there is some $1 \leq j \leq \tilde{K}$ such that $c_i = \tilde{c}_j$ and $F_{a_i}(\z) = F_{\tilde{a}_j}(\z)$ for all $\z \in \mathbb{R}^m$.
\end{definition}
Then, identifiability of finite mixture models is stated as follows.
\begin{proposition}\citep{yakowitz1968identifiability}
\label{prop:mixture_cdf_identifiability}
The finite mixture distribution family $\mathcal{H}_\A$ is identifiable up to permutations, if and only if functions in $\mathcal{F}_\A$ are linearly independent under the finite mixtures. 
\end{proposition}
For our purposes in this work, the above result can be generalised to PDFs defined on $\z\in\R^{m}$. Consider an equivalent family of PDFs:
\begin{equation}\label{eq:pdf_distrib_family}
    \tilde{\mathcal{F}}_\A = \left\{ f_a(\z) = \frac{\partial^m}{\partial z_1\dots\partial z_m} F_a(\z) \mid  a\in\A \right\}.
\end{equation} 
Then, under the assumption that all $F_a$ are absolutely continuous and differentiable almost everywhere, the family $\mathcal{F}_\A$ is linearly independent under finite mixtures if and only if the corresponding family $\tilde{\mathcal{F}}_\A$ is linearly independent under finite mixtures.
\begin{proposition}
\label{prop:mixture_pdf_identifiability} Consider the family $\tilde{\mathcal{F}}_\A$ in Eq.~\eqref{eq:pdf_distrib_family}.
Then the finite mixture distribution family is identifiable up to permutations if and only if functions in $\tilde{\mathcal{F}}_\A$ are linearly independent under finite mixtures.
\end{proposition}

\subsection{Identifiable Latent Variable Models}\label{sec:kivva_background}

In the non-temporal case, many works explore identifiability of latent variable models \citep{khemakhem2020variational,kivva2022identifiability}. Specifically, consider a generative model where its latent variables $\z\in\R^m$ are drawn from a Gaussian mixture prior with $K$ components ($K < +\infty$). Then $\z$ is transformed via a (noisy) piecewise linear mapping to obtain the observation $\x\in\R^n$, $n \geq m$:
\begin{equation}
    \x = f(\z) + \bm{\epsilon},\quad \bm{\epsilon}\sim \mathcal{N}(\bm{0},\bm{\Sigma}), \quad
    \z\sim p(\z) :=\sum_{k=1}^K p(s=k) \mathcal{N}\left(\z\mid\bm{\mu}_k,\bm{\Sigma}_k\right).
\end{equation}
\citet{kivva2022identifiability} established that if the transformation $f$ is \emph{weakly injective} (see definition below), the prior distribution $p(\z)$ is identifiable up to affine transformations\footnote{See Def. 2.2 in \citet{kivva2022identifiability}, or the equivalent adapted to SDSs (Def. \ref{def:identifiability_affine_transformations}).} from the observations. 
If we further assume $f$ is continuous and \emph{injective}, both prior distribution $p(\z)$ and nonlinear mapping $f$ are identifiable up to affine transformations. In Section \ref{sec:switching_dynamical_systems} we extend these results to establish identifiability for SDSs using similar proof strategies.
\paragraph{Notation.} For $(\x_1^*, \dots, \x_M^*)\in\R^{nM}$ and $M<+\infty$, we define $B(\x_1^*, \dots, \x_M^*,\delta) := \{(\x_1, \dots, \x_M)\in\R^{nM}: ||(\x_1, \dots, \x_M) - (\x_1^*, \dots, \x_M^*) || < \delta\}$, where $||\cdot||$ denotes the Euclidean norm on $\R^{nM}$. Furthermore, we extend function notation to sets, i.e., for a set $\mathcal{S}$, $f(\mathcal{S}) := \{f(\x):\x\in\mathcal{S}\}$ and $f^{-1}(\mathcal{S}) := \{\x: f(\x)\in\mathcal{S}\}$.
\begin{definition}
\label{def:weakly_injective} A mapping $f : \mathbb{R}^m \rightarrow \mathbb{R}^n$ is said to be weakly injective if:
\begin{enumerate}[leftmargin=1.75em]
    \item[(i)] there exists $\x_0\in\R^{n}$ and $\delta > 0$ s.t. $|f^{-1}(\{\x\})|=1$ for every $\x\in B(\x_0,\delta)\cap f(\R^m)$, and 
    \item [(ii)] $\{\x\in\R^n:|f^{-1}(\{\x\})|>1\}\subseteq f(\R^m)$ has measure zero with respect to the Lebesgue measure on $f(\R^m)$.
\end{enumerate} 
\end{definition}
\begin{definition}
\label{def:injective}
$f$ is injective if $|f^{-1}(\{\x\})|=1, \forall \x\in f(\R^m)$.
\end{definition}

\subsection{Switching Dynamical Systems}\label{sec:back_sds}

\begin{figure}
  \begin{minipage}{0.65\textwidth}
\includegraphics[width=\linewidth]{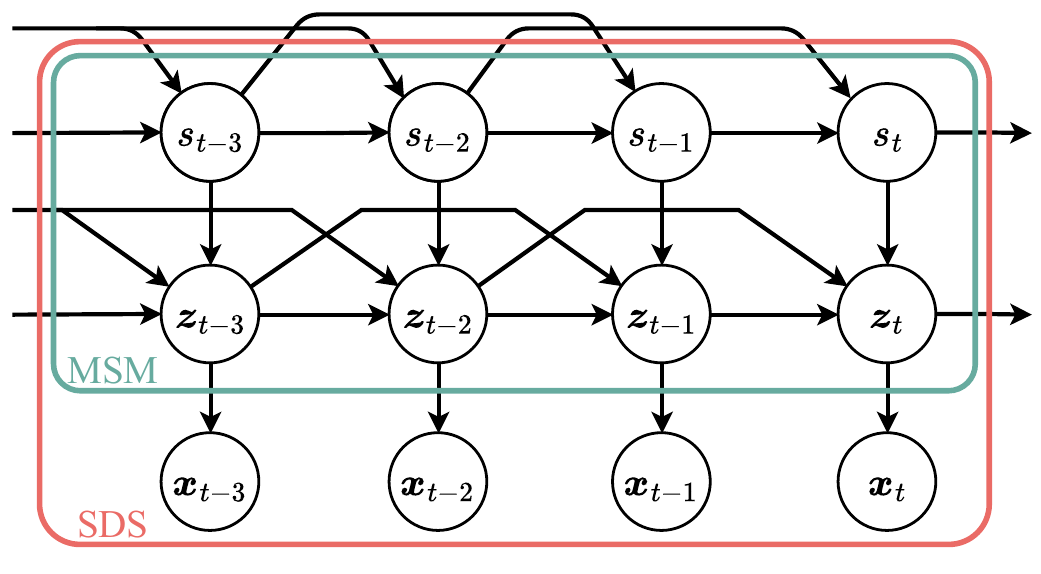}
  \end{minipage}
  \hfill
  \begin{minipage}{0.3\textwidth}
  \caption{The generative models considered in this work. MSM (green) treats $\z_t$ directly as observations, while SDS (red) transforms $\z_t$ into observed $\x_t$.}
    \label{fig:msm_snlds}
  \end{minipage}
\end{figure}

A Switching Dynamical System (SDS), with an example graphical model illustrated in Figure~\ref{fig:msm_snlds}, is a sequential latent variable model with its dynamics governed by both discrete and continuous latent states, $s_t\in\{1,\dots,K\},\ \z_t\in\R^m$, respectively. At each time step $t$, the discrete state $s_t$ determines the regime of the dynamical prior that the current continuous latent variable $\z_t$ should follow, and the observation $\x_t$ is generated from $\z_t$ using a (noisy) nonlinear transformation. This gives the following probabilistic model for the states and the observation variables:
\begin{equation}\label{eq:sds}
p_{\mparam}(\x_{1:T},\z_{1:T},\s_{1:T}) = p_{\mparam}(\s_{1:T}) p_{\mparam}(\z_{1:T} \mid \s_{1:T}) \prod_{t=1}^T  p_{\mparam}(\x_{t} \mid \z_{t}),
\end{equation}
where we use $\mparam$ to denote the model parameters. As $p_{\mparam}(\x_{1:T})$ is intractable, recent works \citep{dong2020collapsed,ansari2021deep} have developed inference techniques based on variational inference \citep{kingma2013autoencoding}. These works consider an additional dependency on the switch $s_t$ from $\x_t$ for more expressive segmentations, which is not considered in our theoretical analysis for simplicity.

The latent dynamic prior $p_{\mparam}(\z_{1:T})$ considers a Markov Switching Model (MSM) which has also been referred to as autoregressive HMM \citep{hamilton1989new,ephraim2005revisiting}. This type of prior uses the latent ``switch'' $s_t$ to condition the distribution of $\z_t$ at each time-step, and the conditional dynamic model of $\z_{1:T}$ given $\s_{1:T}$ follows an autoregressive process, where the likelihood can be computed by marginalising $\s_{1:T}$. For simplicity, we consider $\s_{1:M}:= s_M$ and $p_{\mparam}(\z_{1}, \dots, \z_M \mid s_M)$ as the distribution of the initial $M$ observations, and the MSM with lag $M$ factorises as follows:
\begin{multline}\label{eq:msm-likelihood}
    p_{\mparam}(\z_{1:T}) = \sum_{\s_{M:T}} p_{\mparam}(\s_{M:T}) p_{\mparam}(\z_{1:T} \mid \s_{M:T}), \\ p_{\mparam}(\z_{1:T} \mid \s_{M:T}) = p_{\mparam}(\z_{1}, \dots, \z_M \mid s_M)\prod_{t=M+1}^T p_{\mparam}(\z_t \mid \z_{t-1}, \dots, \z_{t-M}, s_t).
\end{multline}
Note that the structure of the discrete latent state prior $p_{\mparam}(\s_{M:T})$ is not specified, and the identifiability results presented in the next section do not require further assumptions herein. In experiments, we use a first-order Markov process for $p_{\mparam}(\s_{M:T})$, described by a transition matrix $Q\in\R^{K\times K}$ such that $p_{\mparam}(s_t=j\mid s_{t-1}=i) = Q_{ij}$, and an initial distribution $p_{\mparam}(s_M=i)=\pi_i$.
\section{Identifiability Analysis}\label{sec:theory}

In this section, we establish the identifiability of the SDS (Eq.~\eqref{eq:sds}) presented in Section \ref{sec:back_sds} under suitable assumptions. We start by extending finite mixture model identifiability results introduced in Section \ref{sec:finite_mixture} to Markov switching models with finite lag $M$. Then, we leverage ideas from \citet{kivva2022identifiability} which use finite mixture prior for non-temporal generative models; importantly, this theory relies on the use of identifiable finite mixture priors (up to mixture component permutations). 

\paragraph{Identifiability:} Formally, identifiability establishes an injective correspondence between probability distributions $p_{\mparam}(\x)$ and model parameters $\mparam$:
\begin{equation}
    p_{\mparam}(\x) = p_{\tilde{\mparam}}(\x) \iff \mparam = \tilde{\mparam}.
\end{equation}
Given that our models are parametrised by deep neural networks, explicit identifiability of parameters is not possible due to network symmetries (i.e., different parameters map to the same function). Instead, in this paper we refer to identifiability in terms of their functional form, and model mixture components.

\subsection{Identifiable High-Order Markov Switching Models}\label{sec:msm_theory}

\paragraph{Notation.}
We use the following notation for repeated Cartesian products of sets and pointwise products of function families:
\begin{equation}
    \times^N \mathcal{A} := \underbrace{\mathcal{A} \times \dots \times \mathcal{A}}_{N \text{ times}}, \quad 
    \otimes^N \mathcal{P} := \underbrace{\mathcal{P} \otimes \dots \otimes \mathcal{P}}_{N \text{ times}},
\end{equation}
where $\mathcal{A}$ is an index set and $\mathcal{P}$ is a family of functions. The functions in $\mathcal{P}$ may share overlapping variables, resulting in functions defined in the joint domain. For example, consider two families 
\begin{equation}
    \mathcal{P}_{1,\A_1} = \{ p_{1,a_1}(\z_1, \z_2) : a_1 \in \A_1 \}, \quad
\mathcal{P}_{2,\A_2} = \{ p_{2,a_2}(\z_2, \z_3) : a_2 \in \A_2 \},
\end{equation}
with $(\z_1, \z_2), (\z_2, \z_3)\in\R^{2m}$. Then, their pointwise product is defined as:
\begin{equation}
    \mathcal{P}_{1,\A_1} \otimes \mathcal{P}_{2,\A_2} := \Big\{ p_{1,a_1}(\z_1, \z_2)p_{2,a_2}(\z_2, \z_3) \mid (a_1, a_2)\in \A_1\times \A_2 \Big\},
\end{equation}
where functions are defined in $(\z_1, \z_2,\z_3)\in\R^{3m}$. We also write $[K]:=\{1,\dots,K\}$ for index sets from $1$ to $K$.
\paragraph{Finite mixture model view of MSMs.} For simplicity, we drop the subscript $\mparam$.
The MSM $p(\z_{1:T})$ has an equivalent formulation as a finite mixture model, where from Eq.~\eqref{eq:msm-likelihood}, the joint distribution given a sequence of switches $p(\z_{1:T}\mid\s_{M:T})$ is defined from the initial and transition distributions with lag $M$: $p(\z_1, \dots, \z_M \mid s_M)$ and $p(\z_t \mid\z_{t-1}, \dots, \z_{t-M}, s_t)$ for $t>M$ respectively.
Equivalently as in \citet{yakowitz1968identifiability}, we first define function families corresponding to initial and transition distributions. 
\begin{multline}
\Pi_{\A}^M:= \Big\{p_a(\z_{1}, \dots, \z_M) \mid a\in \A \Big\},\quad
\mathcal{P}_{\B}^M := \Big\{ p_b(\z_t \mid \z_{t-1}, \dots,\z_{t-M}) \mid b \in \B \Big\},
\end{multline}
where $\A$ and $\B$ are index sets satisfying mild measure theoretic conditions (similarly as in Section \ref{sec:finite_mixture}). To exemplify, we can consider these families to be infinitely many parametrised functions where each $a\in\A$ corresponds to a different parameter configuration $\mparam_a$. We use different calligraphic symbols (e.g., $\A$, $\B$) because they index functions with different structures and output domains. For example, if the initial distribution family is a Gaussian family, the set $\A$ indexes elements in a finite-dimensional parameter space (means and covariances). In contrast, the set $\B$ used in the transition family indexes elements in a function space, the output of which defines the parameters of a conditional distribution. We provide details of parametrisations below.

We assume a general setting and consider the MSM initial distribution has $K_0$ regimes, where the transition distribution has $K$ regimes. Therefore, the MSM discrete regimes satisfy $s_t \in [K]$ with $K < +\infty$, and $M < t \leq T$, with an initial regime $s_M\in[K_0], \ K_0<+\infty$. Given the above formulation, a Markov switching model considers $K_0$ initial distributions $\Pi_{\A_0}^M$ and $K$ transition distributions $\mathcal{P}_{\B_0}^M$, where $|\A_0| = K_0$,  $|\B_0| = K$, and $ \A_0 \subset \A$, $\B_0 \subset \B$ respectively. Therefore, under these function families we have for each $k\in[K_0]$: $p(\z_{1:M}\mid s_M=k) =p_a(\z_{1:M})$, for some $a\in\A_0$; and for each $k\in[K]$: $p(\z_t\mid \z_{t-1}, \dots, \z_{t-M}, s_t=k)=p_{b_t}(\z_t\mid \z_{t-1:t-M})$ for some $b_t\in\B_0$ with $M < t \leq T$. In other words, the families $\Pi_{\A}^M$ and $\mathcal{P}_{\B}^M$ contain infinitely many functions, and each of the distributions in the MSM aligns with one of such functions. 

We can view the joint distribution $p(\z_{1:T})$ as a mixture of paths, where each path is defined by a realisation of $\s_{M:T}\in [K_0]\times[K]^{T-M}$. Therefore, one can define a bijective \emph{path indexing function} $\varphi: [C]\rightarrow [K_0]\times[K]^{T-M}$, with a total number of mixture paths $C = K_0K^{T-M}$, such that each $i \in [C]$ uniquely retrieves a set of states $\s_{M:T} = \varphi(i)$. Then we can use $c_i=p(\s_{M:T}=\varphi(i))$ to represent the joint probability of the states $\s_{M:T}=\varphi(i)$ under $p$. Then for any $M,T$ such that $M < T < +\infty$, we can construct the family of $M$-order MSMs in finite mixture model sense: 
\begin{multline}\label{eq:msm_finite_mixture}
    \msmfamily := \bigg\{ \sum_{i=1}^{C} c_i p_{a^i}(\z_{1:M})\prod_{t=M+1}^T p_{b^i_t}(\z_t \mid \z_{t-1},\dots,\z_{t-M}) \mid c_i > 0, \\ p_{a^i}\in\Pi^{M}_{\A_0},\quad p_{b^i_t}\in\mathcal{P}^{M}_{\B_0},\quad M< t\leq T,\quad  a^i\in \A_0\subset\A ,\quad b^i_t \in \B_0\subset\B,\quad K_0,K < +\infty, \\  |\A_0|=K_0,\quad |\B_0|=K  , \quad (a^i,b^i_{M+1:T}) \neq (a^j,b^j_{M+1:T}), \forall i\neq j,\quad \sum_{i=1}^C c_i = 1\bigg\}.
\end{multline}
Since Eq.~(\ref{eq:msm_finite_mixture}) requires $(a^i,b^i_{M+1:T}) \neq (a^j,b^j_{M+1:T})$ for any $i\neq j$, this builds an injective mapping $\phi : [C]\rightarrow \A_0\times(\times^{T-M}\B_0)$, such that $\phi(i) = (a^i,b^i_{M+1}, \dots, b^i_{T})$. Combined with the path indexing function, we can use $\phi \circ \varphi^{-1}$ to uniquely map a set of states $\s_{M:T}$ to the indices. 

The above notation illustrates how the MSM extends finite mixture models to temporal settings representing $p(\z_{1:T})$ as a finite mixture of $K_0K^{T-M}$ trajectories, each composed of an initial 
distribution from $\Pi_{\A_0}^M$ and conditional distributions from $\mathcal{P}^M_{\B_0}$. However, this notion differs fundamentally from the i.i.d. mixture model setting in \citet{yakowitz1968identifiability}. To formalise this difference, we define the following \emph{trajectory family}:
\begin{multline}
\mathcal{P}_{\A,\B}^{T,M}:=\Pi_{\A}^M\otimes(\otimes^{T-M} \mathcal{P}_{\B}^M) = \Bigg\{p_{a}(\z_{1:M})\prod_{t=M+1}^T p_{b_t}(\z_t \mid \z_{t-1},\dots,\z_{t-M}) \mid  \\ a \in\A, \ b_t \in \B, \ M < t\leq T\Bigg\}.
\end{multline}
In Section \ref{sec:finite_mixture}, $\mathcal{H}_{\A}$ is defined as the \emph{convex hull} $\mathrm{conv}(\mathcal{F}_\A)$. In contrast, the MSM family defined in Eq.~(\ref{eq:msm_finite_mixture}) forms a strict subset of $\mathrm{conv}(\mathcal{P}_{\A,\B}^{T,M})$, since the mixture components are restricted to finite subsets $\A_0\subset \A$ and $\B_0 \subset \B$. On the other hand, mixtures over $\mathcal{P}_{\A,\B}^{T,M}$ allow arbitrary selections from $\B$ at each time step, yielding a more expressive model family. This generalisation is crucial for our theoretical analysis below, where we prove identifiability of MSM by establishing linear independence under finite mixtures of $\mathcal{P}_{\A,\B}^{T,M}$.

\paragraph{Identifiable Markov switching models.}

Having established the finite mixture model view of Markov switching models, we will use this notation of MSM in the rest of Section \ref{sec:msm_theory}, as we will use finite mixture modelling techniques to establish its identifiability. 
In detail, we first define the identification of $\mathcal{M}^T(\Pi_{\A}^M,\mathcal{P}_{\B}^M)$ as follows. 
\begin{definition}\label{def:identifiability}
We say the family of Markov switching models with order $M$ and length $T$, where $M,T\in\mathbb{Z}^+$ and $M< T$, is \emph{identifiable up to permutations}, if for any $p,\tilde{p}\in \mathcal{M}^T(\Pi_{\A}^M, \mathcal{P}_{\B}^M)$ where  
\begin{itemize}[leftmargin=0cm]
  \setlength{\itemsep}{1pt}
  \setlength{\parskip}{0pt}
  \setlength{\parsep}{0pt}
    \item[] $\quad p(\z_{1:T}) = \sum_{i=1}^{C} c_i p_{a^i}(\z_{1:M}) \prod_{t=M+1}^T p_{b^i_t}(\z_t \mid \z_{t-1},\dots,\z_{t-M}), \quad C=K_0K^{T-M}; \text{ and} $
    \item[] $\quad\tilde{p}(\z_{1:T}) = \sum_{i=1}^{\tilde{C}} \tilde{c}_i p_{\tilde{a}^i}(\z_{1:M}) \prod_{t=M+1}^T p_{\tilde{b}^i_t}(\z_t \mid \z_{t-1},\dots,\z_{t-M}), \quad \tilde{C}=\tilde{K}_0\tilde{K}^{T-M},$
\end{itemize}
with $p(\z_{1:T})=\tilde{p}(\z_{1:T}) \ \forall \z_{1:T}\in\R^{mT}$, we have $C=\tilde{C},\ K_0=\tilde{K}_0,\ K=\tilde{K}$,
and for each $1\leq i \leq C$ there is some $1 \leq j \leq \tilde{C}$ such that:
\begin{enumerate}
    \item $c_i=\tilde{c}_j$;
    \item if $b^i_{t_1}=b^i_{t_2}$ for $M <t_1,t_2 \leq T$,\  $t_1\neq t_2$, then $\tilde{b}^j_{t_1}=\tilde{b}^j_{t_2}$;
    \item $p_{a^i}(\z_{1:M})=p_{\tilde{a}^j}(\z_{1:M}), \  \forall\z_{1:M} \in \R^{mM}$;
    \item $p_{b^i_t}(\z_t \mid \z_{t-1},\dots,\z_{t-M})=p_{\tilde{b}^j_t}(\z_t \mid \z_{t-1},\dots,\z_{t-M})$,\ $\forall(\z_t,\dots, \z_{t-M}) \in \R^{m(M+1)}$.
\end{enumerate}
\end{definition}
The second requirement eliminates the permutation equivalence of e.g., $\s_{1:4} = (1, 2, 3, 2)$; $\tilde{\s}_{1:4} = (3, 1, 2, 3)$ which would be valid in the i.i.d. finite mixture case.

For the purpose of building deep generative models, we seek to define identifiable parametric families and defer the study of the nonparametric case in Appendix \ref{app:proof_identifiability}. In particular we use a nonlinear Gaussian transition family as follows:
\begin{multline}\label{eq:gaussian_transition_family}
    \mathcal{G}_{\B}^M = \Big\{p_b(\z_t\mid \z_{t-1}, \dots, \z_{t-M}) =  \mathcal{N}(\z_t\mid \bm{m}(\z_{t-1}, \dots, \z_{t-M}, b), \\ \bm{\Sigma}(\z_{t-1}, \dots, \z_{t-M}, b) \mid b\in\B\Big\},
\end{multline}
where $\bm{m}(\z_{t-1}, \dots, \z_{t-M}, b): \R^{mM} \rightarrow \R^m$ and $\bm{\Sigma}(\z_{t-1}, \dots, \z_{t-M},b): \R^{mM} \rightarrow \R^{m\times m}$ are nonlinear w.r.t.~$(\z_{t-1}, \dots, \z_{t-M})$ and denote the mean and covariance matrix, respectively. For each $p_b$ to be a density, we assume that $\bm\Sigma(\cdot,b)\succ 0$ for all $b\in\B$. We further require the following \emph{unique indexing} assumption:
\begin{multline}
\label{eq:unique_indexing_conditional_gaussian}
    \forall b \neq b' \in \B, \ \exists \mathcal{Z}\subset\mathbb{R}^{mM}, \mu(\mathcal{Z})>0,\  s.t. \  \forall (\z_{t-1},\dots,\z_{t-M}) \in \mathcal{Z},\  \bm{m}(\z_{t-1},\dots,\z_{t-M}, b) \\
    \neq \bm{m}(\z_{t-1},\dots,\z_{t-M}, b')\  \text{or}\  \bm{\Sigma}(\z_{t-1},\dots,\z_{t-M}, b) \neq \bm{\Sigma}(\z_{t-1},\dots,\z_{t-M}, b'),
\end{multline}
In other words, for such $(\z_{t-1}, \dots, \z_{t-M})$, $p_b(\z_t \mid \z_{t-1},\dots,\z_{t-M})$ and $p_{b'}(\z_t \mid \z_{t-1},\dots,\z_{t-M})$ are two different Gaussian distributions in at least a nonzero measure set. We also introduce a family of \emph{initial distributions} with unique indexing:
\begin{equation}
\label{eq:gaussian_initial_family}
    \mathcal{I}_{\A}^M := \Big\{p_a(\z_{1:M}) = \mathcal{N}(\z_{1:M}\mid \bm{\mu}(a), \bm{\Sigma}_1(a)) \ \mid \ a \in \A \Big\},
\end{equation}
\begin{equation}
\label{eq:unique_indexing_marginal_gaussian}
    \forall a \neq a' \in \A \Leftrightarrow \bm{\mu}(a) \neq \bm{\mu}(a') \text{ or } \bm{\Sigma}_1(a) \neq \bm{\Sigma}_1(a'),
\end{equation}
where, similarly, $\bm\Sigma_1(a)\succ 0$ for all $a\in\A$. The above Gaussian families paired with unique indexing assumptions satisfy conditions that favour identifiability of MSMs.
\begin{theorem}[Identifiability of parametric MSMs]\label{thm:identifiability_main}
Consider the Markov switching model family with order $M$ defined in Equation~(\ref{eq:msm_finite_mixture}) under nonlinear Gaussian families $\mathcal{M}^{T,M}_{NL}:=\mathcal{M}^T(\mathcal{I}_{\A}^M, \mathcal{G}_{\B}^M)$ with $\mathcal{G}^M_{\B}$, $\mathcal{I}^M_{\A}$ defined by Eqs.~(\ref{eq:gaussian_transition_family}),~(\ref{eq:gaussian_initial_family})  respectively.
. Then, the MSM is identifiable in terms of Def. \ref{def:identifiability} under the following assumptions:
\begin{enumerate}
    \item[(m1)] Unique indexing for $\mathcal{G}_{\B}^M$ and $\mathcal{I}_{\A}^M$: Eqs.~(\ref{eq:unique_indexing_conditional_gaussian}) and~(\ref{eq:unique_indexing_marginal_gaussian}) hold;
    \item[(m2)] The mean and covariance in $\mathcal{G}_{\B}^M$, $\bm{m}(\cdot,b): \R^{mM} \rightarrow \R^m$ and $\bm{\Sigma}(\cdot,b): \R^{mM} \rightarrow \R^{m\times m}$, are analytic functions, for any $b\in \B$.
\end{enumerate}

\end{theorem}
\emph{Proof sketch:} See Appendix \ref{app:proof_identifiability} for the full proof. The strategy can be summarised in the following steps. We defer additional intuitions and details to the end of this section, after a series of remarks.
\begin{enumerate}[leftmargin=1.75em]
    \item  Under the finite mixture model view, it suffices to show that the trajectory family, $\mathcal{P}_{\A,\B}^{T,M}$, is linearly independent under finite mixtures \citep{yakowitz1968identifiability}. Given that $\msmfamily$ is a subset of $\mathrm{conv}(\mathcal{P}_{\A,\B}^{T,M})$, identifiability of $\mathrm{conv}(\mathcal{P}_{\A,\B}^{T,M})$ implies the identifiability of $\msmfamily$.
    \item To show linear independence under finite mixtures of $\mathcal{P}_{\A,\B}^{T,M}$, we begin with $T = 2M$ and construct a family of joint densities defined by nonparametric components $\mathcal{P}_{\B}^M$ and $\Pi_{\A}^M$:
\begin{equation*}
    \left\{p_{a}(\z_{1:M})\prod_{t=M+1}^{2M} p_{b_{t}}(\z_{t} \mid \z_{t-1},\dots,\z_{t-M}) \mid a\in\A,\ b_{t}\in\B,\ M <t \leq 2M \right\}.
\end{equation*}
We show in Lemma~\ref{lemma:absolutely_magical_lemma} that this family is linearly independent under finite mixtures.

    \item For $T>2M$, we proceed by induction (Theorem \ref{thm:linear_independence_joint_non_parametric}). Under the induction hypothesis, we reuse the result from Lemma~\ref{lemma:absolutely_magical_lemma} to extend linear independence to longer trajectories.
    \item We show that the Gaussian families $\mathcal{I}^M_{\A}$ and $\mathcal{G}^M_{\B}$ under assumptions (m1-m2) satisfy the \emph{nonparametric} case defined for $\Pi_{\A}^M$ and $\mathcal{P}_{\B}^M$ respectively.
\end{enumerate}
\begin{figure}
  \begin{minipage}{0.35\textwidth}
\resizebox{\linewidth}{!}{\begin{tikzpicture}[x=1pt,y=1pt,font=\LARGE]

\def\outer{
  (85,110) (175,40) (305,50) (420,125)
  (410,255) (300,310) (190,290) (105,215)
}

\def\conv{
  (115,140) (215,100) (335,105) (405,175)
  (375,275) (260,300) (155,245)
}

\def\msm{
  (170,170) (255,155) (345,188) (340,265)
  (268,290) (188,240)
}

\def\nl{
  (205,220) (280,205) (320,210)
  (335,230) (325,255) (290,265) (245,270)
}

\begin{scope}
  \clip \SoftBlobPath{\outer};
  \fill[white] (0,0) rectangle (480,340);
  \fill[
    pattern={Lines[angle=45,distance=8pt,line width=0.6pt]},
    pattern color=black!12
  ] (0,0) rectangle (480,340);
\end{scope}

\SoftBlob{blue!75!black}{blue!10}{\conv}
\SoftBlob{green!55!black}{green!12}{\msm}
\SoftBlob{red!80!black}{red!10}{\nl}

\draw[line width=1.5pt] \SoftBlobPath{\outer};

\node[align=center] at (240,70)
  {General Nonparametric\\Switching Models};

\node[text=blue!75!black] at (235,135)
  {$\mathrm{conv}\!\left(\mathcal{P}_{\mathcal{A},\mathcal{B}}^{T,M}\right)$};
\node[text=blue!75!black] at (240,115)
  {Prop.~\ref{prop:mixture_pdf_identifiability} + Theorem~\ref{thm:linear_independence_joint_non_parametric}};

\node[text=green!55!black] at (230,190)
  {$\mathcal{M}^{T}\!\left(\Pi_{\mathcal A}^{M},\mathcal{P}_{\mathcal B}^{M}\right)$};
\node[text=green!55!black] at (240,170)
  {Theorem~\ref{thm:identifiability_msm}};

\node[text=red!80!black] at (260,245)
  {$\mathcal{M}_{NL}^{T,M}$};
\node[text=red!80!black] at (257,220)
  {Theorem~\ref{thm:identifiability_main}};

\end{tikzpicture}}
  \end{minipage}
  \quad
  \begin{minipage}{0.6\textwidth}
    \caption{
       The red region shows our main result for parametric MSMs with Gaussian transitions (Sec.~\ref{sec:msm_theory}).
The green region indicates our more general identifiability result for nonparametric MSMs (App.~\ref{app:proof_identifiability}).
The blue region highlights that our proof technique extends to any switching model whose family lies within $\mathrm{conv}(\jointfamily)$.
    } 
    \label{fig:sketch_theorems}
  \end{minipage}
\end{figure}
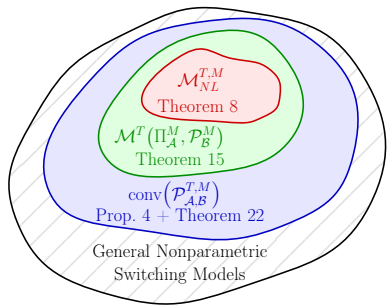
\begin{remark}\label{remark:theory_scope}
Our strategy establishes identifiability by proving linear independence of the trajectory family $\jointfamily$ under finite mixtures (Theorem~\ref{thm:linear_independence_joint_non_parametric}), which, together with Proposition~\ref{prop:mixture_pdf_identifiability}, ensures identifiability in terms of Def.~\ref{def:identifiability_finite_mixture} for any switching model whose family lies in $\mathrm{conv}(\jointfamily)$. Figure~\ref{fig:sketch_theorems} illustrates the resulting hierarchy: 
\begin{itemize}[leftmargin=1.5em]
\setlength{\itemsep}{1pt}
  \setlength{\parskip}{0pt}
  \setlength{\parsep}{0pt}
    \item \textbf{Red region:} Main result (Thm.~\ref{thm:identifiability_main}), 
    covering parametric MSMs with Gaussian transitions.
    \item \textbf{Green region:} Generalisation to nonparametric MSMs 
    (Thm.~\ref{thm:identifiability_msm}), which also includes non-Gaussian transition families.
    \item \textbf{Blue region:} Full theoretical scope of our approach, which includes 
    any switching model whose family lies within $\mathrm{conv}(\jointfamily)$, such as MSMs with time-varying components $K(t)$.
\end{itemize}
\end{remark}

Intuitively, ``switches'' can be considered as discontinuities in the dynamic model and are fully controlled by the discrete states. Assumption (m2) means $\bm{m}(\cdot, b)$ and $\bm{m}(\cdot, b')$ with $b \neq b'$ can only intersect at a set of points with zero Lebesgue measure (similarly for $\bm{\Sigma}(\cdot, b)$), and this smoothness property contributes to the identifiability for the continuous-state transitions $p_{b}(\z_t \mid \z_{t-1}, \dots, \z_{t-M}),\ b\in\mathcal{B}$. The following implications of 
Theorem~\ref{thm:identifiability_main}, when combined, show the profound expressivity of identifiable MSMs:
\begin{itemize}[leftmargin=1.5em]
  \setlength{\itemsep}{1pt}
  \setlength{\parskip}{0pt}
  \setlength{\parsep}{0pt}
    \item Assumption (m2) enables parametrising $p_{b}(\z_t \mid \z_{t-1}, \dots \z_{t-M})$ with e.g., polynomials and neural networks (NNs) with analytic activations (e.g., Softplus). As mentioned, for NNs identifiability applies to the functional form only, since NN weights do not uniquely index NN functions. 
    \item The identifiability result holds independently of the choice of $p(\s_{M:T})$, allowing e.g., nonstationarity and higher-order Markov dependencies. 
\end{itemize} 
Our experiments use a stationary first-order Markov prior for $p(\s_{M:T})$. Its transition parameters are identifiable up to permutations following \citet[Cor.~C.1]{balsells-rodas2024on}. We leave other cases to future work. 

\paragraph{Nonparametric identifiability.}  
The key to our proof technique is the establishment of step (2) in \emph{Proof sketch}, which is proved through a new theoretical result: \emph{Linear Independence on Product functions with M consecutive Overlapping variables} (LIPO-M; Lemma \ref{lemma:absolutely_magical_lemma}), and it can be of independent interest. Unlike Hidden Markov Models (HMMs), in MSMs, $\z_{t}$ and $\z_{t+1}$ are \emph{not} conditionally independent given the discrete states (assuming $M=1$ in this example). Therefore, seminal HMM identifiability results \citep{allman2009identifiability,gassiat2016inference}, which exploit conditional independence of $\z_t$ and $\z_{t+1}$ in HMMs and leverage Kruskal's ``3-way arrays'' technique \citep{kruskal1977three}, do not apply to nonparametric MSMs, rendering our development of new theory necessary.

\paragraph{Linear independence (case $M=1$).} To understand how linear independence under finite mixtures on products with $M$ consecutive variables is achieved, we start with $M=1$ (LIPO-1), presented in \citet{balsells-rodas2024on}, which we provide in Appendix \ref{app:preliminaries_lin_indep} (Lemma \ref{lemma:linear_independence_two_nonlinear_gaussians}) for completeness. Consider the first-order trajectory family on step (2), i.e., $T=2$:
\begin{equation}\label{eq:joint_first_order}
    \Pi_{\A}^1 \otimes \mathcal{P}_{\B}^1 = \bigg\{ p_a(\textcolor{purple}{\z_{1}})p_{b_2}(\textcolor{teal}{\z_{2}}\mid\textcolor{purple}{\z_{1}}) \mid (a,b_2)\in \A\times \B\bigg\},
\end{equation}
where the overlapping variable $\textcolor{purple}{\z_{1}}$ challenges linear independence under finite mixtures. Given linear independence under finite mixtures on the family $\Pi^1_{\mathcal{A}}$ (a3 -- Lemma \ref{lemma:linear_independence_two_nonlinear_gaussians}), the overlapping variable $\textcolor{purple}{\z_{1}}$ complicates linear independence under finite mixtures of $\mathcal{P}_{\mathcal{B}}^1$ when multiplied with $\Pi_{\A}^1$. This issue is resolved by enforcing linear independence under finite mixtures on specific nonzero measure subsets for $\mathcal{P}_{\mathcal{B}}^1$ (a4 -- Lemma \ref{lemma:linear_independence_two_nonlinear_gaussians}), linear dependence implying repeating functions (a5-- Lemma \ref{lemma:linear_independence_two_nonlinear_gaussians}); and continuity for $\textcolor{purple}{\z_{1}}$ (a6-- Lemma \ref{lemma:linear_independence_two_nonlinear_gaussians}). Figure \ref{fig:intuition_lemma_b3} illustrates the intuition behind this result, where linear independence under finite mixtures of the product family can be achieved as long as linear dependence only happens in zero measure sets. For the case $M=1$ the above needs to hold in at least a nonzero measure set of the overlapping variable space.

\begin{figure}
\centering
  \begin{minipage}{0.5\textwidth}
    \centering
    \includegraphics[width=\linewidth]{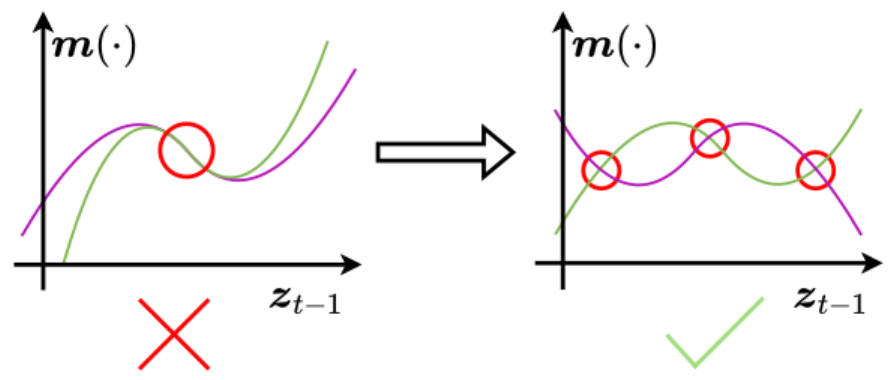}
  \end{minipage}
  \hfill
  \begin{minipage}{0.40\textwidth}
    \caption{Intuition behind Lemma \ref{lemma:linear_independence_two_nonlinear_gaussians}, where linear independence holds if, for any pair of functions (shown in green and purple),  the intersection in the conditioned variable ($\z_{t-1}$) domain is zero-measured.}
    \label{fig:intuition_lemma_b3}
  \end{minipage}
\end{figure}

\paragraph{Linear independence (general case).} Our strategy to generalise to any lag $M < T$ focuses on extending the previous result to cover multiple overlapping variables. 
The challenge is that when the autoregressive order $M$ increases, the number of overlapping variables between consecutive distributions also increases. 
For $M=1$, each new factor only shares a single variable with the previous one, so overlaps remain ``local'' and can be controlled directly. 
For $M>1$, however, a single variable may overlap with several different factors at once. This complicates the proof of linear independence under finite mixtures, as we no longer have a single overlap as in Eq. (\ref{eq:joint_first_order}). To illustrate, for $M=3$, the trajectory family at $T=6$ has the following structure:
\begin{equation}\label{eq:joint_product_3}
    \Pi^3_{\A} \otimes (\otimes^3 \mathcal{P}^3_{\B}) = \Big\{p_a(\textcolor{purple}{\z_{1}},\textcolor{teal}{\z_2},\textcolor{orange}{\z_3})p_{b_4}(\textcolor{violet}{\z_{4}}\mid\textcolor{purple}{\z_{1}},\textcolor{teal}{\z_2},\textcolor{orange}{\z_3}) p_{b_5}(\textcolor{olive}{\z_{5}}\mid\textcolor{violet}{\z_{4}},\textcolor{orange}{\z_3},\textcolor{teal}{\z_2})p_{b_6}(\z_6\mid\textcolor{olive}{\z_{5}},\textcolor{violet}{\z_{4}},\textcolor{orange}{\z_3})\Big\},
\end{equation}
for every $a\in\A$ and $b_4,b_5,b_6\in\B$, where we indicate overlapping variables with different colours. 
Compared to $M=1$ in Eq.~(\ref{eq:joint_first_order}), when multiplying the initial and transition distribution, $p_a(\textcolor{purple}{\z_{1}},\textcolor{teal}{\z_2},\textcolor{orange}{\z_3})p_{b_4}(\textcolor{violet}{\z_{4}}\mid\textcolor{purple}{\z_{1}},\textcolor{teal}{\z_2},\textcolor{orange}{\z_3})$, the overlap $(\textcolor{purple}{\z_{1}},\textcolor{teal}{\z_2},\textcolor{orange}{\z_3})$ has the same structure as in the $M=1$ case, and linear independence for $T=4$ can be established by LIPO-1. However at $T=5$, the next factor $p_{b_5}(\textcolor{olive}{\z_{5}}\mid\textcolor{violet}{\z_{4}},\textcolor{orange}{\z_3},\textcolor{teal}{\z_2})$ introduces an overlap $(\textcolor{violet}{\z_{4}},\textcolor{orange}{\z_3},\textcolor{teal}{\z_2})$ that is no longer confined in a single pair of consecutive terms, as $(\textcolor{orange}{\z_3},\textcolor{teal}{\z_2})$ are present in $p_a(\textcolor{purple}{\z_{1}},\textcolor{teal}{\z_2},\textcolor{orange}{\z_3})$. This branching overlap violates assumption (a3) of LIPO-1, as the product family does not directly reduce to the single-overlap structure in Eq.~\eqref{eq:joint_first_order}. Similar arguments apply for $T=6$, where $\textcolor{orange}{\z_3}$ overlaps with all the functions.

To ensure linear independence under finite mixtures of the trajectory family for any lag $M$, we group variables to force overlaps only between consecutive distributions (similar to Eq.~\eqref{eq:joint_first_order}). For instance, the example with $M=3$ becomes:
\begin{equation}
    \Pi^3_{\A} \otimes (\otimes^3 \mathcal{P}^3_{\B}) = \Big\{p_a(\textcolor{purple}{\z_{1}, \z_2, \z_3}) 
    p_{b_{4:6}}(\textcolor{teal}{\z_{4}, \z_5, \z_6}\mid\textcolor{purple}{\z_{1}, \z_2, \z_3}) \mid a\in\A, b_{4:6}\in(\times^{3}\B)\Big\}.
\end{equation}
As in the case $M=1$ linear independence under finite mixtures holds provided that the product $(\otimes^3 \mathcal{P}^3_{\B})$ satisfies assumptions (a4-a6) in LIPO-1. Therefore, the strategy is to impose these assumptions on $\mathcal{P}^M_\B$, and explore whether they extend to the product families $(\otimes^r \mathcal{P}^M_\B)$, for all $1<r\leq M$. In Appendix \ref{app:extending_assumptions} we prove that such an extension requires strengthening the nonparametric assumptions from (a4-a6) to (b4-b6) (Lemmas~\ref{lemma:extension_b4} and~\ref{lemma:extension_b5}). Notably, assumption (b4) requires linear independence under finite mixtures of the family $\mathcal{P}^{M}_{\B}$ to hold on every subset of a full-measure subset of $\times^M\R^{m}$; whereas as mentioned for $M=1$, only a nonzero measure subset $\mathcal{Y}\subset\R^m$ suffices. Despite being stronger, these assumptions connect to nonlinear Gaussians via (d3) in Theorem~\ref{thm:linear_independence_nonlinear_gaussian}. They also connect to (m2), as it ensures that distinct functions differ almost everywhere.

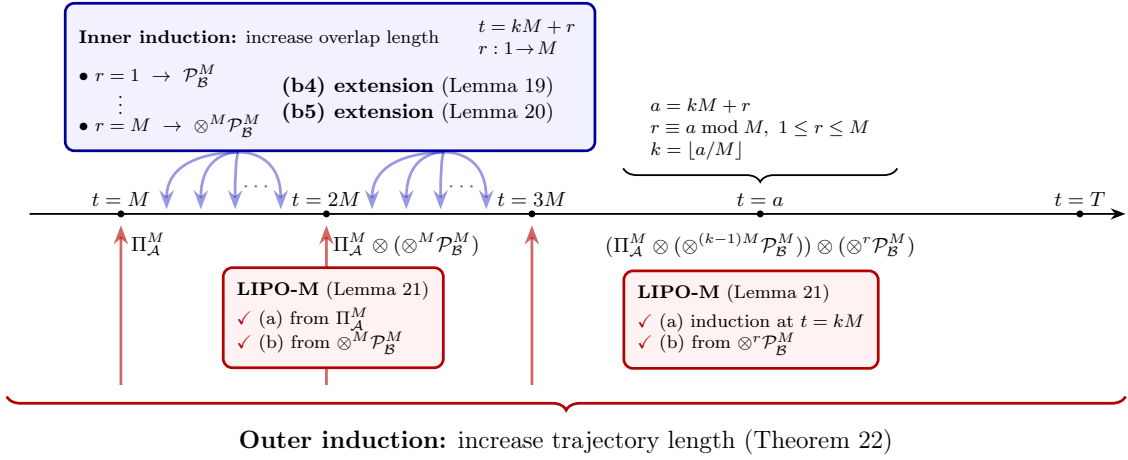
\begin{figure}
    \centering
    \resizebox{\linewidth}{!}{
\tikzset{
  >=Stealth,
  every picture/.style={line width=0.8pt, font=\small},
  tick/.style={circle, fill=black, inner sep=0pt, minimum size=3pt},
  tlabel/.style={font=\footnotesize, inner sep=1pt},
  mathlab/.style={font=\footnotesize, inner sep=2pt},
  box/.style={draw, rounded corners=4pt, inner sep=6pt},
  dim/.style={decorate, decoration={brace, amplitude=6pt}},
  dim*mirror/.style={decorate, decoration={brace, amplitude=6pt, mirror}},
  note/.style={font=\footnotesize, inner sep=1pt, text opacity=1},
  faint/.style={opacity=.75},
}
\tikzset{
  innerbox/.style={box, draw=blue!60!black, very thick, fill=blue!5},
  lipobox/.style={box, draw=red!70!black, very thick, fill=red!5},
  outerbrace/.style={decorate, decoration={brace, amplitude=6pt, mirror}, draw=red!70!black, very thick},
  check/.style={text=red!70!black, font=\scriptsize},
}

\begin{tikzpicture}[x=1pt,y=1pt]
\draw (40,90) -- (520,90) [->];

\foreach \x/\lab in {80/$t=M$, 170/$t=2M$, 260/$t=3M$, 360/$t=a$, 500/$t=T$}{
  \node[tick] at (\x,90) {};
  \node[tlabel, above=2pt] at (\x,90) {\lab};
}

\foreach \x in {80, 170, 260}{
  \draw[->, lipobox, opacity=0.6] (\x,15) -- (\x,85);
}

\node[mathlab, below=6pt] at (92,90) {$\Pi_{\mathcal A}^M$};
\node (prodTwoM) [mathlab, below=6pt, align=center] at (205,90)
  {$\Pi_{\mathcal A}^M\otimes(\otimes^{M}\mathcal P_{\mathcal B}^M)$};
\draw[outerbrace, font=\normalsize] (30,10) -- node[below=10pt] {\textbf{Outer induction:} increase trajectory length (Theorem \ref{thm:linear_independence_joint_non_parametric})} (520,10);

\node[innerbox, anchor=west,font=\scriptsize] (inner) at (55,150) {%
  \begin{minipage}{220pt}
  \raggedright
  \begin{tabular}{@{}l@{\hspace{2em}}l@{}}
    \multirow{2}{*}{\textbf{Inner induction:} increase overlap length}  & 
    $t = kM + r$ \\
    & $r:1\!\to\!M$\\
  \end{tabular}\\[1pt]
    $\bullet\ r=1\ \rightarrow\ \mathcal P_{\mathcal B}^M$\\[-2pt]
    $\qquad \vdots$\\[-2pt]
    $\bullet\ r=M\ \rightarrow\ \otimes^{M}\mathcal P_{\mathcal B}^M$
  \end{minipage}
};
\node[note,faint,align=left] at (210,140) {\textbf{(b4) extension} (Lemma~\ref{lemma:extension_b4})\\
\textbf{(b5) extension} (Lemma~\ref{lemma:extension_b5})};

\coordinate (innersrc) at ($(inner.south)+ (50pt, 0pt)$);

\foreach \xt in {190,205,220}{
  \draw[->, blue!70!black, very thick, opacity=.4]
    (innersrc) .. controls +(0pt,0pt) and +(0,26pt) .. (\xt,92);
}

\node[blue!70!black, opacity=.6] at (230,102) {$\cdots$};

\draw[->, blue!70!black, very thick, opacity=.4]
  (innersrc) .. controls +(0pt,0pt) and +(0,26pt) .. (240,92);

\coordinate (innersrs2) at ($(inner.south)+ (-40pt, 0pt)$);

\foreach \xt in {100,115,130}{
  \draw[->, blue!70!black, very thick, opacity=.4]
    (innersrs2) .. controls +(0pt,0pt) and +(0,26pt) .. (\xt,92);
}

\node[blue!70!black, opacity=.6] at (140,102) {$\cdots$};

\draw[->, blue!70!black, very thick, opacity=.4]
  (innersrs2) .. controls +(0pt,0pt) and +(0,26pt) .. (150,92);

\node[lipobox, align=left, anchor=north, font=\scriptsize] (lipom) 
  at ($(prodTwoM.south)+(-32,-2pt)$) {%
  \textbf{LIPO-M} (Lemma~\ref{lemma:absolutely_magical_lemma})\\[3pt]
  \tick\ (a) from $\Pi_{\mathcal A}^M$\\
  \tick\ (b) from $\otimes^{M}\mathcal P_{\mathcal B}^M$};
\node (prodA) [mathlab, below=6pt, align=center] at (360,90)
  {$(\Pi_{\mathcal A}^M\otimes(\otimes^{(k-1)M}\mathcal P_{\mathcal B}^M)) \otimes (\otimes^{r}\mathcal P_{\mathcal B}^M)$};

\draw[dim*mirror, font=\scriptsize] (300,110) -- node[above=0pt,align=left] {$a = kM + r$\\
$r \equiv a \bmod M,\ 1\le r\le M$\\
$k=\lfloor a/M \rfloor$} (420,110);

\node[lipobox, align=left, anchor=north west,font=\scriptsize] (lipomA)
  at ($(prodA.south west)+(10pt,-3pt)$) {%
  \textbf{LIPO-M} (Lemma~\ref{lemma:absolutely_magical_lemma})\\[3pt]
  \tick\ (a) induction at $t=kM$\\
  \tick\ (b) from $\otimes^{r}\mathcal P_{\mathcal B}^M$};

\end{tikzpicture}}
    \caption{Linear independence under finite mixtures of $\jointfamily$ is proven 
by a nested induction strategy. The inner induction (blue box) constructs the 
product family $(\otimes^r\mathcal{P}^M_\B,\ 1 \leq r \leq M)$, while the outer 
induction (red boxes) increases the length of the trajectory family.}
    \label{fig:linear_independence_induction_sketch}
\end{figure}
\paragraph{Linear independence via nested induction.} To prove LIPO-M and extend the assumptions to product families (Lemmas~\ref{lemma:extension_b4} and~\ref{lemma:extension_b5}), we use induction, with the base case being a single product of consecutive variables ($\mathcal{P}^M_\B\otimes\mathcal{P}^M_\B$). Therefore, the full proof of linear independence under finite mixtures of the trajectory family requires nested induction (see Theorem \ref{thm:linear_independence_joint_non_parametric} in Appendix~\ref{app:lin_indep_joint}). Figure \ref{fig:linear_independence_induction_sketch} illustrates this idea: 
\begin{itemize}
    \item The inner induction constructs the product family $(\otimes^r \mathcal{P}^M_\B)$ with $1\leq r \leq M$ by increasing the number of consecutive transition distributions $r$.
    \item The outer induction increases the length of the trajectory family $\mathcal{P}^{M,t}_{\A,\B}$, with $M<t\leq T$, ensuring that overlaps are addressed as we introduce additional time steps. 
\end{itemize}
A key technical point in Theorem~\ref{thm:linear_independence_joint_non_parametric} is that the induction hypothesis applies only when the trajectory length $t$ is a multiple of $M$. Thus we distinguish two cases:
\begin{itemize}
    \item If $t$ is a multiple of $M$, we apply the induction hypothesis at $t-M$ and then use LIPO-M to combine the trajectory family at time $t-M$ with $M$ consecutive products from $\mathcal{P}_\B^M$.
    \item If $t$ is not a multiple of $M$, compute $r\equiv t\bmod M$. We apply the induction hypothesis at $t-r$ (a multiple of $M$), and then use LIPO-M to combine the trajectory family at time $t-r$ with the $r$ additional products $(\otimes^r \mathcal{P}^M_\B)$.
\end{itemize}
The same logic applies to the base cases $M < t \leq 2M$, where LIPO-M is applied by combining the initial family $\Pi_\A^M$ with $t-M$ consecutive products from $\mathcal{P}_\B^M$.

\subsection{Identifiable Switching Dynamical Systems}\label{sec:switching_dynamical_systems}
Based on identifiable MSMs, our analysis of switching dynamical systems extends 
the setup from \citet{kivva2022identifiability} (Section \ref{sec:kivva_background}) to the temporal case. Assume the prior dynamic model $p(\z_{1:T})$ belongs to the nonlinear Gaussian MSM family $\mathcal{M}^{T,M}_{NL}$ specified by Thm.~\ref{thm:identifiability_main}. At each time step $t$, the latent $\z_t\in\R^m$ generates observed data $\x_t\in\R^n$ via a piecewise linear transformation $f$. The generation process of such SDS model can be expressed as follows:
\begin{equation}\label{eq:generative_model_sns}
\x_t = f(\z_t) + \bm{\varepsilon}_t, \quad \bm{\varepsilon}_t \sim p_{\bm{\varepsilon}}, \quad \z_{1:T}\sim p(\z_{1:T})\in \mathcal{M}^{T,M}_{NL}.
\end{equation}
%
Note that we can also write the decoding process as $\x_{1:T} = F(\z_{1:T}) + \bm{\varepsilon}_{1:T}$, where $F$ is a factored transformation composed of $f$ as defined below. Importantly, this notion allows $F$ to inherit e.g., piecewise linearity and weakly-injectivity properties from $f$.  
\begin{definition}
    We say that a function $F:\R^{mT}\rightarrow\R^{nT}$ is factored if it is composed of $f:\R^m\rightarrow\R^n$, such that for any $\z_{1:T}\in\R^{mT}$, $F(\z_{1:T}) = \left(f(\z_1), \dots, f(\z_T) \right)^\top$.
\end{definition}

The identifiability analysis of the SDS model (Eq.~(\ref{eq:generative_model_sns})) follows two steps (see Figure \ref{fig:factored_trans}). 
(i) Following \citet{kivva2022identifiability}, we establish the identifiability for a prior $p(\z_{1:T}) \in\mathcal{M}^{T,M
}_{NL}$ from $(F_{\#}p)(\x_{1:T})$ in the noiseless case. Here $F_{\#}p$ denotes the pushforward measure of $p$ by $F$. (ii) When the noise $\bm{\varepsilon}_{1:T}$ distribution is known, $F_{\#}p$ is identifiable from $\left(\left(F_{\#}p\right)*p_{\bm{\varepsilon}_{1:T}}\right)(\x_{1:T})$ using convolution tricks from \citet{khemakhem2020variational}\footnote{In particular, identifiability in the noisy case holds under the assumption that the zero set of the characteristic function of $p_{\bm{\varepsilon}}$ is zero measured. We use $p_{\bm{\varepsilon}_{1:T}}$ to denote the distribution of $\bm{\varepsilon}_{1:T}$.}.
In detail, we first define the notion of identifiability for $p\in \mathcal{M}^{T}(\Pi_{\A}^M,\mathcal{P}_{\B}^M)$ given noiseless observations.

\begin{definition}
\label{def:identifiability_affine_transformations}
Given a family of factored transformations $\mathbb{F}$, for $F\in\mathbb{F}$ the prior $p \in  \mathcal{M}^T(\Pi_{\A}^M, \mathcal{P}_{\B}^M)$ is said to be identifiable up to affine transformations, when for any $\tilde{F}\in\mathbb{F}$ and $\tilde{p}\in\mathcal{M}^T(\Pi_{\A}^M, \mathcal{P}_{\B}^M)$ s.t. $F_{\#}p = \tilde{F}_{\#}\tilde{p}$, there exists an invertible factored affine mapping $H: \R^{mT} \rightarrow \R^{mT}$ composed of $h:\R^{m} \rightarrow \R^{m}$, where $p = H_{\#}\tilde{p}$. Equivalently, identifiability up to permutation, scaling and translation refers to having $h(\x)=PD\x + \Bb, \x\in\R^{m}$, where P and D are permutation and diagonal matrices respectively.
\end{definition}

\begin{figure}
  \centering
\includegraphics[width=\linewidth]{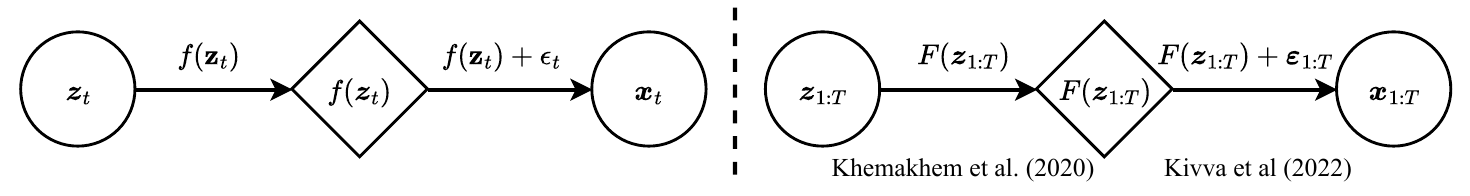}
\vspace{-0.65em}
  \caption{We assume $\z_t$ is transformed via $f$ with noise $\bm{\varepsilon}_t$ at each time step $t$ independently. We view this as a transformation on $\z_{1:T}$ via a factored $F$ with noise $\bm{\varepsilon}_{1:T}$.}
  \label{fig:factored_trans}
  \vspace{-.75em}
\end{figure}
We state the following identifiability result for nonlinear SDSs and defer the proof in Appendix \ref{app:sds}, which is mainly based on \citet{kivva2022identifiability}. 
\begin{theorem}[Identifiability of parametric SDSs]\label{thm:identifiability_sds}
    Assume the observations $\bm{x}_{1:T}$ are generated from a latent variable model following Eq. (\ref{eq:generative_model_sns}), whose prior $p(\z_{1:T})$ belongs to $\mathcal{M}_{NL}^{T,M} = \mathcal{M}^T(\mathcal{I}_{\A}^M,\mathcal{G}_{\B}^M)$ and satisfies (m1 - m2). Assume:
    \begin{thmlist}
        \item[(f1)] $f$ is weakly injective (Def. \ref{def:weakly_injective}) and piecewise linear;
    \end{thmlist}
    and the following assumptions on the prior distribution $p(\z_{1:T})$.
    \begin{thmlist}
        \item[(s1)] At time $t$, the continuous latent variables $\z_t$ are mutually independent given $\z_{t-1}$ and $s_t$. For any $i\neq j$, $z^{(i)}_t \indep z^{(j)}_t \mid \z_{t-1}, s_t$;
        \item[(s2)] There exists $k_1,k_2\in[K]$ and a nonzero measure set $\mathcal{Z}\subset \R^{mM}$ such that the elements in $\{\bm{\Sigma}(\z_{t-1},\dots,\z_{t-M}, k_1)_{ii} / \bm{\Sigma}(\z_{t-1}, \dots, \z_{t-M}, k_2)_{ii} \mid i \in [m]\}$ are distinct for any $(\z_{t-1},\dots,\z_{t-M})\in \mathcal{Z}$.
    \end{thmlist}
    The prior distribution $p(\z_{1:T})$ is identifiable under the above assumptions as follows.
    \begin{thmlist}[leftmargin=1.75em]
        \item[(i)] If (f1) is satisfied, the prior $p(\z_{1:T})$ is identifiable up to affine transformations (Def.~\ref{def:identifiability_affine_transformations}).\label{thm:identifiability_sds:i}
        \item[(ii)] If (f1), (s1), and (s2) are satisfied, the prior $p(\z_{1:T})$ is identifiable up to permutation, scaling, and translation.\label{thm:identifiability_sds:ii}
    \end{thmlist}
\end{theorem}
The above result allows identifiable SDSs to employ e.g., Softplus networks for $p_{\mparam}(\z_t\mid\z_{t-1}, \dots, \z_{t-M},s_t)$ with history-dependent or piecewise stationary noise, and certain (Leaky) ReLU networks for $f$ (see Appendix H of \citet{kivva2022identifiability}). Again for neural networks, identifiability refers only to their functional form.

\paragraph{Relaxing piecewise linearity.}
While our approach relaxes the injectivity assumption on the emission function, we still restrict $f$ to be piecewise linear. Other works, such as \citet{yao2022temporally} assume injective functions, where their results could be adapted to our framework. Specifically, \citet{yao2022temporally} establish identifiability under the following conditions: (1) independence of latent variables (s1 in Theorem~\ref{thm:identifiability_sds}); (2) the existence of a demixing function; and (3) access to $K=2m+1$ components satisfying certain linear independence conditions on vectors formed with derivatives of $\log p(\z_{t}\mid\z_{t-1},\dots,\z_{t-M},k),\ k\in[K]$. However, this strategy either assumes access to the true component label $k$, or requires imposing assumptions directly on the observed conditional distributions $p(\x_{t}\mid\x_{t-1},\dots,\x_{t-M},k),\ k\in[K]$. 
For instance, \citet{song2023temporally} introduces assumptions (m1-m2) directly on the observed distribution, which implicitly constrain both the latent dynamics and the emission function. This lack of separation between latent structure and observation model makes these assumptions difficult to verify or control in practice (e.g., if observed and latent distributions are Gaussian, then $f$ must be affine). Moreover, without additional structural constraints, these strategies cannot fully exploit our identifiability results for MSMs, which are derived at the level of the latent dynamics. To illustrate, consider the joint probability distribution, an equivalent SDS distribution family would be defined as follows:
\begin{multline}
    \Bigg\{\sum^C_i c_i (F_\# p_i)(\x_{1:T}) \mid\ F\in\mathbb{F}, \ \sum_i^C c_i=1,\ c_i>0,\ C=K_0K^{T-M} ,\ K_0,K<+\infty, \\ \  p_i\in\mathcal{P}_{\A_0,\B_0}^{T,M}\subset \mathcal{P}_{\A,\B}^{T,M},\ \A_0\subset \A, \ \B_0\subset\B,\ |\A_0|= K_0,\ |\B_0|=K\Bigg\},
\end{multline}
where each element consists of a factored transformation $F$ applied to a mixture component $p_i$ from an MSM family $\mathcal{M}_{NL}^{T,M}$. Identifiability here requires further analysis, since two different elements in the above family with $F\neq \tilde{F}$ could correspond to different MSMs with different components, $(K_0,K)\neq (\tilde{K}_0,\tilde{K})$, related via inverse CDF mappings: $p = \left(F^{-1}\circ \tilde{F}\right)_\#\tilde{p}$.

\subsection{Regime-dependent Causal Representation Learning}\label{sec:causal_discovery}

In regime-dependent time series \citep{saggioro2020reconstructing}, the SEM is time-dependent. However, it is \emph{causally stationary} \citep{assaad2022survey} when conditioned on discrete latent variables  $s_t\in[K]$ at each time step $t$. In the context of LVMs, the SEM operates on the latent space; and therefore, our framework is connected to causal representation learning \citep{scholkopf2021toward}. When considering SDSs, the corresponding latent causal structure becomes a \emph{regime-dependent causal graph} which we denote as $\mathbf{G}_{1:K}:=\{\mathbf{G}_k\}_{k=1}^K \in  \mathbb{G}^K$. It can be defined as a collection of causal graphs such that for each $k\in[K]$, $\mathbf{G}_{k}\in\mathbb{G}$ encodes the latent SEM at time $t$ if $s_t=k$. To establish connections with the previous MSMs, we assume an \emph{additive noise model} (ANM; \citep{hoyer2008nonlinear}) with absence of instantaneous effects. Therefore, the corresponding regime-dependent SEM in latent space for variable $j\in[m]$ is expressed as follows:
\begin{equation}
    z^{(j)}_t = f^{(j)}_{s_t}\left(\{z^{(i)}_{t-\tau} \mid z^{(i)}_{t-\tau} \in \mathbf{Pa}^{(j)}_{s_t}(\tau), 1\leq \tau \leq M \}\right) + \bm{\varepsilon}^{(j)}_{t}, \quad \bm{\varepsilon}^{(j)}_{t} \sim p_{\bm{\varepsilon}^{(j)}},
\end{equation}
where $\mathbf{Pa}^{(j)}_{s_t}(\tau)$ denotes the parents of variable $j$ with lag $\tau$, associated to $\mathbf{G}_k$ for $s_t=k$.
\begin{definition}\label{def:regime_dependent_graph}
     We say that the regime-dependent graph of a SDS with order $M$ is identifiable up to permutations if for any $p,\tilde{p}\in \mathcal{M}^T(\Pi_{\A}^M, \mathcal{P}_{\B}^M)$, and $F,\tilde{F}\in\mathbb{F}$, with respective regime-dependent graphs $\mathbf{G}_{1:K}\in (\times^{K} \mathbb{G})$ and $\tilde{\mathbf{G}}_{1:\tilde{K}}\in (\times^{\tilde{K}} \mathbb{G})$, such that $(F_{\#}p)(\x_{1:T})=(\tilde{F}_{\#}\tilde{p})(\x_{1:T})$; $K=\tilde{K}$ and there exists a node permutation $\pi\in S_m$ (Def.~\ref{def:graph_node_permutation}) and a graph permutation $\sigma\in S_K$  such that $\mathbf{G}(k) = \pi(\tilde{\mathbf{G}}(\sigma(k)))$ for $k\in [K]$.
\end{definition}
The above node permutation (Def.~\ref{def:graph_node_permutation}) permutes the nodes of each graph in the set $\mathbf{G}_{1:K}$. Our definition aims to recover the inherent graph structure up to permutations of both continuous and discrete latent variables. We establish latent structure identifiability for regime-dependent time series as follows.
\begin{corollary}[Latent structure identifiability]\label{cor:causality_identifiability}
Consider the Markov switching model family with order $M$ defined in Equation (\ref{eq:msm_finite_mixture}) under nonlinear Gaussian families $\mathcal{M}_{NL}^{T,M}$. Assume (m1) from Theorem \ref{thm:identifiability_main} is satisfied;
\begin{enumerate}[leftmargin=1.75cm]
    \item[(m2*)] Heterogeneous noise: the transition distribution mean in  $\bm{m}(\cdot,b): \R^{mM} \rightarrow \R^m$ is analytic, and $\bm{\Sigma}(\cdot,b):\R^{mM} \rightarrow \R^{m\times m}$ is diagonal, analytic, and satisfies (s2);
    \item[(m3)] Minimality: if $z^{(i)}\notin \mathbf{Pa}^{(j)}_{k}(\tau), 1 \leq \tau \leq M$, then the $i$-th dimension of $\bm{m}(\z_{t-1},\dots, \\\z_{t-M}, k)$ is constant with respect to $z^{(j)}_{t-\tau}$.
\end{enumerate}
Then, the regime-dependent graph is identifiable up to permutations (Definition~\ref{def:regime_dependent_graph}).
\end{corollary}
See Appendix \ref{app:causality_proof} for the proof. It suffices to show that the regime dependent structure is recovered up to permutations, given the identifiable transition derivatives from Theorems \ref{thm:identifiability_main} and \ref{thm:identifiability_sds}. Under assumptions (m1-m3), we note that other structures in traditional causal time series analysis, such as the \emph{full time graph} \citep{peters2017elements}, are not identifiable in the context of high-order MSMs; as we have no access to the discrete latents $\s_{M:T}$, but rather their distribution up to permutations.

\section{Estimation}\label{sec:estimation}

We refer to our full approach, which combines the theoretical framework with the inference scheme below, as \emph{identifiable Switching Dynamical System} (iSDS). For real-world applications (Section \ref{sec:real_world}), we also consider MSMs. Together with our assumptions, we denote the corresponding method as the \emph{identifiable Markov Switching Model} (iMSM). 

We adopt the collapsed variational inference setup of \citet{dong2020collapsed}. Parameters are learned by maximising the evidence lower bound (ELBO) \citep{kingma2013autoencoding}. The approximate posterior over latent variables $\{\z_{1:T}, \s_{M:T}\}$ uses a factorisation which incorporates the exact posterior of the discrete latent states conditioned on the continuous latent variables:
\begin{equation}
q_{\vparam,\mparam}(\z_{1:T},\s_{M:T}\mid\x_{1:T}) = \prod_{t=1}^{T} q_{\vparam}(\z_{t}\mid\x_{t}) p_{\mparam}(\s_{M:T}\mid\z_{1:T}).
\end{equation}
The variational posterior over $\z_{1:T}$ factorises across time. This mirrors our injectivity assumptions on the latent variable transformation, and simplifies the estimation. Prior work \citep{ansari2021deep,dong2020collapsed} constructs $q_{\vparam,\mparam}(\z_{1:T}\mid\x_{1:T})$ via a bidirectional RNN on $\x_{1:T}$, followed by a forward RNN. With the exact posterior $p_{\mparam}(\s_{M:T}\mid\z_{1:T})$, we can marginalise $\s_{M:T}$ from the ELBO (see Appendix~\ref{app:estimation_details} for details):
\begin{equation}\label{eq:elbo}
    p_{\mparam}(\x_{1:T}) \geq \mathcal{L}_{\text{ELBO}}:= \mathbb{E}_{q_{\vparam}(\z_{1:T}\mid\x_{1:T})} \Big[\log p_{\mparam}(\x_{1:T}\mid\z_{1:T})\Big]
    - KL\Big(q_{\vparam}(\z_{1:T}\mid\x_{1:T}) || p_{\mparam}(\z_{1:T})\Big).
\end{equation}
We estimate expectations with Monte Carlo samples $\tilde{\z}_{1:T} \sim q_{\vparam}(\z_{1:T}\mid\x_{1:T})$ and apply the reparametrisation trick for backpropagation \citep{kingma2013autoencoding}, enabling stochastic gradient ascent. The gradient results as follows:
\begin{equation}
\nabla\mathcal{L}_{\text{ELBO}} \approx \nabla \log p_{\mparam}(\x_{1:T}\mid\tilde{\z}_{1:T}) - \nabla \log q_{\vparam}(\tilde{\z}_{1:T}\mid \x_{1:T}) + \nabla \log p_{\mparam}(\tilde{\z}_{1:T}).
\end{equation}
Following \citet{dong2020collapsed}, the prior gradient admits an efficient decomposition, similar to standard HMM inference \citep{bishop2006pattern}:
\begin{align}
    \nabla \log p_{\mparam}(\tilde{\z}_{1:T}) &= \mathbb{E}_{p_{\mparam}(\s_{M:T}\mid\tilde{\z}_{1:T})}\Big[ \nabla \log p_{\mparam}(\tilde{\z}_{1:T},\s_{M:T})  \Big] \\
    &= \sum_{t=M+1}^T \sum_{k=1}^K \gamma_{t,k}(\tilde{\z}_{1:T}) \nabla \log p_{\mparam}(\tilde{\z}_t\mid \tilde{\z}_{t-1}, \dots, \tilde{\z}_{t-M}, s_t=k) \\ 
    & + \sum_{k=1}^K \gamma_{M,k}(\tilde{\z}_{1:T}) \nabla \log \pi_k  +  \sum_{t=M+2}^{T}\sum_{k=1}^K\sum_{k'=1}^K \xi_{t,k,k'}(\tilde{\z}_{1:T}) \nabla \log Q_{k',k},
    \label{eq:gradients_z_posterior}
\end{align}
where $\gamma_{t,k}(\tilde{\z}_{1:T})=p_{\mparam}(s_t=k\mid\tilde{\z}_{1:T}) $ and $\xi_{t,k,k'}(\tilde{\z}_{1:T})= p_{\mparam}(s_t=k, s_{t-1}=k'\mid\tilde{\z}_{1:T})$ are computed with forward-backward update rules under a first-order Markov chain over $\s_{M:T}$. See Appendix~\ref{app:estimation_details} for details.
A recent alternative \citep{ansari2021deep} computes $p(\z_{1:T})$ using a forward posterior only, but this requires backpropagation through the entire sequence, considerably increasing computational cost. 

\paragraph{Initialisation and iMSM.}

 \citet{dong2020collapsed} claim that the above approach is prone to state collapse and introduce additional losses with annealing schedules; \citet{ansari2021deep} mitigate state collapse via temperature annealing on discrete logits together with a forward posterior. Such schedules are data-dependent and require careful hyperparameter tuning. In our experience, a strong initialisation of $p_\theta(\z_{1:T})$ is the key to stabilise training and mitigate state collapse. We first project $\x_{1:T}$ into a lower-dimensional space via linear Principal Component Analysis (PCA) and fit a standard MSM to the projected time series following Eq.~\eqref{eq:gradients_z_posterior}. The learned $\pi,\ Q$, and regime-specific parameters are used to initialise $p_{\mparam}(\z_{1:T})$ Similar initialisation strategies have proven effective in prior work \citep{linderman2016recurrent},
 which serve as a foundation for improved training stability in discrete latent variable models. The iMSM follows the same likelihood-based optimisation, using the gradient decomposition in Eq.~\eqref{eq:gradients_z_posterior}.

\paragraph{Sparsity regularisation.} To recover the underlying causal structure, we promote sparsity in the regime-specific mean transitions $\bm{m}(\cdot, k)$ via an element-wise $\ell_1$ penalty on their Jacobians $J_{\bm{m}_k}$ for $k\in[K]$. The regularised objective is
\begin{equation}\label{eq:loss_function_elbo}
\mathcal{L}_{\text{REG}} := \mathcal{L}_{\text{ELBO}} + \eta \sum_{k=1}^K \sum_{i,j}^m \left|\big(J_{\bm{m}_k}\big)_{ij}\right|,
\end{equation}
where $m$ denotes  the latent dimension. While maximising the ELBO is a practical estimation approach, it does not guarantee statistical consistency in general; see e.g., \citet{khemakhem2020variational} for discussion. Consistency analysis is left to future work.
\section{Synthetic Experiments}\label{sec:experiments_msm}

In this section, we present experiments on synthetic data, where access to the ground-truth structure enables us to empirically verify the theoretical results. Specifically, we (i) compare the ability of iSDS to recover identifiable latent variables against recent baselines; (ii) investigate the sample and sequence length efficiency; and (iii) test whether the assumptions in Theorem \ref{thm:identifiability_main} are necessary for identifiability in both MSMs and SDSs.

\paragraph{Data Generation.} We generate data following \citet{balsells-rodas2024on} mostly with additional scenarios for the noise distribution. The regimes follow a first-order Markov chain, with transition matrix $Q$ that assigns $0.9$ probability to remain in the same regime, and $0.1$ probability to transition to the next one. Continuous latent variables $\z_{1:M}$ are initialised from a $K$-component Gaussian Mixture Model with uniform categorical prior. For $t>M$, they evolve according to a Gaussian transition model with $K$ regimes and regime-specific parameters. For each $k\in[K]$, the mean function $\bm{m}(\cdot,k)$ is given by a two-layer MLP with cosine activations, where each layer is a locally linear map \citep{zheng2018dags}, constrained by a randomly generated causal graph $\mathbf{G}_k$. Covariance matrices $\bm{\Sigma}(\cdot,k)$ are diagonal, and we consider three cases:
\begin{itemize}[leftmargin=1.5em]
  \setlength{\itemsep}{5pt}
  \setlength{\parskip}{0pt}
  \setlength{\parsep}{0pt}
    \item Constant noise: fixed, $\bm{\Sigma}(\cdot,k)=\bm{\Sigma}$.
    \item Heterogeneous noise: regime-dependent, $\bm{\Sigma}(\cdot,k)=\bm{\Sigma}(k)$.
    \item History-dependent noise: diagonal entries proportional to the scaled absolute value of the mean.
\end{itemize}
The latter two satisfy assumption (s2), i.e., distinct diagonal ratios across at least one pair of regimes. Observations are generated by passing latent variables through a two-layer MLP with Leaky ReLU activations. Unless otherwise stated, we generate $N=10000$ samples for training, and $1000$ samples for evaluation, both with length $T=100$. 

\paragraph{Metrics.} To evaluate the learned solutions against ground truth, we use the following metrics (details in Appendix \ref{app:metrics}):
\begin{itemize}[leftmargin=1.5em]
  \setlength{\itemsep}{5pt}
  \setlength{\parskip}{0pt}
  \setlength{\parsep}{0pt}
    \item \textbf{Regime estimation:} the component with highest posterior probability is selected, and the $F_1$ score is computed after regime permutation alignment.
    \item \textbf{Continuous latent variable alignment:} we report both weak and strong mean correlation coefficients (MCC). The strong MCC, also defined in \citet{khemakhem2020variational}, is a standard measure in ICA and evaluates alignment up to scaling and permutation. In contrast, weak MCC \citep{kivva2022identifiability} allows affine transformations.
    \item \textbf{Causal structure recovery:} the $F_1$ score between ground-truth and estimated causal graphs. This is obtained by thresholding the Jacobian after regime and continuous latent variable permutation alignment.
    \item \textbf{Transition functions:} the averaged $L_2$ distance and $R^2$ score between ground-truth and estimated means $\bm{m}(\cdot,k)$ after regime permutation and latent variable affine alignment.
\end{itemize}

\paragraph{Baseline comparison.} We compare iSDS against five representative baselines in identifiable representation learning:
\begin{itemize}[leftmargin=1.5em]
  \setlength{\itemsep}{5pt}
  \setlength{\parskip}{0pt}
  \setlength{\parsep}{0pt}
    \item \textbf{iVAE*} \citep{khemakhem2020variational} where ground-truth discrete regimes are provided as auxiliary information (a stronger baseline). 
    \item \textbf{SNICA} \citep{halva2021disentangling}, based on linear switching dynamical systems, where each latent variable follows its own SDS.
    \item \textbf{TDLR} \citep{yao2022temporally}, which assumes non-Gaussianity of temporal latent variables but no switches.
    \item \textbf{NCTRL} \citep{song2023temporally}, based on SDSs, but assuming Gaussianity on the observations, and certain non-Gaussianity on the latent variables.
    \item \textbf{CtrlNS} \citep{song2024causal}, which assumes the number of switches is known and directly depends on the observations.
\end{itemize}

To compare iSDS with other baselines, we explore different dataset settings, A to F, all with $n=10$ observed dimensions and the following characteristics:
\begin{itemize}[leftmargin=2em]
  \setlength{\itemsep}{5pt}
    \setlength{\parskip}{0pt}
  \setlength{\parsep}{0pt}
    \item[A)] $m=3$, $K=3$, $M=1$, constant noise.
    \item[B)] $m=3$, $K=3$, $M=1$, heterogeneous noise.
    \item[C)] $m=3$, $K=3$, $M=1$, history-dependent noise.
    \item[D)] $m=5$, $K=3$, $M=2$, history-dependent noise.
    \item[E)] $m=5$, $K=5$, $M=2$, heterogeneous noise.
    \item[F)] $m=5$, $K=5$, $M=5$, heterogeneous noise.
\end{itemize}

\begin{figure}
    \centering
    \includegraphics[width=\linewidth]{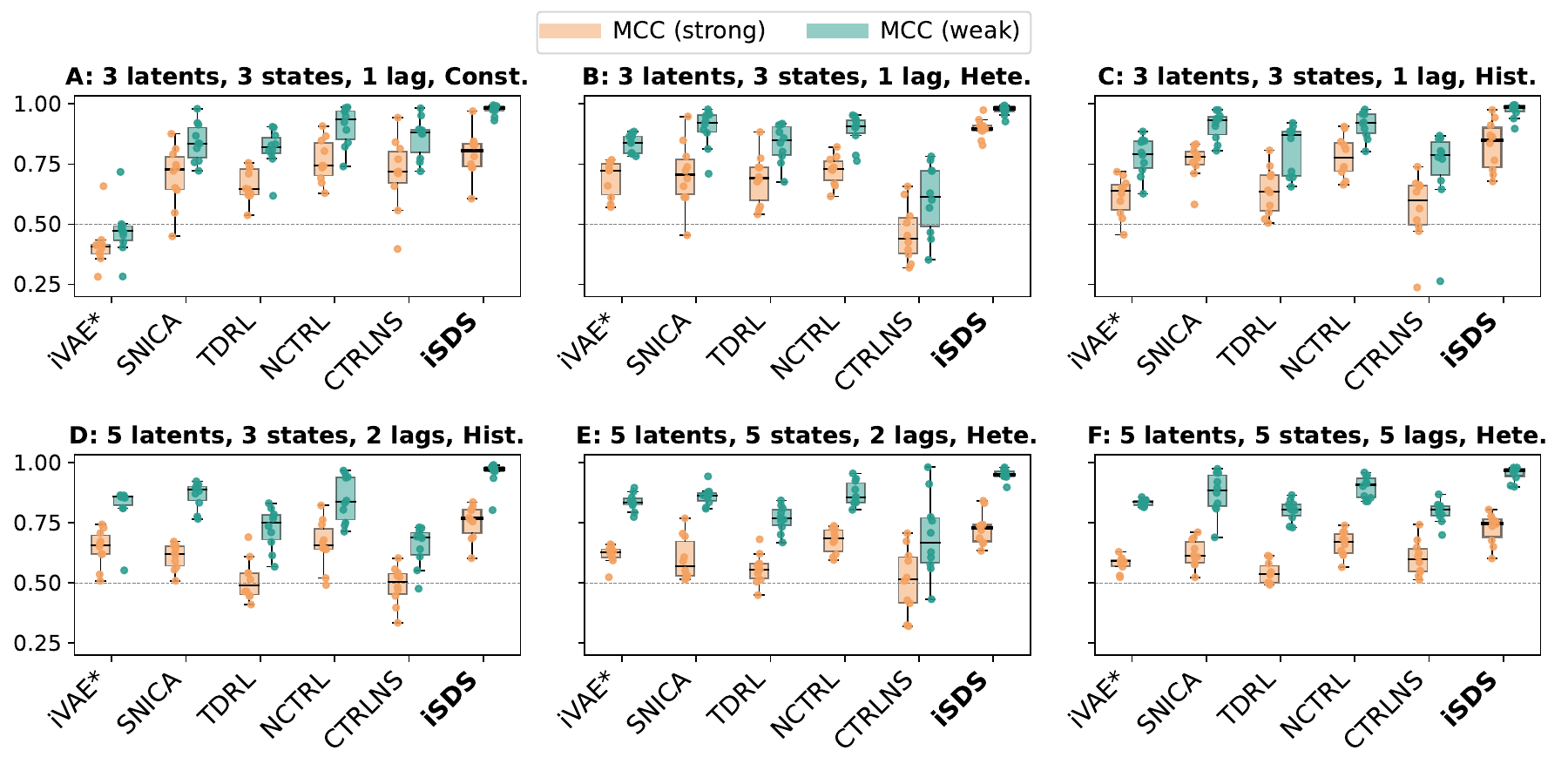}
    \caption{Weak and strong MCC scores across dataset settings A–F (10 data seeds each).}
    \label{fig:synth_exps_baseline_comp}
\end{figure}

Figure~\ref{fig:synth_exps_baseline_comp} shows the weak and strong MCC scores for each configuration (A-F) across 10 dataset seeds. We generally observe that iSDS outperforms the baselines in terms of weak MCC, achieving near-perfect affine alignment in all settings. This is followed by SNICA and NCTRL, as expected since both are SDS-based approaches.
Regarding strong MCC, iSDS again outperforms all baselines, although performance drops for settings A, and D-F. In A-C the only difference lies in noise distribution. Lower strong MCC in A is expected, since identifiability up to scaling and permutation are not guaranteed under constant covariance matrices. In D-F, the lower MCC scores can be attributed to increased complexity from higher lags, dimensions, regimes, and nonstationary noise. 
\begin{figure}
    \centering
    \begin{subfigure}{.64\linewidth}
    \centering
    \includegraphics[width=\linewidth]{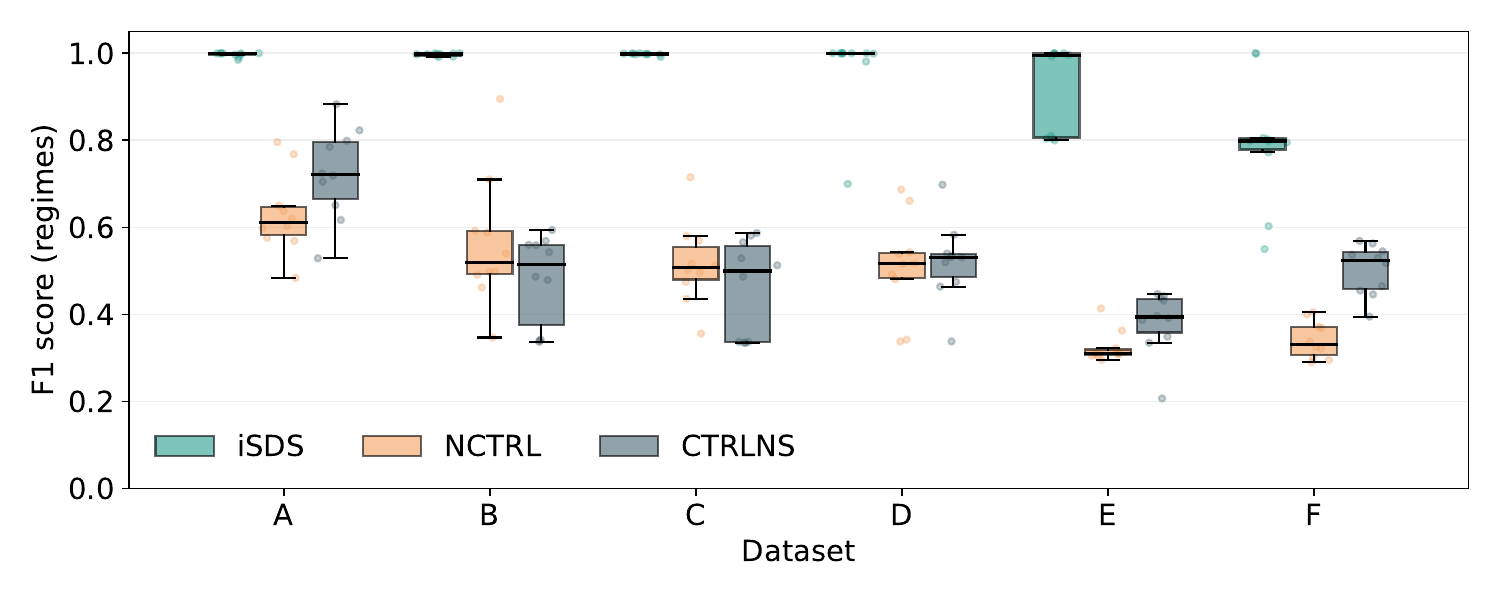}
    \caption{}
    \label{fig:synth_exps_baseline_state}
    \end{subfigure}
    \begin{subfigure}{.32\linewidth}
    \centering
    \includegraphics[width=\linewidth]{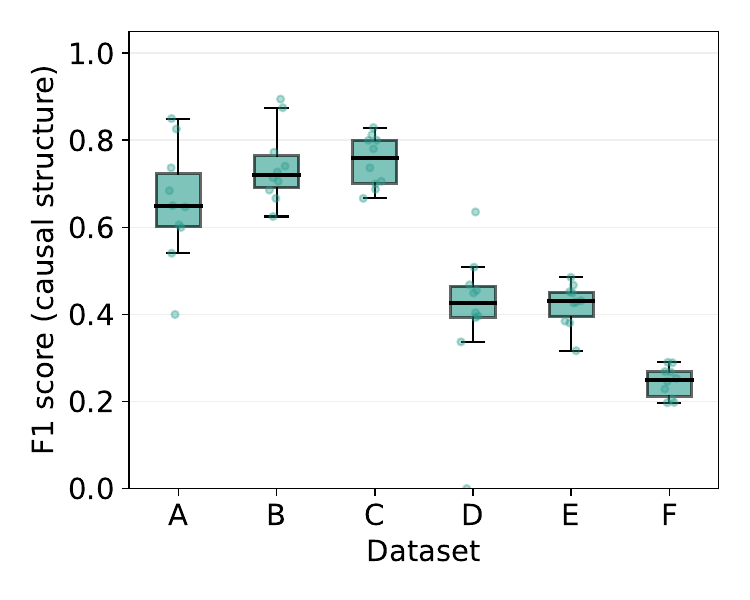}
\caption{}
    \label{fig:synth_exps_causal_str}
    \end{subfigure}

    \caption{(a) Regime estimation and (b) Causal structure recovery across dataset settings A-F (10 data seeds each).}
\end{figure}
\begin{figure}
    \centering
    \begin{subfigure}{.49\linewidth}
    \centering
    \includegraphics[width=\linewidth]{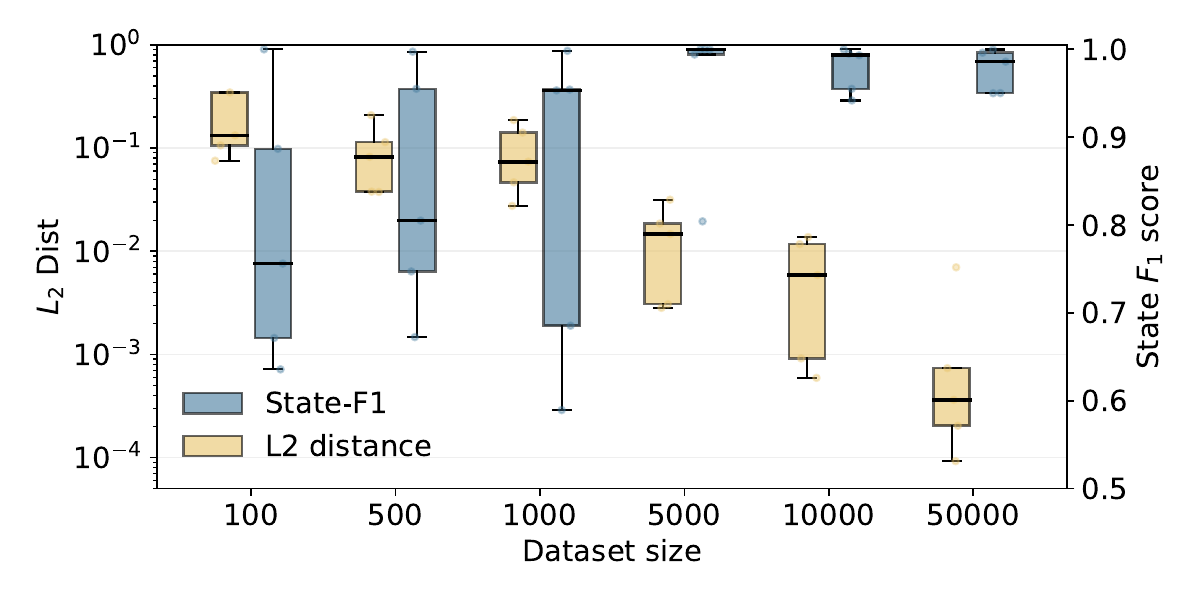}
    \caption{}
    \label{fig:synth_exps_data_size}
    \end{subfigure}
    \hfill
    \begin{subfigure}{.49\linewidth}
    \centering
    \includegraphics[width=\linewidth]{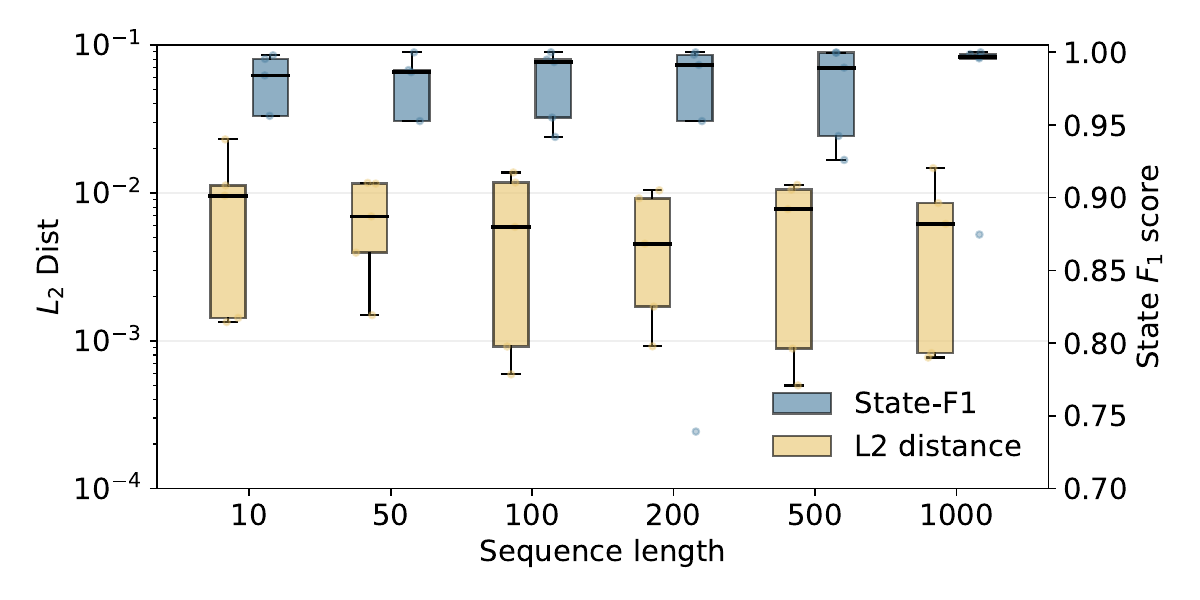}
    \caption{}
    \label{fig:synth_exps_seq_len}
    \end{subfigure}
    \caption{Data efficiency analysis of iSDS by varying (a) dataset size, or (b) sequence length.}
\end{figure}
For regime estimation, Figure~\ref{fig:synth_exps_baseline_state} compares iSDS with NCTRL and CtrlNS as they extract regimes in a comparable way. iSDS consistently achieves superior $F_1$ scores, where performance decreases with more regimes (E and F). NCTRL performs significantly worse, despite being based on SDSs. This can be explained because NCTRL estimates regimes directly from observations under Gaussianity assumptions, which is a model mismatch as the generated observations are generally non-Gaussian. Finally, Figure \ref{fig:synth_exps_causal_str} reports causal structure recovery. In A-C, the $F_1$ score improves with increasing nonstationarity of the noise, consistent with our theoretical results. In contrast, settings D-F show reduced performance in line with the previously reported MCC scores.

\begin{figure}
    \centering
    \begin{minipage}{.49\textwidth}
\includegraphics[width=\linewidth]{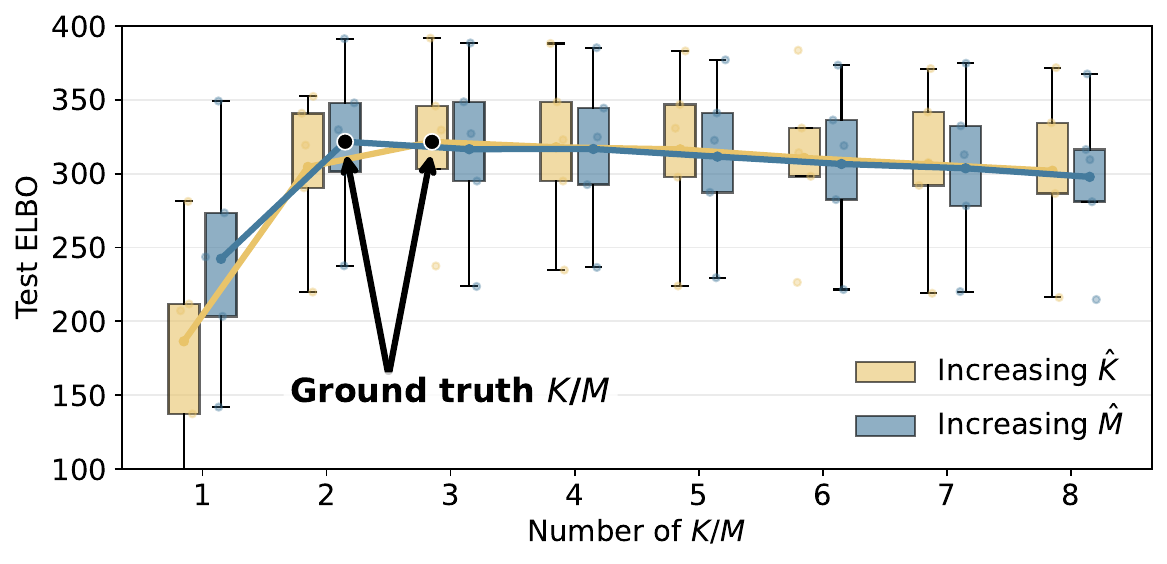}
\caption{Model selection results with iSDS for 5 dataset seeds. Test ELBO for models with increasing number of switches ($\hat{K}$; in yellow), and lags ($\hat{M}$; in blue). Black points indicate ground-truth number of switches and lags ($K=3$ and $M=2$, respectively).}
\label{fig:model_selection}     
\end{minipage}\hfill
    \begin{minipage}{.49\textwidth}
    \vspace{-.35cm}
    \includegraphics[width=\linewidth]{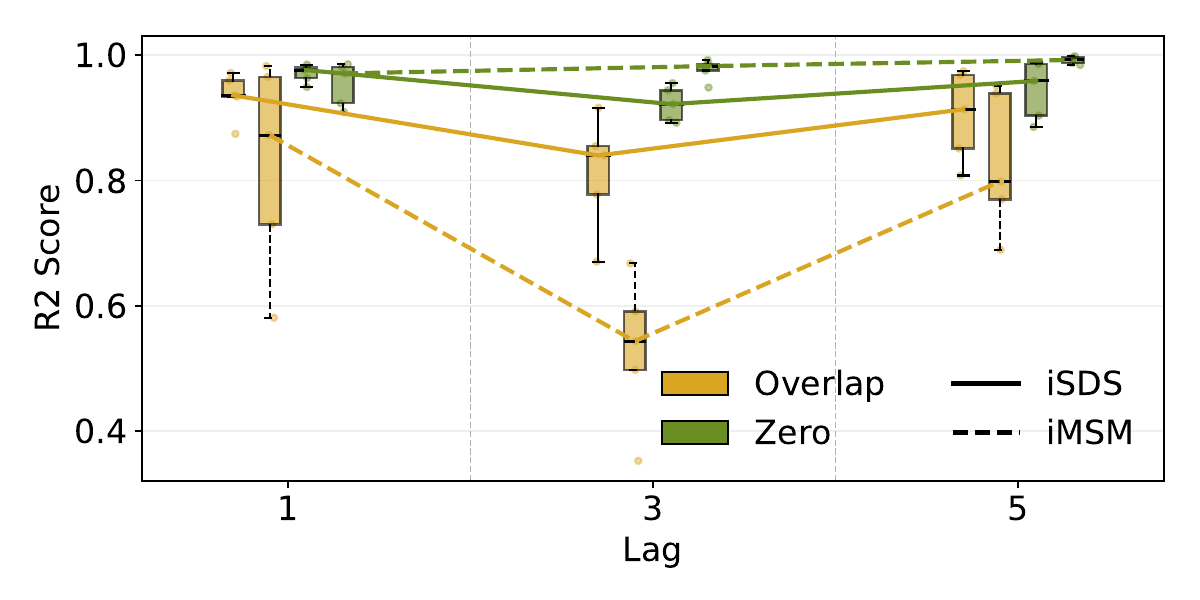}
    \caption{Ablation on intersection assumptions with iMSM and iSDS for 5 dataset seeds. Average $R^2$ of mean functions with zero measure intersections (Zero), and nonzero measure intersections (Overlap) across lags. }
    \label{fig:synth_exps_ablation}     
    \end{minipage}

\end{figure}

\paragraph{Data efficiency analysis.}
To better understand the sample and sequence efficiency of iSDS, we generate data with $n=m=5$ latent and observed dimensions, $M=2$ lags, $K=3$ regimes, and constant noise. We evaluate the $F_1$ score of the estimated regimes, and the average $L_2$ distance of mean functions $\bm{m}(\cdot)$ after affine alignment. Results are shown in Figures~\ref{fig:synth_exps_data_size} and~\ref{fig:synth_exps_seq_len}, where we vary dataset size and sequence length, respectively. In Figure~\ref{fig:synth_exps_data_size}, the sequence length is fixed to $T=100$, and in Figure~\ref{fig:synth_exps_seq_len}, the number of samples is fixed to $N=10000$. 

We observe that larger sample sizes substantially improve performance in terms of both regime recovery and mean function estimation. However, increasing the sequence length has little effect for a sufficient number of samples. The results suggest that, when training iSDS, collecting more sequences is generally more beneficial than increasing their length.

\paragraph{Selecting $K$ and $M$.} Our theoretical framework establishes that for both iMSM and iSDS, the number of regimes $K$ is identifiable. Assuming consistency holds for iSDS, our results facilitate learning the ground truth $K$ by model selection \citep{mclachlan2000finite}. Similar arguments apply for selecting $M$, where for a smaller lag the model becomes misspecified, and for larger $M$, extra lags may be unused. We show this experimentally, and leave more thorough theoretical analyses for future work. To select $K$ and $M$ we use the elbow method \citep{thorndike1953belongs}, which is a standard heuristic for finite mixture models. For iSDS, we use the ELBO on a held-out dataset, rather than the log-likelihood which could be possible for iMSM. Figure \ref{fig:model_selection} showcases a model selection example for iSDS with $M=2$ lags and $K=3$ where we increase the estimated regimes $\hat{K}$ while maintaining the ground-truth lag $M$ (yellow line), increase the estimated lag $\hat{M}$ while maintaining the ground-truth number of regimes (blue line). In both cases, we observe that the highest test ELBO is reached at the ground-truth $M$ and $K$, and the performance plateaus for higher choices of $\hat{M}$ and $\hat{K}$. 

\paragraph{Ablations.} Section \ref{sec:msm_theory} establishes identifiability for parametric MSMs with analytic functions and unique indexing. For $M=1$, a sufficient condition is the existence of an open set where regime-specific mean (and covariance) functions do not intersect. However, for general $M$, we require this set to have full measure. These are sufficient, but not proven necessary, conditions. We assess necessity experimentally by designing two settings for both iMSM and iSDS: (i) \textsc{Zero}, where the ground-truth mean functions intersect at most on a zero measure set; and (ii) \textsc{Overlap}, where functions coincide on a set with nonzero measure but differ elsewhere. We consider $n=m=5$, $K=3$, constant noise, and lags $M\in\{1, 3, 5\}$.

Figure~\ref{fig:synth_exps_ablation} reports the average $R^2$ of the mean functions $\bm{m}(\cdot)$ after regime permutation alignment (with an additional affine alignment for iSDS). Using $R^2$ rather than $L_2$ facilitates comparisons across lags with different input/output scales. As expected, \textsc{Zero} outperforms \textsc{Overlap}. However, \textsc{Overlap} shows strong fits with high $R^2$ scores, far from indicating model failure. For iSDS, the affine alignment can cause the $R^2$ scores to be higher. However, for iMSM we observe strong fits for $M=5$. This suggests that the stronger full measure non-intersection is sufficient but not necessary in practice, as otherwise we would obtain much lower $R^2$ scores (e.g., $R^2<0.5$). Therefore, the decrease in performance in \textsc{Overlap} can be attributed to estimation challenges from the overlap region, where tracking the regimes is more difficult, rather than to a fundamental lack of identifiability.

\section{Applications to Real-World Data}\label{sec:real_world}

Having established the validity of our theoretical results under controlled conditions, we now turn to real-world data, where ground-truth structure is generally unavailable. In this section, we therefore focus on assessing the applicability of the theory developed in Section~\ref{sec:theory} in a fully unsupervised setting. We study datasets from diverse scientific domains and investigate whether assumptions on smoothness and nonstationarity (i.e., assumptions (m2) and (s2)) facilitate learning structured latent representations. Importantly, our identifiability-based framework establishes that learned representations are one-to-one with respect to observations, up to acceptable equivalences (e.g., scaling, permutation). This uniqueness property allows a more trustworthy interpretation of the recovered structures, allowing us to relate them to established domain knowledge. Overall, our aim in this section is to demonstrate how the proposed iMSM and iSDS can support scientific discovery in real-world data.

\subsection{Brain Activity Data}\label{sec:neuroscience}

We first illustrate iMSM on electrocorticography (ECoG) recordings from the NeuroTycho database\footnote{\url{http://neurotycho.org/}}, originally presented by \citet{yanagawa2013}. Signals from 128 electrodes were recorded in a macaque monkey under two conditions: normal wakefulness (\textsc{Awake}) and propofol-induced loss of consciousness (\textsc{Anaesthetised}). Following prior work \citep{mediano2023spectrally}, we expect the awake condition to exhibit more complex dynamics, reflected in more frequent transitions between latent regimes. 

\begin{figure}
    \centering
    \begin{subfigure}{.32\linewidth}
        \centering
        \includegraphics[width=\linewidth]{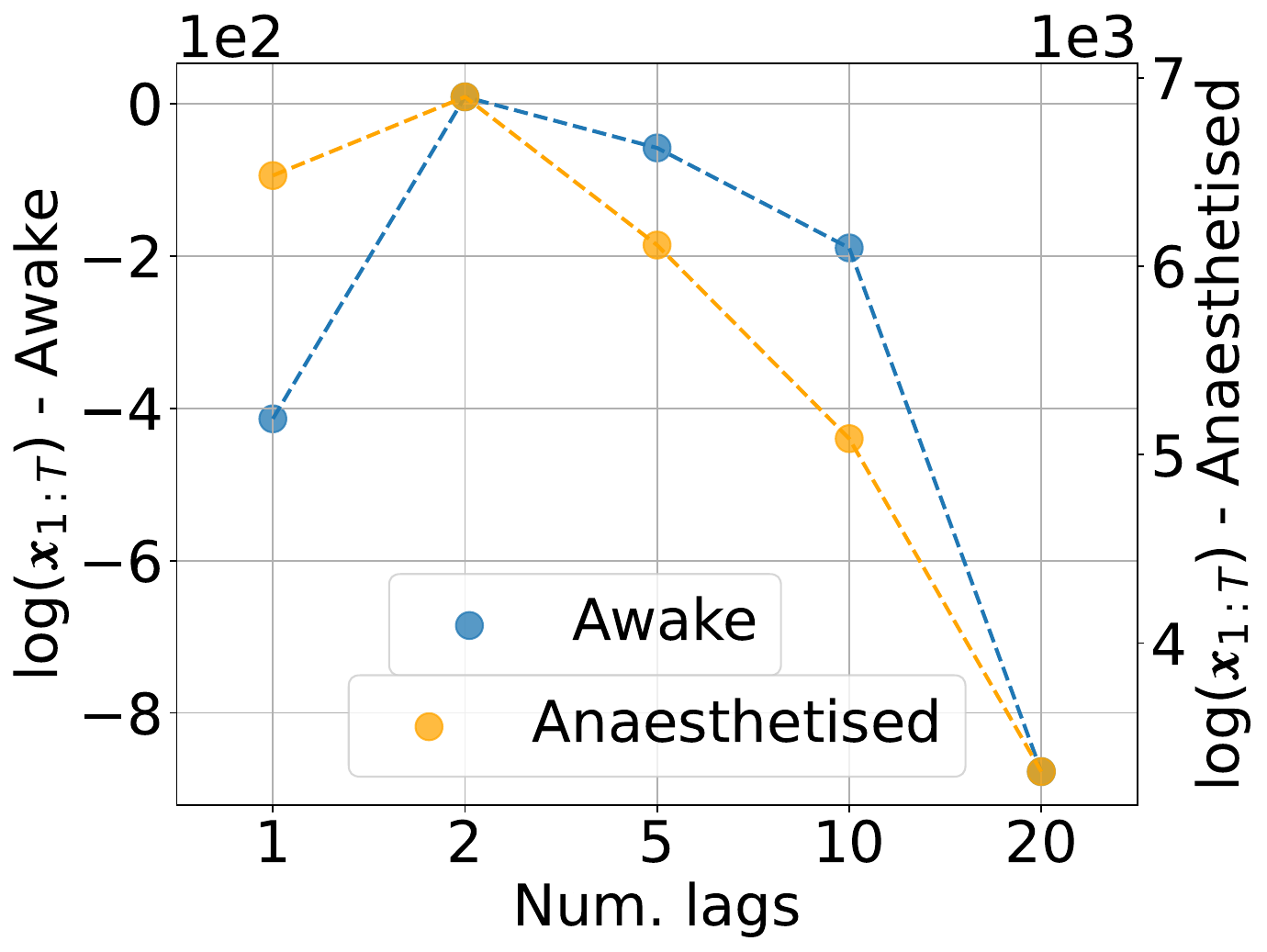}
        \vspace{-1.5em}
        \caption{}
        \label{fig:neuro_loglikelihood}
    \end{subfigure}
        \begin{subfigure}{.65\linewidth}
        \centering
        \includegraphics[width=\linewidth]{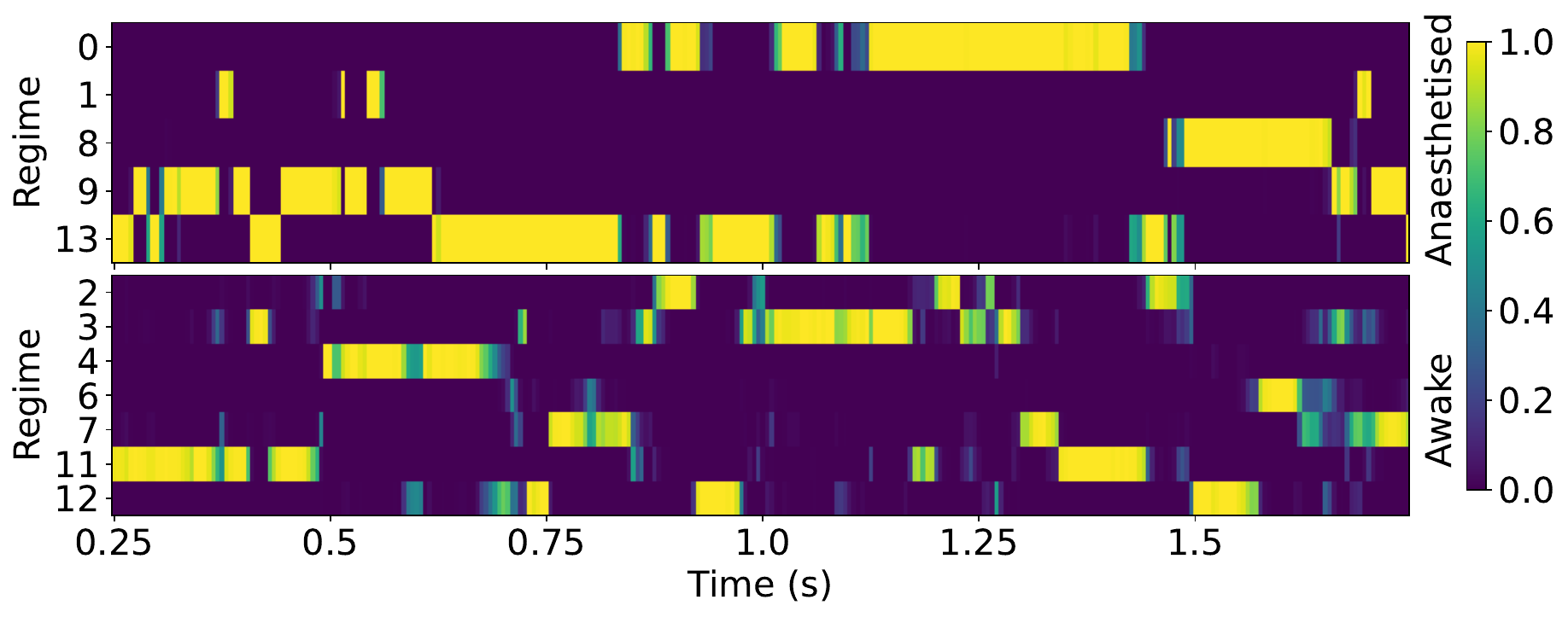}
        \vspace{-1.5em}
        \caption{}
        \label{fig:example_neuro}
    \end{subfigure}
    \vspace{-.5em}
    \caption{(a) Test log-likelihood of ECoG data using different lags and $K=15$ regimes. (b) Example posteriors over regimes for \textsc{Awake} and \textsc{Anaesthetised} conditions. Posteriors are smoothed with a length-$3$ uniform kernel to suppress frame-level artifacts.}
    \label{fig:neuroscience_exp}
    \vspace{-1em}
\end{figure}

The raw ECoG is sampled at 1kHz and comprises \textsc{Awake} and \textsc{Anaesthetised} segments, each lasting 929.8 and 650.65 seconds respectively. We select 21 electrodes approximately covering the visual cortex of the brain. We first apply a second-order Butterworth notch filter at 50Hz to eliminate line noise. Then, we downsample the data to 200Hz, standardise each channel independently, and chunk each sequence into epochs of 2 seconds ($T=400$). This yields $464$ \textsc{Awake} and $325$ \textsc{Anaesthetised} segments, with $50$ from each condition held out for testing.

In both data settings, we set $K=15$ and parametrise transition functions with cosine activations (model selection details in Appendix~\ref{app:neuroscience}). Figure \ref{fig:neuro_loglikelihood} reports test log-likelihoods for iMSM as a function of lag order $M$. As expected, \textsc{Awake} is consistently harder to model (lower log-likelihoods), indicating more complex dynamics. iMSM with $M=2$ achieves the best performance in both settings. Notably, the \textsc{Anaesthetised} condition remains competitive at $M=1$ but degrades rapidly for $M>2$, consistent with simpler dynamics. In contrast, \textsc{Awake} benefits significantly from higher lag orders.

Figure~\ref{fig:example_neuro} illustrates posterior regime occupancy for 2\,s epochs in each condition (smoothed by a length-$3$ uniform kernel). We only show regimes with more than $1\%$ occupancy in the epoch. \textsc{Anaesthetised} shows prolonged stays within a smaller subset of regimes, whereas \textsc{Awake} occupies a larger fraction of regimes with more frequent transitions. Table~\ref{tab:table_neuro} quantifies these differences: relative to \textsc{Anaesthetised}, \textsc{Awake} uses more regimes overall and per sample, exhibits higher transition frequency (Hz), and has shorter regime durations on average.

Crucially, iMSM provides regime identifiability guarantees to ensure the differences in regime usage reflect condition-specific structure supported by the data. This highlights the importance of considering iMSM or iSDS for neuroscience time series such as ECoG, complementing existing methodologies \citep{mediano2023spectrally}.

\begin{table}
    \centering
    \caption{Regime statistics on test visual-cortex ECoG data with $M=2$ and $K=15$. ``Total usage'' is the fraction of regimes occupied at least once across the test set. ``Usage per sample'' is the mean fraction of regimes occupied within a single 2\,s epoch. ``Max\ duration'' is the mean of the per-epoch maximum dwell time.}
    \vspace{-.75em}
    \resizebox{\linewidth}{!}{
    \begin{tabular}{l|c|c|c|c}
        \toprule
        Condition & Transition freq. (Hz) & Total usage & Usage per sample & Max. duration (s) \\
        \midrule
        \textsc{Awake} & $22.05\pm1.06$ & $53.33\%$ & $44\%\pm1.47\%$ & $0.22\pm0.02$\\
        \textsc{Anaesthetised} & $17.71\pm1.29$ & $33.33\%$ & $26.40\% \pm 1.48\%$ & $0.35 \pm 0.04$ \\
        \bottomrule
    \end{tabular}}
    \label{tab:table_neuro}

\end{table}

\subsection{Financial Data}\label{sec:financial_data}

Financial time series are a direct application area for regime-switching models. Early developments in MSMs \citep{hamilton1989new} focused on modelling market-cycle regimes and laid the foundations for regime-switching models in finance. In this section, we study US sector-based stock price data and use iSDS to extract structured latent representations and regimes. Related applications are considered with both MSMs \citep{alizadeh2004markov,cavaliere2014investigating,mahmoudi2022detection} and SDSs \citep{azzouzi1999modelling,carvalho2007simulation,fox2011bayesian}. Empirically, researchers typically work with \emph{log returns} or \emph{log price(-dividend)}.

Log returns $r_t = \log(\frac{p_t}{p_{t-1}})$ are often treated as stationary \citep{cont2001empirical}, and regime-switching mainly focuses on regime-dependent means and/or volatilities. This includes MSMs combined with conditional heteroskedasticity (e.g., GARCH \citep{bollerslev1986generalized}), where the conditional variance depends on latent variables and/or past returns \citep{ang2002international,carvalho2007simulation}. Such models often interpret states as bull/bear or tranquil/turbulent and are widely used for risk assessment and allocation \citep{guidolin2007asset}.

Because prices are typically nonstationary, regime-switching on log price data is often modelled through present-value frameworks that separate company fundamentals (e.g., earnings) from speculative components \citep{balke2009market,chan2021speculative}. These models capture longer-horizon regimes (persistent bear/bull phases), and speculative episodes.

\paragraph{Data modality and identifiability.}

We analyse both data modalities (log returns and log prices) to illustrate our theory in a real-world setting.
Our SDS results imply that regimes are identifiable under suitable modelling conditions (Theorem~\ref{thm:identifiability_sds:i}). For continuous latents, identifiability up to permutation and scaling requires certain heterogeneous or history-dependent noise (Theorem \ref{thm:identifiability_sds:ii}; Corollary \ref{cor:causality_identifiability}). More specifically, it requires for a pair of regimes to have distinct noise ratios across dimensions. In practice, nonstationary log prices are more likely to satisfy the above conditions due to regime and sector-specific scales. In contrast, the noise scales in log returns are normalised, since they are ratios of adjacent prices. This implies that even history-dependent volatility might not meet the required conditions for permutation and scaling identifiability of the latent variables. We therefore evaluate both settings to assess which one captures better structured representations in iSDS.

\paragraph{Dataset and model setting.}

We source 100 U.S. stock prices from Bloomberg Terminal \citep{bloomberg_us_equities_pxlast_2015_2025}, grouped into 10 sectors, and retain daily prices for the past ten years (Jan.\ 2015–May 2025). This yields a sequence of $2630$  time steps. See Appendix \ref{app:finance} for the company list and sector mapping. For training, we draw random subsequences of $T=200$, from the first $2000$ time steps, and we leave the remaining $630$ for testing. We train iSDS with 11 latent variables and 3 regimes: ten sector latents (one per sector) plus a global latent, and three market regimes intended to capture, e.g., bull/bear/sideways or high/mid/low volatility, similar to related works on financial time series \citep{chan2021speculative}. We set $M=5$ and use heterogeneous covariances, i.e., $\bm{\Sigma}(\cdot,k)=\bm{\Sigma}(k)$ for $k\in[K]$. See Appendix \ref{app:training_specs} for additional training details.

\paragraph{Results on log returns.}

\begin{figure}
    \centering
    \includegraphics[width=\linewidth]{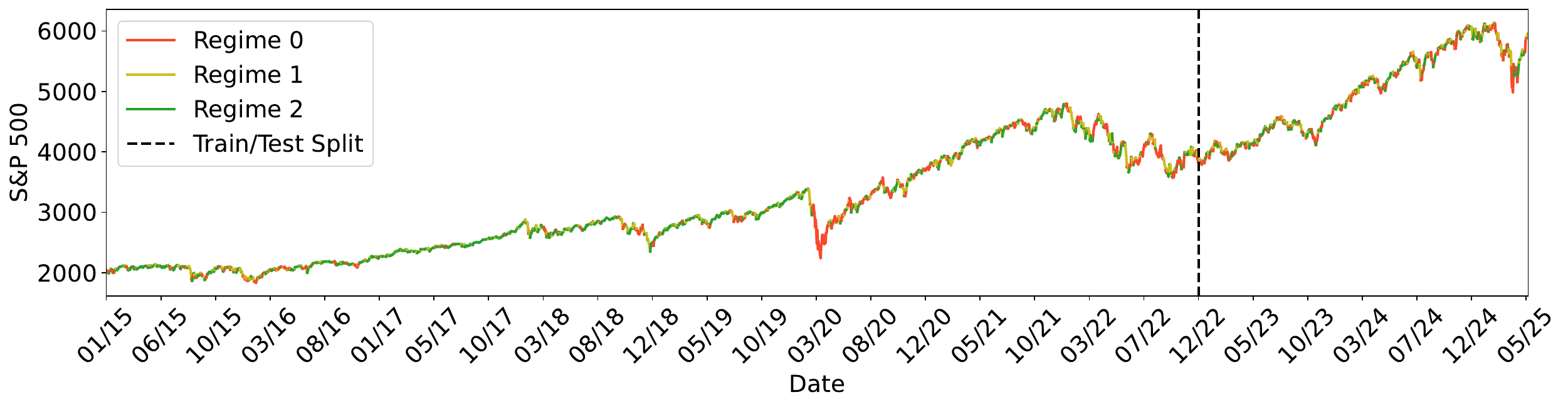}
    \caption{Regime segmentation of log returns overlaid with the S\&P~500; dashed line separates train/test data.}
    \label{fig:sp500regime_log_return}
\end{figure}
\begin{figure}
    \centering
    \includegraphics[width=\textwidth]{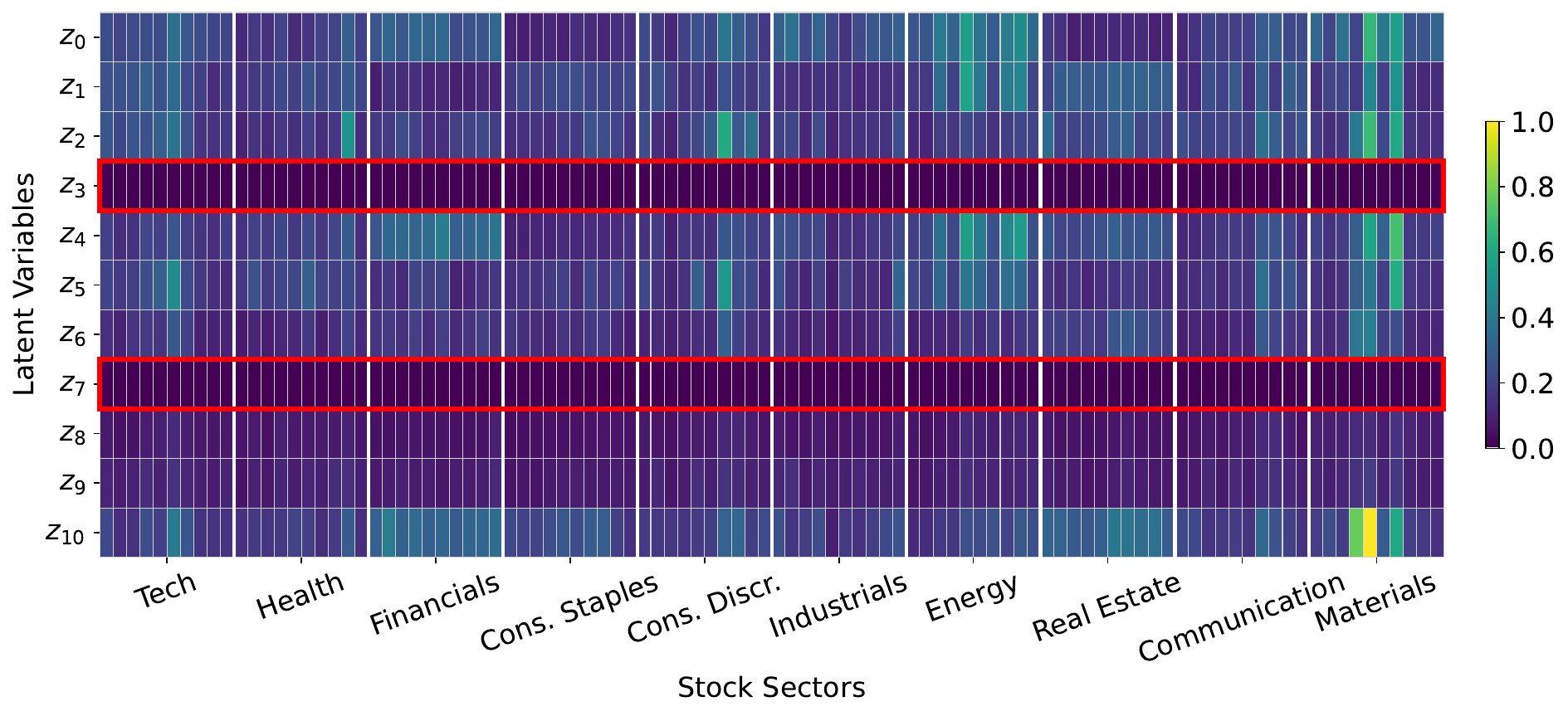}
    \caption{Decoder Jacobian (latents $\to$ assets) for log returns, grouped by sector.}
    \label{fig:jacobian_decoder_log_return}
\end{figure}

Figure~\ref{fig:sp500regime_log_return} shows the training/test segmentation overlaid on the S\&P~500 for context. When trained on log returns, iSDS exhibits regimes capturing short-horizon volatility adaptation, where green, yellow, and red colours correspond to low, medium, and high volatility, respectively. We observe a prolonged low-volatility period from late 2016 to early 2018 and a significantly long high-volatility regime in March 2020 (COVID-19). From 2020 onward, the segmentation consists mainly of medium and high-volatility regimes, consistent with the increased variability of the S\&P~500.

To assess latent structure, Figure~\ref{fig:jacobian_decoder_log_return} displays the decoder Jacobian (latents $\to$ assets), grouped by sector. The pattern shows little sector-specific structure: multiple latent variables mapping to the same sector (e.g., significant concentration around the ``Materials" sector), while some latents (e.g., $z_3$ and $z_7$) are weakly used. This aligns with our theory: the ratio in log returns normalises noise scales across dimensions, preventing disentanglement theoretically.

\paragraph{Results on log prices.} Figure~\ref{fig:sp500regime_log_price} reports the corresponding segmentation for log prices. Here, iSDS extracts longer-horizon regimes: a green state aligned with bull periods, and a yellow state aligned with lateral/bear phases. Compared to the returns model, the inferred regimes are more persistent and track macroeconomic phases, consistent with the literature on present-value frameworks \citep{chan2021speculative}.

The decoder Jacobian (Fig.~\ref{fig:jacobian_decoder_log_price}) shows a significantly clearer structure than in the returns case (cf.\ Fig.~\ref{fig:jacobian_decoder_log_return}). Several latents align with sectors or sector pairs (e.g., $z_5$: Materials/Consumer Staples; $z_5$, $z_7$: Consumer Staples; $z_8$: Real Estate; $z_4$, $z_6$: Health Care). Other latents ($z_2$ and $z_10$) behave as global factors. Some sectors (e.g., Technology and Industrials) spread across multiple latents, which is plausible given their cyclicality and sensitivity to market conditions. In general, the log price model yields a more structured, sector-aware representation, which is aligned with our identifiability expectations. The lack of stronger sector-specific disentanglement is also consistent with our synthetic results, where iSDS did not achieve high strong MCC scores.

\begin{figure}
    \centering
    \includegraphics[width=\linewidth]{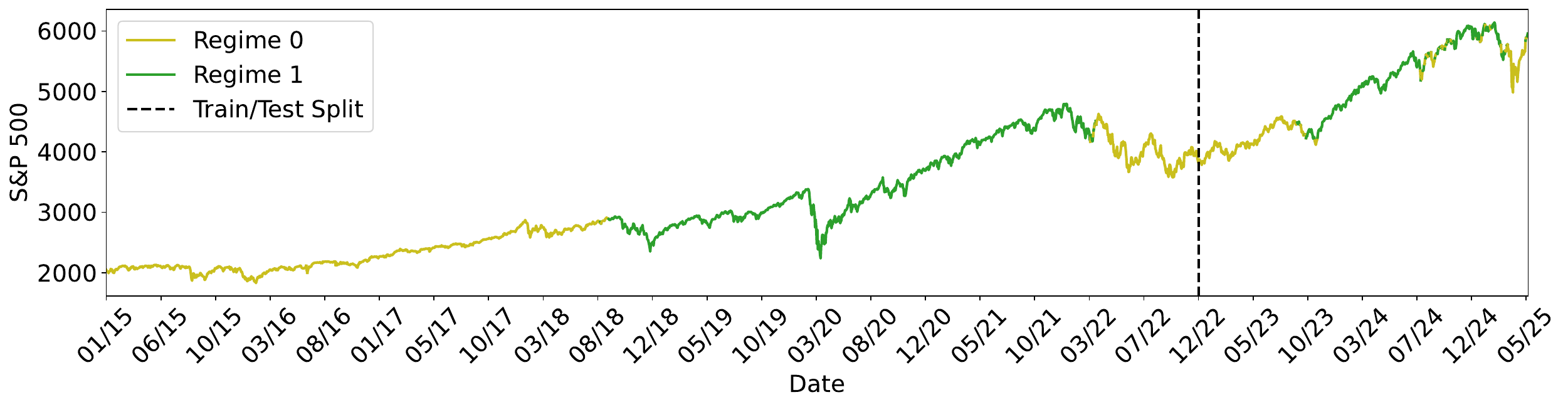}
    \caption{Regime segmentation of log prices overlaid with the S\&P~500; dashed line separates train/test.}
    \label{fig:sp500regime_log_price}
\end{figure}
\begin{figure}
    \centering
    \includegraphics[width=\linewidth]{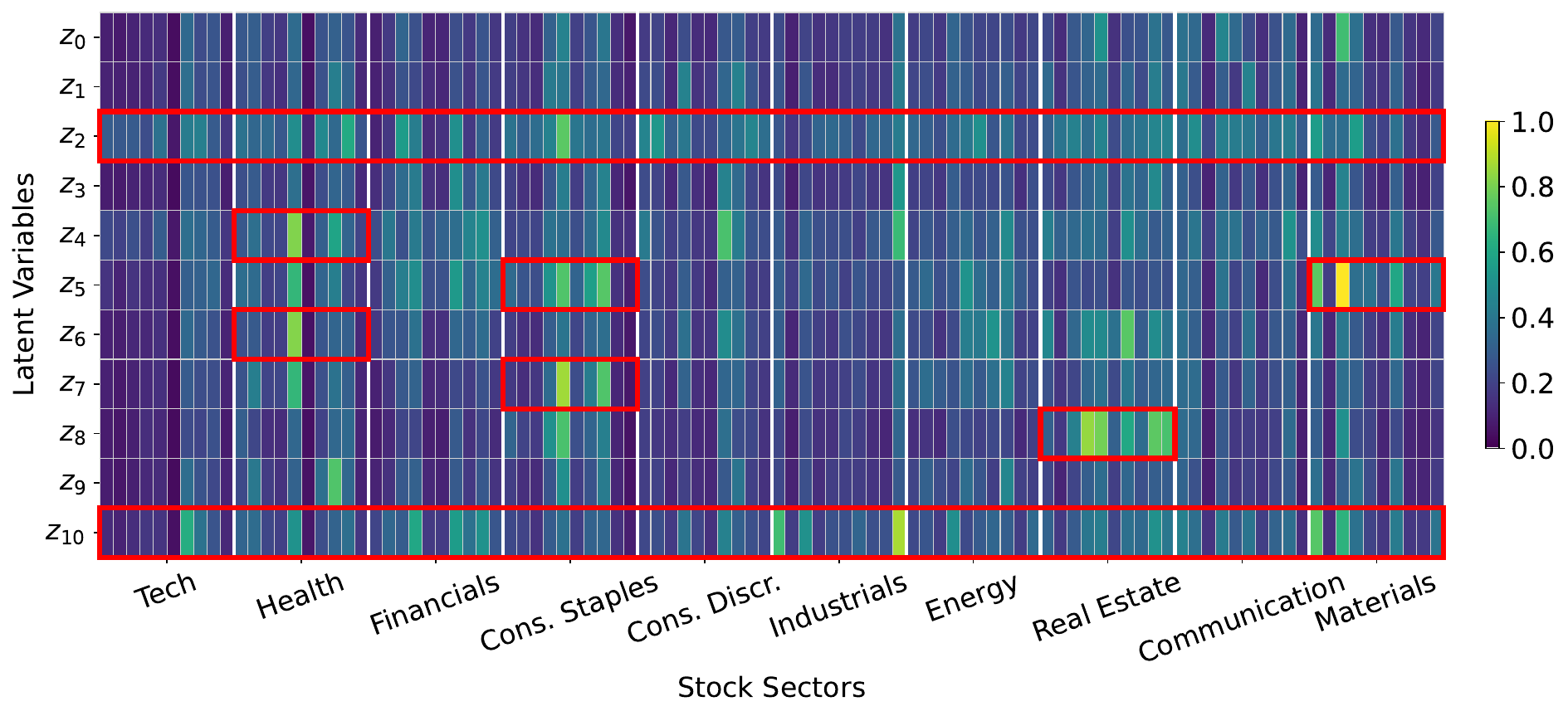}
    \caption{Decoder Jacobian (latents $\to$ assets) for log prices, grouped by sector.}
    \label{fig:jacobian_decoder_log_price}
\end{figure}

\paragraph{Summary.} Comparing both modalities, iSDS on log returns discovers short-time volatility regimes with weak sector disentanglement, while log prices yield more persistent macro regimes with clearer sector-aligned latents. The stronger structured representation in the log price modality is consistent with the nonstationarity of the data and our heterogeneity requirements for identifiable representations. For additional details, in Appendix \ref{app:finance} we also investigate the learned latent structures of iSDS in both data modalities.

\subsection{Climate Data}\label{sec:climate}

Climate data is characterised by strong seasonal patterns and regime-dependent approaches have recently been explored in this context \citep{saggioro2020reconstructing}. Modelling Earth systems is crucial to understand climatological dynamics and to detect extreme events. In this section, we investigate high-dimensional spatio-temporal climate data to demonstrate the applicability of our theoretical framework to this scientific domain. 

We consider the Normalised Difference Vegetation Index (NDVI) over Africa, following a setup similar to that of \citet{varando2022learning}. The NDVI measures vegetation greenness as the normalised ratio between near-infrared and red reflectance, and is widely used as a proxy for vegetation productivity and land-surface phenology \citep{nathalieNDVI2013}. We use the \href{https://www.earthdata.nasa.gov/data/catalog/lpcloud-mod13c1-006}{MODIS/Terra Vegetation Indices 16-Day L3 Global 0.05Deg CMG V006} product \citep{Didan2015-mg} together with the ENSO 3.4 climate index, both provided in preprocessed form by \citet{varando2022learning}. The dataset follows similar preprocessing described in their work: NDVI is reprojected to a 25km Equal-Area Scalable Earth (EASE) grid\footnote{Information on EASE grid reprojections can be found here \url{https://nsidc.org/data/user-resources/help-center/guide-ease-grids}.}, retaining the pixels corresponding to Africa (bounding box $[-20,\ 55]$, $[-37,\ 38]$ in [latitude, longitude]). For each year, the product provides NDVI values from day 1 to 353 every 16 days. We select the period 2001-2020. Missing values are filled temporally (per pixel) using linear interpolation. The resulting dataset is a video tensor with $460$ time steps and $289\times356$ pixels. The ENSO 3.4 index is sampled to match the 16-day temporal resolution. Since the index is available daily, a 16-day sliding window is used, followed by $16\times$ downsampling. With these datasets, we expect our iMSM and iSDS to extract regimes corresponding to seasonal dynamics and, similarly to Granger-PCA \citep{varando2022learning}, investigate whether the learned latent variables align with ENSO 3.4.

\paragraph{Model setting.} We select the $360$ most recent time steps, and reserve the first $100$ for testing and model selection. We train iMSM following \citet{varando2022learning}, where we use PCA on the land pixels (a total of $56196$ pixels) and retain $10$ principal components. For iSDS, we use random $64\times 64$ patches from land pixels, and use CNN-MLP encoder-decoder. 
Here we consider $20$ latent variables to account for additional spatial features on the continent shape. Given the inherent seasonality in climate, we assume known $K=4$ (one regime per season). Through model selection, we set heterogeneous covariances, GELU activations in the transitions, $M=1$ for iMSM, and $M=2$ for iSDS. See Appendices \ref{app:training_specs} and \ref{app:climate_details} for additional training and model selection details, respectively.

\begin{figure}
    \centering
    \includegraphics[width=\linewidth]{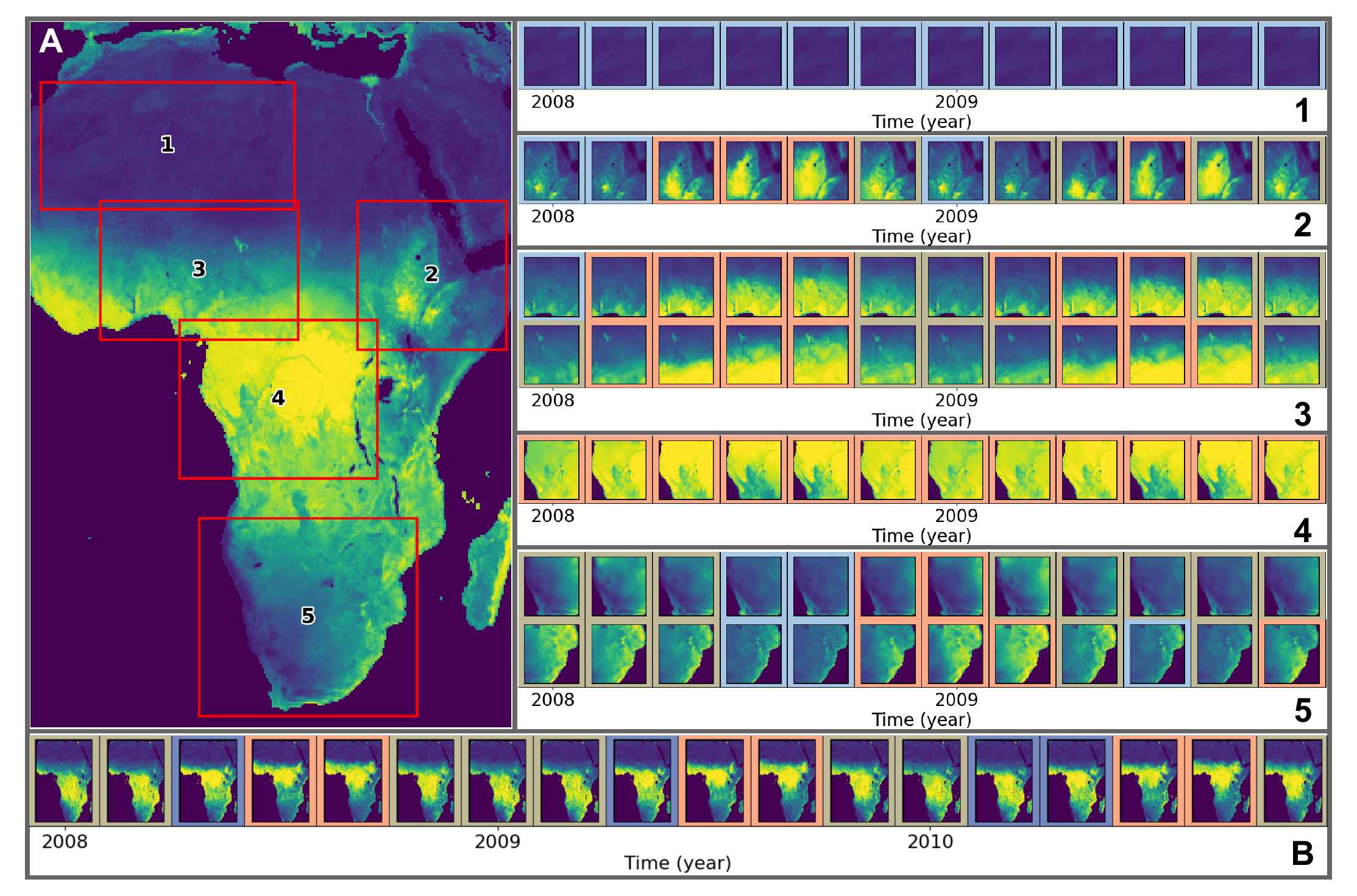}
    \caption{A. Segmentation examples of iSDS across the African continent: (1) Sahara, (2) Ethiopia, (3) Sahel, (4) Congo and, (5) South Africa. B. Segmentation results of iMSM.}
    \label{fig:climate_image}
\end{figure}

\paragraph{Segmentation results.} Figure \ref{fig:climate_image} shows segmentation results of both models. Because iMSM is trained on $10$ principal components, the extracted regimes act globally across the continent. We observe three coherent seasonal regimes recurring each year: Mar-May, Jun-Aug, and Sep-Feb. In contrast, iSDS yields region-dependent regime structures: one regime is predominant over the Sahara and Congo (patches 1 and 4 in Figure \ref{fig:climate_image}A); while two/three regimes alternate over Ethiopia, the Sahel, and South Africa (patches 2, 3, and 5 in Figure \ref{fig:climate_image}B). Although the spatial regime structure differs across regions, the underlying regime effects are shared in iSDS: for example, South Africa's segmentation is approximately the opposite of the Sahel/Ethiopia segmentations. Visualisations of global and regional year-based structures are provided in Appendix \ref{app:climate_details}.

\paragraph{Correlation with ENSO.} To quantitatively assess the learned regime-dependent dynamics, we relate the learned latents to the preprocessed ENSO 3.4 time series. For iMSM (global), we split the time series into seasonal slices determined by the regimes and compute lagged Pearson correlations between each PCA component and ENSO 3.4 over lags ranging from $0$-$240$ days ($16$-day steps; ENSO leads), and also report the ``All times'' baseline (i.e., no splits). For iSDS (regional), we sample $10$ $64 \times 64$ land patches per region, compute latent-ENSO correlation matrices (latents $\times$ lags), and combine them with Fisher $z$-averaging. For each season/region, we report the peak correlation strength and its lag. Full correlation matrices, and additional visualisations are provided in Appendix \ref{app:climate_details}.

\begin{figure}
  \centering
    \begin{minipage}{.4\linewidth}
      \centering
    \captionof{table}{Per-season and global ENSO alignment with iMSM PCs. Peak correlation and lag with ENSO leading.}
    \resizebox{\linewidth}{!}{
    \begin{tabular}{lcccc}
          \toprule
          Season & Correlation & Lag (days) \\
          \midrule
          All times & $0.38$ & $80$ \\
          Mar--May    & $0.49$ & $208$ \\
          Jun--Sep    & $0.37$ & $80$ \\
          Oct--Feb & $\textbf{0.53}$ & $\textbf{80}$ \\
          \bottomrule
        \end{tabular}}
      \label{tab:imsm_season_table}
  \end{minipage}\hfill
  \begin{minipage}{0.55\linewidth}
  \captionof{table}{Per-season ENSO alignment with iSDS latents. For each region, we report averaged peak correlation and lag with ENSO leading.}
  \resizebox{\linewidth}{!}{
       \begin{tabular}{lcccc}
          \toprule
          Region & Season & Correlation & Lag (days) \\
          \toprule
          Sahara & All times & $0.15$ & $48$ \\
          \midrule
          \multirow{3}{*}{Ethiopia} & Jan--Mar    & $\textbf{0.56}$ & $\textbf{144}$ \\
          & Apr--Aug & $0.27$ & $100$ \\
          & Aug--Dec & $0.26$ & $0$ \\
          \midrule
          \multirow{2}{*}{Sahel} & Mar--Aug    & $0.18$ & $112$ \\
          & Sep--Feb & $0.19$ & $0$ \\
          \midrule
          Congo & All times & $0.12$ & $144$ \\
          \midrule
          \multirow{3}{*}{South Africa} & Mar--Jun \& Sep-Nov    & $0.25$ & $64$ \\
          & Jun--Sep & $0.22$ & $128$ \\
          & Dec--Feb & $0.28$ & $192$ \\
          \bottomrule
        \end{tabular}}
    
    \label{tab:isds_season_table}
  \end{minipage}

\end{figure}

The results are summarised in Tables~\ref{tab:imsm_season_table} and~\ref{tab:isds_season_table}. For iMSM (c.f. Table~\ref{tab:imsm_season_table}), we observe regime-dependent splits yield clearer peaks than the ``All times'' baseline. This indicates that aggregating across seasons attenuates the lagged effects of ENSO on the time series. Our results show that iMSM segments the time series into regimes where the effects of ENSO can be perceived with different lags. For iSDS (c.f. Table~\ref{tab:isds_season_table}), the regional responses are heterogeneous. We observe that Sahara, Congo and Sahel are not significantly aligned with ENSO. On the other hand, South Africa shows slightly higher correlation with ENSO in the South African Summer compared to Winter. Furthermore, Ethiopia shows a very high correlation with ENSO during the months of January-March. In climate analysis, researchers study effects of phenomena such as ENSO across lag scales \citep{rouault2022impact}, and within different regions (e.g., Sahel \citep{nicholson2013west} and Ethiopia \citep{Bekele-Biratu}. In general, iMSM and iSDS can provide noteworthy insights on external drivers such as ENSO, as they decouple both seasonal and regional effects.

\paragraph{Summary.} Climate data yields a strong seasonal structure in NDVI over Africa. iMSM captures the global scales, while iSDS reveals region-specific regime structures. Although interpretability of the latent variables is complex in climate time series, our regime identifiability results still enable interpretable regime-conditioned analyses. Extending iSDS with recent identifiable spatiotemporal frameworks could further disentangle regional effects and improve latent space interpretability \citep{wangdiscovering}.
\section{Conclusions}

We developed a unified identifiability framework for regime-switching models with multi-lag dependencies. For Markov switching models, we framed the model as a temporal finite mixture and proved identifiability of both the number of regimes and the regime-specific multi-lag transitions in a nonlinear Gaussian setting, with a thorough nonparametric analysis provided in Appendix~\ref{app:proof_identifiability}. For switching dynamical systems, we established identifiability of latent variables up to permutation and scaling and derived conditions under which regime-dependent latent causal graphs are identifiable up to regime and node permutations. Importantly, the guarantees hold in a fully unsupervised setting via flexible architectural (function class) and noise assumptions. Empirically, we validated the theory on synthetic data and demonstrated iMSM and iSDS across scientific domains, showing that identifiability yields trustworthy interpretations by ensuring unique representations up to acceptable equivalences.

While we provide strong identifiability results for iMSM, our iSDS guarantees assume piecewise linear generative functions and specific noise and function classes, which may not hold universally. On the estimation side, we rely on variational inference and do not prove statistical consistency. Consistency results exist in related literature assuming the variational family has universal approximation properties \citep{khemakhem2020variational}. Empirically, we observe strong synthetic performance, with recovery of regime-dependent causal structures becoming more challenging as dimensionality increases.

Our framework is complementary to existing methodologies and suggests several directions for future work. Some examples include incorporating contemporaneous effects, continuous-time dynamics, and infinite mixture processes, as well as going beyond piecewise linearity. These extensions would broaden the applicability of identifiable regime-switching models for scientific discovery. Furthermore, our identifiability results extend beyond standard MSMs with time-homogeneous Markov switching to models with time-dependent regime distributions, offering room for experimental exploration.


\acks{We thank Gherardo Varando and Gustau Camps-Valls (Image Processing Laboratory, Universitat de Val\`encia) for early discussions on the climate application and for sharing preprocessed datasets used in Section~\ref{sec:climate}. C.~B.-R. also thanks colleagues at Imperial College London for helpful discussions on the theoretical results and contents of this paper. Y.W is supported in part by the Office of Naval Research under grant N00014-23-1-2590, the National Science Foundation under grant No. 2231174, No. 2310831, No. 2428059, No. 2435696, No. 2440954, and a Michigan Institute for Data Science Propelling Original Data Science (PODS) grant.}


\vskip 0.2in
\bibliography{main}

@article{alizadeh2004markov,
  title={A Markov regime switching approach for hedging stock indices},
  author={Alizadeh, Amir and Nomikos, Nikos},
  journal={Journal of Futures Markets: Futures, Options, and Other Derivative Products},
  volume={24},
  number={7},
  pages={649--674},
  year={2004},
  publisher={Wiley Online Library}
}

@article{fisher1915frequency,
  title={Frequency distribution of the values of the correlation coefficient in samples from an indefinitely large population},
  author={Fisher, Ronald A},
  journal={Biometrika},
  volume={10},
  number={4},
  pages={507--521},
  year={1915},
  publisher={JSTOR}
}

@article{Bekele-Biratu,
author = {Bekele-Biratu, Endalkachew and Thiaw, Wassila M. and Korecha, Diriba},
title = {Sub-seasonal variability of the Belg rains in Ethiopia},
journal = {International Journal of Climatology},
volume = {38},
number = {7},
pages = {2940-2953},
year = {2018}
}

@inproceedings{zhao2023revisiting,
  title={Revisiting structured variational autoencoders},
  author={Zhao, Yixiu and Linderman, Scott},
  booktitle={International Conference on Machine Learning},
  pages={42046--42057},
  year={2023},
  organization={PMLR}
}

@misc{bloomberg_us_equities_pxlast_2015_2025,
  author       = {{Bloomberg L.P.}},
  title        = {Historical closing prices (PX\_LAST) for U.S. equities—tickers as listed in this paper},
  howpublished = {Bloomberg Terminal, Imperial College London},
  year         = {2025},
  note         = {Data range: 2015-01-01 to 2025-05-31; Accessed 2025-11-09}
}

@inproceedings{
balsellsrodas2025causal,
title={Causal Discovery from Conditionally Stationary Time Series},
author={Carles Balsells-Rodas and Xavier Sumba and Tanmayee Narendra and Ruibo Tu and Gabriele Schweikert and Hedvig Kjellstrom and Yingzhen Li},
booktitle={Forty-second International Conference on Machine Learning},
year={2025},
url={https://openreview.net/forum?id=j88QAtutwW}
}

@inproceedings{
geadah2024parsing,
title={Parsing neural dynamics with infinite recurrent switching linear dynamical systems},
author={Victor Geadah and International Brain Laboratory and Jonathan W. Pillow},
booktitle={The Twelfth International Conference on Learning Representations},
year={2024},
url={https://openreview.net/forum?id=YIls9HEa52}
}

@inproceedings{
karniol-tambour2024modeling,
title={Modeling state-dependent communication between brain regions with switching nonlinear dynamical systems},
author={Orren Karniol-Tambour and David M. Zoltowski and E. Mika Diamanti and Lucas Pinto and Carlos D Brody and David W. Tank and Jonathan W. Pillow},
booktitle={The Twelfth International Conference on Learning Representations},
year={2024},
url={https://openreview.net/forum?id=WQwV7Y8qwa}
}

@inproceedings{liu2023graph,
  title={Graph switching dynamical systems},
  author={Liu, Yongtuo and Magliacane, Sara and Kofinas, Miltiadis and Gavves, Efstratios},
  booktitle={International Conference on Machine Learning},
  pages={21867--21883},
  year={2023},
  organization={PMLR}
}

@article{rouault2022impact,
  title={Impact of El Ni{\~n}o--Southern oscillation on the Benguela upwelling},
  author={Rouault, M and Tomety, FS},
  journal={Journal of Physical Oceanography},
  volume={52},
  number={10},
  pages={2573--2587},
  year={2022}
}

@article{nicholson2013west,
  title={The West African Sahel: A review of recent studies on the rainfall regime and its interannual variability},
  author={Nicholson, Sharon E},
  journal={International Scholarly Research Notices},
  volume={2013},
  number={1},
  pages={453521},
  year={2013},
  publisher={Wiley Online Library}
}

@inproceedings{wangdiscovering,
  title={Discovering Latent Causal Graphs from Spatiotemporal Data},
  author={Wang, Kun and Varambally, Sumanth and Watson-Parris, Duncan and Ma, Yian and Yu, Rose},
  booktitle={Forty-second International Conference on Machine Learning},
year={2025}
}

@incollection{cavaliere2014investigating,
  title={Investigating stock market behavior using a multivariate Markov-switching approach},
  author={Cavaliere, Giuseppe and Costa, Michele and De Angelis, Luca},
  booktitle={Advances in Latent Variables: Methods, Models and Applications},
  pages={185--196},
  year={2014},
  publisher={Springer}
}

@inproceedings{varando2022learning,
  title={Learning causal representations with Granger PCA},
  author={Varando, Gherardo and Fern{\'a}ndez-Torres, Miguel-{\'A}ngel and Mu{\~n}oz-Mar{\'\i}, Jordi and Camps-Valls, Gustau},
  booktitle={UAI 2022 Workshop on Causal Representation Learning},
  year={2022}
}

@article{nathalieNDVI2013,
author = {Pettorelli, Nathalie},
year = {2013},
month = {10},
pages = {1-208},
title = {The Normalized Difference Vegetation Index},
isbn = {9780199693160},
journal = {The Normalized Difference Vegetation Index},
doi = {10.1093/acprof:osobl/9780199693160.001.0001}
}

@misc{Didan2015-mg,
  title     = "{MOD13C1} {MODIS/Terra} Vegetation Indices 16-Day {L3} Global
               0.05Deg {CMG} {V006}",
  author    = "Didan, Kamel",
  publisher = "NASA Land Processes Distributed Active Archive Center",
  year      =  2015
}

@article{bollerslev1986generalized,
  title={Generalized autoregressive conditional heteroskedasticity},
  author={Bollerslev, Tim},
  journal={Journal of econometrics},
  volume={31},
  number={3},
  pages={307--327},
  year={1986},
  publisher={Elsevier}
}

@article{chan2021speculative,
  title={Speculative bubbles in present-value models: A Bayesian Markov-switching state space approach},
  author={Chan, Joshua CC and Santi, Caterina},
  journal={Journal of Economic Dynamics and Control},
  volume={127},
  pages={104101},
  year={2021},
  publisher={Elsevier}
}

@article{balke2009market,
  title={Market fundamentals versus rational bubbles in stock prices: A Bayesian perspective},
  author={Balke, Nathan S and Wohar, Mark E},
  journal={Journal of Applied Econometrics},
  volume={24},
  number={1},
  pages={35--75},
  year={2009},
  publisher={Wiley Online Library}
}

@article{guidolin2007asset,
  title={Asset allocation under multivariate regime switching},
  author={Guidolin, Massimo and Timmermann, Allan},
  journal={Journal of Economic Dynamics and Control},
  volume={31},
  number={11},
  pages={3503--3544},
  year={2007},
  publisher={Elsevier}
}

@article{ang2002international,
  title={International asset allocation with regime shifts},
  author={Ang, Andrew and Bekaert, Geert},
  journal={The review of financial studies},
  volume={15},
  number={4},
  pages={1137--1187},
  year={2002},
  publisher={Oxford University Press}
}

@inproceedings{azzouzi1999modelling,
  title={Modelling financial time series with switching state space models},
  author={Azzouzi, Mehdi and Nabney, Ian T},
  booktitle={Proceedings of the IEEE/IAFE 1999 Conference on Computational Intelligence for Financial Engineering (CIFEr)(IEEE Cat. No. 99TH8408)},
  pages={240--249},
  year={1999},
  organization={IEEE}
}

@article{carvalho2007simulation,
  title={Simulation-based sequential analysis of Markov switching stochastic volatility models},
  author={Carvalho, Carlos M and Lopes, Hedibert F},
  journal={Computational Statistics \& Data Analysis},
  volume={51},
  number={9},
  pages={4526--4542},
  year={2007},
  publisher={Elsevier}
}

@article{fox2011bayesian,
  title={Bayesian nonparametric inference of switching dynamic linear models},
  author={Fox, Emily and Sudderth, Erik B and Jordan, Michael I and Willsky, Alan S},
  journal={IEEE Transactions on signal processing},
  volume={59},
  number={4},
  pages={1569--1585},
  year={2011},
  publisher={IEEE}
}

@article{mahmoudi2022detection,
  title={Detection of structural regimes and analyzing the impact of crude oil market on Canadian stock market: Markov regime-switching approach},
  author={Mahmoudi, Mohammadreza and Ghaneei, Hana},
  journal={Studies in Economics and Finance},
  volume={39},
  number={4},
  pages={722--734},
  year={2022},
  publisher={Emerald Publishing Limited}
}

@article{haas2004new,
  title={A new approach to Markov-switching GARCH models},
  author={Haas, Markus and Mittnik, Stefan and Paolella, Marc S},
  journal={Journal of financial Econometrics},
  volume={2},
  number={4},
  pages={493--530},
  year={2004},
  publisher={Oxford University Press}
}

@article{ghahramani2000variational,
  title={Variational learning for switching state-space models},
  author={Ghahramani, Zoubin and Hinton, Geoffrey E},
  journal={Neural computation},
  volume={12},
  number={4},
  pages={831--864},
  year={2000},
  publisher={MIT Press}
}

@article{kruskal1977three,
  title={Three-way arrays: rank and uniqueness of trilinear decompositions, with application to arithmetic complexity and statistics},
  author={Kruskal, Joseph B},
  journal={Linear algebra and its applications},
  volume={18},
  number={2},
  pages={95--138},
  year={1977},
  publisher={Elsevier}
}

@inproceedings{locatello2019challenging,
  title={Challenging common assumptions in the unsupervised learning of disentangled representations},
  author={Locatello, Francesco and Bauer, Stefan and Lucic, Mario and Raetsch, Gunnar and Gelly, Sylvain and Sch{\"o}lkopf, Bernhard and Bachem, Olivier},
  booktitle={international conference on machine learning},
  pages={4114--4124},
  year={2019},
  organization={PMLR}
}

@inproceedings{
balsells-rodas2024on,
title={On the Identifiability of Switching Dynamical Systems},
author={Carles Balsells-Rodas and Yixin Wang and Yingzhen Li},
booktitle={Forty-first International Conference on Machine Learning},
year={2024},
url={https://openreview.net/forum?id=Eew3yUQQtE}
}

@article{scholkopf2021toward,
  title={Toward causal representation learning},
  author={Sch{\"o}lkopf, Bernhard and Locatello, Francesco and Bauer, Stefan and Ke, Nan Rosemary and Kalchbrenner, Nal and Goyal, Anirudh and Bengio, Yoshua},
  journal={Proceedings of the IEEE},
  volume={109},
  number={5},
  pages={612--634},
  year={2021},
  publisher={IEEE}
}

@article{song2023temporally,
  title={Temporally disentangled representation learning under unknown nonstationarity},
  author={Song, Xiangchen and Yao, Weiran and Fan, Yewen and Dong, Xinshuai and Chen, Guangyi and Niebles, Juan Carlos and Xing, Eric and Zhang, Kun},
  journal={Advances in Neural Information Processing Systems},
  volume={36},
  pages={8092--8113},
  year={2023}
}

@article{song2024causal,
  title={Causal temporal representation learning with nonstationary sparse transition},
  author={Song, Xiangchen and Li, Zijian and Chen, Guangyi and Zheng, Yujia and Fan, Yewen and Dong, Xinshuai and Zhang, Kun},
  journal={Advances in Neural Information Processing Systems},
  volume={37},
  pages={77098--77131},
  year={2024}
}

@article{yakowitz1968identifiability,
  title={On the identifiability of finite mixtures},
  author={Yakowitz, Sidney J and Spragins, John D},
  journal={The Annals of Mathematical Statistics},
  volume={39},
  number={1},
  pages={209--214},
  year={1968},
  publisher={Institute of Mathematical Statistics}
}

@inproceedings{linderman2017bayesian,
  title={Bayesian learning and inference in recurrent switching linear dynamical systems},
  author={Linderman, Scott and Johnson, Matthew and Miller, Andrew and Adams, Ryan and Blei, David and Paninski, Liam},
  booktitle={Artificial Intelligence and Statistics},
  pages={914--922},
  year={2017},
  organization={PMLR}
}

@article{johnson2016composing,
  title={Composing graphical models with neural networks for structured representations and fast inference},
  author={Johnson, Matthew J and Duvenaud, David K and Wiltschko, Alex and and Datta, Sandeep R and Adams, Ryan P },
  journal={Advances in neural information processing systems},
  volume={29},
  year={2016}
}

@article{linderman2016recurrent,
  title={Recurrent switching linear dynamical systems},
  author={Linderman, Scott W and Miller, Andrew C and Adams, Ryan P and Blei, David M and Paninski, Liam and Johnson, Matthew J},
  journal={arXiv preprint arXiv:1610.08466},
  year={2016}
}

@article{halva2021disentangling,
  title={Disentangling identifiable features from noisy data with structured nonlinear ICA},
  author={H{\"a}lv{\"a}, Hermanni and Le Corff, Sylvain and Leh{\'e}ricy, Luc and So, Jonathan and Zhu, Yongjie and Gassiat, Elisabeth and Hyvarinen, Aapo},
  journal={Advances in Neural Information Processing Systems},
  volume={34},
  pages={1624--1633},
  year={2021}
}

@article{zheng2018dags,
  title={Dags with no tears: Continuous optimization for structure learning},
  author={Zheng, Xun and Aragam, Bryon and Ravikumar, Pradeep K and Xing, Eric P},
  journal={Advances in neural information processing systems},
  volume={31},
  year={2018}
}

@inproceedings{khemakhem2020variational,
  title={Variational autoencoders and nonlinear ica: A unifying framework},
  author={Khemakhem, Ilyes and Kingma, Diederik and Monti, Ricardo and Hyvarinen, Aapo},
  booktitle={International Conference on Artificial Intelligence and Statistics},
  pages={2207--2217},
  year={2020},
  organization={PMLR}
}

@inproceedings{kingma2013autoencoding,
  author       = {Diederik P. Kingma and
                  Max Welling},
  editor       = {Yoshua Bengio and
                  Yann LeCun},
  title        = {Auto-Encoding Variational Bayes},
  booktitle    = {2nd International Conference on Learning Representations, {ICLR} 2014,
                  Banff, AB, Canada, April 14-16, 2014, Conference Track Proceedings},
  year         = {2014},
  url          = {http://arxiv.org/abs/1312.6114},
}

@article{gassiat2016inference,
  title={Inference in finite state space non parametric hidden Markov models and applications},
  author={Gassiat, {\'E}lisabeth and Cleynen, Alice and Robin, Stephane},
  journal={Statistics and Computing},
  volume={26},
  number={1},
  pages={61--71},
  year={2016},
  publisher={Springer}
}

@techreport{an2013identifiability,
  title={Identifiability and inference of hidden Markov models},
  author={An, Yonghong and Hu, Yingyao and Hopkins, Johns and Shum, Matt},
  year={2013},
  institution={Technical report}
}

@inproceedings{dong2020collapsed,
  title={Collapsed amortized variational inference for switching nonlinear dynamical systems},
  author={Dong, Zhe and Seybold, Bryan and Murphy, Kevin and Bui, Hung},
  booktitle={International Conference on Machine Learning},
  pages={2638--2647},
  year={2020},
  organization={PMLR}
}

@article{saggioro2020reconstructing,
  title={Reconstructing regime-dependent causal relationships from observational time series},
  author={Saggioro, Elena and de Wiljes, Jana and Kretschmer, Marlene and Runge, Jakob},
  journal={Chaos: An Interdisciplinary Journal of Nonlinear Science},
  volume={30},
  number={11},
  pages={113115},
  year={2020},
  publisher={AIP Publishing LLC}
}

@article{lindgren1978markov,
  title={Markov regime models for mixed distributions and switching regressions},
  author={Lindgren, Georg},
  journal={Scandinavian Journal of Statistics},
  pages={81--91},
  year={1978},
  publisher={JSTOR}
}

@incollection{paszke2017automatic,
title = {PyTorch: An Imperative Style, High-Performance Deep Learning Library},
author = {Paszke, Adam and Gross, Sam and Massa, Francisco and Lerer, Adam and Bradbury, James and Chanan, Gregory and Killeen, Trevor and Lin, Zeming and Gimelshein, Natalia and Antiga, Luca and Desmaison, Alban and Kopf, Andreas and Yang, Edward and DeVito, Zachary and Raison, Martin and Tejani, Alykhan and Chilamkurthy, Sasank and Steiner, Benoit and Fang, Lu and Bai, Junjie and Chintala, Soumith},
booktitle = {Advances in Neural Information Processing Systems 32},
pages = {8024--8035},
year = {2019},
publisher = {Curran Associates, Inc.},
}

@inproceedings{
yao2022temporally,
title={Temporally Disentangled Representation Learning},
author={Weiran Yao and Guangyi Chen and Kun Zhang},
booktitle={Advances in Neural Information Processing Systems},
editor={Alice H. Oh and Alekh Agarwal and Danielle Belgrave and Kyunghyun Cho},
year={2022},
url={https://openreview.net/forum?id=Vi-sZWNA_Ue}
}

@inproceedings{
yao2022learning,
title={Learning Temporally Causal Latent Processes from General Temporal Data},
author={Weiran Yao and Yuewen Sun and Alex Ho and Changyin Sun and Kun Zhang},
booktitle={International Conference on Learning Representations},
year={2022},
url={https://openreview.net/forum?id=RDlLMjLJXdq}
}

@inproceedings{kingma2014adam,
  author    = {Diederik P. Kingma and
               Jimmy Ba},
  editor    = {Yoshua Bengio and
               Yann LeCun},
  title     = {Adam: {A} Method for Stochastic Optimization},
  booktitle = {3rd International Conference on Learning Representations, {ICLR} 2015,
               San Diego, CA, USA, May 7-9, 2015, Conference Track Proceedings},
  year      = {2015},
  bibsource = {dblp computer science bibliography, https://dblp.org}
}

@article{gassiat2020identifiability,
  title={Identifiability and consistent estimation of nonparametric translation hidden Markov models with general state space},
  author={Gassiat, Elisabeth and Le Corff, Sylvain and Leh{\'e}ricy, Luc},
  journal={The Journal of Machine Learning Research},
  volume={21},
  number={1},
  pages={4589--4628},
  year={2020},
  publisher={JMLRORG}
}

@book{fruhwirth2006finite,
  title={Finite mixture and Markov switching models},
  author={Fr{\"u}hwirth-Schnatter, Sylvia and Fr{\`e}uhwirth-Schnatter, Sylvia},
  volume={425},
  year={2006},
  publisher={Springer}
}

@inproceedings{hyvarinen2017nonlinear,
  title={Nonlinear ICA of temporally dependent stationary sources},
  author={Hyvarinen, Aapo and Morioka, Hiroshi},
  booktitle={Artificial Intelligence and Statistics},
  pages={460--469},
  year={2017},
  organization={PMLR}
}

@article{thorndike1953belongs,
  title={Who belongs in the family?},
  author={Thorndike, Robert L},
  journal={Psychometrika},
  volume={18},
  number={4},
  pages={267--276},
  year={1953},
  publisher={Springer-Verlag}
}

@misc{murphy1998switching,
  title={Switching kalman filters},
  author={Murphy, Kevin P},
  year={1998},
  publisher={Unpublished technical report, UC Berkeley}
}

@article{cont2001empirical,
  title={Empirical properties of asset returns: stylized facts and statistical issues},
  author={Cont, Rama},
  journal={Quantitative finance},
  volume={1},
  number={2},
  pages={223},
  year={2001},
  publisher={IOP Publishing}
}

@inproceedings{hyvarinen2019nonlinear,
  title={Nonlinear ICA using auxiliary variables and generalized contrastive learning},
  author={Hyvarinen, Aapo and Sasaki, Hiroaki and Turner, Richard},
  booktitle={The 22nd International Conference on Artificial Intelligence and Statistics},
  pages={859--868},
  year={2019},
  organization={PMLR}
}

@book{mclachlan2000finite,
  title={Finite mixture models},
  author={McLachlan, Geoffrey J and Peel, David},
  year={2000},
  publisher={John Wiley \& Sons}
}

@inproceedings{
lippe2023causal,
title={Causal Representation Learning for Instantaneous and Temporal Effects in Interactive Systems},
author={Phillip Lippe and Sara Magliacane and Sindy L{\"o}we and Yuki M Asano and Taco Cohen and Efstratios Gavves},
booktitle={The Eleventh International Conference on Learning Representations },
year={2023},
url={https://openreview.net/forum?id=itZ6ggvMnzS}
}

@inproceedings{lippe2023biscuit,
  title = 	 {{BISCUIT}: Causal Representation Learning from Binary Interactions},
  author =       {Lippe, Phillip and Magliacane, Sara and L{\"o}we, Sindy and Asano, Yuki M and Cohen, Taco and Gavves, Efstratios},
  booktitle = 	 {Proceedings of the Thirty-Ninth Conference on Uncertainty in Artificial Intelligence},
  pages = 	 {1263--1273},
  year = 	 {2023},
  editor = 	 {Evans, Robin J. and Shpitser, Ilya},
  volume = 	 {216},
  series = 	 {Proceedings of Machine Learning Research},
  month = 	 {31 Jul--04 Aug},
  publisher =    {PMLR},
  url = 	 {https://proceedings.mlr.press/v216/lippe23a.html},
}

@inproceedings{poritz1982linear,
  title={Linear predictive hidden Markov models and the speech signal},
  author={Poritz, A},
  booktitle={ICASSP'82. IEEE International Conference on Acoustics, Speech, and Signal Processing},
  volume={7},
  pages={1291--1294},
  year={1982},
  organization={IEEE}
}

@article{hamilton1989new,
  title={A new approach to the economic analysis of nonstationary time series and the business cycle},
  author={Hamilton, James D},
  journal={Econometrica: Journal of the econometric society},
  pages={357--384},
  year={1989},
  publisher={JSTOR}
}

@article{mityagin2015zero,
  title={The zero set of a real analytic function},
  author={Mityagin, Boris},
  journal={arXiv preprint arXiv:1512.07276},
  year={2015}
}

@article{allman2009identifiability,
  title={Identifiability of parameters in latent structure models with many observed variables},
  author={Allman, Elizabeth S and Matias, Catherine and Rhodes, John A},
  journal={The Annals of Statistics},
  volume={37},
  number={6A},
  pages={3099--3132},
  year={2009},
  publisher={Institute of Mathematical Statistics}
}

@inproceedings{Saxena2021ClockworkVA,
  title={Clockwork Variational Autoencoders},
  author={Vaibhav Saxena and Jimmy Ba and Danijar Hafner},
  booktitle={Neural Information Processing Systems},
  year={2021}
}

@inproceedings{Li2018DisentangledSA,
  title={Disentangled Sequential Autoencoder},
  author={Yingzhen Li and Stephan Mandt},
  booktitle={International Conference on Machine Learning},
  year={2018}
}

@article{comon1994independent,
  title={Independent component analysis, a new concept?},
  author={Comon, Pierre},
  journal={Signal processing},
  volume={36},
  number={3},
  pages={287--314},
  year={1994},
  publisher={Elsevier}
}

@article{hyvarinen1999nonlinear,
  title={Nonlinear independent component analysis: Existence and uniqueness results},
  author={Hyv{\"a}rinen, Aapo and Pajunen, Petteri},
  journal={Neural networks},
  volume={12},
  number={3},
  pages={429--439},
  year={1999},
  publisher={Elsevier}
}

@article{kivva2022identifiability,
  title={Identifiability of deep generative models without auxiliary information},
  author={Kivva, Bohdan and Rajendran, Goutham and Ravikumar, Pradeep and Aragam, Bryon},
  journal={Advances in Neural Information Processing Systems},
  volume={35},
  pages={15687--15701},
  year={2022}
}

@book{peters2017elements,
  title={Elements of causal inference: foundations and learning algorithms},
  author={Peters, Jonas and Janzing, Dominik and Sch{\"o}lkopf, Bernhard},
  year={2017},
  publisher={The MIT Press}
}

@article{chung2015recurrent,
  title={A recurrent latent variable model for sequential data},
  author={Chung, Junyoung and Kastner, Kyle and Dinh, Laurent and Goel, Kratarth and Courville, Aaron C and Bengio, Yoshua},
  journal={Advances in neural information processing systems},
  volume={28},
  year={2015}
}

@inproceedings{
babaeizadeh2018stochastic,
title={Stochastic Variational Video Prediction},
author={Mohammad Babaeizadeh and Chelsea Finn and Dumitru Erhan and Roy H. Campbell and Sergey Levine},
booktitle={International Conference on Learning Representations},
year={2018},
url={https://openreview.net/forum?id=rk49Mg-CW},
}

@book{bishop2006pattern,
author = {Bishop, Christopher M.},
title = {Pattern Recognition and Machine Learning (Information Science and Statistics)},
year = {2006},
isbn = {0387310738},
publisher = {Springer-Verlag},
address = {Berlin, Heidelberg}
}

@article{mediano2023spectrally,
title={Spectrally and temporally resolved estimation of neural signal diversity},
url={http://dx.doi.org/10.7554/eLife.88683.1},
DOI={doi.org/10.7554/eLife.88683.1},
journal={eLife},
author={Mediano, Pedro A.M. and Rosas, Fernando E. and Luppi, Andrea I. and Noreika, Valdas and Seth, Anil K. and Carhart-Harris, Robin L. and Barnett, Lionel and Bor, Daniel},
year={2023},
month=mar }

@article{yanagawa2013,
    doi = {10.1371/journal.pone.0080845},
    author = {Yanagawa, Toru AND Chao, Zenas C. AND Hasegawa, Naomi AND Fujii, Naotaka},
    journal = {PLOS ONE},
    publisher = {Public Library of Science},
    title = {Large-Scale Information Flow in Conscious and Unconscious States: an ECoG Study in Monkeys},
    year = {2013},
    month = {11},
    volume = {8},
    url = {https://doi.org/10.1371/journal.pone.0080845},
    pages = {null},
    number = {11},

}

@article{assaad2022survey,
  title={Survey and evaluation of causal discovery methods for time series},
  author={Assaad, Charles K and Devijver, Emilie and Gaussier, Eric},
  journal={Journal of Artificial Intelligence Research},
  volume={73},
  pages={767--819},
  year={2022}
}

@article{hoyer2008nonlinear,
  title={Nonlinear causal discovery with additive noise models},
  author={Hoyer, Patrik and Janzing, Dominik and Mooij, Joris M and Peters, Jonas and Sch{\"o}lkopf, Bernhard},
  journal={Advances in neural information processing systems},
  volume={21},
  year={2008}
}

@inproceedings{halva2020hidden,
  title={Hidden markov nonlinear ica: Unsupervised learning from nonstationary time series},
  author={H{\"a}lv{\"a}, Hermanni and Hyvarinen, Aapo},
  booktitle={Conference on Uncertainty in Artificial Intelligence},
  pages={939--948},
  year={2020},
  organization={PMLR}
}

@article{ephraim2005revisiting,
  title={Revisiting autoregressive hidden Markov modeling of speech signals},
  author={Ephraim, Yariv and Roberts, William JJ},
  journal={IEEE Signal processing letters},
  volume={12},
  number={2},
  pages={166--169},
  year={2005},
  publisher={IEEE}
}

@inproceedings{
rahmani2025causal,
title={Causal Temporal Regime Structure Learning},
author={Abdellah Rahmani and Pascal Frossard},
booktitle={The 28th International Conference on Artificial Intelligence and Statistics},
year={2025},
}

@inproceedings{lachapelle2022disentanglement,
  title={Disentanglement via mechanism sparsity regularization: A new principle for nonlinear ICA},
  author={Lachapelle, S{\'e}bastien and Rodriguez, Pau and Sharma, Yash and Everett, Katie E and Le Priol, R{\'e}mi and Lacoste, Alexandre and Lacoste-Julien, Simon},
  booktitle={Conference on Causal Learning and Reasoning},
  pages={428--484},
  year={2022},
  organization={PMLR}
}

@inproceedings{morioka2021independent,
  title={Independent innovation analysis for nonlinear vector autoregressive process},
  author={Morioka, Hiroshi and H{\"a}lv{\"a}, Hermanni and Hyvarinen, Aapo},
  booktitle={International conference on artificial intelligence and statistics},
  pages={1549--1557},
  year={2021},
  organization={PMLR}
}

@article{hyvarinen2016unsupervised,
  title={Unsupervised feature extraction by time-contrastive learning and nonlinear ica},
  author={Hyvarinen, Aapo and Morioka, Hiroshi},
  journal={Advances in neural information processing systems},
  volume={29},
  year={2016}
}

@article{fraccaro2017disentangled,
  title={A disentangled recognition and nonlinear dynamics model for unsupervised learning},
  author={Fraccaro, Marco and Kamronn, Simon and Paquet, Ulrich and Winther, Ole},
  journal={Advances in neural information processing systems},
  volume={30},
  year={2017}
}

@article{ansari2021deep,
  title={Deep explicit duration switching models for time series},
  author={Ansari, Abdul Fatir and Benidis, Konstantinos and Kurle, Richard and Turkmen, Ali Caner and Soh, Harold and Smola, Alexander J and Wang, Bernie and Januschowski, Tim},
  journal={Advances in Neural Information Processing Systems},
  volume={34},
  pages={29949--29961},
  year={2021}
}

\newpage

\appendix

\section{Proof of Theorem \ref{thm:identifiability_main}}\label{app:proof_identifiability}
\paragraph{Sketch of the proof:} We organise the proof strategy into 4 steps.
\begin{enumerate}
    \item We show the requirement for identifiability is linear independence of the trajectory family (Appendix \ref{app:linear_indep_render_id}).
    \item We provide linear independence results for products of nonparametric functions. We first start with $M=1$ (Appendix \ref{app:preliminaries_lin_indep}), and we then strengthen the assumptions (Appendix \ref{app:extending_assumptions}) to allow products of $M$ functions (Appendix \ref{app:lippo_m}).
    \item We prove linear independence of the trajectory family for nonparametric transitions of order $M$ using induction (Appendix \ref{app:lin_indep_joint}).
    \item We show the parametric assumptions (m1), (m2) satisfy linear independence of the trajectory family (Appendix \ref{app:parametric_assumptions}).
\end{enumerate}

\subsection{Linear independence renders identifiability}\label{app:linear_indep_render_id}

\begin{theorem}\label{thm:identifiability_msm} 
Assume the functions in $\mathcal{P}^{T,M}_{\A,\B}$ are linearly independent under finite mixtures, then the distribution family $\mathcal{M}^T(\Pi_{\A}^M,\mathcal{P}_{\B}^M)$ is identifiable as defined in Def \ref{def:identifiability}.
\end{theorem}
\begin{proof}
By Proposition \ref{prop:mixture_pdf_identifiability}, linear independence of $\jointfamily$ under finite mixtures implies identifiability up to permutation in the sense of Definition \ref{def:identifiability_finite_mixture}. Then, for $p(\z_{1:T})$ and $\tilde{p}(\z_{1:T})$ from Definition \ref{def:identifiability}, we have $C=\tilde{C}$, and for every $1\leq i \leq C$, there exists $1\leq j \leq \tilde{C}$ such that $c_i=\tilde{c}_j$ and:
\begin{multline}
p_{a^i}(\z_{1:M})\prod_{t=M+1}^T p_{b_t^i}(\z_t | \z_{t-1}, \dots, \z_{t-M}) = \\ p_{\tilde{a}^j}(\z_{1:M})\prod_{t=M+1}^T p_{\tilde{b}_t^j}(\z_t | \z_{t-1}, \dots, \z_{t-M}), \text{  } \forall \z_{1:T} \in \mathbb{R}^{Tm}.    
\end{multline}
Given that we have conditional PDFs, if the joint distributions are equal on $\z_{1:T}$, then the distributions on $\z_{1:T-1}$ are also equal:
\begin{multline}
p_{a^i}(\z_{1:M})\prod_{t=M+1}^{T-1} p_{b_t^i}(\z_t | \z_{t-1}, \dots, \z_{t-M}) = \\
p_{\tilde{a}^j}(\z_{1:M})\prod_{t=M+1}^{T-1} p_{\tilde{b}_t^j}(\z_t | \z_{t-1}, \dots, \z_{t-M}), \text{  } \forall \z_{1:T-1} \in \mathbb{R}^{(T-1)m}.    
\end{multline}
Therefore, we have $p_{b_T^i}(\z_T | \z_{T-1}, \dots, \z_{T-M}) = p_{\tilde{b}_T^j}(\z_T | \z_{T-1}, \dots, \z_{T-M})$ for all $\z_t, \dots, \z_{t-M}\in \R^m$.  We can follow the same reasoning for other time indices to have $p_{b_t^i}(\z_t | \z_{t-1}, \dots, \z_{t-M}) = p_{\tilde{b}_t^j}(\z_t | \z_{t-1}, \dots, \z_{t-M})$ for all $t > M, \z_t, \dots, \z_{t-M}\in \R^m$. Similar logic applies to the initial distribution, where we have $p_{a^i}(\z_{1:M})=p_{\tilde{a}^j}(\z_{1:M})$ for all $\z_1,\dots,\z_M\in\R^m$; and we have $K_0=\tilde{K}_0$, $K=\tilde{K}$.

Finally, if there exists $t_1 \neq t_2$ such that $b_{t_1}^i = b_{t_2}^i$ but $\tilde{b}_{t_1}^j \neq \tilde{b}_{t_2}^j$, we have for any $\bm{\alpha}\in \mathbb{R}^{Mm},  \bm{\beta} \in\R^m$:
\begin{align*}
p_{\tilde{b}_{t_1}^j}(\z_{t_1} = \bm{\beta} | \z_{t_1 -1:t_1 - M} = \bm{\alpha}) &= p_{b_{t_1}^i}(\z_{t_1} = \bm{\beta} | \z_{t_1 -1:t_1 - M} = \bm{\alpha}) \\
&= p_{b_{t_2}^i}(\z_{t_2} = \bm{\beta} | \z_{t_2 -1:t_2-M} = \bm{\alpha}) \\
&= p_{\tilde{b}_{t_2}^j}(\z_{t_2} = \bm{\beta} | \z_{t_2 -1:t_2-M} = \bm{\alpha}),
\end{align*}
which implies linear dependence under finite mixtures of $\mathcal{P}^M_{\B}$. We note this contradicts the assumption of linear independence under finite mixtures of the trajectory family $\jointfamily$: such assumption says we should have
\begin{equation}
\sum_{i}\sum_{j} \gamma_{ij} p_{b^i_{1:T-1}}(\z_{1:T-1}) p_{b^j_T}(\z_T | \z_{T-1:T-M}) = 0, \quad \forall \z_{1:T} \in \mathbb{R}^{Tm},
\end{equation}
if and only if $\gamma_{ij}=0$ for all $1 \leq i \leq C/K, 1 \leq j \leq K$. However, given linear dependence on $\mathcal{P}^M_{\B}$, we can rearrange the terms and set $\gamma_{i,j} = \gamma_j \neq 0, 1\leq j \leq K$ such that the above equation is satisfied. Therefore a contradiction is reached and we have $b_{t_1}^i = b_{t_2}^i \Leftrightarrow \tilde{b}_{t_1}^j = \tilde{b}_{t_2}^j$ for all $M < t_1, t_2 \leq T$.
\end{proof}
The next step is to show conditions under which the trajectory family $\jointfamily$ is linearly independent under finite mixtures.

\subsection{Linear independence of nonparametric product families}
To establish linear independence of the nonparametric trajectory family, we analyze consecutive products of conditional distributions. For instance, when $M=1$, considering two consecutive variables gives the product $$\{p(\textcolor{purple}{\z_{t}}\mid\z_{t-1},s_t)p(\textcolor{teal}{\z_{t+1}}\mid\textcolor{purple}{\z_{t}},s_{t+1})\}$$ with one overlapping variable $\z_t$. Increasing $M$ increases the number of overlaps when considering $M+1$ consecutive variables. The following product
\begin{equation}\label{eq:joint_product_M}
    \left\{\prod_{\ell=0}^{M} p(\z_{t+\ell} \mid \z_{t-1+\ell}, \dots, \z_{t-M+\ell}, s_{t+\ell}): (s_{t}, \dots, s_{t+M})\in[K]^{M+1}\right\},
\end{equation} 
contains two types of overlapping variables for $M>1$:
\begin{enumerate}
    \item $M$ variables that appear both as observed variables and as conditioning variables, namely $\z_{t+M-1}, \dots, \z_{t}$ in Eq. (\ref{eq:joint_product_M}).
    \item $M-1$ variables that appear only as conditioning variables in multiple factors, namely $\z_{t-1}, \dots, \z_{t-M+1}$ in Eq. (\ref{eq:joint_product_M}).
\end{enumerate}

For $M>1$, the additional overlaps between the conditioning variables $(\z_{t-1}, \dots, \z_{t-M+1})$ complicate the analysis, as these prevent a direct extension of the case $M=1$ presented in \citet{balsells-rodas2024on}.

\subsubsection{Preliminaries}\label{app:preliminaries_lin_indep}

We recall the central result for the case $M=1$ from \citet{balsells-rodas2024on}, which we name
\emph{\textbf{L}inear \textbf{I}ndependence on \textbf{P}roduct functions with $1$ consecutive \textbf{O}verlap} (LIPO-1). It provides conditions under which linear independence is preserved for products with a single overlapping variable.

\begin{lemma}[LIPO-1]
\label{lemma:linear_independence_two_nonlinear_gaussians}
Consider two scalar real-valued function families $\mathcal{U}_I := \{u_i(\y, \x) | i \in I \}$ and $\mathcal{V}_J := \{v_j(\z, \y) | j \in J \}$ with $\x \in \mathcal{X}, \y \in \mathbb{R}^{d_y}$ and $\z \in \mathbb{R}^{d_z}$. We further make the following assumptions:
\begin{itemize}
    \item[(a1)] Positive function values: $u_i(\y, \x) > 0$ for all $i \in I, (\y, \x) \in \mathbb{R}^{d_y} \times \mathcal{X}$. Similar positive function values assumption applies to $\mathcal{V}_J$: $v_j(\z, \y) > 0$ for all $j \in J, (\z, \y) \in \mathbb{R}^{d_z} \times \mathbb{R}^{d_y}$.
    \item[(a2)] Unique indexing: for $\mathcal{U}_I$, $i \neq i' \in I \Leftrightarrow \exists \ \x, \y \text{ s.t. } u_i(\y, \x) \neq u_{i'}(\y, \x)$. Similar unique indexing assumption applies to $\mathcal{V}_J$;
    \item[(a3)] Linear independence \textcolor{black}{under finite mixtures} on specific nonzero measure subsets for $\mathcal{U}_I$: for any nonzero measure subset $\mathcal{Y} \subset \mathbb{R}^{d_y}$, $\mathcal{U}_I$ contains linearly independent functions \textcolor{black}{under finite mixtures} on $(\y, \x) \in \mathcal{Y} \times \mathcal{X}$. 
    \item[(a4)] Linear independence \textcolor{black}{under finite mixtures} on specific nonzero measure subsets for $\mathcal{V}_J$: there exists a nonzero measure subset $\mathcal{Y} \subset \mathbb{R}^{d_y}$, such that for any nonzero measure subsets $\mathcal{Y}' \subset \mathcal{Y}$ and $\mathcal{Z} \subset \mathbb{R}^{d_z}$, $\mathcal{V}_J$ contains linearly independent functions \textcolor{black}{under finite mixtures} on $(\z, \y) \in \mathcal{Z} \times \mathcal{Y}'$;
    \item[(a5)] Linear dependence \textcolor{black}{under finite mixtures} for subsets of functions in $\mathcal{V}_J$ implies repeating functions: for any $\bm{\beta} \in \mathbb{R}^{d_y}$, any nonzero measure subset $\mathcal{Z} \subset \mathbb{R}^{d_z}$ and any subset $J_0 \subset J$ \textcolor{black}{such that $|J_0| < +\infty$}, $\{v_j(\z, \y = \bm{\beta}) | j \in J_0 \}$ contains linearly dependent functions on $\z \in \mathcal{Z}$ only if $ \exists \ j \neq j' \in J_0$ such that $v_j(\z, \bm{\beta}) = v_{j'}(\z, \bm{\beta})$ for all $\z \in \textcolor{black}{\mathbb{R}^{d_z}}$.
    \item[(a6)] Continuity for $\mathcal{V}_J$: for any $j \in J$, $v_j(\z, \y)$ is continuous in $\y \in \mathbb{R}^{d_y}$.
\end{itemize}
Then for any nonzero measure subset $\mathcal{Z} \subset \mathbb{R}^{d_z}$, $\mathcal{U}_I \otimes \mathcal{V}_J := \{v_j(\z, \y) u_i(\y, \x) | i \in I, j \in J \}$ 
contains linear independent functions under finite mixtures defined on $(\x, \y, \z) \in \mathcal{X} \times \mathbb{R}^{d_y} \times \mathcal{Z}$.
\end{lemma}

\begin{proof}
Assume this sufficiency statement is false, then there exists a nonzero measure subset $\mathcal{Z} \subset \mathbb{R}^{d_z}$, $S_0 \subset I \times J$ with $|S_0| < +\infty$ and a set of nonzero values $\{ \gamma_{ij} \in \mathbb{R} | (i, j) \in S_0 \}$, such that 
\begin{equation}
\label{eq:lipo1_assumed_linear_dependence_gaussian}
\begin{aligned}
\sum_{ (i, j) \in S_0} \gamma_{ij} v_j(\z, \y) u_i(\y, \x) = 0, &\quad \forall (\x, \y, \z) \in \mathcal{X} \times \mathbb{R}^{d_y} \times \mathcal{Z}.
\end{aligned}
\end{equation}
Note that the choices of $S_0$ and $\gamma_{ij}$ are independent of any $\x, \y, \z$ values, but might be dependent on $\mathcal{Z}$. By assumptions (a1), the index set $S_0$ contains at least 2 different indices $(i, j)$ and $(i', j')$.
In particular, $S_0$ contains at least 2 different indices $(i, j)$ and $(i', j')$ with $j \neq j'$, otherwise we can extract the common term $v_j(\z, \y)$ out:
$$\sum_{ (i, j) \in S_0} \gamma_{ij} v_j(\z, \y) u_i(\y, \x) = v_j(\z, \y) \left( \sum_{i: (i, j) \in S_0} \gamma_{ij}  u_i(\y, \x) \right) = 0, \quad \forall (\x, \y, \z) \in \mathcal{X} \times \mathbb{R}^{d_y} \times \mathcal{Z},$$
and as there exist at least 2 different indices $(i', j)$ and $(i, j)$ in $S_0$, we have at least one $i' \neq i$, and the above equation contradicts assumptions (a1) - (a3).

Now define $J_0 = \{j \in J | \exists (i, j) \in S_0 \}$ the set of all possible $j$ indices that appear in $S_0$, and from $|S_0| < +\infty$ we have $|J_0| < +\infty$ as well. We rewrite the linear combination equation (Eq.~(\ref{eq:lipo1_assumed_linear_dependence_gaussian})) for any $\bm{\beta} \in \mathbb{R}^{d_y}$ as
\begin{equation}
\label{eq:lipo1_linear_dependence_division}
\sum_{j \in J_0} \left(\sum_{ i: (i, j) \in S_0} \gamma_{ij} u_{i}(\y = \bm{\beta}, \x) \right) v_j(\z, \y = \bm{\beta}) = 0, \quad \forall (\x, \z) \in \mathcal{X} \times \mathcal{Z}. 
\end{equation}
From assumption (a3) we know that the set $\mathcal{Y}_0 := \{\bm{\beta} \in \mathbb{R}^{d_y} | \sum_{ i: (i, j) \in S_0} \gamma_{ij} u_{i}(\y = \bm{\beta}, \x) = 0, \forall \x \in \mathcal{X} \}$ can only have zero measure in $\mathbb{R}^{d_y}$. Write $\mathcal{Y} \subset \mathbb{R}^{d_y}$ the nonzero measure subset defined by assumption (a4), we have $\mathcal{Y}_1 := \mathcal{Y} \backslash \mathcal{Y}_0 \subset \mathcal{Y}$ also has nonzero measure and satisfies assumption (a4). Combined with assumption (a1), we have for each $\bm{\beta} \in \mathcal{Y}_1$, there exists $\x \in \mathcal{X}$ such that $\sum_{ i: (i, j) \in S_0} \gamma_{ij} u_{i}(\y = \bm{\beta}, \x) \neq 0$ for at least two $j$ indices in $J_0$. 
For each $\bm{\beta} \in \mathcal{Y}_1$, $\{v_j(\z, \y = \bm{\beta}) | j \in J_0 \}$ contains linearly dependent functions on $\z \in \mathcal{Z}$.
Now under assumption (a5), we can split the index set $J_0$ into subsets indexed by $k \in K(\bm{\beta})$ as follows, such that within each index subset $J_k(\bm{\beta})$ the functions with the corresponding indices are equal:
\begin{equation}
\label{eq:lipo1_j_index_split_def}
\begin{aligned}
J_0 = \cup_{k \in K(\bm{\beta})} J_k(\bm{\beta}), \quad J_k(\bm{\beta}) \cap J_{k'}(\bm{\beta}) = \emptyset, \forall k \neq k' \in K(\bm{\beta}), \\
j \neq j' \in J_k(\bm{\beta}) \quad \Leftrightarrow \quad v_j(\z, \y = \bm{\beta}) = v_{j'}(\z, \y = \bm{\beta}), \quad \forall \z \in \mathcal{Z}.
\end{aligned}
\end{equation}
Then we can rewrite Eq.~(\ref{eq:lipo1_linear_dependence_division}) for any $\bm{\beta} \in \mathcal{Y}_1$ as
\begin{equation}
\sum_{k \in K(\bm{\beta})} \left(\sum_{j \in J_k(\bm{\beta})}\sum_{ i: (i, j) \in S_0} \gamma_{ij} u_{i}(\y = \bm{\beta}, \x) v_j(\z, \y = \bm{\beta}) \right) = 0, \quad \forall (\x, \z) \in \mathcal{X} \times \mathcal{Z}.   
\end{equation}
Recall from Eq.~(\ref{eq:lipo1_j_index_split_def}) that $v_j(\bm{z}, \bm{y} = \bm{\beta})$ and $v_{j'}(\bm{z}, \bm{y} = \bm{\beta})$ are the same functions on $\bm{z} \in \mathcal{Z}$ iff.~$j \neq j'$ are in the same index set $J_k(\bm{\beta})$. This means if Eq.~(\ref{eq:lipo1_assumed_linear_dependence_gaussian}) holds, then for any $\bm{\beta} \in \mathcal{Y}_1$, under assumptions (a1) and (a5),
\begin{equation}
\label{eq:lipo1_zero_constraint_given_beta}
\sum_{j \in J_k(\bm{\beta})}\sum_{ i: (i, j) \in S_0} \gamma_{ij} u_{i}(\y = \bm{\beta}, \x) = 0, \quad \forall \x \in \mathbb{R}^d, \quad k \in K(\bm{\beta}).  
\end{equation}
Define $C(\bm{\beta}) = \min_k |J_k(\bm{\beta})|$ the minimum cardinality count for $j$ indices in the $J_k(\bm{\beta})$ subsets. Choose $\bm{\beta}^* \in \arg\min_{\bm{\beta} \in \textcolor{black}{\mathcal{Y}_1}} C(\bm{\beta})$:
\begin{itemize}
    \item[1.] We have $C(\bm{\beta}^*) < |J_0|$ and $|K(\bm{\beta}^*)| \geq 2$. Otherwise for all $j \neq j' \in J_0$ we have $v_j(\z, \y = \bm{\beta}) = v_{j'}(\z, \y = \bm{\beta})$ for all $\z \in \mathcal{Z}$ and $\bm{\beta} \in \textcolor{black}{\mathcal{Y}_1}$, so that they are linearly dependent on $(\z, \y) \in \mathcal{Z} \times \textcolor{black}{\mathcal{Y}_1}$, a contradiction to assumption (a4) by setting $\mathcal{Y}' = \textcolor{black}{\mathcal{Y}_1}$.
    \item[2.] Now assume $|J_{1}(\bm{\beta}^*)| = C(\bm{\beta}^*)$ w.l.o.g.. \textcolor{black}{From assumption (a5), we know that for any $j \in J_1(\bm{\beta}^*)$ and $j' \in J_0 \backslash J_{1}(\bm{\beta}^*)$, $v_j(\z, \y = \bm{\beta}) = v_{j'}(\z, \y = \bm{\beta})$ only on zero measure subset of $\mathcal{Z}$ at most. Then as $|J_0| < +\infty$ and $\mathcal{Z} \subset \mathbb{R}^{d_z}$ has nonzero measure, there exist $\z_0 \in \mathcal{Z}$ and $\delta > 0$} such that 
    $$|v_j(\z = \z_0, \y = \bm{\beta}^*) - v_{j'}(\z = \z_0, \y = \bm{\beta}^*)| \geq \delta, \quad \forall j \in J_{1}(\bm{\beta}^*), \forall j' \in J_0 \backslash J_{1}(\bm{\beta}^*).$$
    Under assumption (a6), there exists $\epsilon_j > 0$ such that we can construct an $\epsilon$-ball $B(\bm{\beta}^*, \epsilon_j)$ using $\ell_2$-norm, such that 
    $$|v_j(z = \z_0, \y = \bm{\beta}^*) - v_j(\z = \z_0, \y = \bm{\beta})| \leq \delta/3, \quad \forall \bm{\beta} \in B(\bm{\beta}^*, \epsilon_j).$$
Choosing a suitable $0 < \epsilon \leq \min_{j \in J_0} \epsilon_j$ (note that $\min_{j \in J_0} \epsilon_j > 0$ as $|J_0| < +\infty$) and constructing an $\ell_2$-norm-based $\epsilon$-ball $B(\bm{\beta}^*, \epsilon) \subset \textcolor{black}{\mathcal{Y}_1}$, we have for all $j \in J_{1}(\bm{\beta}^*), j' \in J_0 \backslash J_{1}(\bm{\beta}^*)$, $j' \notin J_1(\bm{\beta})$ for all $\bm{\beta} \in B(\bm{\beta}^*, \epsilon)$ due to
    $$|v_j(\z = \z_0, \y = \bm{\beta}) - v_{j'}(\z = \z_0, \y = \bm{\beta})| \geq \delta/3, \quad \forall \bm{\beta} \in B(\bm{\beta}^*, \epsilon).$$
    So this means for the split $\{ J_k(\bm{\beta}) \}$ of any $\bm{\beta} \in B(\bm{\beta}^*, \epsilon)$, we have $ J_1(\bm{\beta}) \subset J_1(\bm{\beta}^*)$ 
    and therefore $ |J_1(\bm{\beta})| \leq |J_1(\bm{\beta}^*)|$. Now by definition of $\bm{\beta}^* \in \arg\min_{\bm{\beta} \in \mathcal{Y}} C(\bm{\beta})$ and $|J_{1}(\bm{\beta}^*)| = C(\bm{\beta}^*)$, we have $J_1(\bm{\beta}) = J_1(\bm{\beta}^*)$ for all $\bm{\beta} \in B(\bm{\beta}^*, \epsilon)$. 
    \item[3.] One can show that $|J_{1}(\bm{\beta}^*)| = 1$, otherwise by definition of the split (Eq.~(\ref{eq:lipo1_j_index_split_def})) and the above point, there exists $j \neq j' \in J_{1}(\bm{\beta}^*)$ such that $v_j(\z, \y = \bm{\beta}) = v_{j'}(\z, \y = \bm{\beta})$ for all $\z \in \mathcal{Z}$ and $\bm{\beta} \in B(\bm{\beta}^*, \epsilon)$, a contradiction to assumption (a4) by setting $\mathcal{Y}' = B(\bm{\beta}^*, \epsilon)$. 
    Now assume that $j \in J_1(\bm{\beta}^*)$ is the only index in the subset, then the fact proved in the above point that $J_1(\bm{\beta}) = J_1(\bm{\beta}^*)$ for all $\bm{\beta} \in B(\bm{\beta}^*, \epsilon)$ means
    \begin{equation*}
    \sum_{ i: (i, j) \in S_0} \gamma_{ij} u_{i}(\y = \bm{\beta}, \x) = 0, \quad \forall \x \in \mathcal{X}, \quad \forall \bm{\beta} \in B(\bm{\beta}^*, \epsilon),
    \end{equation*}
    again a contradiction to assumption (a3) by setting $\mathcal{Y} = B(\bm{\beta}^*, \epsilon)$. 
\end{itemize}
The above 3 points indicate that Eq.~(\ref{eq:lipo1_zero_constraint_given_beta}) cannot hold for all $\bm{\beta} \in \mathcal{Y}_1$ (and therefore for all $\bm{\beta} \in \mathcal{Y}$) under assumptions (a3) - (a6), therefore a contradiction is reached.
\end{proof}

The proof shows that under (a1)–(a4), linear dependence could in principle occur for every value of the overlapping variable $\mathbf y$, which then requires additional assumptions (a5)–(a6).

\paragraph{Why does LIPO-1 not solve the general case $M>1$?}
As noted in Section \ref{sec:msm_theory}, LIPO-1 applies directly when $M=1$ and, more generally, when $T=M+1$. For given initial and transition families, we set $$u_i(\y,\x):=p_a(\z_{1}, \dots,\z_M), \quad \x:=\emptyset,\quad \y:=(\z_1,\dots,\z_M),\quad i=a,$$ and $$v_j(\z,\y):=p_b(\z_{M+1}\mid \z_{M},\dots,\z_1), \quad \z:=\z_{M+1}, \quad \y:=(\z_1,\dots,\z_M), \quad j=b,$$ 
so that the overlapping variable is $\y=(\z_1,\dots,\z_M)$ and the assumptions (a1)–(a6) are satisfied.
For $T>M+1$, this approach fails because the variable $\z$ in one step must become $\y$ in the next step to apply LIPO-1. For example, with $T=M+2$, we require $v'(\z',\y')$ to satisfy (a4-a6). Therefore we define
$$v'_j(\z',\y'):=p_b(\z_{M+2}\mid \z_{M+1},\dots,\z_2),\quad  \y':=(\z_2,\dots,\z_{M+1}), \quad \z':=\z_{M+2}.$$
However, from the previous step we have $\z:=\z_{M+1}$, and the only $u'_i(\y,\x)$ that preserves (a3) defines
$$u'_i(\y',\x') := p_a(\z_{1}, \dots,\z_M)p_b(\z_{M+1}\mid \z_{M},\dots,\z_1), \quad \x':=(\z_1,\dots,\z_M), \quad \y':=\z_{M+1}.$$
We note this $\y'$ is incompatible with the required $(\z_2,\dots,\z_{M+1})$ in $v_j'(\z',\y')$, therefore LIPO-1 cannot be used. To ensure $\y'=\z$ we can define $v_j(\z,\y)$ as a product over $M$ consecutive transitions: \begin{multline*}
    v_j(\z,\y):=\prod_{\ell=1}^M p_{b_\ell}(\z_{t+\ell} \mid \z_{t-M+\ell:t-1+\ell}), \\ \y:=(\z_{t-M+1},\dots, \z_{t}),\ \z:=(\z_{t+1}, \dots, \z_{t+M}),\ j := (b_1, \dots, b_M).
\end{multline*} With this definition, the overlap variables now satisfy $\y'=\z$, so the product family has the structure required to apply LIPO-1 sequentially. The remaining task is to verify that this product family still satisfies assumptions (a1)–(a2) and (a4)–(a6).

\subsubsection{Extending assumptions on product families}\label{app:extending_assumptions}

\noindent\textbf{Notation.} 
In the main text and above, we write the coordinate variables in ascending order. However, for better readability, in this section we write variables in descending order, $\z_M, \dots, \z_1$ and $\y_M, \dots \y_1$, to mirror the product family construction. We also use the following letters to denote distinct notions in our proofs: $r$ denotes the number of consecutive products, $\ell$ denotes the product factor index, and $q$ denotes the class label in the split (used later in LIPO-M, Lemma~\ref{lemma:absolutely_magical_lemma}). We also remove $d_y,d_z$ (used in LIPO-1), and instead we use $d$ to denote the dimensionality of a single variable $\y_i,i\in[M]$ or $\z$.

Below we explore assumptions on nonparametric families with $M+1$ variables (aligned with some transition distribution $p(\z_t\mid\z_{t-1}, \dots, \z_{t-M}, s_t)$, such that the product of $M$ consecutive variables satisfies (a1), (a2), and (a4-a6). This would allow us to use LIPO-1 (or an equivalent extension) to prove linear independence under finite mixtures of the trajectory family $\jointfamily$. We note assumption (a2) is redundant, as it holds for $U_I$ and $V_J$ if (a3) and (a4) hold, respectively. Therefore, we remove (a2) from our theoretical analysis for simplicity. Given a family $\mathcal{V}_J=\{v_j(\z,\y_M,\dots,\y_1)| j\in J\}$ with variables $(\z,\y_M,\dots,\y_1)$ defined on $\times^{M+1} \R^{d}$, we provide the assumption modifications following the enumeration presented in LIPO-1 (Lemma \ref{lemma:linear_independence_two_nonlinear_gaussians}). We keep (a3) unchanged as we will not modify the family $\mathcal{U}_I$.
\begin{enumerate}[leftmargin=2.25cm]
\item[(a1) $\rightarrow$ (b1)] $v_j(\z,\y_M,\dots,\y_1) > 0$ for all $j \in J, (\z,\y_M,\dots,\y_1) \in (\times^{M+1} \R^{d})$.
\item[(a4) $\rightarrow$ (b4)] There exists a full-measure set $\mathcal{Y}\subset (\times^M \R^d)$ (i.e., $\mu((\times^M \mathbb{R}^d)\setminus \mathcal Y)=0$) such that for every  nonzero measure sets $\mathcal{Y}' \subset \mathcal{Y}$, and $\mathcal{Z} \subset \R^{d}$, $\mathcal{V}_J$ contains linearly independent functions \textcolor{black}{under finite mixtures} on $(\z,\y_M,\dots,\y_1) \in \mathcal{Z} \times \mathcal{Y}'$;
\item[(a5) $\rightarrow$ (b5)] For any $(\bm{\beta}_\ell,\dots,\bm{\beta}_1) \in ( \times^\ell \R^{d})$, any $\ell$ with $1\leq \ell \leq M$, any nonzero measure subsets $\mathcal{Z} \subset \mathbb{R}^{d}$, $\mathcal{Y}\subset (\times^{M-\ell}\R^{d})$, and any subset $J_0 \subset J$ \textcolor{black}{such that $|J_0| < +\infty$}: $\{v_j(\z,\y_M,\dots, \y_{\ell+1},\y_\ell=\bm{\beta}_\ell,\dots, \y_1=\bm{\beta}_1) | j \in J_0 \}$ contains linearly dependent functions on $(\z, \y_M,\dots,\y_{\ell+1}) \in \mathcal{Z}\times\mathcal{Y}$ only if $ \exists \ j \neq j' \in J_0$ such that $v_j(\z,\y_M,\dots, \y_{\ell+1},\y_\ell=\bm{\beta}_\ell,\dots, \y_1=\bm{\beta}_1) = v_{j'}(\z,\y_M,\dots, \y_{\ell+1},\y_\ell=\bm{\beta}_\ell,\dots, \y_1=\bm{\beta}_1)$ for all $(\z, \y_M,\dots,\y_{\ell+1}) \in \times^{M-\ell+1}\mathbb{R}^{d}$.
\item[(a6) $\rightarrow$ (b6)] For any $j \in J$, $v_j(\z,\y_M,\dots, \y_1)$ is continuous in $(\y_M,\dots,\y_1) \in ( \times^M \R^{d})$.
\end{enumerate}
Contrary to the modifications of (a1) and (a6), we note that the extension of (a4) and (a5) to $M$ consecutive products requires them to be strengthened respectively to (b4) and (b5). For (b4), the nonzero measure subset $\mathcal{Y}$ must now have full measure. This differs from the case $M=1$ where $\mathcal{Y}$ can have any nonzero measure, and the intuition behind this requirement can be explained from the appearance of conditioned-conditioned overlaps for $M>1$ (See Eq.~(\ref{eq:joint_product_M})). For (b5), linear dependence only implied by repeating functions for fixed $\y$ is strengthened, and now, this must be satisfied for repeating functions for certain fixed subgroups of variables from $\y_M, \dots, \y_1$. 

We provide the following definitions which will be used in the results below.

\begin{definition}[Projection onto coordinates]\label{def:set-projection}
Let $\mathcal{A} \subseteq (\times^r \mathbb{R}^d)$ and $I = \{i_1, \dots, i_\ell\} \subset [r]$, with $1 \le \ell \le r$. 
The projection of $\mathcal{A}$ onto the coordinates in $I$ is
\[
\pi_I(\mathcal{A}) := \{ (\x_i)_{i\in I} \in (\times^\ell \mathbb{R}^d) \;|\; (\x_1, \dots, \x_r) \in \mathcal{A} \} \subseteq (\times^\ell \mathbb{R}^d).
\]
\end{definition}
\begin{definition}[Conditioned projection]\label{def:reverse-projection}
Let $\mathcal A \subseteq (\times^r \mathbb R^d)$ with variables $(\x_1,\dots,\x_r)$. 
Given $I=\{i_1,\dots,i_\ell\}\subset[r]$ and assignments $(\tilde{\x}_{i_1},\dots,\tilde{\x}_{i_\ell})$, define
\begin{multline*}
\mathcal A\big|_{(\x_{i_1}=\tilde{\x}_{i_1},\dots,\mathbf x_{i_\ell}=\tilde{\x}_{i_\ell})}
:= \Big\{ (\x_j)_{j\notin I} \in (\times^{r-\ell}\mathbb R^d) : \\
(\y_1,\dots,\y_r)\in\mathcal A,\ 
y_i=\tilde{\x}_i\ \forall i\in I,\ 
y_j=\x_j\ \forall j\notin I
\Big\}.
\end{multline*}
\end{definition}

Assumption (b4) extends to consecutive product function families as follows.

\begin{lemma}[Extension of (b4)]
\label{lemma:extension_b4}
Assume a scalar function family $\mathcal{V}_J=\{v_j(\z,\y_M,\dots,\y_1)\\| j\in J\}$
defined on $(\z,\y_M,\dots,\y_1) \in \times^{M+1} \R^{d}$, such that it satisfies (b1), (b4), (b5), and (b6).

Then, for any $r$ with $2\leq r \leq M$, there exists a full-measure set $\mathcal{Y}\subset (\times^M \R^d)$ such that for every nonzero measure sets $\mathcal{Y}' \subset \mathcal{Y}$, and $\mathcal{Z}_r \subset (\times^r \R^{d})$, the product family 
$$\otimes^r \mathcal{V}_J := \{v_{j_r}(\z_r,\dots,\z_1,\y_M,\dots, \y_{r})\dots v_{j_1}(\z_1,\y_M,\dots,\y_1)|(j_r,\dots,j_1)\in \times^r J\},$$ 
is linearly independent under finite mixtures on $(\z_r,\dots,\z_1,\y_M,\dots,\y_1)\in \mathcal{Z}_r \times \mathcal{Y}'$.
\end{lemma}

\begin{proof}
    We prove the above statement by induction, where we start with the base case ($r=2$) as follows. The proof strategy for the base case largely follows the LIPO-1 proof.

    \paragraph{Case $r=2$.} 

    Assume the statement is false. Then, for every set $\mathcal{Y}\subset (\times^M \R^d)$ with $\mu(\mathcal{Y}) = \mu(\times^M \R^d)$, there exist nonzero measure sets $\mathcal{Z}_2 \subset (\R^{d} \times \R^{d})$ and $\mathcal{Y}'\subset \mathcal{Y}$ such that $\mathcal{V}_J \otimes \mathcal{V}_J$ contains linearly dependent functions. We fix $\mathcal{Y}$ from (b4) and consider any choice of $\mathcal{Y'}$ and $\mathcal{Z}_2$. This means there exists an index set $S_0\subset (J\times J), |S_0|< +\infty$ and the corresponding nonzero coefficients $\{\gamma_{j_1 j_2}\in\R , (j_1,j_2)\in S_0\}$, such that the following linear dependence condition hold:
    \begin{multline}\label{eq:linear_independence_b4_extension}
        \sum_{(j_1,j_2)\in S_0} \gamma_{j_1 j_2} v_{j_2}(\z_2 , \z_1 , \y_M, \dots, \y_2 ) v_{j_1}(\z_1 , \y_M, \dots , \y_1 ) = 0, \\ \forall (\z_2 ,\z_1 ,\y_M, \dots ,\y_1)\in (\mathcal{Z}_2 \times \mathcal{Y}'),
        \end{multline}
The choices of $S_0$ and $\gamma_{j_1 j_2}$ values might depend on the choice of $\mathcal{Z}_2$ and $\mathcal{Y}'$, and $\mathcal{Y}'$ depends on the choice of full-measure $\mathcal{Y}$. 

From (b1), the set $S_0$ contains at least two different indices $(j_1,j_2)$ and $(j_1',j_2')$, with $j_2\neq j_2'$. To see this, assume we have $j_2=j_2'$ for all possible $(j_1,j_2), (j_1',j_2') \in S_0$. Then we can factor out $v_{j_2}(\z_2 , \z_1 , \y_M, \dots, \y_2 )$ and rewrite the linear dependence equation as
    \begin{multline}\label{eq:linear_independence_2_js_b4_extension}
        v_{j_2}(\z_2 , \z_1 , \y_M, \dots, \y_2 ) \sum_{j_1:(j_1,j_2)\in S_0} \gamma_{j_1 j_2}  v_{j_1}(\z_1 , \y_M, \dots , \y_1 ) = 0, \\ \forall (\z_2 ,\z_1 ,\y_M, \dots ,\y_1)\in (\mathcal{Z}_2 \times \mathcal{Y}'),
        \end{multline}
which contradicts (b4) since $\mathcal{Y}'\subset\mathcal{Y}$.

Now we define $J_0:=\{j_2\in J | \exists (j_1,j_2)\in S_0\}$, $|J_0|< +\infty$. Then, linear dependence can occur for any $(\bm{\beta}_M,\dots,\bm{\beta}_1) \in (\pi_{\{1\}}(\mathcal{Z}_2) \times \pi_{\{M,\dots,2\}}(\mathcal{Y}'))$, where $\pi$ denotes set projections defined in Def.~\ref{def:set-projection}.
 \begin{multline}
 \sum_{j_2\in J_0}\Bigg( \sum_{j_1:(j_1,j_2)\in S_0} \gamma_{j_1 j_2}  v_{j_1}(\z_1=\bm{\beta}_{M} , \y_M =\bm{\beta}_{M-1}, \dots , \y_2=\bm{\beta}_{1},\y_1 ) \Bigg) \\ v_{j_2}(\z_2 , \z_1 = \bm{\beta}_M , \y_M=\bm{\beta}_{M-1}, \dots, \y_2 = \bm{\beta}_1 )= 0, \\  \forall (\z_2 ,\y_1)\in (\mathcal{Z}' \times \mathcal{Y}'_1).
 \end{multline}
Where both $\mathcal{Z}'=\mathcal{Z}_{2}|_{(\z_1=\bm{\beta}_M)}$ and $\mathcal{Y}'_1 = \mathcal{Y}'|_{(\y_M=\bm{\beta}_{M-1}, \dots, \y_2
=\bm{\beta}_1)}$ are conditioned projections (Def.~\ref{def:reverse-projection}) of $\mathcal{Z}_2$ and $\mathcal{Y}'$ respectively and depend on $(\bm{\beta}_M, \dots, \bm{\beta}_1)$. These conditioned projections are nonempty by construction since 
$(\bm{\beta}_M,\dots,\bm{\beta}_1)\in (\pi_{\{1\}}(\mathcal Z_2)\times \pi_{\{M,\dots,2\}}(\mathcal Y'))$. Now, we define the following set.
\begin{multline}
\mathcal{D}_0 := \Bigg\{(\bm{\beta}_M,\dots,\bm{\beta}_1) \in \Big(\pi_{\{1\}}(\mathcal{Z}_2) \times \pi_{\{M,\dots,2\}}(\mathcal{Y}')\Big)\bigg| \\ \sum_{j_1:(j_1,j_2)\in S_0}\gamma_{j_1,j_2}v_{j_1}(\z_1=\bm{\beta}_{M} , \y_M =\bm{\beta}_{M-1}, \dots , \y_2=\bm{\beta}_{1},\y_1 ) = 0, \ \forall\y_1\in\mathcal{Y}'_1\Bigg\}.
\end{multline}
From (b4) we require linear independence for every $\mathcal{Y}' \subset \mathcal{Y}$, where $\mathcal{Y}$ has full measure; and every $\mathcal{Z}'\subset \R^d$. Therefore, linear dependence as described above can happen at most in a measure zero set as $\mathcal{Y}_0 :=(\times^M \R^d) \setminus\mathcal{Y}$ has zero measure in $(\times^M \R^d)$. Therefore, the set 
\[
\mathcal{D}_0 \subset \pi_{\{1\}}(\mathcal{Z}_2) \times \pi_{\{M,\dots,2\}}(\mathcal{Y}_0)
\]
also has measure zero due to \(\mathcal{Y}_0\) having measure zero.

Now define $\mathcal{D}:= (\pi_{\{1\}}(\mathcal{Z}_2) \times \pi_{\{M,\dots,2\}}(\mathcal{Y}') ) \setminus \mathcal{D}_0$ which is nonzero measured. From assumption (b1), we have $\forall (\z_1=\bm{\beta}_{M} , \y_M =\bm{\beta}_{M-1}, \dots , \y_2=\bm{\beta}_{1})\in \mathcal{D}$, there exists $\y_1\in \mathcal{Y}'|_{(\y_M=\bm{\beta}_{M-1}, \dots, \y_2
=\bm{\beta}_1)}$ such that 
$$\tilde{\gamma}_{j_2}(\bm{\beta}_M,\dots,\bm{\beta}_1):=\sum_{j_1:(j_1,j_2)\in S_0}\gamma_{j_1j_2}v_{j_1}(\z_1=\bm{\beta}_{M} , \y_M =\bm{\beta}_{M-1}, \dots , \y_2=\bm{\beta}_{1},\y_1)\neq 0$$for at least two different $j_2\in J_0$ as shown in Eq.~(\ref{eq:linear_independence_2_js_b4_extension}). This implies linear dependence of $\{v_{j_2}(\z_2 , \z_1 = \bm{\beta}_M , \y_M=\bm{\beta}_{M-1}, \dots, \y_2 = \bm{\beta}_1 )\}$ on $\z_2\in \mathcal{Z}_{2}|_{(\z_1=\bm{\beta}_M)}$, $\forall (\z_1 = \bm{\beta}_M , \y_M=\bm{\beta}_{M-1}, \dots, \y_2 = \bm{\beta}_1 )\in \mathcal{D}$ by referring back to Eq.~(\ref{eq:linear_independence_b4_extension}).

Under assumption (b5), we have linear dependence from repeating functions. Therefore, we can split $J_0$ into subsets indexed by $k\in K(\bm{\beta}_M,\dots,\bm{\beta}_1)$, such that the $v_{j}$ functions within each subset $J_k(\bm{\beta}_M,\dots,\bm{\beta}_1)$ are equal as described in Eq.~\eqref{eq:b4_extension_splits_2}:
\begin{multline}
\label{eq:b4_extension_splits_1}
J_0 = \cup_{k \in K(\bm{\beta}_M,\dots,\bm{\beta}_1)} J_k(\bm{\beta}_M,\dots,\bm{\beta}_1), \text{  }
J_k(\bm{\beta}_M,\dots,\bm{\beta}_1) \cap J_{k'}(\bm{\beta}_M,\dots,\bm{\beta}_1) = \emptyset, \\
\forall k \neq k' \in K(\bm{\beta}_M,\dots,\bm{\beta}_1),
\end{multline}
\begin{multline}
\label{eq:b4_extension_splits_2}
 j \neq j' \in J_k(\bm{\beta}_M,\dots,\bm{\beta}_1) \quad \Leftrightarrow \quad \\
v_{j}(\z_2 , \z_1 = \bm{\beta}_M , \dots, \y_2 = \bm{\beta}_1 ) = v_{j'}(\z_2 , \z_1 = \bm{\beta}_M , \dots, \y_2 = \bm{\beta}_1 ), \quad \forall \z_2 \in \mathcal{Z}_{2}|_{(\z_1=\bm{\beta}_M)}.
\end{multline}

Then we can rewrite the linear dependence condition for any $(\bm{\beta}_M,\dots,\bm{\beta}_1) \in \mathcal{D}$ as
    \begin{multline}
        \sum_{k \in K(\bm{\beta}_M,\dots,\bm{\beta}_1)} \Bigg( \sum_{j_2 \in J_k(\bm{\beta}_M,\dots,\bm{\beta}_1)}  \sum_{j_1 : (j_1,j_2)\in S_0} \gamma_{j_1 j_2}  v_{j_1}(\z_1=\bm{\beta}_{M} , \y_M =\bm{\beta}_{M-1}, \dots , \y_2=\bm{\beta}_{1},\y_1 ) \\
         v_{j_2}(\z_2 , \z_1 = \bm{\beta}_M , \dots, \y_2 = \bm{\beta}_1 ) \Bigg)= 0, \quad  \forall (\z_2, \y_1)\in \mathcal{Z}_{2}|_{(\z_1=\bm{\beta}_M)} \times \mathcal{Y}'|_{(\y_M=\bm{\beta}_{M-1}, \dots, \y_2 =\bm{\beta}_1)}
        \end{multline}
Recall from that $v_{j}(\z_2 , \z_1 = \bm{\beta}_M , \dots, \y_2 = \bm{\beta}_1 )$ and $v_{j'}(\z_2 , \z_1 = \bm{\beta}_M , \dots, \y_2 = \bm{\beta}_1 )$ are the same functions on $\z_2 \in \mathcal{Z}_{2}|_{(\z_1=\bm{\beta}_M)} $ iff.~$j \neq j'$ are in the same index set $J_k(\bm{\beta}_M,\dots,\bm{\beta}_1)$. This means if linear independence holds, then for any $(\bm{\beta}_M,\dots,\bm{\beta}_1)\in \mathcal{D}$, under assumptions (b1) and (b5),
\begin{multline}
\sum_{j_2 \in J_k(\bm{\beta}_M,\dots,\bm{\beta}_1)}  \sum_{j_1 : (j_1,j_2)\in S_0} \gamma_{j_1 j_2}  v_{j_1}(\z_1=\bm{\beta}_{M} , \y_M =\bm{\beta}_{M-1}, \dots , \y_2=\bm{\beta}_{1},\y_1 ) = 0, \\
\quad \forall  \y_1\in \mathcal{Y}'|_{(\y_M=\bm{\beta}_{M-1}, \dots, \y_2 =\bm{\beta}_1)}, \quad k \in K(\bm{\beta}_{M-1}, \dots, \bm{\beta}_1 ).  
\end{multline}

Define $C(\bm{\beta}_M,\dots,\bm{\beta}_1) = \min_k |J_k(\bm{\beta}_M,\dots,\bm{\beta}_1)|$ the minimum cardinality count for $j_2$ indices in the $J_k(\bm{\beta}_M,\dots,\bm{\beta}_1)$ subsets. 

Choose $(\bm{\beta}_M^*,\dots,\bm{\beta}_1^*) \in \arg\min_{(\bm{\beta}_M,\dots,\bm{\beta}_1) \in \mathcal{D}} C(\bm{\beta}_M,\dots,\bm{\beta}_1) $:
\begin{itemize}
    \item[1.] We have $C(\bm{\beta}_M^*,\dots,\bm{\beta}_1^*)  < |J_0|$ and $|K(\bm{\beta}_M^*,\dots,\bm{\beta}_1^*)| \geq 2$. Otherwise for all $j \neq j' \in J_0$ we have $v_{j}(\z_2 , \z_1 = \bm{\beta}_M , \dots, \y_2 = \bm{\beta}_1 ) = v_{j'}(\z_2 , \z_1 = \bm{\beta}_M , \dots, \y_2 = \bm{\beta}_1 )$ for all $\z_2 \in \mathcal{Z}_{2}|_{(\z_1=\bm{\beta}_M)}$ and $(\bm{\beta}_M,\dots,\bm{\beta}_1) \in\mathcal{D})$, so that they are linearly dependent on $(\z_2, \z_1, \y_M, \dots, \y_2) \in \mathcal{Z}' \times \mathcal{D}$ for some nonzero measure $\mathcal{Z}' \subseteq \pi_{\{2\}}(\mathcal{Z}_2)$, a contradiction to assumption (b4) by setting $\mathcal{Y}' = \mathcal{D}$, which holds no matter the choice of $\mathcal{Z}_2$ as $\mathcal{Y}'\subset \mathcal{Y}$ where $\mathcal{Y}$ has full measure.
    \item[2.] Now assume $|J_{1}(\bm{\beta}_M^*,\dots,\bm{\beta}_1^*)| = C(\bm{\beta}_M^*,\dots,\bm{\beta}_1^*)$ w.l.o.g.. From assumption (b5), we know that for any $j \in J_1(\bm{\beta}_M^*,\dots,\bm{\beta}_1^*)$ and $j' \in J_0 \backslash J_{1}(\bm{\beta}_M^*,\dots,\bm{\beta}_1^*)$, $v_{j}(\z_2 , \z_1 = \bm{\beta}_M , \dots, \y_2 = \bm{\beta}_1 ) = v_{j'}(\z_2 , \z_1 = \bm{\beta}_M , \dots, \y_2 = \bm{\beta}_1 )$ only on a zero measure subset of $\mathcal{Z}_{2}|_{(\z_1=\bm{\beta}_M)}$ at most. Then as $|J_0| < +\infty$ and $\mathcal{Z}_{2}|_{(\z_1=\bm{\beta}_M)} \subset \mathbb{R}^{d}$ has nonzero measure, there exist $\z_0 \in \mathcal{Z}_{2}|_{(\z_1=\bm{\beta}_M)}$ and $\delta > 0$ such that 
    \begin{multline}
        |v_{j}(\z_2 = \z_0 , \z_1 = \bm{\beta}_M^* , \dots, \y_2 = \bm{\beta}_1^* ) - v_{j'}(\z_2 = \z_0 , \z_1 = \bm{\beta}_M^* , \dots, \y_2 = \bm{\beta}_1^* )| \geq \delta, \quad \\
        \forall j \in J_{1}(\bm{\beta}_M^*,\dots,\bm{\beta}_1^*), \forall j' \in J_0 \backslash J_{1}(\bm{\beta}_M^*,\dots,\bm{\beta}_1^*)
    \end{multline}
    Under assumption (b6), there exists $\epsilon_j > 0$ such that we can construct an $\epsilon$-ball $B(\bm{\beta}_M^*,\dots,\bm{\beta}_1^*, \epsilon_j)$ using $\ell_2$-norm, such that 
    \begin{multline}
     |v_{j}(\z_2 = \z_0 , \z_1 = \bm{\beta}_M^* , \dots, \y_2 = \bm{\beta}_1^* ) - v_{j}(\z_2 = \z_0 , \z_1 = \bm{\beta}_M , \dots, \y_2 = \bm{\beta}_1)| \leq \delta/3, \\
    \quad \forall (\bm{\beta}_M,\dots,\bm{\beta}_1) \in B(\bm{\beta}_M^*,\dots,\bm{\beta}_1^*, \epsilon_j).
    \end{multline}
    Choosing a suitable $0 < \epsilon \leq \min_{j \in J_0} \epsilon_j$ (note that $\min_{j \in J_0} \epsilon_j > 0$ as $|J_0| < +\infty$) and constructing an $\ell_2$-norm-based $\epsilon$-ball $B(\bm{\beta}_M^*,\dots,\bm{\beta}_1^*, \epsilon) \subset \mathcal{D}$, we have for all $j \in J_{1}(\bm{\beta}_M^*,\dots,\bm{\beta}_1^*), j' \in J_0 \backslash J_{1}(\bm{\beta}_M^*,\dots,\bm{\beta}_1^*)$, $j' \notin J_1(\bm{\beta}_M,\dots,\bm{\beta}_1)$ for all $(\bm{\beta}_M,\dots,\bm{\beta}_1) \in B(\bm{\beta}_M^*,\dots,\bm{\beta}_1^*, \epsilon)$ due to
    \begin{multline}
     |v_{j}(\z_2 = \z_0 , \z_1 = \bm{\beta}_M , \dots, \y_2 = \bm{\beta}_1 ) - v_{j'}(\z_2 = \z_0 , \z_1 = \bm{\beta}_M , \dots, \y_2 = \bm{\beta}_1)| \geq \delta/3,   \\
      \quad \forall (\bm{\beta}_M,\dots,\bm{\beta}_1) \in B(\bm{\beta}_M^*,\dots,\bm{\beta}_1^*, \epsilon).
    \end{multline}
    So this means for the split $\{ J_k(\bm{\beta}_M,\dots,\bm{\beta}_1) \}$ of any $(\bm{\beta}_M,\dots,\bm{\beta}_1) \in B(\bm{\beta}_M^*,\dots,\bm{\beta}_1^*, \epsilon)$, we have $ J_1(\bm{\beta}_M,\dots,\bm{\beta}_1) \subset J_1(\bm{\beta}^*)$ 
    and therefore $ |J_1(\bm{\beta}_M,\dots,\bm{\beta}_1)| \leq |J_1(\bm{\beta}_M^*,\dots,\bm{\beta}_1^*)|$. Now by definition of $(\bm{\beta}_M^*,\dots,\bm{\beta}_1^*)\in \arg\min_{(\bm{\beta}_M,\dots,\bm{\beta}_1) \in \mathcal{D}} C(\bm{\beta}_M,\dots,\bm{\beta}_1)$ and $|J_{1}(\bm{\beta}_M^*,\dots \\,\bm{\beta}_1^*)| = C(\bm{\beta}_M^*,\dots,\bm{\beta}_1^*)$, we have $J_1(\bm{\beta}_M,\dots,\bm{\beta}_1) = J_1(\bm{\beta}_M^*,\dots,\bm{\beta}_1^*)$ for all $(\bm{\beta}_M,\dots,\bm{\beta}_1) \\\in B(\bm{\beta}_M^*,\dots,\bm{\beta}_1^*, \epsilon)$. 
    \item[3.] One can show that $|J_{1}(\bm{\beta}_M^*,\dots,\bm{\beta}_1^*)| = 1$, otherwise by definition of the splits and the above point, there exists $j \neq j' \in J_{1}(\bm{\beta}_M^*,\dots,\bm{\beta}_1^*)$ such that $v_{j}(\z_2 , \z_1 = \bm{\beta}_M , \dots, \y_2 = \bm{\beta}_1 ) = v_{j'}(\z_2 , \z_1 = \bm{\beta}_M , \dots, \y_2 = \bm{\beta}_1 )$ for all $\z_2 \in \mathcal{Z}_{2}|_{(\z_1=\bm{\beta}_M)}$ and $(\bm{\beta}_M,\dots,\bm{\beta}_1) \in B(\bm{\beta}_M^*,\dots,\bm{\beta}_1^*, \epsilon)$, a contradiction to assumption (b4) by setting $\mathcal{Y}^* :=B(\bm{\beta}_M^*,\dots,\bm{\beta}_1^*, \epsilon)\cap \mathcal{Y}$. 
    Now assume that $j \in J_1(\bm{\beta}_M^*,\dots,\bm{\beta}_1^*)$ is the only index in the subset, then the fact proved in the above point that $J_1(\bm{\beta}_M,\dots,\bm{\beta}_1) = J_1(\bm{\beta}_M^*,\dots,\bm{\beta}_1^*)$ for all $(\bm{\beta}_M,\dots,\bm{\beta}_1) \in \mathcal{Y}^*$ means
    \begin{multline}
    \label{eq:b4_extension_linear_dependence_contradiction}
    \sum_{ j_1: (j_1, j_2) \in S_0} \gamma_{j_1j_2} v_{j_1}(\z_1=\bm{\beta}_{M} , \y_M =\bm{\beta}_{M-1}, \dots , \y_2=\bm{\beta}_{1},\y_1 )  = 0, \\
    \quad \forall \y_1\in \mathcal{Y}'|_{(\y_M=\bm{\beta}_{M-1}, \dots, \y_2 =\bm{\beta}_1)}, \quad \forall (\bm{\beta}_M, \dots, \bm{\beta}_1) \in \mathcal{Y}^*,
    \end{multline}
    which again contradicts (b4) as we find linear dependence at least in a nonzero measure set.
\end{itemize}

The above 3 points indicate that for any choice of $\mathcal{Z}_2$ and $\mathcal{Y}$, we find linear dependence of $\mathcal{V}_J$ in at least a nonzero measure set. This is a direct contradiction of assumption (b4), where linear independence under finite mixtures must hold for any nonzero measure subset of $\mathcal{Y}\times \R$ with full-measure $\mathcal{Y}$.

We note the assumption of $\mathcal{Y}$ having full measure in (b4), as opposed to (a4) is crucial here. Otherwise, the resulting $\epsilon$-ball could lie outside of the choice of $\mathcal{Y}$, as it is constructed from projections of $\mathcal{Z}_2\subset(\R^d\times\R^d)$ and $\mathcal{Y'}\subset\mathcal{Y}$. Therefore, we would have $\mu(B(\bm{\beta}_M^*,\dots,\bm{\beta}_1^*, \epsilon)\cap\mathcal{Y})=0$ and the above contradiction would not hold.

    \paragraph{Case $r > 2$.} Now assume the statement holds for $r-1$, and again prove the case by contradiction. Assume the statement is false for $r > 2$. Then, for every set $\mathcal{Y}\subset (\times^M \R^d)$ with $\mu(\mathcal{Y}) = \mu(\times^M \R^d)$, there exists nonzero measure sets $\mathcal{Z}_r \subset (\times^r \R^d)$ and $\mathcal{Y}'\subset \mathcal{Y}$ such that the family $\otimes^r \mathcal{V}_J$ contains linearly dependent functions. As before, we fix $\mathcal{Y}$ from (b4) and consider any choice of $\mathcal{Y}'$ and $\mathcal{Z}_r$. Then, there exists $S_0 \subset (\times^{r}J),\ |S_0| < +\infty$ such that
    \begin{multline}
    \sum_{(j_r,\dots j_1)\in S_0}\gamma_{j_r,\dots,j_1}v_{j_r}(\z_r,\dots,\z_1,\y_M,\dots, \y_{r})\dots v_{j_1}(\z_1,\y_M,\dots,\y_1) = 0, \\ \forall (\z_r,\dots ,\z_1 ,\y_M, \dots ,\y_1)\in (\mathcal{Z}_r \times \mathcal{Y}'),
    \end{multline}
    with $\{\gamma_{j_r, \dots, j_1}\in\R , (j_r,\dots,j_1)\in S_0\}$ a set of nonzero values which, as before, might depend on the choice of $\mathcal{Z}_r$ and $\mathcal{Y}'\subset\mathcal{Y}$. The previous equality can be arranged as follows
\begin{multline}\label{eq:b4_extension_linear_dependence_relation_n}
    \sum_{(j_r,\dots j_1)\in S_0}\gamma_{j_r,\dots,j_1}v_{j_r}(\z_r,\dots,\z_1,\y_M,\dots, \y_{r}) v_{j_{r-1},\dots,j_1}(\z_{r-1},\dots,\z_1,\y_M,\dots,\y_1) = 0, \\ \forall (\z_r,\dots ,\z_1 ,\y_M, \dots ,\y_1)\in (\mathcal{Z}_r \times \mathcal{Y}'),
\end{multline}
    where $v_{j_{r-1},\dots,j_1}:=v_{j_{r-1}}(\z_{r-1},\dots,\z_1,\y_M,\dots, \y_{r-1})\dots v_{j_1}(\z_1,\y_M,\dots,\y_1)$ for $j_i \in J, 1\leq i < r$ denotes the functions on the family $\otimes^{r-1} \mathcal{V}_J$ and satisfies the following: 
    \begin{enumerate}
        \item[(b1)] Positive function values: $v_{j_{r-1},\dots,j_1}(\z_{r-1},\dots,\z_1,\y_M,\dots,\y_1) > 0$ for all $(j_{r-1}, \dots, j_1) \in (\times^{r-1} J), (\z_{r-1},\dots,\z_1,\y_M,\dots,\y_1) \in (\times^{M+r-1} \R^{d})$;
    \end{enumerate}
    and from the induction hypothesis:
    \begin{enumerate}
        \item[(b4*)] There exists a full-measure set $\mathcal{Y}\subset (\times^M \R^d)$ such that for every nonzero measure sets $\mathcal{Y}' \subset \mathcal{Y}$, and $\mathcal{Z}_{r-1} \subset (\times^{r-1} \R^{d})$, the family $\otimes^{r-1} \mathcal{V}_J$ contains linearly independent functions under finite mixtures in $(\z_{r-1},\dots,\z_1,\y_M,\dots,\y_1)\in \mathcal{Z}_{r-1} \times \mathcal{Y}'$.
    \end{enumerate}
The strategy here is to reduce this case to the above base case ($n=2$), where importantly we need to show: 
\begin{enumerate}
    \item[(1)]  $S_0$ contains at least two $(j_r, \dots, j_1)$ and $(j_r',\dots, j_1')$ with $j_r \neq j_r'$; and
    \item[(2)] the overlapping variables between the product of families, $(\z_{r-1}, \dots, \z_1, \y_M, \dots, \y_r) \in (\pi_{\{r-1,\dots,1\}}(\mathcal{Z}_r) \times \pi_{\{M, \dots, r\}}(\mathcal{Y}'))$ do not cause linear dependence in nonzero measure sets for any choice of $\mathcal{Z}_r$ and $\mathcal{Y}'\subset \mathcal{Y}$.
\end{enumerate}
We can simply see (1) holds from (b1) and (b4) by following the same logic from Eq.~\eqref{eq:linear_independence_2_js_b4_extension}, where the linear dependence relation contradicts (b4) if we have $j_r=j_r'$. Therefore, we define $J_0:=\{j_r\in J | \exists (j_r, \dots,j_1)\in S_0\}$, $|J_0|< +\infty$. Then, linear dependence can occur for any $(\bm{\beta}_M,\dots,\bm{\beta}_1) \in (\pi_{\{r-1,\dots,1\}}(\mathcal{Z}_r) \times \pi_{\{M, \dots, r\}}(\mathcal{Y}'))\subset (\times^{M}\R^d)$, where $\pi$ denotes set projections defined in Def.~\ref{def:set-projection}.
 \begin{multline}
 \sum_{j_r\in J_0} v_{j_r}(\z_r, \z_{r-1}=\bm{\beta}_{M},\dots,\z_1=\bm{\beta}_{M-r+2},\y_M=\bm{\beta}_{M-r+1}\dots, \y_{r}=\bm{\beta}_1) \\ \Bigg( \sum_{(j_{r-1},\dots,j_1):(j_r,\dots,j_1)\in S_0} \gamma_{j_r,\dots, j_1} v_{j_{r-1},\dots,j_1}(\z_{r-1}=\bm{\beta}_{M},\dots,\z_1=\bm{\beta}_{M-r+2},\y_M=\bm{\beta}_{M-r+1}, \\ \dots, \y_{r}=\bm{\beta}_1, \y_{r-1},\dots,\y_1) \Bigg) = 0,  \\ \forall (\z_r ,\y_{r-1}, \dots, \y_1)\in \Big(\mathcal{Z}_{r}|_{(\z_{r-1}=\bm{\beta}_M, \dots, \z_1=\bm{\beta}_{M-r+2})} \times \mathcal{Y}'|_{(\y_M=\bm{\beta}_{M-r+1}, \dots, \y_r=\bm{\beta}_1)}\Big).
 \end{multline}
Where as in the base case, both $\mathcal{Z}_{r}|_{(\z_{r-1}=\bm{\beta}_M, \dots, \z_1=\bm{\beta}_{M-r+2})}$ and $\mathcal{Y}'|_{(\y_M=\bm{\beta}_{M-r+1}, \dots, \y_r=\bm{\beta}_1)}$ are conditioned projections (Def.~\ref{def:reverse-projection}) from $\mathcal{Z}_r$ and $\mathcal{Y}'$ respectively  and are dependent on $\bm{\beta}_M, \dots, \bm{\beta}_1$. As before, these sets are never empty, and we can establish the following equivalence between variables and sets (given the choice of $\mathcal{Z}_r$ and $\mathcal{Y}'\subset \mathcal{Y}$) with respect to the base case ($r=2$).
\begin{itemize}[leftmargin=1.5em]
    \item $\z_{r}$ is equivalent to $\z_2$, and therefore $\mathcal{Z}_{r}|_{(\z_{r-1}=\bm{\beta}_M, \dots, \z_1=\bm{\beta}_{M-r+2})}$ is equivalent to $\mathcal{Z}_{2}|_{(\z_1=\bm{\beta}_M)}$.
    \item $\z_{r-1}, \dots, \z_1$ are equivalent to $\z_1$, and therefore $\pi_{\{r-1, \dots, 1\}}(\mathcal{Z}_r)$ is equivalent to $\pi_{\{1\}}(\mathcal{Z}_2)$.
    \item $\y_{M}, \dots, \y_r$ are equivalent to $\y_M, \dots, \y_2$, and therefore $\pi_{\{M, \dots, r\}}(\mathcal{Y}')$ is equivalent to $\pi_{\{M, \dots, 2\}}(\mathcal{Y}')$.
    \item $\y_{r-1}, \dots, \y_1$ are equivalent to $\y_1$, and therefore $\mathcal{Y}'|_{(\y_M=\bm{\beta}_{M-r+1}, \dots,\y_r=\bm{\beta}_1)}$ is equivalent to $\mathcal{Y}'|_{(\y_{M}=\bm{\beta}_{M-1}, \dots, \y_2=\bm{\beta}_1)}$.
\end{itemize}
Given the above equivalences, the arguments for contradiction given linear dependence for any $(\bm{\beta}_M,\dots,\bm{\beta}_1) \in (\pi_{\{r-1,\dots,1\}}(\mathcal{Z}_r) \times \pi_{\{M, \dots, r\}}(\mathcal{Y}'))$ still hold. The only difference is that the dimensionality of the projections and conditioned projections will change (except for $\z_r\in\mathcal{Z}_{r}|_{(\z_{r-1}=\bm{\beta}_M, \dots, \z_1=\bm{\beta}_{M-r+2})}$). Therefore, the set
\begin{multline}
\mathcal{D}_0 := \Bigg\{(\bm{\beta}_M,\dots,\bm{\beta}_1) \in \Big(\pi_{\{r-1,\dots,1\}}(\mathcal{Z}_r) \times \pi_{\{M, \dots, r\}}(\mathcal{Y}')\Big)\bigg| \\ \sum_{(j_{r-1},\dots,j_1):(j_r,\dots,j_1)\in S_0} \gamma_{j_r,\dots, j_1} v_{j_{r-1},\dots,j_1}(\z_{r-1}=\bm{\beta}_{M},\dots,\z_1=\bm{\beta}_{M-r+2},\y_M=\bm{\beta}_{M-r+1}, \\ \dots, \y_{r}=\bm{\beta}_1, \y_{r-1},\dots,\y_1)  = 0, \forall(\y_{r-1}, \dots, \y_1)\in\mathcal{Y}'|_{(\y_M=\bm{\beta}_{M-r+1}, \dots,\y_r=\bm{\beta}_1)}\Bigg\}
\end{multline}
has zero measure under assumption (b4), which implies that for any $(\z_{r-1}=\bm{\beta}_{M},\dots, \z_1=\bm{\beta}_{M-r+2},\y_M=\bm{\beta}_{M-r+1}, \dots, \y_{r}=\bm{\beta}_1)\in \mathcal{D}$, with $\mathcal{D}:= (\pi_{\{r-1,\dots,1\}}(\mathcal{Z}_r) \times \pi_{\{M, \dots, r\}}(\mathcal{Y}')) \setminus \mathcal{D}_0$, there exists $(\y_{r-1}, \dots, \y_1)\in\mathcal{Y}'|_{(\y_M=\bm{\beta}_{M-r+1}, \dots,\y_r=\bm{\beta}_1)}$ such that
\begin{multline}
    \tilde{\gamma}_{j_r}(\bm{\beta}_M,\dots,\bm{\beta}_1):=\sum_{(j_{r-1},\dots,j_1):(j_r,\dots,j_1)\in S_0} \gamma_{j_r,\dots, j_1} v_{j_{r-1},\dots,j_1}(\z_{r-1}=\bm{\beta}_{M},\dots,\\ \z_1=\bm{\beta}_{M-r+2},\y_M=\bm{\beta}_{M-r+1}, \dots, \y_{r}=\bm{\beta}_1, \y_{r-1},\dots,\y_1)\neq 0
\end{multline}
for at least two $j_r\in J_0$. As before, this implies linear dependence of $\{v_{j_r}(\z_r, \z_{r-1}=\bm{\beta}_{M},\dots,\z_1=\bm{\beta}_{M-r+2},\y_M=\bm{\beta}_{M-r+1}\dots, \y_{r}=\bm{\beta}_1)\}$ on $\z_r\in\mathcal{Z}_{r}|_{(\z_{r-1}=\bm{\beta}_M, \dots, \z_1=\bm{\beta}_{M-r+2})}$, $\forall ( \z_{r-1}=\bm{\beta}_{M},\dots,\z_1=\bm{\beta}_{M-r+2},\y_M=\bm{\beta}_{M-r+1}\dots, \y_{r}=\bm{\beta}_1)\in \mathcal{D}$. Therefore, using assumptions (b5-b6) we can split $J_0$ subsets indexed by $K(\bm{\beta}_M, \dots, \bm{\beta}_1)$ such that the functions in each subset are equal, similarly as in Eqs.~(\ref{eq:b4_extension_splits_1}-\ref{eq:b4_extension_splits_2}). Following previous arguments from the case $r=2$ we choose $(\bm{\beta}_M^*,\dots,\bm{\beta}_1^*) \in \arg\min_{(\bm{\beta}_M,\dots,\bm{\beta}_1) \in \mathcal{D}} C(\bm{\beta}_M,\dots,\bm{\beta}_1)$, and assumption (b6) establishes linear dependence of the elements $v_{j_{r-1},\dots,j_1}(\cdot)$ in at least a nonzero measure set (Eq.~\eqref{eq:b4_extension_linear_dependence_contradiction})

The linear dependence of $\otimes^{r-1}\mathcal{V}_J$ on a nonzero measure set contradicts the induction hypothesis (b4*), as $\mathcal{Y}$ has full measure. Hence, Eq.~\eqref{eq:b4_extension_linear_dependence_relation_n} cannot hold for any choice of $\mathcal{Y}'$ and $\mathcal{Z}_r$, and a contradiction is reached.
\end{proof}

Regarding (b5), we observe that linear dependence on $M$ products of functions given fixed subsets of variables does not imply repeating functions anymore. Instead, we show that at least one component of the product family is linearly independent under finite mixtures.

\begin{lemma}[Extension of (b5)]\label{lemma:extension_b5}
Assume a distribution family $\mathcal{V}_J=\{v_j(\z,\y_M,\dots,\y_1)| j\in J\}$ with variables $\z,\y_M,\dots,\y_1$ defined on $\times^{M+1} \R^{d}$, such that it satisfies (b1), (b4), (b5), (b6).

Then, for any $r$ with $2\leq r \leq M$, any $(\bm{\beta}_M,\dots,\bm{\beta}_1) \in ( \times^M \R^{d})$, any nonzero measure $\mathcal{Z}_r\subset (\times^{r}\R^{d})$, and $S_0 \subset (\times^r J)$ such that $|S_0| <+\infty$ and the family 
\begin{multline*}
\Big\{v_{j_r}(\z_r,\dots,\z_1,\y_M=\bm{\beta}_M,\dots, \y_r=\bm{\beta}_r)\dots \\ v_{j_1}(\z_1,\y_M=\bm{\beta}_M,\dots, \y_1=\bm{\beta}_1)|(j_r,\dots,j_1)\in S_0\Big\}
\end{multline*}
contains linearly dependent functions only if $\exists \ell \leq r$ such that \begin{multline*}
\Big\{v_{j_\ell}(\z_\ell, \dots, \z_1, \y_M=\bm{\beta}_M,\dots, \y_\ell=\bm{\beta}_\ell)|j_\ell\in J_\ell\Big\},\\ \text{with } J_\ell:=\big\{j_\ell\in J|\exists(j_r,\dots,j_1)\in S_0\big\}
\end{multline*}
contains linearly dependent functions on $(\z_\ell, \dots, \z_1) \in \pi_{\{\ell, \dots, 1\}}(\mathcal{Z}_r)$.
\end{lemma}
\begin{proof}
We prove the statement by induction starting from the base case ($r=2$) as follows.

\paragraph{Case $r=2$.} Fix $\mathcal{Z}_2\subset \R^{d}\times \R^{d}$, and assume that for any $(\bm{\beta}_M,\dots,\bm{\beta}_1) \in ( \times^M \R^{d})$, and $S_0\subset J\times J$, the family  $$\Big\{v_{j_2}(\z_2, \z_1,\y_M=\bm{\beta}_{M},\dots, \y_2=\bm{\beta}_2)v_{j_1}(\z_1,\y_M=\bm{\beta}_M,\dots, \y_1=\bm{\beta}_1)|(j_2,j_1)\in S_0\Big\}$$ contains linearly dependent functions. 
From (b1, b4-b6), we can connect these assumptions to the ones defined in Lemma \ref{lemma:linear_independence_two_nonlinear_gaussians} by setting $u_i(\y=\z_1, \x=\emptyset)=v_{j_1}(\z_1,\y_M=\bm{\beta}_M,\dots, \y_1=\bm{\beta}_1)$, $i=j_1$, and $v_j(\z=\z_2, \y=\z_1)=v_{j_2}(\z_2, \z_1,\y_M=\bm{\beta}_{M},\dots, \y_2=\bm{\beta}_2)$, $j=j_2$, where (a1), (a5) and (a6) are satisfied. Therefore, by Lemma \ref{lemma:linear_independence_two_nonlinear_gaussians}, we know that either one of the families contains linearly dependent functions (from contradictions to (a3) or (a4) in Lemma \ref{lemma:linear_independence_two_nonlinear_gaussians}).

\paragraph{Case $r > 2$.}

Fix $\mathcal{Z}_r\subset (\times^{r}\R^{d})$, and assume that for any $(\bm{\beta}_M,\dots,\bm{\beta}_1) \in ( \times^M \R^{d})$, and $S_0 \subset (\times^r J)$, the family
\begin{multline*}
 \Big\{v_{j_r}(\z_r,\dots,\z_1,\y_M=\bm{\beta}_M,\dots, \y_r=\bm{\beta}_r)\dots \\ v_{j_1}(\z_1,\y_M=\bm{\beta}_M,\dots, \y_1=\bm{\beta}_1)|(j_r,\dots,j_1)\in S_0\Big\}   
\end{multline*}
contains linearly dependent functions. 
As before, from (b1, b4-b6) and by setting 
\begin{multline*}
u_i(\y=(\z_{r-1}, \dots, \z_1), \x=\emptyset)=v_{j_{r-1}}(\z_{r-1},\dots,\z_1,\y_M=\bm{\beta}_M,\dots, \y_{r-1}=\bm{\beta}_{r-1})\dots \\ v_{j_1}(\z_1,\y_M=\bm{\beta}_M,\dots, \y_1=\bm{\beta}_1), \quad i=(j_{r-1},\dots, j_1), 
\end{multline*}
and $v_j(\z=\z_{r}, \y=(\z_{r-1}, \dots, \z_1))=v_{j_r}(\z_r,\dots \z_1,\y_M=\bm{\beta}_M,\dots, \y_r=\bm{\beta}_r), \quad j=j_r$, we can establish correspondence to assumptions on Lemma \ref{lemma:linear_independence_two_nonlinear_gaussians}. We know (a1) and (a6) are satisfied directly from (b1) and (b6). (a5) is satisfied by setting $\ell=M+1-r$ in (b5). Then, from contradictions to either (a3) or (a4) in Lemma \ref{lemma:linear_independence_two_nonlinear_gaussians}, either one of the families must contain linearly dependent functions.
Therefore, either
\begin{equation*}
    \Big\{v_{j_{r}}(\z_{r},\dots,\z_1,\y_M=\bm{\beta}_M,\dots, \y_{r}=\bm{\beta}_{r})|j_{r}\in J_r\Big\}, \text{ with }J_r=\big\{j_r\in J|\exists(j_r,\dots,j_1)\in S_0\big\}
\end{equation*}
contains linearly dependent functions on $(\z_{r},\dots,\z_1) \in \pi_{\{r,\dots,1\}}(\mathcal{Z}_r)$, or from the induction hypothesis there is some $\ell \leq r-1$ such that 
\begin{multline*}
    \Big\{v_{j_{\ell}}(\z_{\ell}, \dots,\z_1,\y_M=\bm{\beta}_M,\dots, \y_{\ell}=\bm{\beta}_\ell)|j_\ell\in J_\ell\Big\},\\ \text{ with }J_\ell=\big\{j_\ell\in J|\exists(j_r,\dots,j_1)\in S_0\big\}
\end{multline*}
contains linearly dependent functions on $(\z_\ell,\dots, \z_1) \in \pi_{\{\ell, \dots, 1\}}(\mathcal{Z}_r)$.
\end{proof}

\subsubsection{Linear independence on products with M overlapping variables}\label{app:lippo_m}

Given that assumption (b5) does not extend similarly as (b4), we generalise LIPO-1 to LIPO-M (Linear Independence on Product functions with $M$ consecutive Overlapping variables) using Lemma \ref{lemma:extension_b5} to show linear independence under finite mixtures when $V_J$ is defined as a product of $M$ functions with consecutive variables. 

\textbf{Notation:} We use the following block notation $\z_{\ell:1}$, $\bm{\beta}_{M:\ell}$, $\y_{M:\ell}$ to denote tuples of variables $(\z_\ell, \dots,\z_1)$, $(\bm{\beta}_M, \dots,\bm{\beta}_\ell)$, $(\y_M,\dots,\y_\ell)$, respectively.

\begin{lemma}[LIPO-M]
\label{lemma:absolutely_magical_lemma}
Assume two families $\mathcal{U}_I=\{u_i(\y,\x)| i\in I\}$ and $\mathcal{V}_J=\{v_j(\z,\y)| j\in J\}$, with $\x\in\mathcal{X}$, $\y\in\R^{d_y}$, and $\z\in\R^{d_z}$, with $d_y=d\cdot M$, $d_z=d\cdot r$, and $1 \leq r \leq M$. We further assume:
\begin{itemize}
    \item[(a)] The family $\mathcal{U}_I$ satisfies assumptions (a1) and (a3).
    \item[(b)] Each element in the family $\mathcal{V}_J$ is defined as a product of $r$ functions, all of which belong to the same family $\mathcal{W}_K=\{w_k(\z,\y_M,\dots,\y_1)| k\in K\}$, with $\z_1\in\R^d$, $(\y_M, \dots, \y_1) \in (\times^M \R^d)$, and is expressed as follows
    \begin{multline*}
        \mathcal{V}_J := \otimes^r \mathcal{W}_K = \{w_{k_r}(\z_r, \dots, \z_1, \y_M, \dots, \y_r)\dots w_{k_1}(\z_1,\y_M,\dots,\y_1)\mid \\
        (k_r,\dots, k_1)\in (\times^r K) \}
    \end{multline*}
    where $J:= (\times^r K)$, and $\z:=(\z_r, \dots, \z_1)$, $\y:=(\y_M, \dots, \y_1)$. The family $\mathcal{W}_K$ satisfies assumptions (b1), (b4), (b5), and (b6).
\end{itemize}
Then for any nonzero measure subset $\mathcal{Z} \subset \mathbb{R}^{d_z}$, $\mathcal{U}_I \otimes \mathcal{V}_J := \{u_i(\y, \x)v_{j}(\z, \y) | i \in I, j \in J \}$ 
contains linearly independent functions under finite mixtures defined on $(\x, \y, \z) \in \mathcal{X} \times \mathbb{R}^{d_y}\times \mathcal{Z}$.
\end{lemma}

\begin{proof}
Assume this sufficiency statement is false, then there exist a nonzero measure subset $\mathcal{Z} \subset \mathbb{R}^{d_z}$, $S_0 \subset I \times J$ with $|S_0| < +\infty$ and a set of nonzero values $\{ \gamma_{ij} \in \mathbb{R} | (i,j) \in S_0 \}$, such that 
\begin{equation}
\label{eq:linear_dependence_division}
\sum_{ (i, j) \in S_0} \gamma_{ij}v_{j}(\z, \y)u_i(\y, \x) = 0, \\
\quad \forall (\x, \y, \z) \in \mathcal{X}\times \mathbb{R}^{d_y} \times \mathcal{Z}.
\end{equation}
For convenience, we denote the elements of $J$ as indices despite being a tuple. Note that the choices of $S_0$ and $\gamma_{ij}$ are independent of any $(\x, \y, \z)$ values, but might be dependent on $\mathcal{Z}$. We wish to follow the logic presented in LIPO-1 (Lemma \ref{lemma:linear_independence_two_nonlinear_gaussians}) under (a1-a6). Below we establish some connections. First, we know that the family $\mathcal{V}_J$ as defined in (b) is a product of $r$ elements with $1\leq r \leq M$, all of which satisfy (b1) and (b4-b5). Lemma \ref{lemma:extension_b4}, shows that there exists some full-measure $\mathcal{Y}$, such that for any subset $\mathcal{Z}\times\mathcal{Y}'\subset\R^{d_z}\times\mathcal{Y}$, the family $\mathcal{V}_J$ is linearly independent under finite mixtures. This is stronger than (a4) in Lemma \ref{lemma:linear_independence_two_nonlinear_gaussians} because (b4) provides a full-measure $\mathcal{Y}$, while (a4) only requires a nonzero measure $\mathcal{Y}$. Furthermore, it also satisfies (a1) as it is defined as a product of $r$ functions that satisfy (b1), with $1\leq r \leq M$. Therefore, we follow the same arguments as in Lemma \ref{lemma:linear_independence_two_nonlinear_gaussians}, which show that the index set $S_0$ contains at least 2 different indices $(i, j)$ and $(i, j')$ with $j
\neq j'$. Moreover, we define $J_0 = \{j \in J | \exists (i, j) \in S_0 \}$ the set of all possible $j$ indices that appear in $S_0$, where $|J_0| < +\infty$ and the following set $\mathcal{Y}_0 := \{\bm{\beta} \in \mathbb{R}^{d_y} | \sum_{ i: (i, j) \in S_0} \gamma_{ij} u_{i}(\y = \bm{\beta}, \x) = 0, \forall \x \in \mathcal{X} \}$ can only have zero measure in $\mathbb{R}^{d_y}$ from (a3). Again following Lemma \ref{lemma:linear_independence_two_nonlinear_gaussians}, we have a nonzero measure set $\mathcal{Y}' := \mathcal{Y} \backslash \mathcal{Y}_0 \subset \mathcal{Y}$ where we choose $\mathcal{Y}$ from the full-measure set in (b4). Then, we have for each $\bm{\beta} \in \mathcal{Y}'$, there exists $\x \in \mathcal{X}$ such that $\sum_{ i: (i, j) \in S_0} \gamma_{ij} u_{i}(\y = \bm{\beta}, \x) \neq 0$ for at least two $j$ indices in $J_0$. This implies for each $\bm{\beta} \in \mathcal{Y}'$, $\{v_j(\z, \y = \bm{\beta}) | j \in J_0 \}$ contains linearly dependent functions on $\z \in \mathcal{Z}$.

We will now write $J$, $S_0$ and $J_0$ tuple format, as we need to expand the product family $\mathcal{V}_J$. Therefore, we now regard $j=(k_r,\dots,k_1)$, and re-define $J_0=\{(k_r,\dots,k_1)\in J|\exists (i, k_r,\dots,k_1)\in S_0\}$. From (b4-b6), Lemma \ref{lemma:extension_b5} shows that the product of $r$ functions ($1\leq r \leq M$) that composes the family $\mathcal{V}_J$ implies linearly dependent functions in at least one of the components. That is, for some $\ell \leq r$,  $\{w_{k_\ell}(\z_{\ell}, \dots, \z_1,\y_{M:\ell}=\bm{\beta}_{M:\ell})\mid k_\ell\in K^{(\ell)}_0\}$, with $\y_{M:\ell}:=(\y_M,\dots,\y_\ell)$, $\bm{\beta}_{M:\ell}:=(\bm{\beta}_M,\dots,\bm{\beta}_\ell)$, and $K^{(\ell)}_0 := \{k_\ell\in K|\exists(i,k_r,\dots,k_1)\in S_0\}$, contains linearly dependent functions on $(\z_\ell,\dots,\z_1) \in \pi_{\{\ell,\dots, 1\}}(\mathcal{Z})$. Under assumption (b5), we can split the index set $K_0^{(\ell)}$ into subsets indexed by $q \in L^{(\ell)}(\bm{\beta}_{M:\ell})$ as follows, such that within each index subset $K_q^{(\ell)}(\bm{\beta}_{M:\ell})$ the functions with the corresponding indices are equal:
\begin{equation}
\label{eq:j_index_split_def}
\begin{aligned}
K_0^{(\ell)}& = \cup_{q \in L^{(\ell)}(\bm{\beta}_{M:\ell})} K_q^{(\ell)}(\bm{\beta}_{M:\ell}), \quad K_q^{(\ell)}(\bm{\beta}_{M:\ell}) \cap K_{q'}^{(\ell)}(\bm{\beta}_{M:\ell}) = \emptyset, \forall q \neq q' \in L^{(\ell)}(\bm{\beta}_{M:\ell}), \\
& k_\ell \neq k_\ell' \in  K_q^{(\ell)}(\bm{\beta}_{M:\ell}) \quad \Leftrightarrow \quad w_{k_\ell}(\z_{\ell}, \dots, \z_1,\y_{M:\ell}=\bm{\beta}_{M:\ell}) = \\ & \qquad  w_{k'_\ell}(\z_{\ell}, \dots, \z_1,\y_{M:\ell}=\bm{\beta}_{M:\ell}), \forall (\z_\ell,\dots,\z_1) \in \pi_{\{\ell,\dots, 1\}}(\mathcal{Z}).
\end{aligned}
\end{equation}
Then we can rewrite Eq.~(\ref{eq:linear_dependence_division}) for any $\bm{\beta} \in \mathcal{Y}'$ as
\begin{multline}
\sum_{q \in L^{(\ell)}(\bm{\beta}_{M:\ell})} \Bigg(\sum_{k_\ell \in K^{(\ell)}_q(\bm{\beta}_{M:\ell})}\sum_{ (i, k_r, \dots, k_{\ell+1}, k_{\ell-1}, \dots, k_1): (i, k_r, \dots, k_1) \in S_0} \gamma_{i,k_r,\dots,k_1} u_{i}(\y = \bm{\beta}, \x) \\ w_{k_r}(\z_r,\dots,\z_1, \y_M=\bm{\beta}_M)\dots  w_{k_1}(\z_1, \y_M=\bm{\beta}_M,\dots, \y_1=\bm{\beta}_1) \Bigg) = 0, \\ \quad \forall (\x, \z_r, \dots, \z_1) \in \mathcal{X} \times \mathcal{Z}.  
\end{multline}
Note that we use $\gamma_{i,k_r,\dots,k_1} = \gamma_{ij}$ for any $(i,j)\in S_0$ and $(k_r,\dots,k_1)\in J_0$, such that $j = (k_r,\dots,k_1)$. Recall from Eq.~(\ref{eq:j_index_split_def}) that $w_{k_\ell}(\z_{\ell}, \dots, \z_1,\y_{M:\ell}=\bm{\beta}_{M:\ell})$ and $w_{k'_\ell}(\z_{\ell}, \dots, \z_1,\\ \y_{M:\ell}=\bm{\beta}_{M:\ell})$ are the same functions on $(\z_\ell, \dots, \z_1) \in \pi_{\{\ell,\dots, 1\}}(\mathcal{Z})$ iff.~$k_\ell \neq k_\ell'$ are in the same index set $K_l^{(\ell)}(\bm{\beta})$. This means if Eq.~(\ref{eq:linear_dependence_division}) holds, then for any $\bm{\beta} \in \mathcal{Y}'$, under assumptions (b1) and (b5),
\begin{multline}
\label{eq:zero_constraint_given_beta}
\sum_{k_\ell \in K^{(\ell)}_q(\bm{\beta}_{M:\ell})}\sum_{ (i, k_r, \dots, k_{\ell+1}, k_{\ell-1}, \dots, k_1): (i, k_r, \dots, k_1) \in S_0} \gamma_{i,k_r,\dots,k_1} u_{i}(\y = \bm{\beta}, \x) w_{k_r}(\z_r,\dots,\z_1,\\ \y_M=\bm{\beta}_M) \dots w_{k_{\ell+1}}(\z_{\ell+1},\dots,\z_1, \y_{M:\ell+1}=\bm{\beta}_{M:\ell+1})w_{k_{\ell-1}}(\z_{\ell-1},\dots,\z_1,\\ \y_{M:\ell-1}=\bm{\beta}_{M:\ell-1}) \dots w_{k_1}(\z_1, \y_{M:1}=\bm{\beta}_{M:1}) = 0 ,\\ \quad \forall (\x, \z_r, \dots, \z_1) \in \mathcal{X} \times \mathcal{Z},\text{  } q \in L^{(\ell)}(\bm{\beta}_{M:\ell}).
\end{multline}

Define $C^{(\ell)}(\bm{\beta}_{M:\ell}) = \min_q |K^{(\ell)}_q(\bm{\beta}_{M:\ell})|$ the minimum cardinality count for $k_\ell$ indices in the $K^{(\ell)}_q(\bm{\beta}_{M:\ell})$ subsets. Choose $\bm{\beta}^*_{M:\ell} \in \arg\min_{\bm{\beta}_{M:\ell} \in \pi_{ \{M,\dots, \ell\}}(\mathcal{Y}')} C^{(\ell)}(\bm{\beta}_{M:\ell})$:
\begin{itemize}
    \item[1.] We have $C^{(\ell)}(\bm{\beta}^*_{M:\ell}) < |K^{(\ell)}_0|$ and $|L^{(\ell)}(\bm{\beta}^*_{M:\ell})| \geq 2$. Otherwise for all $k_\ell \neq k_\ell' \in K^{(\ell)}_0$ we have $w_{k_\ell}(\z_\ell,\dots,\z_1, \y_{M:\ell}=\bm{\beta}_{M:\ell}) = w_{k_\ell'}(\z_\ell,\dots,\z_1, \y_{M:\ell}=\bm{\beta}_{M:\ell})$ for all $(\z_\ell, \dots, \z_1) \in \pi_{\{\ell,\dots, 1\}}(\mathcal{Z})$ and $\bm{\beta}_{M:\ell} \in \pi_{\{M,\dots,\ell\}}(\mathcal{Y}')$, so that they are linearly dependent on $(\z_\ell, \dots, \z_1, \y_{M:\ell}) \in \pi_{\{\ell,\dots, 1\}}(\mathcal{Z}) \times \pi_{ \{M,\dots, \ell\}}(\mathcal{Y}')$, a contradiction to assumption (b4).
    \item[2.] Now assume $|K_1^{(\ell)}(\bm{\beta}^*_{M:\ell})| = C^{(\ell)}(\bm{\beta}^*_{M:\ell})$ w.l.o.g.. From assumption (b5), we know that for any $k_\ell \in K_1^{(\ell)}(\bm{\beta}^*_{M:\ell})$ and $k_\ell' \in K_0^{(\ell)} \backslash K_{1}^{(\ell)}(\bm{\beta}^*_{M:\ell})$, $w_{k_\ell}(\z_\ell, \dots, \z_1, \y_{M:\ell}=\bm{\beta}_{M:\ell}) = w_{k_\ell'}(\z_\ell, \dots, \z_1, \y_{M:\ell}=\bm{\beta}_{M:\ell})$ only on zero measure subset of $\pi_{\{\ell, \dots, 1\}}(\mathcal{Z})$ at most. Then as $|K_0^{(\ell)}| < +\infty$ and $\pi_{ \{\ell,\dots,1\}}(\mathcal{Z}) \subset (\times^{\ell} \mathbb{R}^{d})$ has nonzero measure, there exist $(\z^{(0)}_\ell, \dots, \z^{(0)}_1) \in \pi_{\{\ell,\dots,1\}}(\mathcal{Z})$ and $\delta > 0$ such that 
    \begin{multline}
    |w_{k_\ell}(\z_{\ell:1}=\z_{\ell:1}^{(0)}, \y_{M:\ell}=\bm{\beta}^*_{M:\ell}) - w_{k_\ell'}(\z_{\ell:1}=\z_{\ell:1}^{(0)}, \y_{M:\ell}=\bm{\beta}^*_{M:\ell})| \geq \delta, \\ \forall k_\ell \in K_{1}^{(\ell)}(\bm{\beta}^*_{M:\ell}), \forall k_\ell' \in K_0^{(\ell)} \backslash K_{1}^{(\ell)}(\bm{\beta}^*_{M:\ell}).
    \end{multline}
    Under assumption (b6), there exists $\epsilon_{k_\ell} > 0$ such that we can construct an $\epsilon$-ball $B(\bm{\beta}^*_{M:\ell}, \epsilon_{k_\ell})$ using $\ell_2$-norm, such that 
    \begin{multline}
    |w_{k_\ell}(\z_{\ell:1}=\z_{\ell:1}^{(0)}, \y_{M:\ell}=\bm{\beta}^*_{M:\ell}) - w_{k_\ell}(\z_{\ell:1}=\z_{\ell:1}^{(0)}, \y_{M:\ell}=\bm{\beta}_{M:\ell})| \leq \delta/3, \\ \text{ 
 } \forall \bm{\beta}_{M:\ell} \in B(\bm{\beta}^*_{M:\ell}, \epsilon_{k_\ell}).
    \end{multline}
    Choosing a suitable $0 < \epsilon'_\ell \leq \min_{k_\ell \in K_0^{(\ell)}} \epsilon_{k_\ell}$ (note that $\min_{k_\ell \in K_0^{(\ell)}} \epsilon_{k_\ell} > 0$ as $|K_0^{(\ell)}| < +\infty$) and constructing an $\ell_2$-norm-based $\epsilon$-ball $B(\bm{\beta}^*_{M:\ell}, \epsilon'_\ell) \subset \pi_{\{M,\dots, \ell\}}(\mathcal{Y}')$, we have for all $k_\ell \in K_{1}^{(\ell)}(\bm{\beta}^*_{M:\ell}), k_\ell' \in K_0^{(\ell)} \backslash K_{1}^{(\ell)}(\bm{\beta}^*_{M:\ell})$, $k_\ell' \notin K_1^{(\ell)}(\bm{\beta}_{M:\ell})$ for all $\bm{\beta}_{M:\ell} \in B(\bm{\beta}^*_{M:\ell}, \epsilon'_\ell)$ due to
    \begin{multline}
    |w_{k_\ell}(\z_{\ell:1}=\z_{\ell:1}^{(0)}, \y_{M:\ell}=\bm{\beta}_{M:\ell}) - w_{k_\ell'}(\z_{\ell:1}=\z_{\ell:1}^{(0)}, \y_{M:\ell}=\bm{\beta}_{M:\ell})| \geq \delta/3, \\ \text{ 
 } \forall \bm{\beta}_{M:\ell} \in B(\bm{\beta}^*_{M:\ell}, \epsilon'_\ell).
    \end{multline}
    So this means for the split $\{ K_l^{(\ell)}(\bm{\beta}_{M:\ell}) \}$ of any $\bm{\beta} \in B(\bm{\beta}^*_{M:\ell}, \epsilon)$, we have $ K_1^{(\ell)}(\bm{\beta}_{M:\ell}) \\ \subset K_1^{(\ell)}(\bm{\beta}^*_{M:\ell})$ 
    and therefore $ |K_1^{(\ell)}(\bm{\beta}_{M:\ell})| \leq |K_1^{(\ell)}(\bm{\beta}^*_{M:\ell})|$. Now by definition of $\bm{\beta}^*_{M:\ell} \in \arg\min_{\bm{\beta}_{M:\ell} \in \pi_{\{M,\dots,\ell\}}(\mathcal{Y})} C^{(\ell)}(\bm{\beta}_{M:\ell})$ and $|K_1^{(\ell)}(\bm{\beta}^*_{M:\ell})| = C^{(\ell)}(\bm{\beta}^*_{M:\ell})$, we have $K_1^{(\ell)}(\bm{\beta}_{M:\ell}) = K_1^{(\ell)}(\bm{\beta}^*_{M:\ell})$ for all $\bm{\beta}_{M:\ell} \in B(\bm{\beta}^*_{M:\ell}, \epsilon'_\ell)$. 
    \item[3.] One can show that $|K_1^{(\ell)}(\bm{\beta}^*_{M:\ell})| = 1$, otherwise by definition of the split (Eq.~(\ref{eq:j_index_split_def})) and the above point, there exists $k_\ell \neq k_\ell' \in K_1^{(\ell)}(\bm{\beta}^*_{M:\ell})$ such that $w_{k_\ell}(\z_{\ell}, \dots, \z_1,\y_{M:\ell}=\bm{\beta}_{M:\ell})=w_{k'_\ell}(\z_{\ell}, \dots, \z_1,\y_{M:\ell}=\bm{\beta}_{M:\ell})$ for all $(\z_\ell, \dots, \z_1) \in \pi_{\{\ell,\dots,1\}}(\mathcal{Z})$ and $\bm{\beta}_{M:\ell} \in B(\bm{\beta}^*_{M:\ell}, \epsilon'_\ell)$, again a contradiction to assumption (b4). 
 \end{itemize}   
    Given the above 3 points and assuming $k_\ell \in K_1^{(\ell)}(\bm{\beta}^*_{M:\ell})$ is the only index in the subset, then the fact proved in the above point that $K_1^{(\ell)}(\bm{\beta}_{M:\ell}) = K_1^{(\ell)}(\bm{\beta}^*_{M:\ell})$ means
\begin{multline}\label{eq:lipo_M_big_equation_linear_dependence}
\sum_{ (i, k_r, \dots, k_{\ell+1}, k_{\ell-1}, \dots, k_1): (i, k_r, \dots, k_1) \in S_0} \gamma_{i,k_r,\dots,k_1} u_{i}(\y = \bm{\beta}, \x) w_{k_r}(\z_r,\dots,\z_1, \y_M=\bm{\beta}_M)\cdot\\ \ldots \cdot w_{k_{\ell+1}}(\z_{\ell+1},\dots,\z_1, \y_{M:\ell+1}=\bm{\beta}_{M:\ell+1})w_{k_{\ell-1}}(\z_{\ell-1},\dots,\z_1, \y_{M:\ell-1}=\bm{\beta}_{M:\ell-1})\cdot  \\ \ldots \cdot w_{k_1}(\z_1, \y_{M:1}=\bm{\beta}_{M:1}) = 0 ,\quad \forall (\x, \z_r, \dots, \z_1) \in \mathcal{X} \times \mathcal{Z}, \forall \bm{\beta}_{M:\ell} \in B(\bm{\beta}^*_{M:\ell}, \epsilon'_\ell);
\end{multline}
where the term $w_{k_{\ell}}(\z_{\ell},\dots,\z_1, \y_{M:\ell}=\bm{\beta}_{M:\ell})$ can be factored out from the above equation thanks to (b1). Given that earlier we established linear dependence of $\{v_j(\z, \y = \bm{\beta}) | j \in J_0 \}$ for each $\bm{\beta} \in \mathcal{Y}'$, We define $\mathcal{Y}_\ell:=(B(\bm{\beta}^*_{M:\ell}, \epsilon'_\ell)\times \pi_{\{\ell-1,\dots, 1\}}(\mathcal{Y}'))\setminus\mathcal{Y}_0$. And from Eq.~\eqref{eq:lipo_M_big_equation_linear_dependence}, for each $\bm{\beta}_{M:1}\in\mathcal{Y}_\ell$ the family 
\begin{multline*}
    \Big\{w_{k_r}(\z_r,\dots,\z_1, \y_M=\bm{\beta}_M) \dots w_{k_{\ell+1}}(\z_{\ell+1},\dots,\z_1, \y_{M:\ell+1}=\bm{\beta}_{M:\ell+1})\cdot \\ w_{k_{\ell-1}}(\z_{\ell-1},\dots,\z_1, \y_{M:\ell-1}=\bm{\beta}_{M:\ell-1}) \dots w_{k_1}(\z_1, \y_{M:1}=\bm{\beta}_{M:1})|, \\(k_r,\dots, k_{\ell+1}, k_{\ell-1},\dots,k_1)\in \{k_r,\dots, k_{\ell+1}, k_{\ell-1},\dots,k_1)| \exists(i,k_r,\dots, k_1)\in S_0\}\Big\}
\end{multline*}
contains linearly dependent functions. The above family is a product of $r-1$ functions, all of which belong to the same family $\mathcal{W}_K$ which satisfies (b1), (b4), (b5) and (b6). Therefore, we can use Lemma \ref{lemma:extension_b5} again. Then, we know for each $\bm{\beta}_{M:1}\in\mathcal{Y}_\ell$, there is at least one $\ell'\leq r$ with $\ell'\neq \ell$ such that $\{w_{k_{\ell'}}(\z_{\ell'}, \dots, \z_1,\y_{M:\ell'}=\bm{\beta}_{M:\ell'})| k_\ell'\in K^{(\ell')}_0\}$, with $K^{(\ell')}_0=\{k_{\ell'}\in K|\exists(i,k_r,\dots,k_1)\in S_0\}$, contains linear dependent functions under finite mixtures. 

As before, we create splits as in Eq.~\eqref{eq:j_index_split_def}, choose
$$ \bm{\beta}^*_{M:\ell'} \in \arg\min_{\bm{\beta}_{M:\ell'} \in \pi_{\{M,\dots,\ell'\}}(\mathcal{Y})} C^{(\ell')}(\bm{\beta}_{M:\ell'}),
$$
and construct $B(\bm{\beta}^*_{M:\ell'},\epsilon'_{\ell'})\subset\pi_{\{M,\dots,\ell'\}}(\mathcal{Y}_\ell)$ with $\epsilon'_\ell >\epsilon'_{\ell'}>0$, such that by (b4), (b6) and Eq.~\eqref{eq:j_index_split_def}, the corresponding split satisfies $|K_1^{(\ell')}(\bm{\beta}_{M:\ell'})| =1$ for all $\bm{\beta}_{M:\ell'}\in B(\bm{\beta}^*_{M:\ell'},\epsilon'_{\ell'})$ (we again assume $|K_1^{(\ell')}(\bm{\beta}^*_{M:\ell'})| = C^{(\ell')}(\bm{\beta}^*_{M:\ell'})$ w.l.o.g.). Therefore we can factor out $\{w_{k_{\ell'}}(\z_{\ell'}, \dots, \z_1,\y_{M:\ell'}=\bm{\beta}_{M:\ell'})| k_{\ell'}\in K^{(\ell')}_0\}$ from Eq.~\eqref{eq:lipo_M_big_equation_linear_dependence} for all $\bm{\beta}_{M:\ell'}\in B(\bm{\beta}^*_{M:\ell'},\epsilon'_{\ell'})$. 

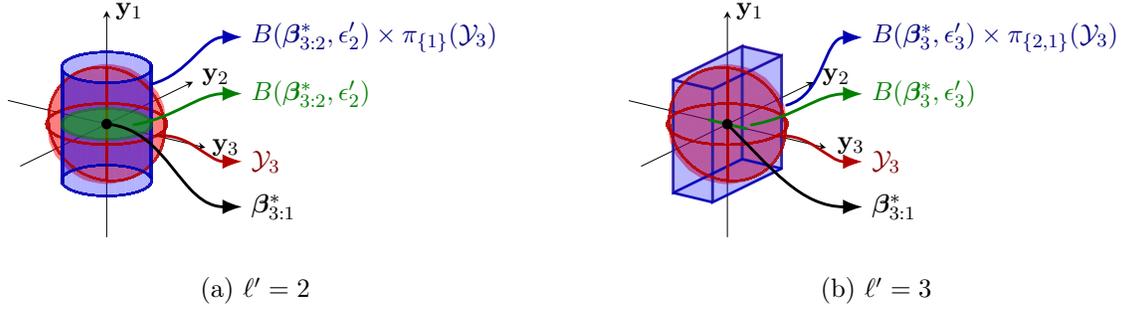
\begin{figure}
  \centering
  \begin{subfigure}{0.45\textwidth}
    \centering
    \resizebox{\linewidth}{!}{
    \begin{tikzpicture}
\begin{axis}[
  view={35}{20},
  clip=false,
  axis lines=center,
  xlabel={$\y_3$}, ylabel={$\y_2$}, zlabel={$\y_1$},
  ylabel style={
  yshift=2pt, xshift=0pt,   
  anchor=west
},
  xlabel style={
  yshift=0pt, xshift=-1pt,   
  anchor=west
},
  zlabel style={  
  anchor=west
},
  ticks=none,
  unit vector ratio=1 1 1,
  xmin=-2, xmax=2,
  ymin=-2.5, ymax=2.5,
  zmin=-2, zmax=2
]

\def\R{1}
\def\RC{0.75}

\addplot3[
  surf,
  domain=0:180, y domain=0:360,
  samples=13, samples y=25,      
  shader=flat,
  draw=none,
  faceted color=none,
  fill=red, opacity=0.3, draw opacity=0.0,
]
({\R*sin(x)*cos(y)}, {\R*sin(x)*sin(y)}, {\R*cos(x)});

\addplot3[
  surf,
  domain=0:360, y domain={-\R}:{\R},
  samples=25, samples y=5,      
  shader=flat,
  draw=none,
  faceted color=none,
  fill=blue, opacity=0.3, draw opacity=0.0,
]
({\RC*cos(x)}, {\RC*sin(x)}, {y});

\addplot3[domain=0:360, samples=60, thick, red!70!black, opacity=0.9]
({\R*cos(x)}, {\R*sin(x)}, {0});

\foreach \phi in {20,200} {
  \addplot3[domain=0:180, samples=30, thick, red!70!black, opacity=0.85]
  ({\R*sin(x)*cos(\phi)}, {\R*sin(x)*sin(\phi)}, {\R*cos(x)});
}

\foreach \alpha in {45,225} {
  \addplot3[domain=-\R:\R, samples=30, very thick, blue!70!black, opacity=0.9]
  ({\RC*cos(\alpha)}, {\RC*sin(\alpha)}, {x});
}

\addplot3[domain=0:360, samples=30, thick, blue!70!black, opacity=0.9]
({\RC*cos(x)}, {\RC*sin(x)}, { \R});
\addplot3[domain=0:360, samples=30, thick, blue!70!black, opacity=0.9]
({\RC*cos(x)}, {\RC*sin(x)}, { -\R});

\addplot3[domain=0:360, samples=30, thick, green!50!black, opacity=0.6, fill=green!70!black]
({\RC*cos(x)}, {\RC*sin(x)}, {0});
\addplot3[mark=*, mark size=2pt, color=black] coordinates {(0,0,0)};

\def\zc{0}   
\def\rc{0.6*\RC}
\coordinate (CircleFrom) at (axis cs:\xC,\yC,\zc);

\coordinate (SphereFrom) at (axis cs:\R,0,0);
\pgfmathsetmacro{\xC}{\rc*cos(30)}
\pgfmathsetmacro{\yC}{\rc*sin(30)}
\pgfmathsetmacro{\cx}{\RC*cos(25)}
\pgfmathsetmacro{\cy}{\RC*sin(25)}
\tikzset{annarrow/.style={-Latex, very thick, line cap=round}}
\coordinate (CylFrom) at (axis cs:\cx,\cy,\RC);

\coordinate (OriginFrom) at (axis cs:0,0,0);

\coordinate (Right1) at (rel axis cs:.75,1, 0.5);
\coordinate (Right2) at (rel axis cs:.75,1,0.20);
\coordinate (Right3) at (rel axis cs:.75,1,0.75);
\coordinate (Right4) at (rel axis cs:.75,1,0.0);

\draw[annarrow, green!50!black] (CircleFrom) to[out=10,in=180] (Right1);
\draw[annarrow, red!70!black]   (SphereFrom) to[out=5, in=180] (Right2);
\draw[annarrow, blue!70!black]  (CylFrom)    to[out=15,in=180] (Right3);
\draw[annarrow, black]          (OriginFrom) to[out=-5,in=180] (Right4);

\node[anchor=west, text=green!50!black] at (Right1) {$B(\bm{\beta}^*_{3:2}, \epsilon'_{2})$};
\node[anchor=west, text=red!50!black]   at (Right2)  {$\mathcal{Y}_3$};
\node[anchor=west, text=blue!50!black]  at (Right3)  {$B(\bm{\beta}^*_{3:2}, \epsilon'_{2})\times \pi_{\{1\}}(\mathcal{Y}_3)$};
\node[anchor=west]                      at (Right4)  {$\bm{\beta}_{3:1}^*$};

\end{axis}
\end{tikzpicture}}
    \vspace{-.5em}
    \caption{$\ell'=2$}
  \end{subfigure}
  \hfill
  \begin{subfigure}{0.45\textwidth}
    \centering
    \resizebox{\linewidth}{!}{
    \begin{tikzpicture}
\begin{axis}[
  view={35}{20},
  clip=false,
  axis lines=center,
  xlabel={$\y_3$}, ylabel={$\y_2$}, zlabel={$\y_1$},
  ylabel style={
  yshift=2pt, xshift=0pt,   
  anchor=west
},
  xlabel style={
  yshift=1pt, xshift=1pt,   
  anchor=west
},
  zlabel style={  
  anchor=west
},
  ticks=none,
  unit vector ratio=1 1 1,
  xmin=-2, xmax=2,
  ymin=-2.5, ymax=2.5,
  zmin=-2, zmax=2
]

\def\R{1}
\def\RC{0.4}
\addplot3[
  surf,
  domain=0:180, y domain=0:360,
  samples=13, samples y=25,      
  shader=flat,
  draw=none,
  faceted color=none,
  fill=red, opacity=0.3, draw opacity=0.0,
]
({\R*sin(x)*cos(y)}, {\R*sin(x)*sin(y)}, {\R*cos(x)});


\addplot3[
  surf, domain=-1:1, y domain=-1:1,
  samples=2, samples y=2,
  shader=flat, draw=none, faceted color=none,
  fill=blue, fill opacity=0.15
] ({\RC},{x},{y});

\addplot3[
  surf, domain=-1:1, y domain=-1:1,
  samples=2, samples y=2,
  shader=flat, draw=none, faceted color=none,
  fill=blue, fill opacity=0.15
] ({-\RC},{x},{y});

\addplot3[
  surf, domain=-\RC:\RC, y domain=-1:1,
  samples=2, samples y=2,
  shader=flat, draw=none, faceted color=none,
  fill=blue, fill opacity=0.15
] ({x},{1},{y});

\addplot3[
  surf, domain=-\RC:\RC, y domain=-1:1,
  samples=2, samples y=2,
  shader=flat, draw=none, faceted color=none,
  fill=blue, fill opacity=0.15
] ({x},{-1},{y});

\addplot3[
  surf, domain=-\RC:\RC, y domain=-1:1,
  samples=2, samples y=2,
  shader=flat, draw=none, faceted color=none,
  fill=blue, fill opacity=0.15
] ({x},{y},{1});

\addplot3[
  surf, domain=-\RC:\RC, y domain=-1:1,
  samples=2, samples y=2,
  shader=flat, draw=none, faceted color=none,
  fill=blue, fill opacity=0.15
] ({x},{y},{-1});

\draw[very thick, blue!70!black, opacity=0.8]
  (axis cs:-\RC,-1,-1) -- (axis cs:-\RC,-1, 1) --
  (axis cs:-\RC, 1, 1) -- (axis cs:-\RC, 1,-1) -- cycle;
\draw[very thick, blue!70!black, opacity=0.8]
  (axis cs: \RC,-1,-1) -- (axis cs: \RC,-1, 1) --
  (axis cs: \RC, 1, 1) -- (axis cs: \RC, 1,-1) -- cycle;
\draw[very thick, blue!70!black, opacity=0.8]
  (axis cs:-\RC,-1,-1) -- (axis cs: \RC,-1,-1)
  (axis cs:-\RC, 1,-1) -- (axis cs: \RC, 1,-1)
  (axis cs:-\RC,-1, 1) -- (axis cs: \RC,-1, 1)
  (axis cs:-\RC, 1, 1) -- (axis cs: \RC, 1, 1);

\addplot3[very thick, green!50!black] coordinates {(-\RC,0,0) (\RC,0,0)};
\addplot3[domain=0:360, samples=60, thick, red!70!black, opacity=0.9]
({\R*cos(x)}, {\R*sin(x)}, {0});

\foreach \phi in {20,200} {
  \addplot3[domain=0:180, samples=30, thick, red!70!black, opacity=0.85]
  ({\R*sin(x)*cos(\phi)}, {\R*sin(x)*sin(\phi)}, {\R*cos(x)});
}

\addplot3[mark=*, mark size=2pt, color=black] coordinates {(0,0,0)};

\def\zc{0}   
\def\rc{\RC}
\coordinate (CircleFrom) at (axis cs:\xC,\yC,\zc);

\coordinate (SphereFrom) at (axis cs:\R,0,0);
\pgfmathsetmacro{\xC}{\rc*cos(30)}
\pgfmathsetmacro{\yC}{\rc*sin(30)}
\pgfmathsetmacro{\cx}{\R*cos(25)}
\pgfmathsetmacro{\cy}{\R*sin(25)}
\tikzset{annarrow/.style={-Latex, very thick, line cap=round}}
\coordinate (CylFrom) at (axis cs:\cx,\cy,\rc);

\coordinate (OriginFrom) at (axis cs:0,0,0);

\coordinate (Right1) at (rel axis cs:.75,1, 0.5);
\coordinate (Right2) at (rel axis cs:.75,1,0.20);
\coordinate (Right3) at (rel axis cs:.75,1,0.75);
\coordinate (Right4) at (rel axis cs:.75,1,0.0);

\draw[annarrow, green!50!black] (CircleFrom) to[out=10,in=180] (Right1);
\draw[annarrow, red!70!black]   (SphereFrom) to[out=5, in=180] (Right2);
\draw[annarrow, blue!70!black]  (CylFrom)    to[out=15,in=180] (Right3);
\draw[annarrow, black]          (OriginFrom) to[out=-45,in=180] (Right4);

\node[anchor=west, text=green!50!black] at (Right1) {$B(\bm{\beta}^*_{3}, \epsilon'_{3})$};
\node[anchor=west, text=red!50!black]   at (Right2)  {$\mathcal{Y}_3$};
\node[anchor=west, text=blue!50!black]  at (Right3)  {$B(\bm{\beta}^*_{3}, \epsilon'_{3})\times \pi_{\{2,1\}}(\mathcal{Y}_3)$};
\node[anchor=west]                      at (Right4)  {$\bm{\beta}_{3:1}^*$};

\end{axis}
\end{tikzpicture}}
    \vspace{-.5em}
    \caption{$\ell'=3$}
  \end{subfigure}
  \vspace{-.5em}
  \caption{Illustrations of nonzero intersection of subsets to construct $\mathcal{Y}_{\ell'}$ with $M=3,\ \ell=1$.}
  \label{fig:proof_sketch_sphere_cube_cylinder}
\end{figure}

Now we construct a new nonzero set of admissible $\bm{\beta}$-values such that we can factor out the $\ell'$-th factor of the product family $\{v_j(\cdot)\}$. We also need to ensure that the new set $\mathcal{Y}_{\ell'}$ is a subset of $\mathcal{Y}_\ell$, so that all previously factored terms remain valid. Define
\begin{equation}\label{eq:lipo_M_admissible_beta_values}
    \mathcal{Y}_{\ell'}:=\mathcal{Y}_\ell\cap\Big(B(\bm{\beta}^*_{M:\ell'}, \epsilon'_{\ell'})\times \pi_{\{\ell'-1,\dots, 1\}}(\mathcal{Y}_\ell)\Big),
\end{equation}
where, by construction, $\mathcal{Y}_{\ell'}\subseteq\mathcal{Y}_\ell$. The set also has nonzero measure, since $B(\bm{\beta}^*_{M:\ell'}, \epsilon'_{\ell'})\subset \pi_{\{M,\dots,\ell'\}}(\mathcal{Y}_\ell)$, and the product with $\pi_{\{\ell'-1,\dots,1\}}(\mathcal{Y}_\ell)$ results in a nonzero measure cylinder where at least a nonzero measure subset is in $\mathcal{Y}_{\ell}$, because the cylinder's base lies in $\pi_{\{M,\dots,\ell'\}}(\mathcal{Y}_\ell)$. 
Figure~\ref{fig:proof_sketch_sphere_cube_cylinder} illustrates this for $M=3$ and $\ell=3$, where the red sphere depicts $\mathcal{Y}_3$. We observe that the product between $B(\bm{\beta}^*_{M:\ell'}$ and the projection $\pi_{\{\ell'-1,\dots,1\}}(\mathcal{Y}_\ell)$ results in a cylinder (or a cube if $\ell'=1$), and their intersection with $\mathcal{Y}_\ell$ is nonzero measure. In this example, $\bm{\beta}_{M:\ell}$ across $\ell\in[M]$ are aligned: i.e., for $\ell\neq\ell'$ and $\bm{\beta}^{*(\ell)}_{M:\ell}$, $\bm{\beta}^{*(\ell')}_{M:\ell'}$, we have $\bm{\beta}^{*(\ell)}_{M:\ell'} = \bm{\beta}^{*(\ell')}_{M:\ell'}$ (we prove this alignment later). However, the nonzero measure intersection does not rely on alignment: it already follows from $B(\bm{\beta}^*_{M:\ell'}, \epsilon'_{\ell'})\subset \pi_{\{M,\dots,\ell'\}}(\mathcal{Y}_\ell)$.

Therefore, for each $\bm{\beta} \in \mathcal{Y}_{\ell'}$, the family is still linearly dependent under finite mixtures $\{v_j(\z, \y = \bm{\beta}) | j \in J_0 \}$, where two terms in the product $v_j$ have only one element ($w_{k_\ell}$ and $w_{k_{\ell'}}$). By repeating this procedure over all remaining indices, we obtain a nested sequence: $\mathcal{Y}_{final}\subseteq\dots\subseteq\mathcal{Y}_{\ell'}\subseteq\mathcal{Y}_\ell$ where for each $\bm{\beta}\in\mathcal{Y}_{final}$, every factor in $\{v_j(\z, \y = \bm{\beta}) | j \in J_0 \}$ has been reduced to a single element. Then, Eq.~\eqref{eq:lipo_M_big_equation_linear_dependence} reduces to
\begin{equation}
\sum_{ i: (i, k_r, \dots, k_1) \in S_0} \gamma_{i,k_r,\dots,k_1} u_{i}(\y = \bm{\beta}, \x) = 0, \quad \forall (\x, \y) \in \mathcal{X} \times \mathcal{Y}_{final}.
\end{equation}
Because $|K_1^{(\ell)}(\bm{\beta}_{M:\ell})| \leq |K_1^{(\ell)}(\bm{\beta}^*_{M:\ell})|$ for all $\bm{\beta}_{M:\ell} \in B(\bm{\beta}^*_{M:\ell}, \epsilon)$, we may choose the $\bm{\beta}^*_{M:\ell}$ so that the overlapping coordinates agree. That is, for any $\ell\neq \ell'$, with $1 \leq \ell < \ell'\leq r$ w.l.o.g., we can obtain two $\bm{\beta}^{*(\ell)}_{M:\ell}$, $\bm{\beta}^{*(\ell')}_{M:\ell'}$, such that $\bm{\beta}^{*(\ell)}_{M:\ell'} = \bm{\beta}^{*(\ell')}_{M:\ell'}$. Given alignment of $\bm{\beta}^*_{M:\ell}$ and since $\epsilon'_\ell>0$ is arbitrarily small, the final nested domain can be taken to be $\mathcal{Y}_{final}=B(\bm{\beta}^*,\epsilon'_{final})$,
with $\min_{\ell\in[r]}\epsilon'_{\ell} > \epsilon'_{final} > 0$ and $\bm{\beta}^* := \bm{\beta}^*_{M:1}$. This results in a nonzero measure set in which the family $\{u_i(\cdot)\}$ is linearly dependent, contradicting (a3). Therefore, a contradiction is reached.
\end{proof}

\subsubsection{Linear independence of nonparametric trajectory families}\label{app:lin_indep_joint}

In LIPO-M, we state results for generic per-variable dimension $d$, so that $\y\in\R^{dM}$ and $\z\in\R^{dr}$ with $1\leq r \leq M$. In the nonparametric trajectory setting we set $d=m$ because $\z_t\in\R^m$. Below we present linear independence under finite mixtures for nonparametric trajectory families.
\begin{theorem}
\label{thm:linear_independence_joint_non_parametric}
Define the following trajectory family for a trajectory $(\z_1,\dots,\z_T)\in\R^{Tm}$:
\begin{multline}
\jointfamily = \Bigg\{ p_{a,b_{M+1:T}}(\z_{1:T}) = p_{a}(\z_{1:M})\prod_{t=M+1}^T p_{b_t}(\z_t | \z_{t-1},\dots, \z_{t-M}), \\ \quad p_a \in \Pi_\A^M,\quad p_{b_t} \in \mathcal{P}_\B^M,\quad t = M+1, ..., T \Bigg\},
\end{multline}
and assume $\Pi_\A^M$ and $\mathcal{P}_\B^M$ satisfy the following assumptions,
\begin{itemize}
    \item[(c1)] $\Pi^M_\A$ satisfies (a1), (a3), and
    \item[(c2)] $\mathcal{P}_\B^M$ satisfies (b4-b6).
\end{itemize}
Then the following statement holds: For any $T > M$ such that $r=M$ if $(T \mod M) = 0$, or $r=(T \mod M)$, otherwise; and any subset $\mathcal{Z}_r\subset (\times^r\mathbb{R}^{m})$, the trajectory family contains linearly independent distributions under finite mixtures for $(\z_{1:T-r}, \z_{T-r+1:T}) \in \mathbb{R}^{(T-r)m} \times \mathcal{Z}_r$.
\end{theorem}
\begin{proof}
We first prove the statement for any $T>M$ such that $T \mod M = 0$ by induction as follows.

\underline{$T = 2M$}: The result can be proved using Lemma \ref{lemma:absolutely_magical_lemma} by setting in the proof,
\begin{equation*}
u_i(\y = \z_{1:M}, \x = \z_0) = \pi_{a}(\z_{1:M}), i = a,
\end{equation*}
and
\begin{equation*}
    v_j(\z = \z_{M+1:2M}, \y = \z_{1:M}) = \prod_{\ell=1}^M p_{b_{M+\ell}}(\z_{M+\ell} | \z_{\ell:M-1+\ell}), j=b_{M+1:2M},
\end{equation*}
where $j$ is a tuple index, and $\z_{\ell:M-1+\ell}:=(\z_\ell, \z_{\ell+1},\dots,\z_{M-1+\ell})$ is a sliding window of length $M$.
We observe that the Lemma holds using assumptions (c1-c2) directly as $v_j(\z = \z_{M+1:2M}, \y = \z_{1:M})$ is a product of $M$ functions that satisfies (b) in LIPO-M (Lemma \ref{lemma:absolutely_magical_lemma}).

\underline{$T = rM, r\in\mathbb{Z}^{+}, r> 2$}: Assume the statement holds for the trajectory family when $T = \tau - M$.
Note that we can write $p_{a,b_{M+1:\tau}}(\z_{1:\tau})$ as
\begin{equation}
p_{a,b_{M+1:\tau}}(\z_{1:\tau}) = p_{a,b_{M+1:\tau-M}}(\z_{1:\tau-M}) \prod_{\ell=1}^M p_{b_{\tau-M+\ell}}(\z_{\tau-M+\ell} | \z_{\tau-2M+\ell:\tau-M-1+\ell}).
\end{equation}
Then the statement for $T = \tau$ can be proved using LIPO-M (Lemma \ref{lemma:absolutely_magical_lemma}) by setting
\begin{equation*}
u_i(\y = \z_{\tau - 2M+1 : \tau - M}, \x = \z_{1:\tau-2M}) = p_{a,b_{M+1:\tau-M}}(\z_{1:\tau-M}), i = a,b_{M+1:\tau-M},
\end{equation*}
and
\begin{multline}\label{eq:v_family_proof_multiple}
    v_j(\z = \z_{\tau-M+1:\tau}, \y = \z_{\tau - 2M +1 : \tau - M}) = \\ \prod_{\ell=1}^M p_{b_{\tau-M+\ell}}(\z_{\tau-M+\ell} | \z_{\tau-2M+\ell:\tau-M-1+\ell}),  j=b_{\tau-M+1:\tau}.
\end{multline}
Note now both $i$ and $j$ denote tuple indices. The family spanned with $p_{a,b_{M+1:\tau-M}}(\z_{1:\tau-M}), i = a,b_{M+1:\tau-M}$ satisfies (a1) from (c1), and from the induction hypothesis it is linearly independent under finite mixtures, which is exactly (a3). The family in Eq.~(\ref{eq:v_family_proof_multiple}) is a product of $M$ functions where (b1) and (b4-b6) are satisfied from (c2), which imply (b) in Lemma \ref{lemma:absolutely_magical_lemma}.

Given the above result we proceed to prove the case where $T \mod M \neq 0$.

\underline{$T \mod M \neq 0$}: Let $r$ be the remainder of $T/M$: $T\mod M = r$. From the previous result we know the statement holds for $T - r$, i.e., when $T$ is a multiple of $M$. We can write $p_{a,b_{M+1:T}}(\z_{1:T})$ as follows
\begin{equation}
p_{a,b_{M+1:T}}(\z_{1:T}) = p_{a,b_{M+1:T-r}}(\z_{1:T-r}) \prod_{\ell=1}^r p_{b_{T-r+\ell}}(\z_{T-r+\ell} | \z_{T-r+\ell-M:T-r+\ell-1}).
\end{equation}
We can again prove the statement with Lemma \ref{lemma:absolutely_magical_lemma} by setting 
\begin{equation*}
u_i(\y = \z_{T - r-M+1 : T - r}, \x = \z_{1:T-r-M}) = p_{a,b_{M+1:T-r}}(\z_{1:T-r}), i = a,b_{M+1:T-r},
\end{equation*}
and
\begin{multline}\label{eq:v_family_proof_no_multiple}
    v_j(\z = \z_{T-r+1:T}, \y = \z_{T - r-M +1 : T - r}) = \\ \prod_{\ell=1}^r p_{b_{T-r+\ell}}(\z_{T-r+\ell} | \z_{T-r+\ell-M:T-r+\ell-1}), j=b_{T-r+1:T}.
\end{multline}
The family spanned with $p_{a,b_{M+1:T-r}}(\z_{1:T-r}), i = a,b_{M+1:T-r}$ satisfies (a1) from (c1). (a3) is satisfied as follows:
\begin{itemize}
    \item For $M < T<2M$, we have $T-r=M$ and (a3) holds directly from (c1).
    \item For $T > 2M$, we have $(T-r) \mod M = 0$ and (a3) holds from the induction hypothesis at $T-r$.    
\end{itemize}
The family in Eq.~(\ref{eq:v_family_proof_no_multiple}) is a product of $r$ function families which satisfy (b1) and (b4-b6) from (c2). This implies (b) in Lemma \ref{lemma:absolutely_magical_lemma}.

We showed that the conditions in LIPO-M are satisfied for all possible cases of $T$, and therefore by LIPO-M, the resulting trajectory family is linearly independent under finite mixtures.
\end{proof}

\subsection{Linear independence under finite mixtures with assumptions (m1) and (m2)}\label{app:parametric_assumptions}

Recall assumptions (m1) and (m2):
\begin{enumerate}
    \item[(m1)] Unique indexing for $\mathcal{G}_{\B}^M$ and $\mathcal{I}_{\A}^M$: Eqs. (\ref{eq:unique_indexing_conditional_gaussian}) and (\ref{eq:unique_indexing_marginal_gaussian}) hold;
    \item[(m2)] The mean and covariance in $\mathcal{G}_{\B}^M$, $\bm{m}(\cdot,b): \R^{mM} \rightarrow \R^m$ and $\bm{\Sigma}(\cdot,b): \R^{mM} \rightarrow \R^{m\times m}$, are analytic functions, for any $b\in \B$.
\end{enumerate}

Below we explore whether these parametric conditions ensure linear independence under finite mixtures of the trajectory family (and thus, identifiability). The unique indexing assumption of the Gaussian transition family (Eq.~\eqref{eq:unique_indexing_conditional_gaussian}) implies linear independence as shown below.
\begin{proposition}
\label{prop:gaussian_linear_independence}
Functions in $\mathcal{G}_\B$ are linearly independent under finite mixtures on variables $(\z_t, \dots, \z_{t-M})$ if the unique indexing assumption (Eq.~(\ref{eq:unique_indexing_conditional_gaussian})) holds.
\end{proposition}
\begin{proof}
Assume the statement is false and w.l.o.g. $M=1$. Then there exists $\B_0 \subset \B$ with $|\B_0|<+\infty$ and a set of nonzero values $\{ \gamma_b | b \in \B_0 \}$, such that 
$$\sum_{b \in \B_0} \gamma_b \mathcal{N}(\z_t; \bm{m}(\z_{t-1}, b), \bm{\Sigma}(\z_{t-1}, b)) = 0, \quad \forall \z_t, \z_{t-1} \in \mathbb{R}^m.$$
In particular, this equality holds for any $\z_{t-1} \in \mathbb{R}^m$, meaning that a weighted sum of Gaussian distributions (defined on $\z_t$) equals to zero. Note that \citet{yakowitz1968identifiability} proved that multivariate Gaussian distributions with different means and/or covariances are linearly independent. Therefore the equality above implies for any $\z_{t-1}$
$$\bm{m}(\z_{t-1}, b) = \bm{m}(\z_{t-1}, b') \quad \text{and} \quad \bm{\Sigma}(\z_{t-1}, b) = \bm{\Sigma}(\z_{t-1}, b') \quad \forall b, b' \in \B_0, b \neq b',$$
a contradiction to the unique indexing assumption.
\end{proof}

The transition and initial Gaussian distribution families defined in Eqs. (\ref{eq:gaussian_transition_family}) and (\ref{eq:gaussian_initial_family}) respectively align with assumptions (a) and (b) in LIPO-M as follows. 

\begin{proposition}
\label{prop:conditional_gaussian_properties}
The conditional Gaussian distribution family $\mathcal{G}_{\B}$ (Eq.~(\ref{eq:gaussian_transition_family})), under the unique indexing assumption (Eq.~(\ref{eq:unique_indexing_conditional_gaussian})), satisfies assumptions (b1), (b5), and (b6) in Lemmas \ref{lemma:extension_b4}, \ref{lemma:extension_b5}, and \ref{lemma:absolutely_magical_lemma}, if we define $\mathcal{V}_J := \mathcal{G}_{\B}, \z: = \z_t$ and $(\y_M,\dots,\y_1) := (\z_{t-1},\dots,\z_{t-M})$.
\end{proposition}
\begin{proposition}
\label{prop:marginal_gaussian_properties}
The initial Gaussian distribution family $\mathcal{I}_{\A}$ (Eq.~(\ref{eq:gaussian_initial_family})), under the unique indexing assumption (Eq.~(\ref{eq:unique_indexing_marginal_gaussian})), satisfies assumptions (a1), (a3) in Lemma \ref{lemma:absolutely_magical_lemma}, if we define $\mathcal{U}_I := \mathcal{I}_{\A}, \y: = (\z_1, \dots, \z_M)$ and $\x = \emptyset$.
\end{proposition}
To see why (b5) holds under the Gaussian transition family $\mathcal{G}_\B$, we note from Prop. \ref{prop:gaussian_linear_independence} that linear dependence occurs only if the unique indexing assumption does not hold. Therefore, we can fix a subset of the $(\z_{t-1}, \dots, \z_{t-M})$ variables, such that the resulting mean and covariance functions violate the unique indexing assumption, which would imply linear dependence on the resulting function family. To verify (b4), we require the following zero-measure intersection of moments result.
\begin{proposition}\label{prop:intersection_mom}
Assume Gaussian family transitions $\mathcal{G}_\B$ under unique indexing defined by Eq. (\ref{eq:unique_indexing_conditional_gaussian}), with zero-measure intersection of moments:  For any $b,b'\in \B$ with $b \neq b'$,  the set 
\begin{multline*}
\mathcal{X}_{b,b'}:=\bigg\{ (\z_{t-1},\dots, \z_{t-M}) \in \mathbb{R}^{mM} \mid \bm{m}(\z_{t-1},\dots,\z_{t-M}, b) = \bm{m}(\z_{t-1},\dots,\z_{t-M}, b'),\\
 \bm{\Sigma}(\z_{t-1},\dots,\z_{t-M}, b) = \bm{\Sigma}(\z_{t-1},\dots,\z_{t-M}, b') \bigg\}
\end{multline*}
 has zero measure. Then (b4) holds if $V_J:=\mathcal{G}_\B$, $\z:=\z_t$, and $(\y_1,\dots,\y_M):=(\z_{t-1},\dots,\z_{t-M})$.
\end{proposition}
\begin{proof}
    We prove the statement by contradiction. Define $\mathcal{V}_J := \mathcal{G}_{\B}$, $\z:=\z_t$, and $(\y_1,\dots,\y_M):=(\z_{t-1},\dots,\z_{t-M})$ from (b4); and w.l.o.g. set $\mathcal{Y}:= \R^{mM}$. If (b4) does not hold, then there exist nonzero measure sets $\mathcal{Y}'\subset\mathcal{Y}$ and $\mathcal{Z}\subset \R^m$ such that $\mathcal{G}_{\B}$ contains linearly dependent functions. From Proposition \ref{prop:gaussian_linear_independence}, this means there exists an index set $\B_0\subset \B$ with $|\B_0|<+\infty$ such that the unique indexing assumption (Eq. (\ref{eq:unique_indexing_conditional_gaussian})) does not hold in $\mathcal{Y}'$. From the intersection of moments assumption, we know that this set is $\mathcal{Y}'=\bigcup_{b\neq b'\in \B_0} \mathcal{X}_{b,b'}$, which has zero measure. This contradicts the requirement for a nonzero measure set $\mathcal{Y}'\subset\mathcal{Y}$ in (b4).
\end{proof}

Finally, we present linear independence under finite mixtures of the trajectory family in the parametric case.
\begin{theorem}
\label{thm:linear_independence_nonlinear_gaussian}
Define the following trajectory family under the nonlinear Gaussian model
\begin{multline}\label{eq:joint_distrib_fam}
\Bigg\{ p_{a,b_{M+1:T}}(\z_{1:T}) = p_{a}(\z_{1:M})\prod_{t=M+1}^T p_{b_t}(\z_t | \z_{t-1},\dots, \z_{t-M}),\\  \quad p_{a} \in \mathcal{I}^M_{\A}, \quad p_{b_t} \in \mathcal{G}^M_{\B},\quad t = M+1, ..., T \Bigg\},
\end{multline}
with $\mathcal{G}^M_{\B}$, $\mathcal{I}^M_{\A}$ defined by Eqs. (\ref{eq:gaussian_transition_family}), (\ref{eq:gaussian_initial_family}) respectively. Assume:
\begin{itemize}
    \item[(d1)] Unique indexing for $\mathcal{G}_{\B}^M$ and $\mathcal{I}_{\A}^M$: Eqs.~(\ref{eq:unique_indexing_conditional_gaussian}), ~(\ref{eq:unique_indexing_marginal_gaussian}) hold;
    \item[(d2)] The functions in $\mathcal{G}_{\B}$ are continuous with respect to $(\z_{t-1},\dots,\z_M) \in \mathbb{R}^{mM}$;
    \item[(d3)] Zero-measure intersection of moments:  For any $b,b'\in \B$ with $b \neq b'$,  the set 
\begin{multline*}
\mathcal{X}_{b,b'}:=\bigg\{ (\z_{t-1},\dots, \z_{t-M}) \in \mathbb{R}^{mM} \mid  \bm{m}(\z_{t-1},\dots,\z_{t-M}, b) = \bm{m}(\z_{t-1},\dots,\z_{t-M}, b'),\\
 \bm{\Sigma}(\z_{t-1},\dots,\z_{t-M}, b) = \bm{\Sigma}(\z_{t-1},\dots,\z_{t-M}, b') \bigg\}
\end{multline*}
    has zero measure.
\end{itemize}
Then, the trajectory family contains linearly independent distributions under finite mixtures for $(\z_{1:T-r}, \z_{T-r+1}, \dots,\z_T) \in \mathbb{R}^{(T-r)m} \times (\times^r \mathbb{R}^m)$, where $r=M$ if $T\mod M =0$, or $r=T\mod M$ otherwise.
\end{theorem}
\begin{proof}
Note that (d2) and (b6) are equivalent. Assumptions (a1), (a3), (b1), and (b5) are satisfied due to Propositions \ref{prop:conditional_gaussian_properties}, \ref{prop:marginal_gaussian_properties}. Assumption (b4) holds due to assumption (d3) via Prop \ref{prop:intersection_mom}. Then, the statement holds by Theorem \ref{thm:linear_independence_joint_non_parametric}.
\end{proof}

\subsection{Conclusion}

Below we conclude the proof of Theorem \ref{thm:identifiability_main}.
\begin{proof}
    We first clarify the alignment from assumptions (m1-m2) to (d1-d3) in Theorem \ref{thm:linear_independence_nonlinear_gaussian}. (d1) and (m1) are equivalent. (d2) is satisfied from (m2) as analytic functions are $\mathcal{C}^{\infty}$, and therefore continuous. From \citet{mityagin2015zero}, we know the zero set of an analytic function is measure zero. Therefore, (m2) implies that the intersections between on means or covariances for any $b,b'\in\mathcal{B}$, with $b\neq b'$ are also measure zero, which implies (d3). Therefore, from Theorem \ref{thm:linear_independence_nonlinear_gaussian}, the trajectory family under the Gaussian transition family $\mathcal{G}_\B$ is linearly independent under finite mixtures. By Theorem \ref{thm:identifiability_msm}, linear independence under finite mixtures of the trajectory family implies identifiability of the MSM in terms of Def. \ref{def:identifiability}. 
\end{proof}

\section{Proof of SDS identifiability}\label{app:sds}

We extend the results from \citet{kivva2022identifiability} and \citet{balsells-rodas2024on} with our multi-lag MSM family as a prior distribution for $\z_{1:T}$. The strategy requires finding some open set where the transformations $F$ and $G$ from two identically distributed SDSs are invertible. Then, identifiability can be shown thanks to piecewise linearity of $F$ and $G$, and the multi-lag MSM being an analytic distribution and closed under affine transformations.

\subsection{Preliminaries}

We need to introduce some definitions and results that will be used in the proof. These have been previously defined in \citet{kivva2022identifiability}.

\begin{definition}
    Let $D_0\subseteq D\subseteq \R^n$ be open sets. Let $f_0:D_0\rightarrow\R$. We say that an analytic function $f:D\rightarrow\R$ is an analytic continuation of $f_0$ onto $D$ if $f(\x) = f_0(\x)$ for every $\x\in D_0$.
\end{definition}
\begin{definition}
    Let $\x_0\in\R^m$ and $\delta>0$. Let $p:B(\x_0,\delta)\rightarrow\R$. Define
    $$\text{Ext}(p):\R^m\rightarrow\R$$
    to be the unique analytic continuation of $p$ on the entire space $\R^m$ if such a continuation exists, and to be 0 otherwise.
\end{definition}
\begin{definition}
    Let $D_0\subset D$ and $p:D\rightarrow \R$ be a function. We define $p|_{D_0}:D\rightarrow \R$ to be a restriction of $p$ to $D_0$, namely a function that satisfies $p|_{D_0}(\x)=p(\x)$ for every $\x\in D_0$.
\end{definition}
\begin{definition}
    Let $f:\R^m\rightarrow\R^n$ be a piecewise affine function. We say that a point $\x\in f(\R^m)\subseteq\R^n$ is generic with respect to $f$ if the pre-image $f^{-1}(\{\x\})$ is finite and there exists $\delta>0$, such that $f:B(\z,\delta)\rightarrow\R^{n}$ is affine for every $\z\in f^{-1}(\{\x\})$.
\end{definition}

We also need to show that the points in the pre-image of a piecewise factored mapping $F$ can be computed using the multi-lag MSM prior under assumptions (m1-m2), which is then used to show that there exists a nonzero measure set where both $F$ and $G$ are invertible. We follow Lemma~D.6 from \citet{balsells-rodas2024on}, where we observe that including additional lags does not alter the statement. We also adapt Corollary~D.7 and Theorem~D.8 from \citet{balsells-rodas2024on} to multi-lag MSMs, where the proofs require minor clarifications.

\begin{lemma}[adapted from \citet{balsells-rodas2024on}]\label{lemma:extension_msm_preimage}
    Consider a random variable $\z_{1:T}$ distributed according to the multi-lag MSM family $\mathcal{M}^{T,M}_{NL}$ with assumptions (m1-m2). Let us consider $f:\mathbb{R}^m \rightarrow \mathbb{R}^m$, a piecewise affine mapping which generates the random variable $\x_{1:T} = F(\z_{1:T})$ as $\x_t = f(\z_t), 1\leq t \leq T$. Also, consider $\x^{(0)} \in \mathbb{R}^{m}$ a generic point with respect to $f$. Then, $\x^{(0)}_{1:T} = \{\x^{(0)},\dots,\x^{(0)}\}\in\mathbb{R}^{Tm}$ is also a generic point with respect to $F$ and the number of points in the pre-image $F^{-1}(\{\x^{(0)}_{1:T}\})$ can be computed as
    $$\left|F^{-1}\left(\left\{\x^{(0)}_{1:T}\right\}\right)\right| = \lim_{\delta\rightarrow 0} \int_{\x_{1:T}\in\mathbb{R}^{Tm}}\text{Ext}\left(p|_{B(\x^{(0)}_{1:T},\delta)}\right) d\x_{1:T}$$
\end{lemma}
\begin{proof}
    If $\x^{(0)}\in\mathbb{R}^{m}$ is a generic point with respect to $f$, $\x^{(0)}_{1:T}$ is also a generic point with respect to $F$ since the pre-image is $F^{-1}(\{\x^{(0)}_{1:T}\})$ now larger but still finite. In other words, $F^{-1}(\{\x^{(0)}_{1:T}\})$ is the Cartesian product $\times^T\mathcal{Z}$, where $\mathcal{Z}=\{\z_1,\z_2,\dots,\z_\ell\}$ are the points in the pre-image $f^{-1}(\{\x^{(0)}\})$.  Considering this, we have well defined affine mappings $G_{i_1,\dots,i_T}:B(\z_{i_1},\dots,\z_{i_T}, \epsilon)\rightarrow\mathbb{R}^m$, $i_t\in[\ell]$ for $1\leq t \leq T$, such that $G_{i_1,\dots,i_T}=F(\z_{1:T}), \forall\z_{1:T}\in B(\z_{i_1},\dots,\z_{i_T},\epsilon)$. This affine mapping $G_{i_1,\dots,i_T}$ is factored: $g_{i_t}(\z_t) = f(\z_t), \  \forall \z_t\in B(\z_i,\epsilon)$,
$$G_{i_1,\dots,i_T}=\begin{pmatrix}
A_{i_1} & \dots & \bm{0}\\
\vdots & \ddots & \vdots \\
\bm{0} & \dots & A_{i_T}
\end{pmatrix} \begin{pmatrix}
    \z_1 \\
    \vdots \\
    \z_T
\end{pmatrix} + \begin{pmatrix}
    b_{i_1} \\
    \vdots \\
    b_{i_T}
\end{pmatrix},$$
with invertible $A_{i_1},\dots,A_{i_T}$. Let $\delta_0 > 0$ such that
$$B(\x^{(0)}_{1:T},\delta_0) \subseteq \bigcap_{i_1, \dots, i_T}^\ell G_{i_1,\dots,i_T}\big(B(\z_{i_1},\dots,\z_{i_T}, \epsilon)\big).$$
We can compute the likelihood for every $\x_{1:T}\in B(\x^{(0)}_{1:T},\delta')$ with $0 < \delta' < \delta_0$ using Prop.~C.2 from \citet{balsells-rodas2024on} where the MSM is closed under factored affine transformations. We can assume $M=1$ w.l.o.g., since for lag $M>1$, the MSM is still closed under factored affine transformations,
\begin{multline*}
p|_{B(\x^{(0)}_{1:T},\delta)} =  \sum_{i_1,\dots,i_T}^\ell \sum_{j_1,\dots,j_T}^K p\left(s_{1:T} = \{j_1, \dots, j_T\}\right) \N\left(\x_1; A_{i_1}\bm{\mu}(j_1) + \Bb_{i_1}, A_{i_1}\bm{\Sigma}_{1}(j_1)A_{i_1}^T\right)\\
     \prod_{t=2}^T \N\left(\x_t; A_{i_t}\bm{m}\left(A_{i_{t-1}}^{-1}\left(\x_{t-1} - \Bb_{i_{t-1}}\right),j_t\right) + \Bb_{i_t}, A_{i_t}\bm{\Sigma}\Big( 
A_{i_{t-1}}^{-1}\left(\x_{t-1} - \Bb_{i_{t-1}} \right) ,j_t\Big)A_{i_t}^T\right),
\end{multline*}
where we use $\x'_{t-1}= A_{i_{t-1}}^{-1}\left(\x_{t-1} - \Bb_{i_{t-1}} \right)$, and for higher lags $M>1$ we will have $(\x'_{t-1}, \dots, \x'_{t-M})=\left( A_{i_{t-1}}^{-1}\left(\x_{t-1} - \Bb_{i_{t-1}} \right), \dots, A_{i_{t-M}}^{-1}\left(\x_{t-M} - \Bb_{i_{t-M}} \right)\right)$ on the arguments of $\bm{m}(\cdot)$ and $\bm{\Sigma}(\cdot)$ in the equation above. The previous density is an analytic function which is defined on an open neighbourhood of $\x^{(0)}_{1:T}$. Then from the identity theorem of analytic functions the resulting density defines the analytic extension of $p|_{B(\x^{(0)}_{1:T},\delta)}$ on $\mathbb{R}^{Tm}$. Then, we have
\begin{align*}
&\int_{\x_{1:T}\in\mathbb{R}^{Tm}} \text{Ext}\left(p|_{B(\x^{(0)}_{1:T},\delta)}\right) d\x_{1:T} \\
& \quad=\int_{\x_{1:T}\in\mathbb{R}^{Tm}} \sum_{i_1,\dots,i_T}^\ell \sum_{j_1,\dots,j_T}^K p\left(s_{1:T} = \{j_1, \dots, j_T\}\right) \N\left(\x_1; A_{i_1}\bm{\mu}(j_1) + \Bb_{i_1},  A_{i_1}\bm{\Sigma}_{1}(j_1)A_{i_1}^T\right)\\
    &\qquad \prod_{t=2}^T \N\Big(\x_t; A_{i_t}\bm{m}\left(\x'_{t-1},j_t\right) + \Bb_{i_t}, A_{i_t}\bm{\Sigma}\left(\x'_{t-1} ,j_t\right)A_{i_{t}}^T\Big)\, d\x_{1:T}\\
 &\quad=  \sum_{i_1,\dots,i_T}^\ell \int_{\x_{1:T}\in\mathbb{R}^{Tm}} \sum_{j_1,\dots,j_T}^K p\left(s_{1:T} = \{j_1, \dots, j_T\}\right) \N\left(\x_1; A_{i_1}\bm{\mu}(j_1) + \Bb_{i_1}, A_{i_1}\bm{\Sigma}_{1}(j_1)A_{i_1}^T\right)\\
    &\qquad \prod_{t=2}^T \N\Big(\x_t; A_{i_t}\bm{m}\left(\x'_{t-1},j_t\right) + \Bb_{i_t}, A_{i_t}\bm{\Sigma}\left(\x'_{t-1} ,j_t\right)A_{i_{t}}^T\Big)\, d\x_{1:T}\\
 &\quad= \sum_{i_1,\dots,i_T}^\ell 1 = \ell^T = \left|F^{-1}(\{\x^{(0)}_{1:T}\})\right|.
\end{align*}
\end{proof}

As mentioned, this leads to having an open set where both $F$ and $G$ mappings are invertible.

\begin{corollary}[adapted from \citet{balsells-rodas2024on}]\label{cor:invertibility}
    Let $F$, $G : \mathbb{R}^{Tm}\rightarrow\mathbb{R}^{Tn}$ be factored piecewise affine mappings, with $\x_t := f(\z_t)$ and  $\x_t' := g(\z_t')$,  for $1\leq t \leq T$. Assume $f$ and $g$ are weakly injective (Def. \ref{def:weakly_injective}). Let $\z_{1:T}$ and $\z_{1:T}'$ be distributed according to the identifiable multi-lag MSM family $\mathcal{M}^{T,M}_{NL}$ with assumptions (m1-m2). Assume $F(\z_{1:T})$ and $G(\z_{1:T}')$ are identically distributed. Assume that for $\x_0\in\R^{n}$ and $\delta>0$, $f$ is invertible on $B(\x_0,2\delta)\cap f(\R^m)$.

    Then, for $\x^{(0)}_{1:T}=\{\x_0,\dots,\x_0\}\in\R^{Tn}$ there exists $\x^{(1)}_{1:T}\in B(\x^{(0)}_{1:T},\delta)$ and $\delta_1 >0$ such that both $F$ and $G$ are invertible on $B(\x^{(1)}_{1:T},\delta_1)\cap F(\R^{Tm})$.
\end{corollary}
\begin{proof}
    Identical to \citet{balsells-rodas2024on}, replacing Lemma~D.6 there with Lemma \ref{lemma:extension_msm_preimage} (multi-lag version), which shows that the points in the pre-image of $G$ can be computed thanks to $\z'_{1:T}$ following a multi-lag MSM under assumptions (m1-m2) as almost every point in the pre-image of $G$ is generic, leading to invertibility of both $F$ and $G$ for a nonzero measure set.
\end{proof}

\begin{theorem}[adapted from \citet{balsells-rodas2024on}]\label{thm:identifiability-affine-msm}
    Let $F$, $G : \mathbb{R}^{Tm}\rightarrow\mathbb{R}^{Tn}$ be factored piecewise affine mappings, with $\x_t := f(\z_t)$ and  $\x_t' := g(\z_t')$,  for $1\leq t \leq T$. Let $\z_{1:T}$ and $\z_{1:T}'$ be distributed according to the identifiable multi-lag MSM family $\mathcal{M}^{T,M}_{NL}$ under assumptions (m1-m2). Assume $F(\z_{1:T})$ and $G(\z_{1:T}')$ are identically distributed, and that there exists $\x^{(0)}_{1:T}\in\R^{Tn}$ and $\delta>0$ such that $F$ and $G$ are invertible on $B(\x^{(0)}_{1:T}, \delta) \cap F(\mathbb{R}^{Tm})$. Then there exists an invertible factored affine transformation $H$ such that $H(\z_{1:T}) = \z_{1:T}'$.
\end{theorem}

\begin{proof}
    Since a multi-lag MSM is closed under affine transformations \citep[Prop~C.2]{balsells-rodas2024on}, and the distributions of the multi-lag MSM family $\mathcal{M}^{T,M}_{NL}$ under assumptions (m1-m2) are analytic functions, the proof is identical to \citet{balsells-rodas2024on}.
\end{proof}

Now we have all the elements to prove Theorem \ref{thm:identifiability_sds:i} (i).

\subsection{Proof of Theorem \ref{thm:identifiability_sds:i} (i)}

\begin{proof}
    It follows from \citet{balsells-rodas2024on}. Assume there exists another model that generates the same distribution from Eq.(\ref{eq:generative_model_sns}), with a prior $p'\in \mathcal{M}^{T,M}_{NL}$ under assumptions (m1-m2), and nonlinear emission $F'$, composed by $f'$ which is weakly injective and piecewise linear: i.e., $(F\#p )(\x_{1:T}) = (F'\#p' )(\x_{1:T})$.
    
    From weakly-injectivity, at least for some $\x_0\in\R^{n}$ and $\delta>0$, $f$ is invertible on $B(\x_0,2\delta)\cap f(\R^m)$. This satisfies the preconditions from Corollary \ref{cor:invertibility}, which implies there exists $\x^{(1)}_{1:T}\in B(\x^{(0)}_{1:T},\delta)$ and $\delta_1 >0$ such that both $F$ and $F'$ are invertible on $B(\x^{(1)}_{1:T},\delta_1)\cap F(\R^{Tm})$. Thus, by Theorem \ref{thm:identifiability-affine-msm}, there exists an affine transformation $H$ such that $H(\z_{1:T}) = \z_{1:T}'$, i.e., $p\in\mathcal{M}^{T,M}_{NL}$ is identifiable up to affine transformations.
\end{proof}

\subsection{Proof of Theorem \ref{thm:identifiability_sds:ii} (ii)}
\begin{proof}
    It follows from Theorem \ref{thm:identifiability_sds:i} (i) and Theorem E.1 from \citet{kivva2022identifiability}, which is applicable for any $t\in\{M+1,\dots,T\}$ as $\z_t | \z_{t-1}, \dots \z_{t-M}$ is a GMM. Note that $t\in\{1,\dots,M\}$ is already covered by \citet{kivva2022identifiability} as it reduces to the initial distribution case. 
    From Theorem \ref{thm:identifiability_sds:i} (i) we have identifiability up to affine transformations: i.e., we have some $\z'_{1:T}$ such that  $\z'_t = A \z_{t} + \Bb, t\in\{1,\dots,T\}$. Given (s1) and (s2), for any $t\in\{M+1,\dots,T\}$, the conditions of Theorem E.1 in \citet{kivva2022identifiability} are satisfied in a nonzero measure set in the space of $(\z_{t-1}, \dots \z_{t-M})$. Therefore, we can find $\tilde{A}$ such that $\tilde{\z}_t =\tilde{A}^{-1}\z'_t = \tilde{P}\tilde{D}\z_t + \tilde{\Bb}$, where $\tilde{P}$ and $\tilde{D}$ are permutation and scaling matrices respectively.
\end{proof}

\section{Proof of Corollary \ref{cor:causality_identifiability}}\label{app:causality_proof}
We first provide the definition of a node permutation in the context of temporal graph structures.
\begin{definition}\label{def:graph_node_permutation}
    Let $\mathbf{G}=(V,E_{1:M})$ be a graph, where from Section \ref{sec:causal_discovery}, $\mathbf{G}\in\mathbb{G}$ has $V:=[m]$ as nodes and contains information about the temporal edges with maximum lag $M$. A node permutation $\pi\in S_m$ on $\mathbf{G}$ results in a permuted graph $\tilde{\mathbf{G}}=\pi(\mathbf{G})$ with the same set of vertices, but with a permuted set of edges, for each $1\leq \tau \leq M$, as follows:
    \begin{equation}
        \tilde{\mathbf{G}}=(\tilde{V},\tilde{E}_{1:M}), \quad \tilde{V}=\big\{\pi(1),\dots, \pi(m)\big\}, \quad \tilde{E}_{\tau}= \Big\{\big(\pi(i),\pi(j) \big)\mid \big( i,j \big)\in E_{\tau}\Big\}
    \end{equation}
\end{definition}

\begin{proof}
    Assume two SDSs $p(\x_{1:T}),\tilde{p}(\x_{1:T})$ with corresponding regime-dependent graphs $\mathbf{G}_{1:K}$, $\tilde{\mathbf{G}}_{1:\tilde{K}}$. 

    Given assumption (m1) and (m2*), we recover $p(\z_{1:T})$ up to permutation and scaling from Theorem \ref{thm:identifiability_sds}. Combined with Theorem \ref{thm:identifiability_main}, we have identifiability up to permutations from Def. \ref{def:identifiability}, which implies $K=\tilde{K}$. We focus on condition 4 where for each $1 \leq i \leq C$, there is some $1\leq j \leq \tilde{C}$ such that:
    \begin{equation}
        p_{b^i_t}(\z_t | \z_{t-1:t-M})=\tilde{p}_{\tilde{b}^j_t}(\z'_t | \z'_{t-1:t-M}), \quad\forall(\z'_t,\dots, \z'_{t-M}) \in \R^{m(M+1)}.
    \end{equation}
    Then, at each time step $t$ there exists a permutation $\sigma$ such that for each $k\in[K]$, $p(\z_t | \z_{t-1:t-M}, s_t=k)=\tilde{p}(\z'_t | \z'_{t-1:t-M}, s_t=\sigma(k)), \quad\forall(\z_t,\dots, \z_{t-M}) \in \R^{m(M+1)}$. We know from condition 2 that $\sigma$ is constant for $t > M$. Since we know the distributions are Gaussian, from Gaussian identifiability \citep{yakowitz1968identifiability} we have
\begin{align}\label{eq:transition_function_proof_structure}
    \bm{m}(\z_{t-1},\dots, \z_{t-M}, k) & = PD\tilde{\bm{m}}(\z'_{t-1},\dots, \z'_{t-M}, \sigma(k)) + \Bb  \quad k\in[K],
\end{align}
where each $\z'_t = D^{-1}P^T\z_t + \tilde{\Bb}$ since the MSM is closed under affine transformations \citep[Prop.~C.2]{balsells-rodas2024on}.
Now, given that the functions $\bm{m}(\cdot,k), k\in[K]$ are analytic, the Jacobian will preserve similar permutation equivalences. W.l.o.g. we fix $\tau\in[M]$, and $k\in[K]$ and denote the Jacobian of the transition w.r.t. $\z_{t-\tau}$ as $J_{\bm{m}(\cdot,k), \tau}$. We compute the derivative w.r.t. $\z_{t-\tau}$ on Eq. (\ref{eq:transition_function_proof_structure}):
\begin{align}
    \frac{\partial \bm{m}(\z_{t-1},\dots, \z_{t-M}, k)}{\partial \z_{t-\tau}} &= PD \frac{\partial \tilde{\bm{m}}(\z'_{t-1},\dots, \z'_{t-M}, \sigma(k))}{\partial \z'_{t-\tau}}\frac{\partial \z'_{t-\tau}}{\partial \z_{t-\tau}},\\
    J_{\bm{m}(\cdot,k), \tau} &= P D J_{\tilde{\bm{m}}(\cdot,\sigma(k)), \tau}D^{-1}P^T.
\end{align}
Given the minimality assumption (m3), we can directly map $J_{\bm{m}(\cdot,k), \tau}$ and $J_{\tilde{\bm{m}}(\cdot,\sigma(k)), \tau}$ to some adjacency matrix via thresholding entries with absolute value greater than 0. Given $D$ is diagonal, the structure of $DJ_{\tilde{\bm{m}}(\cdot,\sigma(k)), \tau}D^{-1}$ is preserved. The resulting set of graph adjacency matrices $A_{k,\tau},\tilde{A}_{k,\tau}$ for regime $k$ and lag $\tau$ relate as follows
\begin{equation}
    A_{k,\tau} = P \tilde{A}_{\sigma(k), \tau} P^T,
\end{equation}
where the permutation matrix $P$ is associated to some permutation $\pi\in S_m$, and reflects a permutation of the rows and columns of $A_{k,\tau}$ following $\pi$. Therefore, the above operation is equivalent to a node permutation as defined in Def.~\ref{def:graph_node_permutation}, and we have $\mathbf{G}_{k}=\pi(\tilde{\mathbf{G}}_{\sigma(k)})$, for each $k\in [K]$.
\end{proof}

\section{Estimation Details}\label{app:estimation_details}

The ELBO presented in Eq.~\eqref{eq:elbo} is derived as follows:
\begin{align*}
\log p_{\mparam}&(\x_{1:T}) =\log \int \sum_{\s_{M:T}} p_{\mparam}(\x_{1:T}, \z_{1:T}, \s_{M:T})d\z_{1:T}  \\
& \geq \mathbb{E}_{q_{\vparam,\mparam}(\z_{1:T},\s_{M:T}| \x_{1:T})}\left[\log\frac{p_{\mparam}(\x_{1:T},\z_{1:T}| \s_{M:T})p_{\mparam}(\s_{M:T})}{q_{\vparam}(\z_{1:T},\s_{M:T}| \x_{1:T})}  \right]\\
& \geq \mathbb{E}_{q_{\vparam}(\z_{1:T}| \x_{1:T})}\left[\log\frac{p_{\mparam}(\x_{1:T}|\z_{1:T})}{q_{\vparam}(\z_{1:T}| \x_{1:T})}  + \mathbb{E}_{p_{\mparam}(\s_{M:T}| \z_{1:T})}\left[\log\frac{p_{\mparam}(\z_{1:T}|\s_{M:T})p_{\mparam}(\s_{M:T})}{p_{\mparam}(\s_{M:T}| \z_{1:T})}\right]\right]\\
& \geq \mathbb{E}_{q_{\vparam}(\z_{1:T}| \x_{1:T})}\left[\log\frac{p_{\mparam}(\x_{1:T}|\z_{1:T})}{q_{\vparam}(\z_{1:T}| \x_{1:T})} +   \log p_{\mparam}(\z_{1:T})\right]\\
& \geq \mathbb{E}_{q_{\vparam}(\z_{1:T}|\x_{1:T})} \Big[\log p_{\mparam}(\x_{1:T}|\z_{1:T})\Big]
    - KL\Big(q_{\vparam}(\z_{1:T}|\x_{1:T}) || p_{\mparam}(\z_{1:T})\Big).
\end{align*}

Given samples $\tilde{\z}\sim q_{\vparam}(\z_{1:T}|\x_{1:T})$, the posterior distributions $\gamma_{t,k}(\tilde{\z}_{1:T})$ and $\xi_{t,k,k'}(\tilde{\z}_{1:T})$ used in Eq.~\eqref{eq:gradients_z_posterior} can be computed using forward and backward messages \citep{bishop2006pattern}: $\{ \alpha_{t,k}(\tilde{\z}_{1:t}) = p_{\mparam}(\tilde{\z}_{1:t},s_t=k) \}$, and $\{ \beta_{t,k}(\tilde{\z}_{t+1:T}) = p_{\mparam}(\tilde{\z}_{t+1:T}|s_t=k) \}$, respectively:
\begin{equation}
    \gamma_{t,k}(\tilde{\z}_{1:T}) = \frac{\alpha_{t,k}(\tilde{\z}_{1:t})\beta_{t,k}(\tilde{\z}_{t+1:T})}{\sum_{k=1}^K\alpha_{t,k}(\tilde{\z}_{1:t})\beta_{t,k}(\tilde{\z}_{t+1:T})},
\end{equation}
\begin{equation}
    \xi_{t,k,k'}(\tilde{\z}_{1:T}) = \frac{\alpha_{t-1,k'}(\tilde{\z}_{1:t-1})\beta_{t,k}(\tilde{\z}_{t+1:T})p(\tilde{\z}_{t}|\tilde{\z}_{t-1}, \dots, \tilde{\z}_{t-M}, s_t=k)p(s_t=k|s_{t-1}=k')}{\sum_{k=1}^K\alpha_{t,k}(\tilde{\z}_{1:t})\beta_{t,k}(\tilde{\z}_{t+1:T})}
\end{equation}
where $\alpha_{t,k}(\tilde{\z}_{1:t})$ and $\beta_{t,k}(\tilde{\z}_{t+1:T})$ are computed as follows:
\begin{equation}
    \alpha_{t,k}(\tilde{\z}_{1:t}) = p_{\mparam}(\tilde{\z}_t | \tilde{\z}_{t-1:t-M}, s_t = k)  
\times \sum_{k'=1}^K p_{\mparam}(s_t = k | s_{t-1} = k') \alpha_{t-1,k'}(\tilde{\z}_{1:t-1}),
\end{equation}
\begin{equation}
   \beta_{t,k}(\tilde{\z}_{t+1:T}) = \sum_{k'=1}^K p_{\mparam}(\tilde{\z}_{t+1} | \tilde{\z}_{t:t+1-M}, s_{t+1} = k') 
\times p_{\mparam}(s_{t+1} = k' | s_{t} = k) \beta_{t+1,k'}(\tilde{\z}_{t+1:T}). 
\end{equation}

\section{Experiment Details}\label{app:experiment_details}

The code for reproducing our experiments can be found here: \url{https://github.com/charlio23/identifiable-SDS}.

\subsection{Synthetic Data Generation}

We provide additional details of the data generation process. The initial Gaussian components are sampled from $\mathcal{N}(0,0.7^2\mathbf{I})$. For the mean transition functions $\bm{m}(\cdot,k)$, we already indicate the use of locally connected networks \citep{zheng2018dags} where network dependencies follow a regime-dependent causal structure. The networks consist of two-layer MLPs with 16 hidden units and cosine activations. The covariance matrix $\bm{\Sigma}(\cdot, k)$ is designed as follows depending on the setting:
\begin{itemize}[leftmargin=1.5em]
  \setlength{\itemsep}{5pt}
  \setlength{\parskip}{0pt}
  \setlength{\parsep}{0pt}
    \item Constant noise: fixed, $\bm{\Sigma}= 0.01\mathbf{I}$.
    \item Heterogeneous noise: $\bm{\Sigma}(k)= \text{diag}( \bm{\sigma}^{2}_k)$, where for each $k\in[K]$ and $i\in[m]$, $\sigma^{2}_{k,i}\sim U(0.005, 0.08)$, such that assumption (s2) holds. That is, we check whether there are at least two regimes $k_1,k_2\in[K]$ with distinct diagonal ratios across dimensions, and re-sample otherwise.
    \item History-dependent noise: $\bm{\Sigma}(\cdot, k)= \text{diag}(C_k\cdot \mathrm{sigmoid}(\bm{m}(\cdot, k)))$, where $C_k\sim U(0.05, 0.1)$.
\end{itemize}
For iSDS, we further generate observations using two-layer Leaky ReLU networks with 8 hidden units. Regarding ablations, the assumptions we explore use the following networks for the mean transitions:
\begin{itemize}[leftmargin=1.5em]
  \setlength{\itemsep}{5pt}
  \setlength{\parskip}{0pt}
  \setlength{\parsep}{0pt}
    \item \textsc{Overlap}: piecewise analytic functions with cosine activations, where we force the same function across states if the $L_2$ norm of the conditioned trajectory at time $t$ $(\z_{t-1}, \dots, \z_{t-M})$ is between 3 and 5; and
    \item \textsc{Zero}: Random networks with cosine activations.
\end{itemize}

\subsection{Training Specifications}\label{app:training_specs}

The experiments are implemented in PyTorch \citep{paszke2017automatic} and performed on NVIDIA RTX 2080 Ti GPUs, except for the experiments on climate data with SDSs, where we used NVIDIA GeForce RTX 3090 GPUs. In all our experiments, we use Adam optimiser \citep{kingma2014adam} with a learning rate of $7\cdot 10^{-3}$ for iMSM, and $5\cdot 10^{-4}$ for iSDS. Other specific configurations are dependent on the dataset and indicated below. Similar to related approaches \citep{halva2020hidden}, we use random restarts to achieve better parameter estimates: $3$ for iMSM initialisation (see below), and $5$ for the full models. 

To train iSDS, we generally follow a 4-stage procedure.
\begin{itemize}
    \item \textbf{Initialisation (iMSM):} We initialise the latent dynamical prior by training iMSM on the observations after dimensionality reduction via PCA. To train iMSM, we use a batch size of $100$, set learning rate $7\cdot10^{-3}$, and decrease it by a factor of 0.5 when the likelihood plateaus, up to twice. We use 3 random restarts, and select the best model on training likelihood to initialise iSDS.
    \item \textbf{Encoder-decoder pre-training:} We find an initialisation for the encoder and decoder by freezing the gradients on the dynamical prior. 
    \item \textbf{Warmup phase:} iSDS is trained jointly but we keep the switch parameters $\pi,\ Q$ frozen. This helps us prevent state collapse.
    \item \textbf{Final phase:} All the networks are fine-tuned. 
\end{itemize}
Related works \citep{dong2020collapsed,ansari2021deep} introduce annealing to force uniform regime posterior assignments and ensure the model uses all the switches for inference. Instead, we find that combining initialisation from iMSM and random restarts prevents state collapse. Below we describe the specific configurations and network architecture choices for each dataset.
\paragraph{Synthetic data.} For training, we use a batch size of 100, initialise the latent iMSM as described above, pre-train the encoder and decoder for 40 epochs, and warmup for 10 additional epochs. We then unfreeze all the gradients and train for 700 epochs with a learning rate of $5\cdot10^{-4}$ and decay of 0.8 every 200 epochs. The transition means and covariances follow the architecture from the data generation process. The encoder and decoder are two-layer MLPs with 128 hidden units and Leaky ReLU activations. We regularise the Jacobian transitions with $\eta=0.05$ in the loss function Eq.~\eqref{eq:loss_function_elbo}.

\paragraph{Brain activity data.} Training follows the initialisation phase of iSDS, and the transitions consist of two-layer MLPs with $32$ hidden units with GELU activations.

\paragraph{Financial data.} Training largely follows the synthetic data procedure with minor modifications, and we use the same setting for log prices and log returns. In the initialisation phase, we use the whole time series to encourage better regime learning, with a learning rate of $7\cdot10^{-3}$ which is reduced by a factor of 0.5 if likelihood plateaus for 10 iterations. iMSM training stops after the model plateaus 3 times. The rest of the phases follow the procedure described in the synthetic data with subsampled time series of length $T=200$. We train for $100$k iterations, with a learning rate of $5\cdot 10^{-4}$ and a decay of 0.8 every $40$k iterations. Pre-train and warmup phases last for $2$k and $8$k iterations respectively. The architecture consists of encoder-decoder from two-layer MLPs with 256 hidden dimensions and Leaky ReLU activations, and mean transitions of two-layer MLPs with 32 hidden dimensions and GELU activations, and regime-dependent diagonal covariances. Here we consider $\eta=0$ in the loss function Eq.~\eqref{eq:loss_function_elbo}.

\paragraph{Climate data.} Training follows the financial time series, where initialisation trains iMSM on the whole time series with $T=360$ time steps. We later use the iMSM trained for the interpretability analysis presented in Section~\ref{sec:climate}. To train the rest of iSDS, we subsample sequences of length $T=100$. We pre-train the encoder-decoder for $250$k iterations, with a learning rate of $5\cdot10^{-4}$ and a decay of 0.8 every $100$k iterations. We then warmup iSDS for $4000$ iterations, and we run the final phase for $120$k iterations, with a decay of 0.8 every $20$k iterations. For the network architecture, the transitions are the same as used in financial time series, and we use CNNs for inference and generation. For inference, we use a $9$-layer CNN with increasing channels, kernel size $3$, and padding $1$, Leaky ReLU activations, and we alternate between using stride $2$ and $1$. We start with $64$ channels and double the channels every two layers, until we have $512$ channels. We then run a three-layer MLP with Leaky ReLU activations and $256$ hidden dimensions. For generation, we use a similar network, starting with a two-layer MLP with Leaky ReLU activations and $256$ hidden dimensions, and use transposed convolutions instead of convolutions (with the same configuration as before). We also consider $\eta=0$ in this setting.

\subsection{Evaluation metrics}\label{app:metrics}

In Section \ref{sec:experiments_msm} we outline how each metric is computed. Here we provide additional details for the $L_2$ distance and $R^2$ score on the regime-specific mean functions $\bm{m}(\cdot,k)$, and for recovering the latent regime-dependent causal structure.

\paragraph{Mean function metrics.} Let $f,g:\R^m \rightarrow \R^m$ and define the squared $L_2$ distance as follows:
\begin{equation}
    d^2_{L_2}(f, g) := \int_{\x\in\R^m}\big|\big|f(\x) - g(\x) \big|\big|^2d\x \approx \frac{1}{I}\sum_{i=1}^{I}\big|\big|f\left(\x^{(i)}\right) - g\left(\x^{(i)}\right) \big|\big|^2,
\end{equation}
where $\x^{(i)}$ denote samples from held-out data. Because models are identifiable up to a permutation of the regimes and an affine transformation of the continuous latents (or permutation and scaling), we first align the regimes by choosing the permutation $\sigma$ that maximises the $F_1$ score of regime estimation using posterior assignments from $p(\s_{M:T}|\z_{1:T})$. Recall that under an affine transformation, the multi-lag MSM preserves the following equivalence between mean functions:
\begin{equation}\label{eq:affine_alignment_mean}
    \bm{m}(\z_{t-1},\dots, \z_{t-M},k) = A\hat{\bm{m}}\Big(A^{-1}(\z_{t-1} - \Bb),\dots, A^{-1}(\z_{t-M} - \Bb),\sigma(k)\Big) + \Bb,
\end{equation}
where $A,\ \Bb$ denote the affine transformation parameters. To evaluate functions within the appropriate regime domain, we partition the held-out trajectories into $K$ sets $\{\mathcal{Z}_k\}_{k\in[K]}$, where each set $\mathcal{Z}_k$ contains variables such that for each $M < t \leq T$,  $\{\z_{t-1}^{(i)}, \dots, \z_{t-M}^{(i)}\}\in\mathcal{Z}_k$ if $k=\arg\max p(s^{(i)}_t|\z_{1:T}^{(i)})$, where index $i$ denotes a sample in the held-out dataset. We use $1000$ held-out trajectories. We then compute the error as follows:
\begin{equation}\label{eq:l2_metric}
   L_2 \text{ err iSDS}:= \min_{A,\Bb}\frac{1}{K} \sum_{k=1}^K d^2_{L_2}\Big(\bm{m}\big(\cdot,k\big), A\hat{\bm{m}}\big(\cdot,k\big) + \Bb\Big),
\end{equation}
where inputs to $\bm{m}(\cdot,k)$ are drawn from $\mathcal{Z}_k$, and inputs to $\hat{\bm{m}}(\cdot,k)$ are mapped according to Eq.~\eqref{eq:affine_alignment_mean}.

For $R^2$ scores, we follow a similar procedure, and compute a variance-normalised $L_2$ discrepancy:
\begin{equation}
    d_{R^2}(f, g) := 1 - \frac{d^2_{L_2}(f,g)}{Var(f)}, \quad Var(f) = \frac{1}{I}\sum_i \Big|\Big|f\big(\x^{(i)}\big) - \frac{1}{I}\sum_j f\big(\x^{(j)}\big)\Big|\Big|^2.
\end{equation}
In ablations, we observe low $F_1$ regime scores due to the introduced function overlaps. Therefore, for iMSMs we choose the regime permutation which maximises average ${R^2}$: 
\begin{equation}
   R^2 \text{ err iMSM} := \max_{\sigma\in S_K}\frac{1}{K} \sum_{k=1}^K d_{R^2}\Big(\bm{m}\big(\cdot,k\big), \hat{\bm{m}}\big(\cdot,\sigma(k)\big)\Big).
\end{equation}
For iSDS, we follow the same procedure, but first align the latents via an affine map:
\begin{equation}
   R^2 \text{ err iSDS} := \max_{\sigma\in S_K}\frac{1}{K} \sum_{k=1}^K d_{R^2}\Big(\bm{m}\big(\cdot,k\big), A\hat{\bm{m}}\big(\cdot,\sigma(k)\big) + \Bb\Big),
\end{equation}
with inputs to $\hat{\bm{m}}\big(\cdot)$ rescaled as in Eq.~\eqref{eq:affine_alignment_mean}. Here we evaluate on the same 1000 held-out trajectories; unlike the $L_2$ computation, we do not pre-sort samples by regimes in the ablations (due to the low $F_1$ score).

\paragraph{Causal structure recovery.}
Let $\bm{J}_{\hat{\bm{m}}(\cdot,k)}(\z_{t-1}, \dots, \z_{t-M})$ denote the Jacobian of $\hat{\bm{m}}(\cdot,k)$ at $(\z_{t-1}, \dots, \z_{t-M})$. We estimate the regime-dependent latent causal graph $\hat{\mathbf{G}}_{1:K}$ by thresholding the mean absolute Jacobian over samples assigned to regime $k$. Using the same partition $\mathcal{Z}_k$ with size $|\mathcal{Z}_k|=N_k$, we define: 
\begin{equation}
    \hat{\mathbf{G}}_k := \1\left(  \frac{1}{N_k}\sum_{i=1}^{N_k}\left|\bm{J}_{\hat{\bm{m}}(\cdot,k)} \left(\z_{t-1}^{(i)}, \dots, \z_{t-M}^{(i)}\right)\right| > \tau \right), \quad (\z_{t-1}^{(i)}, \dots, \z_{t-M}^{(i)})\in\mathcal{Z}_k,
\end{equation}
where we use $\tau=0.05$ in our experiments, and $\1(\cdot)$ denotes the indicator function, which equals 1 if the argument is true and 0 otherwise.  For lag $M$, $\hat{\mathbf{G}}_k$ has shape $m\times (mM)$ (rows: current variables; columns: lagged variables across all $M$ lags). Regimes are aligned by the permutation that maximises the regime $F_1$ score (as above). For latent alignment, we decompose the affine map fitted for the $L_2$ distance as $A\approx DP$, with $P$ a permutation matrix and $D$ diagonal. $P$ is obtained by selecting maximal-absolute entries per row. For each lag $\tau\in[M]$ with $\hat{E}_{k,\tau}$ denoting the lag-$\tau$ adjacency matrix of $\hat{\mathbf{G}}_k$, we permute with $P\hat{E}_{\sigma(k),\tau}P^T$ for latent variable alignment. Finally, we report the average $F_1$ score between the aligned estimate and the ground-truth graphs $\mathbf{G}_{1:K}$, averaged over regimes.

\subsection{Brain Activity Data}\label{app:neuroscience}

\begin{figure}
\centering
  \begin{minipage}{0.35\textwidth}
    \centering
    \includegraphics[width=\linewidth]{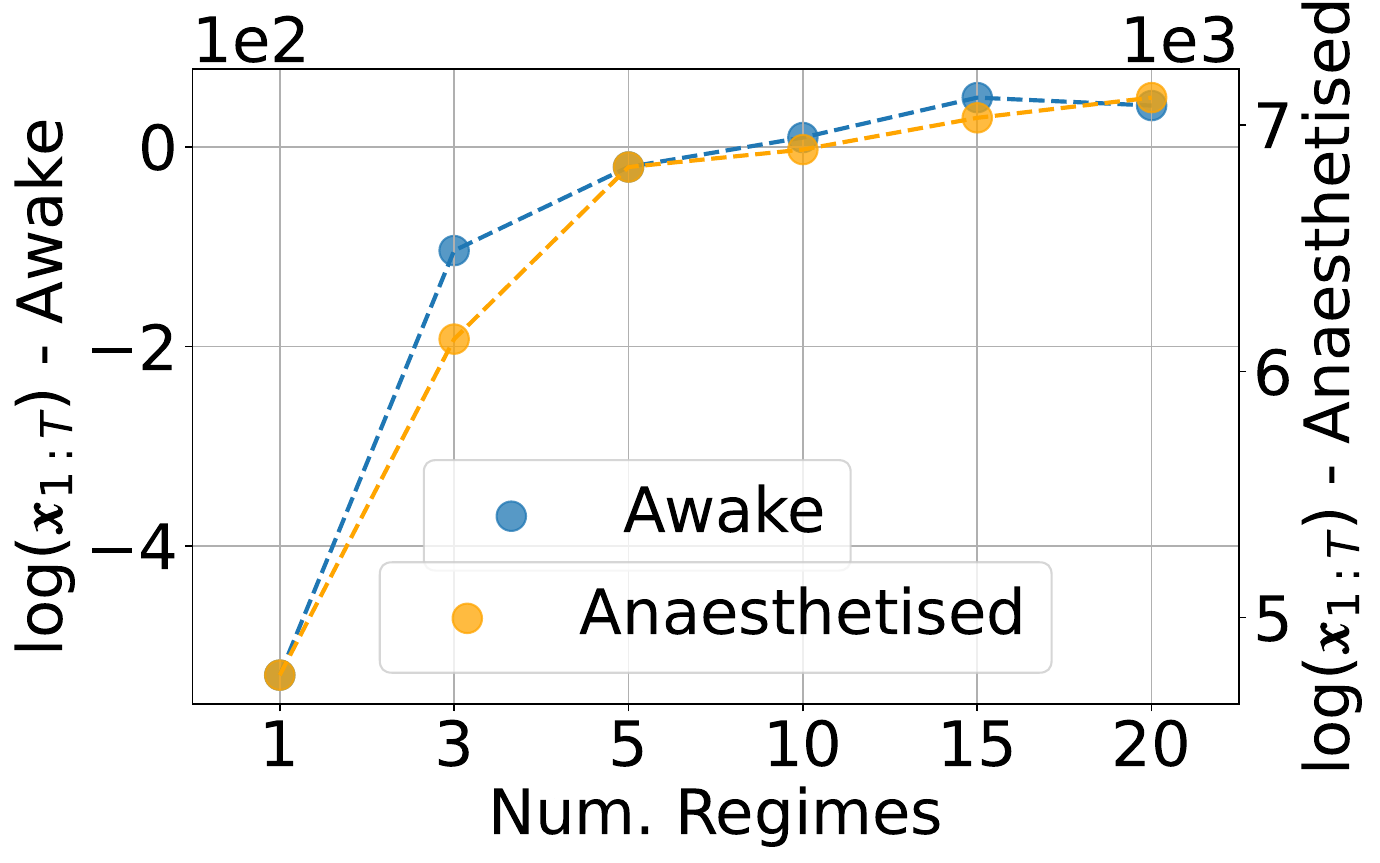}
  \end{minipage}
  \quad
  \begin{minipage}{0.58\textwidth}
    \caption{Test log-likelihood using different regimes and $M=2$ lags for \textsc{Awake} and \textsc{Anaesthetised} conditions.}
    \label{fig:neuroscience_model_selection}
  \end{minipage}
\end{figure}

Figure~\ref{fig:neuroscience_model_selection} reports test log-likelihoods for the \textsc{Awake} and \textsc{Anaesthetised} conditions with increasing regimes $K$. Performance improves with larger $K$ and plateaus around $K\approx10$. We adopt an overcomplete setting with $K=15$ in the main experiments, and expect the models to use a subset the regimes in practice. As reported in Table~\ref{tab:table_neuro} (Section \ref{sec:neuroscience}), the total regime usage is $53.33\%$ for \textsc{Awake} and $33.33\%$ for \textsc{Anaesthetised}, corresponding to 8 and 5 active regimes. Allowing a larger candidate set of regimes helps prevent state collapse, as unused regimes remain inactive, and test likelihood does not degrade.

\paragraph{Metrics from Table \ref{tab:table_neuro}.} Each sequence last 2\,s, with $T=400$ time steps. We compute all metrics from the smoothed posterior over regimes (uniform kernel of length 3).
\begin{itemize}[leftmargin=1.35em]
  \setlength{\itemsep}{5pt}
  \setlength{\parskip}{0pt}
  \setlength{\parsep}{0pt}
  \item \textbf{Transition frequency (Hz):} Regime changes per second, averaged across epochs.
  \item \textbf{Total usage:} Fraction of regimes whose occupancy exceeds 1\% of time points in at least one test epoch.
  \item \textbf{Usage per sample:} Mean (over epochs) fraction of regimes whose per-epoch occupancy exceeds 1\% of time points.
\item\textbf{Max duration (s):} Maximum regime dwell time, averaged across epochs.
\end{itemize}

\subsection{Financial Data}\label{app:finance}

Below we list the tickers of the 100 U.S. companies we consider in our experiments organised by sector.

\begin{itemize}[leftmargin=1.35em]
  \setlength{\itemsep}{5pt}
  \setlength{\parskip}{0pt}
  \setlength{\parsep}{0pt}
    \item \textbf{Tech:} AAPL, MSFT, GOOGL, META, AMZN, NVDA, INTC, CSCO, IBM, ORCL
    \item \textbf{Health:} JNJ, PFE, MRK, ABBV, BMY, LLY, AMGN, GILD, ILMN, UNH
    \item \textbf{Financials:} JPM, BAC, WFC, C, GS, MS, USB, PNC, TFC, COF
    \item \textbf{Consumer Staples:} KO, PEP, PG, CL, KMB, MDLZ, GIS, K, WMT, COST
    \item \textbf{Consumer Discretionary:} NKE, SBUX, MCD, DIS, BKNG, TJX, GAP, KMX, ROST, YUM
    \item \textbf{Industrials:} BA, CAT, MMM, GE, HON, LMT, RTX, UNP, CSX, DAL
    \item \textbf{Energy:} XOM, CVX, COP, PSX, DVN, OXY, KMI, EOG, HES, VLO
\item \textbf{Real Estate:} SPG, PLD, PSA, AVB, EQR, WELL, VTR, DLR, O, ESS
\item \textbf{Communication:} VZ, T, TMUS, CMCSA, CHTR, DIS, NFLX, PARA, EA, GOOGL
\item \textbf{Materials:} DD, SHW, LYB, NEM, FCX, NUE, AA, MLM, CRH, CE    
\end{itemize}

We also examine latent transition structure across regimes for log returns (Fig.~\ref{fig:graph_transition_log_return}), and the edge-intensity distribution (Fig.~\ref{fig:histogram_log_return}). The histogram is computed using the average absolute off-diagonal Jacobian entries for each regime, and it is concentrated near zero. This indicates weak latent interactions. For interpretability, we show the higher-intensity subset (right of the red threshold). In Regimes~0 and ~1, $z_0$ and $z_1$ dominate the latent dynamics (consistent with global market factors). Furthermore, the high-volatility regime (Regime~0), exhibits denser interactions than the medium and low-volatility regimes (Regime~2), consistent with the need for stronger dynamical adjustments during turbulent periods.

\begin{figure}
    \centering

        \begin{subfigure}{0.24\textwidth}
            \centering
            \includegraphics[width=\textwidth]{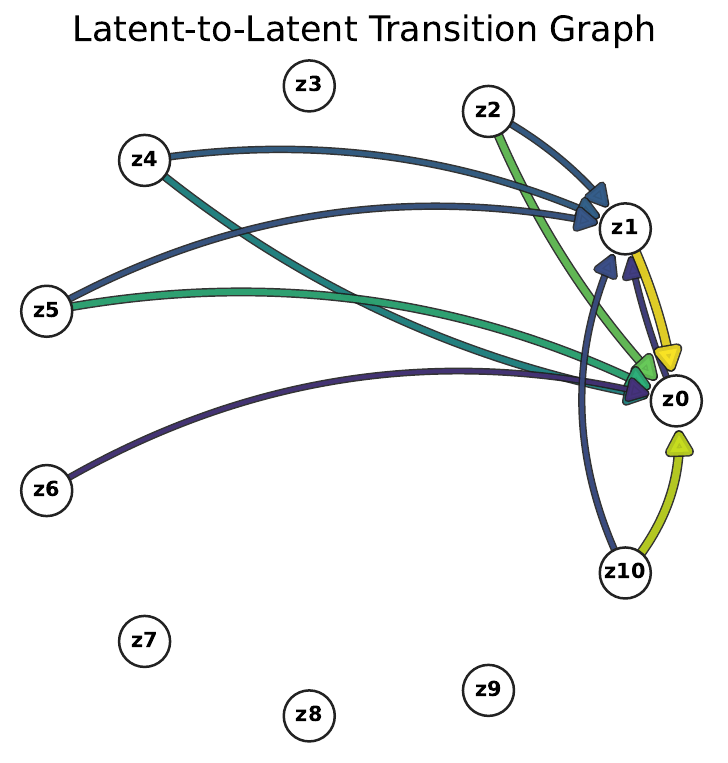}
            \vspace{-.5cm}
            \caption{Regime 0.}
        \end{subfigure}%
        \hfill
        \begin{subfigure}{0.24\textwidth}
            \centering
            \includegraphics[width=\textwidth]{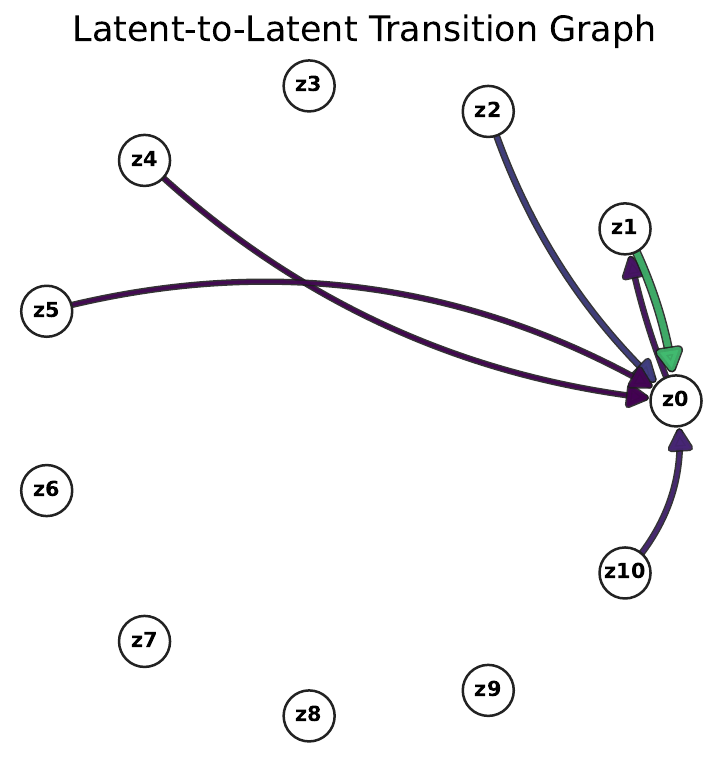}
            \vspace{-.5cm}
            \caption{Regime 1.}
        \end{subfigure}
        \begin{subfigure}{0.24\textwidth}
            \centering
            \includegraphics[width=\textwidth]{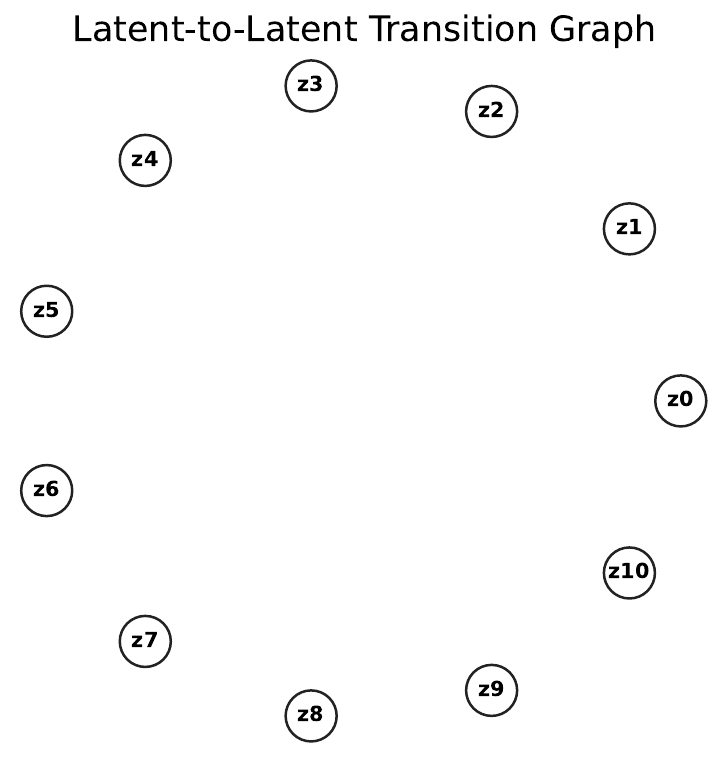}
            \vspace{-.5cm}
            \caption{Regime 2.}
        \end{subfigure}%
        \hfill
        \begin{subfigure}{0.24\textwidth}
            \centering
            \includegraphics[width=\textwidth]{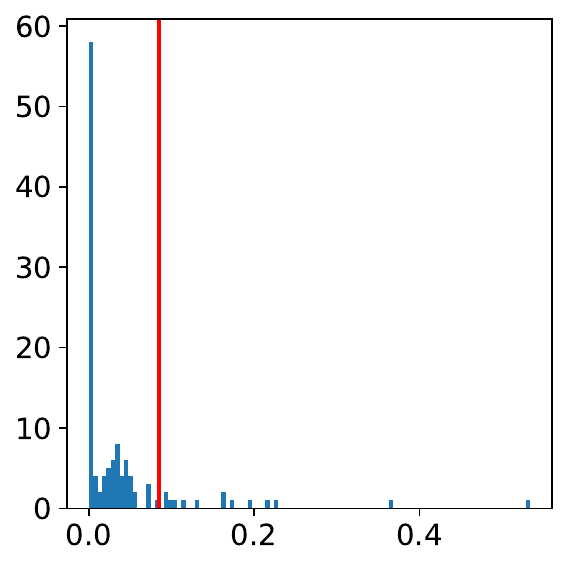}
            \vspace{-.5cm}
            \caption{Edge histogram.}
            \label{fig:histogram_log_return}
        \end{subfigure}
        \vspace{-.25cm}
        \caption{(a-c) Latent transition graphs by regimes and (d) edge intensities (log returns).}
        \label{fig:graph_transition_log_return}
\end{figure}

For log price data, Figure~\ref{fig:graph_transition_log_price} shows regime-specific latent transition graphs and the corresponding edge-intensity histogram (Fig. \ref{fig:histogram_log_price}). The edge distribution is centred near $0.1$ with a long tail of stronger interactions. The latent dynamics differ across regimes, with the lateral/bear regime (Regime~0) exhibiting simpler and weaker dynamics compared to the bull regime (Regime~1). The two global factors ($z_2$ and $z_10$) are involved differently across regimes. Notably, interactions are concentrated on $z_5$ and $z_8$, which are aligned with Consumer Staples, Real Estate, and Materials. This suggests these sectors require relatively complex, regime-dependent adjustments (e.g., Consumer Staples as a defensive sector; Real Estate as rate-sensitive).

\begin{figure}
    \centering
        \begin{subfigure}{0.24\textwidth}
            \centering
            \includegraphics[width=\textwidth]{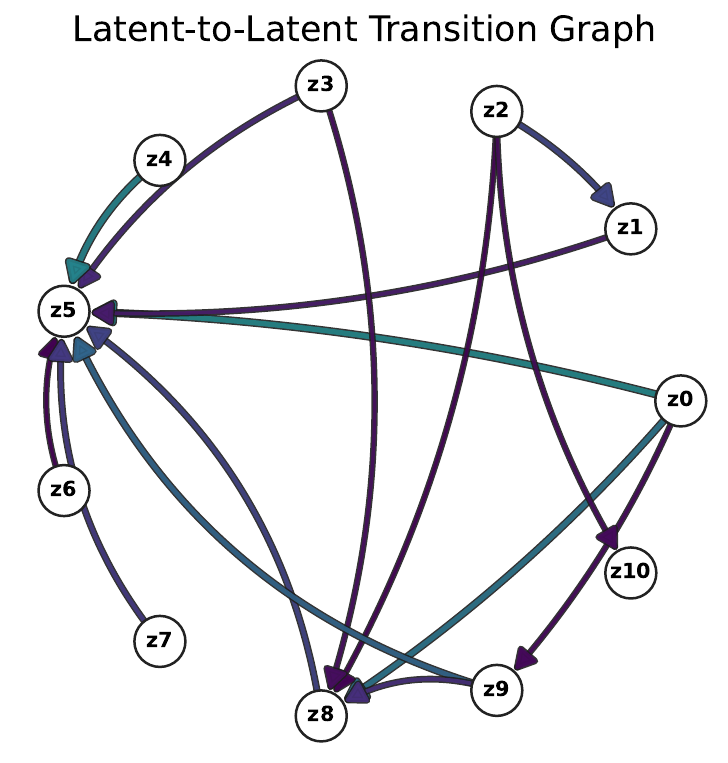}
            \vspace{-.5cm}
            \caption{Regime 0.}
        \end{subfigure}
        \begin{subfigure}{0.24\textwidth}
            \centering
            \includegraphics[width=\textwidth]{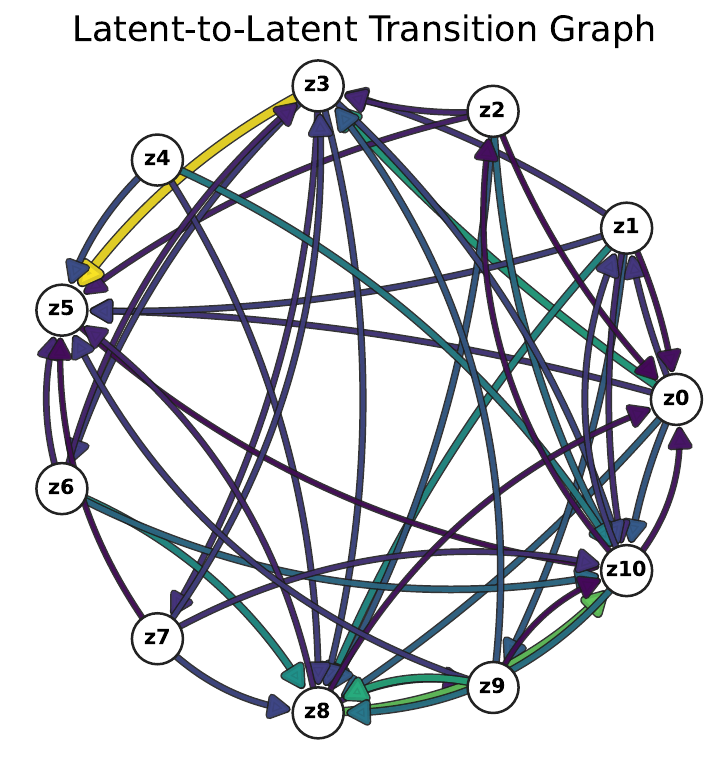}
            \vspace{-.5cm}
            \caption{Regime 1.}
        \end{subfigure}%
        \begin{subfigure}{0.24\textwidth}
            \centering
            \includegraphics[width=\textwidth]{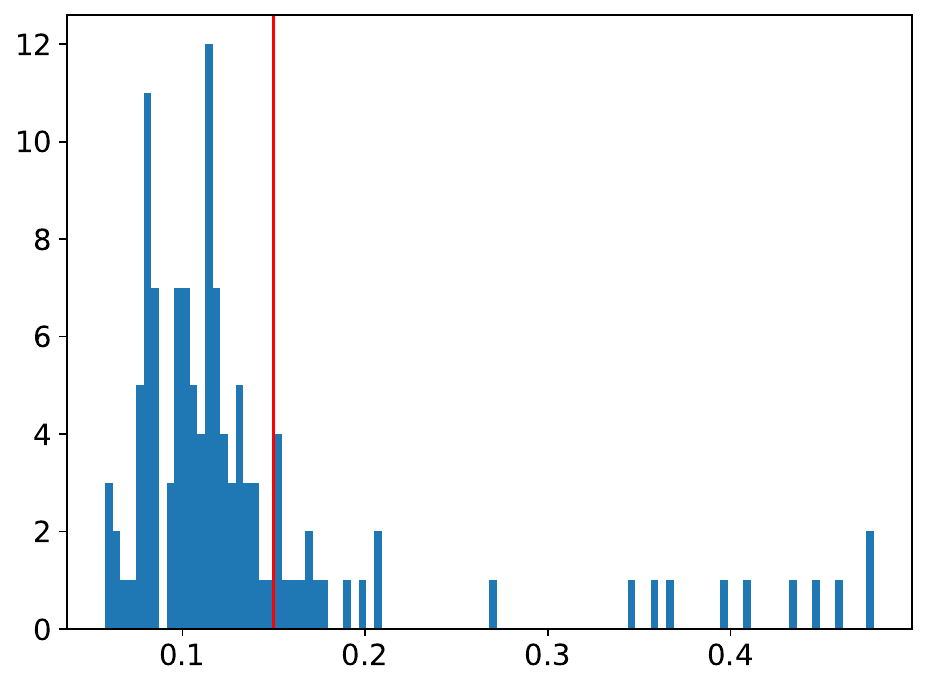}
            \vspace{-.5cm}
            \caption{Edge histogram.}
            \label{fig:histogram_log_price}
        \end{subfigure}
        \vspace{-.25cm}
        \caption{(a-b) Latent transition graphs by regime and (c) edge intensities (log price).}
        \label{fig:graph_transition_log_price}

\end{figure}

\subsection{Climate data}\label{app:climate_details}

\begin{figure}
\centering
  \begin{minipage}{0.38\textwidth}
    \centering
    \includegraphics[width=\linewidth]{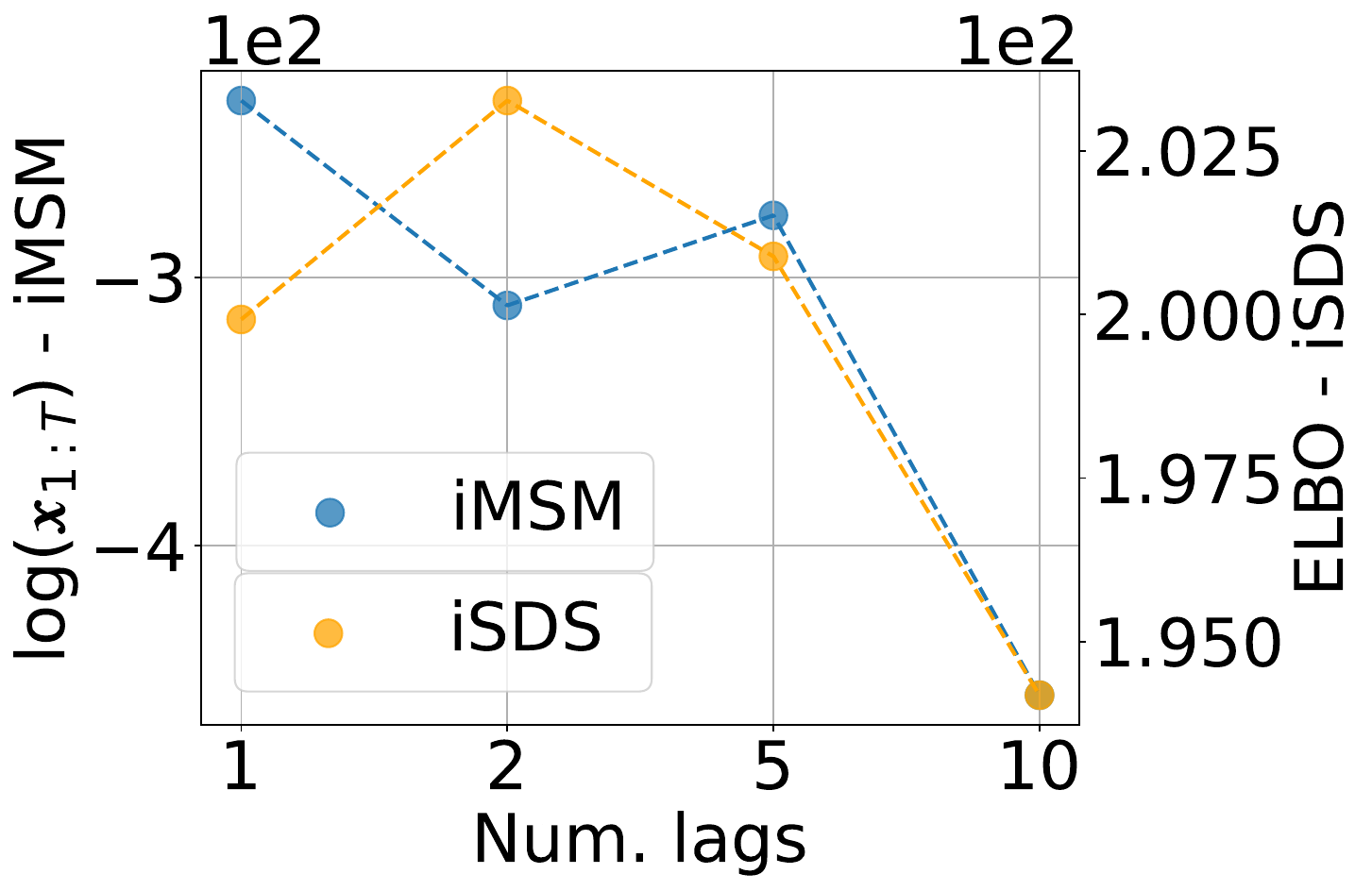}
  \end{minipage}
  \quad
  \begin{minipage}{0.58\textwidth}
    \caption{Test log-likelihood (blue) and ELBO (orange) for iMSM and iSDS respectively with increasing $M$. Quantities are reported per observation size for better comparison (i.e., $10$ for iMSM, $64\times64$ for iSDS).}
    \label{fig:climate_model_selection}
  \end{minipage}
\end{figure}

\paragraph{Model selection.} The data settings differ between iMSM and iSDS. For iMSM, we follow \citet{varando2022learning} and project land pixels onto $10$ principal components. For iSDS, we train on random $64{\times}64$ land patches with subsequences of length $T=100$. We fix $K=4$ regimes and vary the number of lags $M$. Figure~\ref{fig:climate_model_selection} reports held-out performance as a function of $M$: test log-likelihood and ELBO (higher is better in both cases). To make scales comparable, we report both metrics per observation (i.e., per PC for iMSM and per pixel for iSDS). We observe that iMSM performs best at $M=1$, while iSDS benefits from higher orders, with $M=2$ performing best. We note that iMSM operates on globally mixed signals coming from very high-dimensional data (land pixels), and therefore increasing $M$ can quickly overfit unseen sequences. In contrast, iSDS models local dynamics with shared transitions, and parameter sharing regularises additional lag parameters.

\paragraph{Regime calendars.} In Section~\ref{sec:climate}, Tables \ref{tab:imsm_season_table} and \ref{tab:isds_season_table} report ENSO correlations by season for iMSM and iSDS respectively. We visualise the within-year regime occupancy of iMSM in Figure \ref{fig:iMSM_year_structure}. As mentioned, the dynamics are structured yearly within 3 periods: from March to May, from June to September, and from October to February. The structures remain consistent for unobserved years. For iSDS, we observe 5 types of recurrent year-based patterns depending on location. Figures~\ref{fig:iSDS_year_structure_ET}–\ref{fig:iSDS_year_structure_SA} show Ethiopia, Sahel, and South Africa. We omit calendars for Sahara and Congo, as they predominantly use a single regime. Across regions, the reported months in Table~\ref{tab:isds_season_table} agree with the visual occupancy, which remain consistent for unseen years.

\begin{figure}
  \begin{subfigure}{0.46\textwidth}
    \centering
    \includegraphics[width=\linewidth]{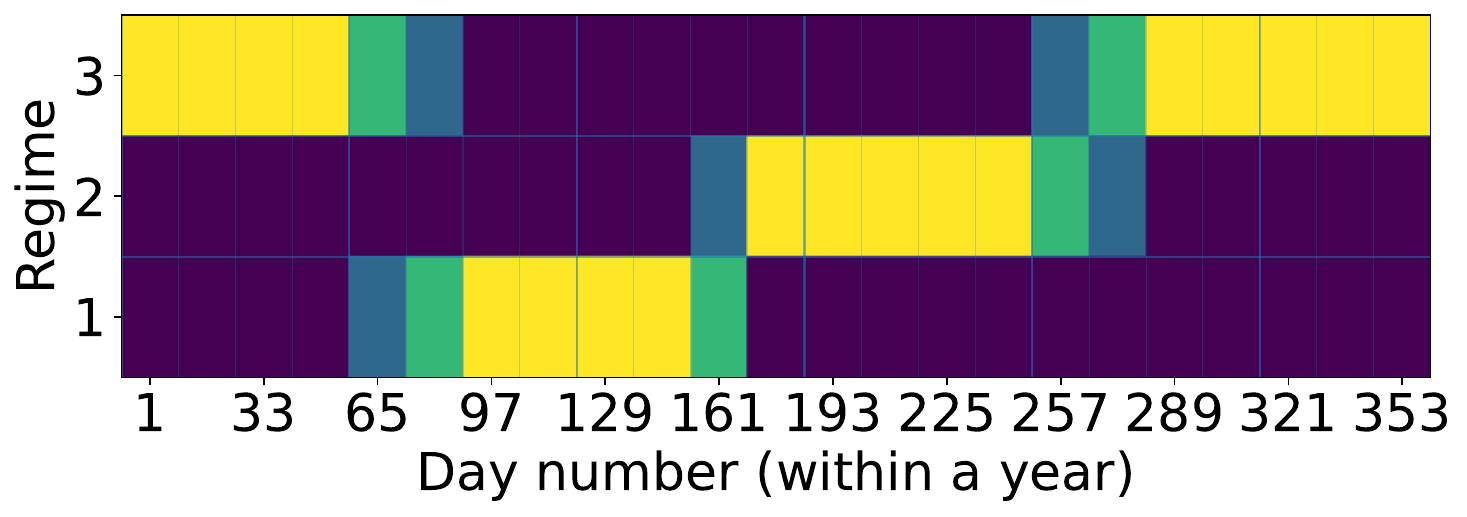}
    \vspace{-.65cm}
    \caption{Jan. 2001 -- Dec. 2004 (held-out).}
  \end{subfigure}
  \begin{subfigure}{0.525\textwidth}
    \includegraphics[width=\linewidth]{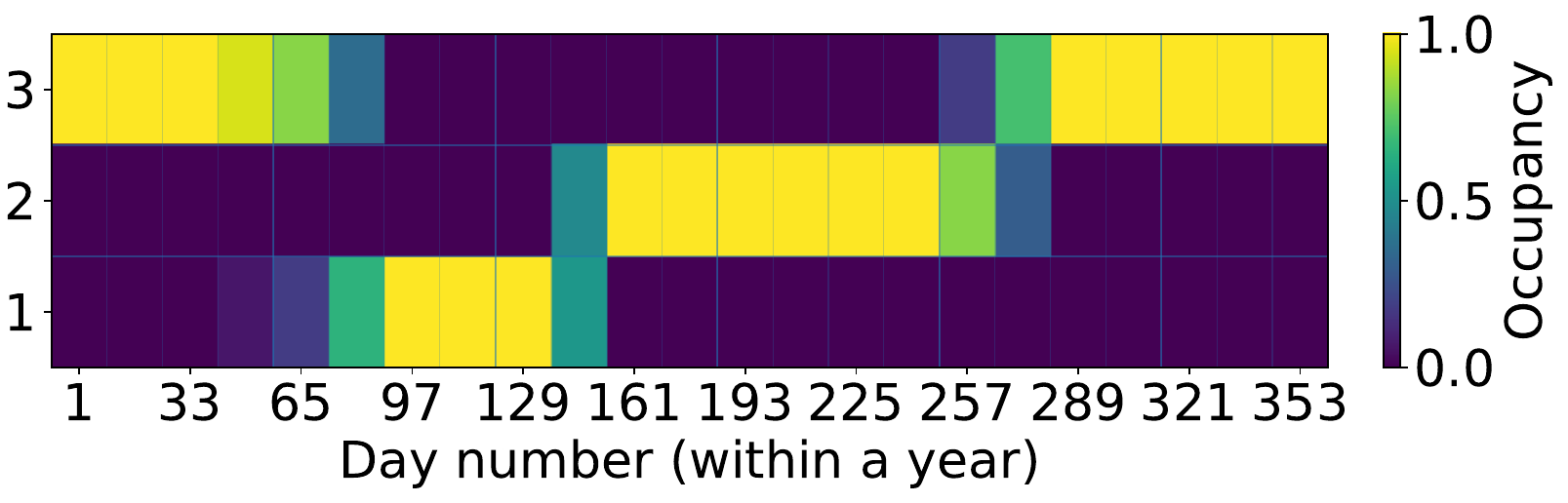}
    \vspace{-.65cm}
    \caption{Jan. 2005 -- Dec. 2020.}
  \end{subfigure}
  \vspace{-.25cm}
  \caption{Year-based structures from land pixels using iMSM.}
  \label{fig:iMSM_year_structure}
\end{figure}
\begin{figure}
  \begin{subfigure}{0.46\textwidth}
    \centering
    \includegraphics[width=\linewidth]{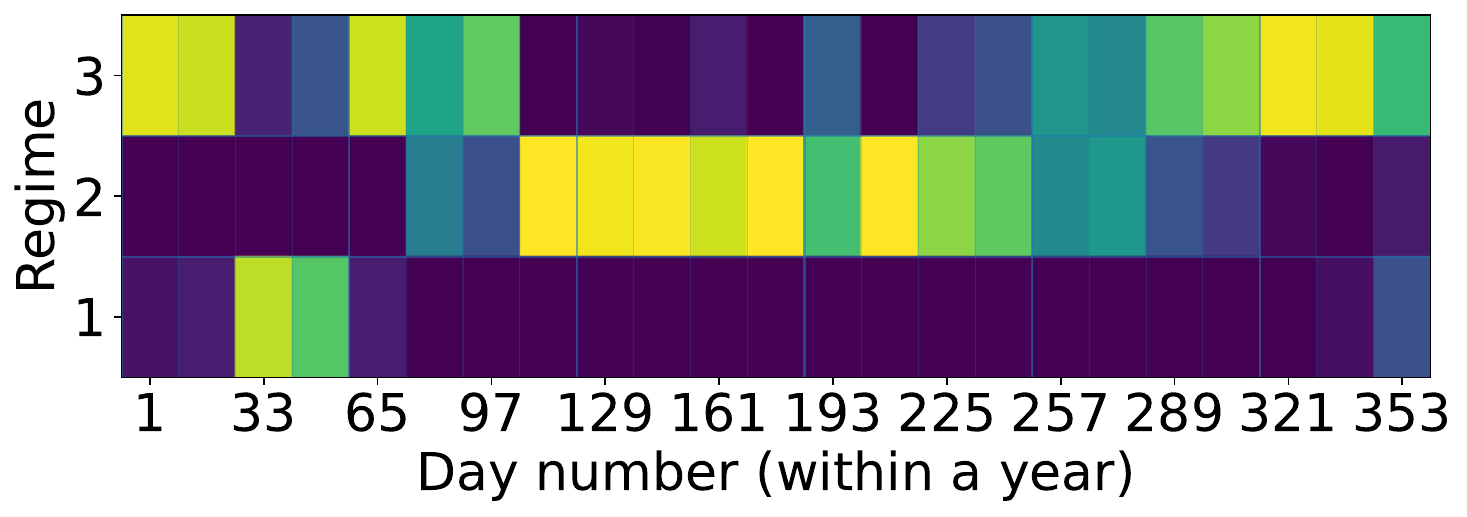}
    \vspace{-.65cm}
    \caption{Jan. 2001 -- Dec. 2004 (held-out).}
  \end{subfigure}
  \begin{subfigure}{0.525\textwidth}
    \includegraphics[width=\linewidth]{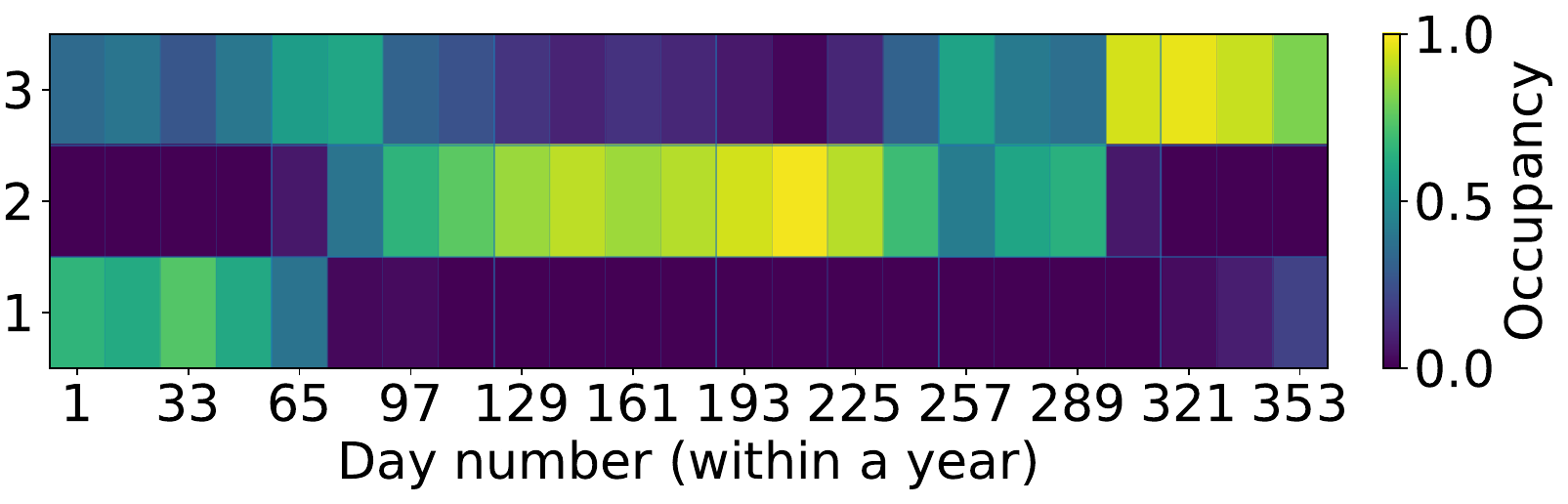}
    \vspace{-.65cm}
    \caption{Jan. 2005 -- Dec. 2020.}
  \end{subfigure}
  \vspace{-.25cm}
  \caption{Year-based structures on Ethiopia uncovered by iSDS.}
  \label{fig:iSDS_year_structure_ET}
\end{figure}
\begin{figure}
  \begin{subfigure}{0.46\textwidth}
    \centering
    \includegraphics[width=\linewidth]{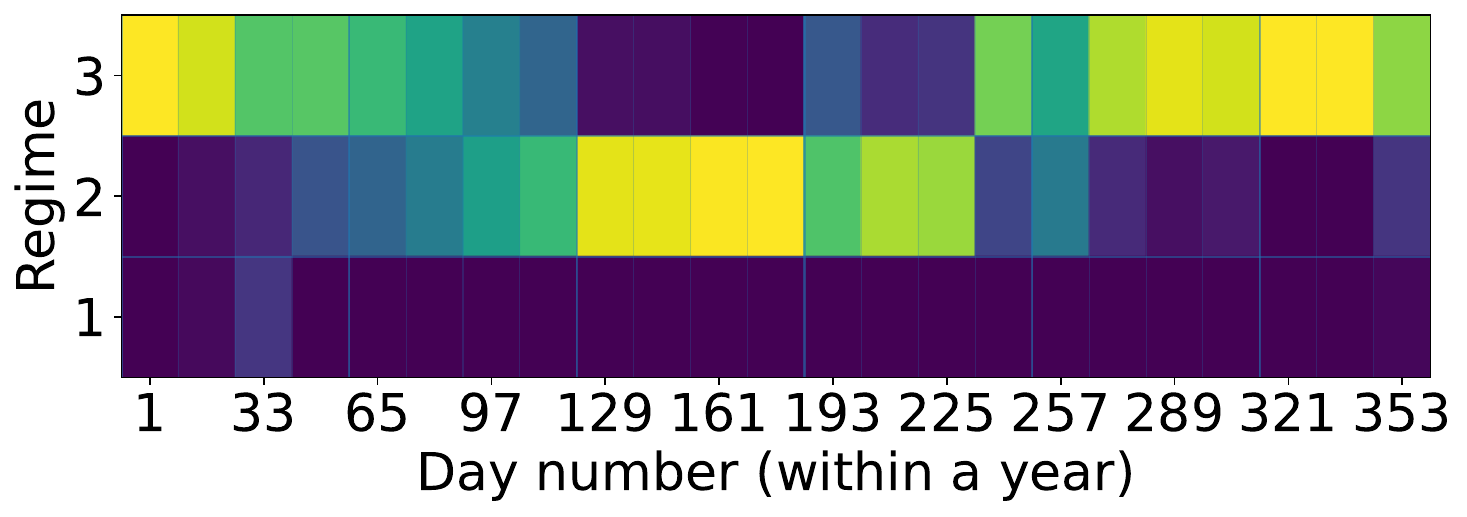}
    \vspace{-.65cm}
    \caption{Jan. 2001 -- Dec. 2004 (held-out).}
  \end{subfigure}
  \begin{subfigure}{0.525\textwidth}
    \includegraphics[width=\linewidth]{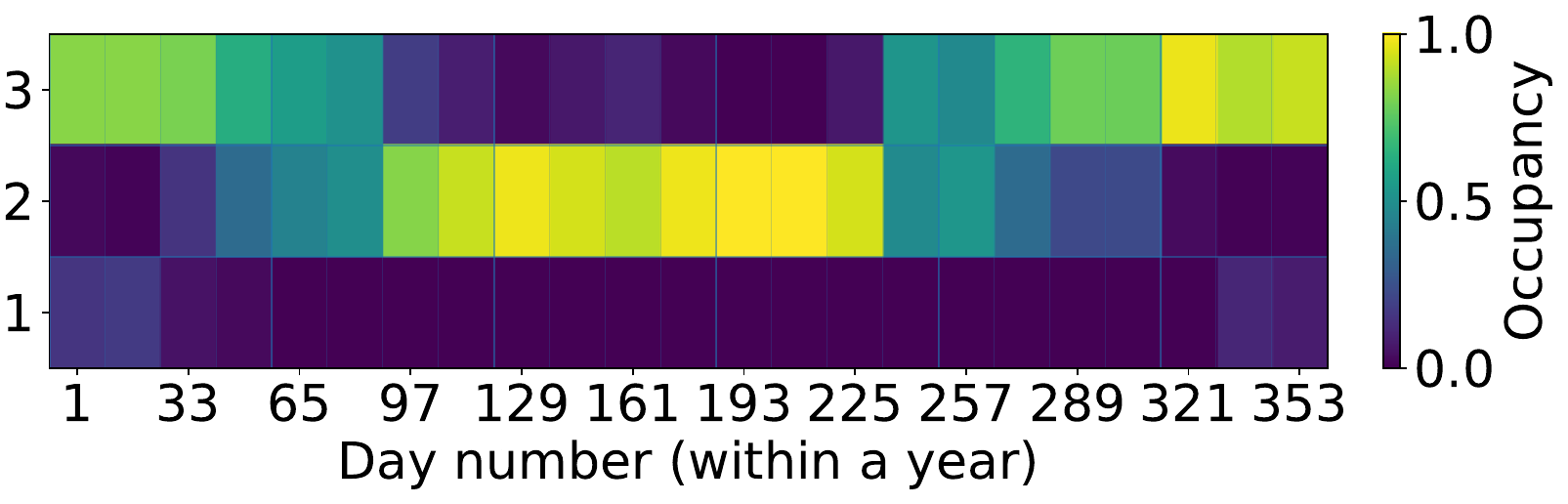}
    \vspace{-.65cm}
    \caption{Jan. 2005 -- Dec. 2020.}
  \end{subfigure}
  \vspace{-.25cm}
  \caption{Year-based structures on Sahel uncovered by iSDS.}
  \label{fig:iSDS_year_structure_Sahel}
\end{figure}
\begin{figure}
  \begin{subfigure}{0.46\textwidth}
    \centering
    \includegraphics[width=\linewidth]{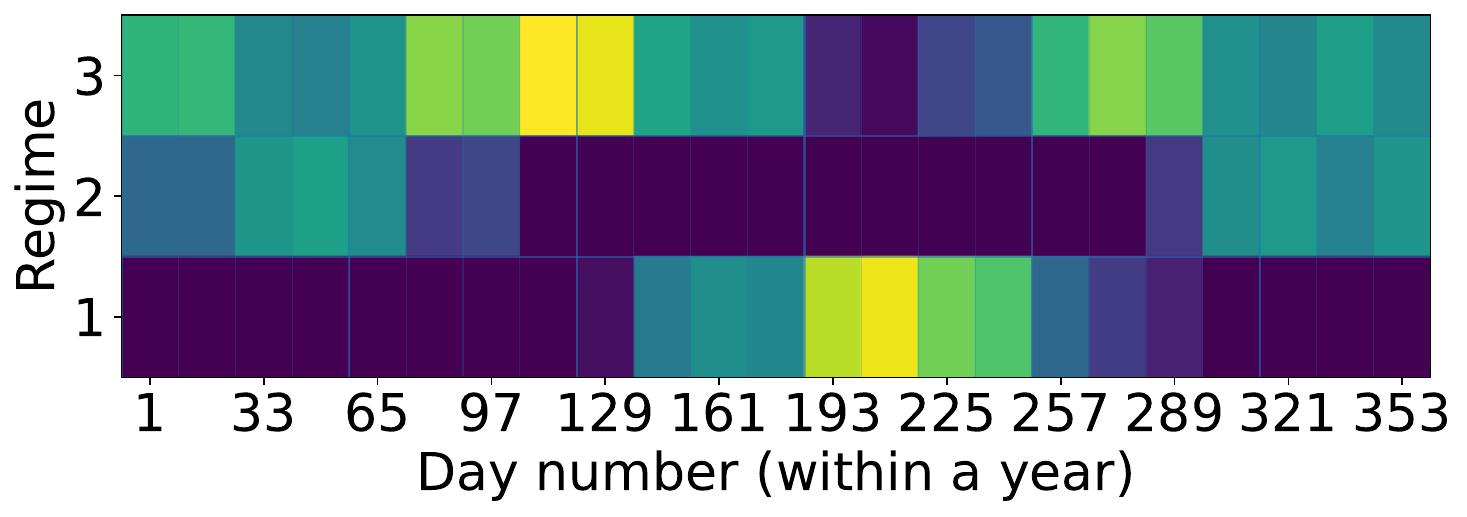}
    \vspace{-.65cm}
    \caption{Jan. 2001 -- Dec. 2004 (held-out).}
  \end{subfigure}
  \begin{subfigure}{0.525\textwidth}
    \includegraphics[width=\linewidth]{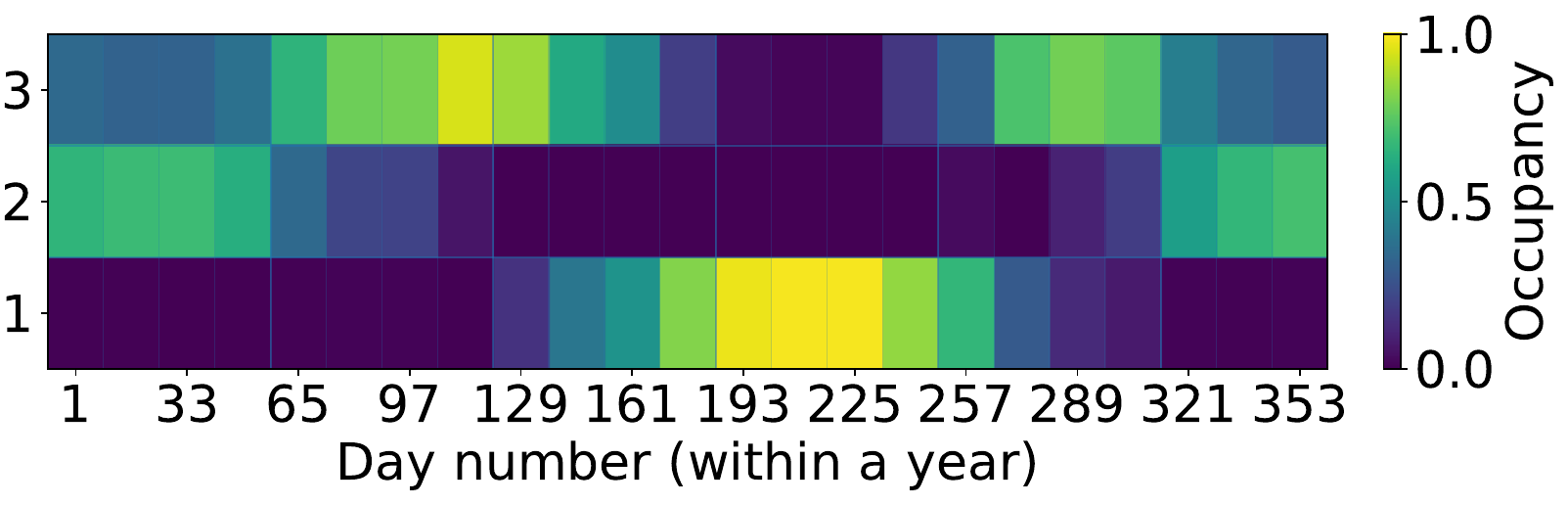}
    \vspace{-.65cm}
    \caption{Jan. 2005 -- Dec. 2020.}
  \end{subfigure}
  \vspace{-.25cm}
  \caption{Year-based structures on South Africa uncovered by iSDS.}
  \label{fig:iSDS_year_structure_SA}
\end{figure}

\begin{figure}
    \centering
    \begin{subfigure}{0.24\textwidth}
    \centering
    \includegraphics[width=\linewidth]{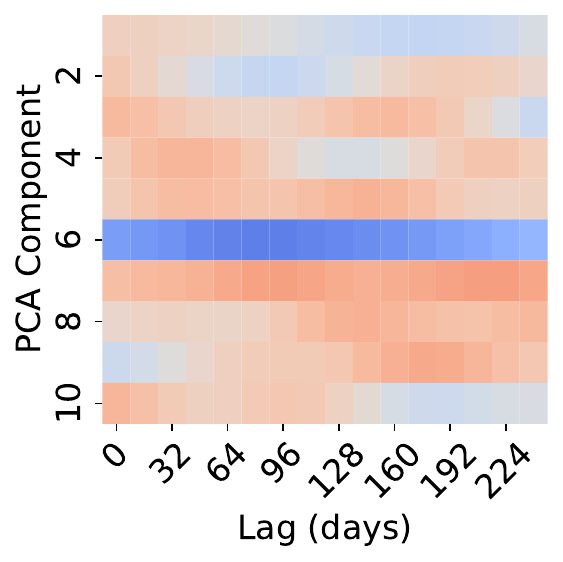}
    \vspace{-.75cm}
    \caption{All times.}
  \end{subfigure}
  \begin{subfigure}{0.22\textwidth}
    \includegraphics[width=\linewidth]{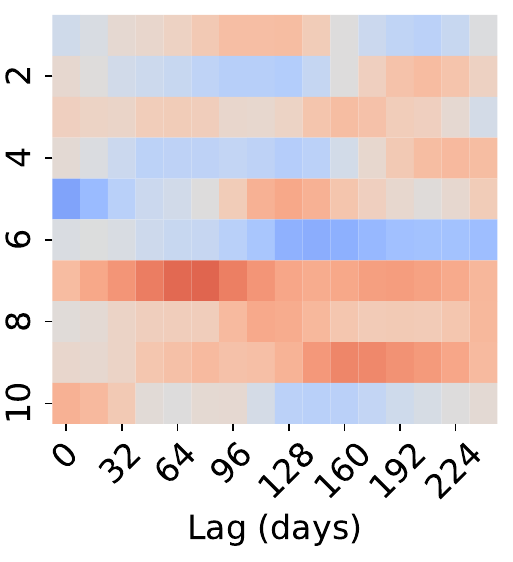}
    \vspace{-.75cm}
    \caption{Jun.--Sep.}
  \end{subfigure}
  \begin{subfigure}{0.22\textwidth}
    \centering
    \includegraphics[width=\linewidth]{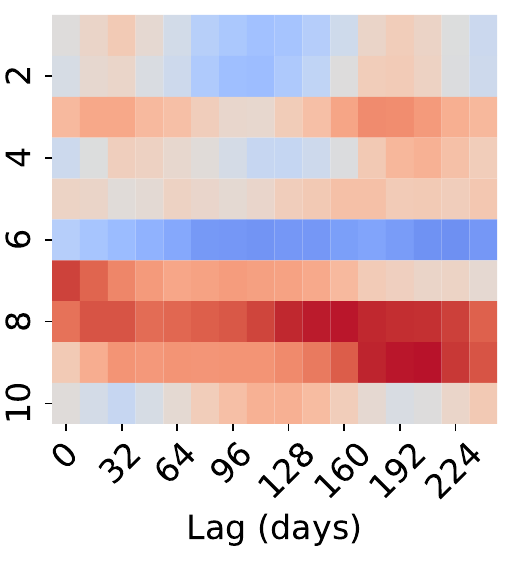}
    \vspace{-.75cm}
    \caption{Mar.--May.}
  \end{subfigure}
  \begin{subfigure}{0.2825\textwidth}
    \centering
    \includegraphics[width=\linewidth]{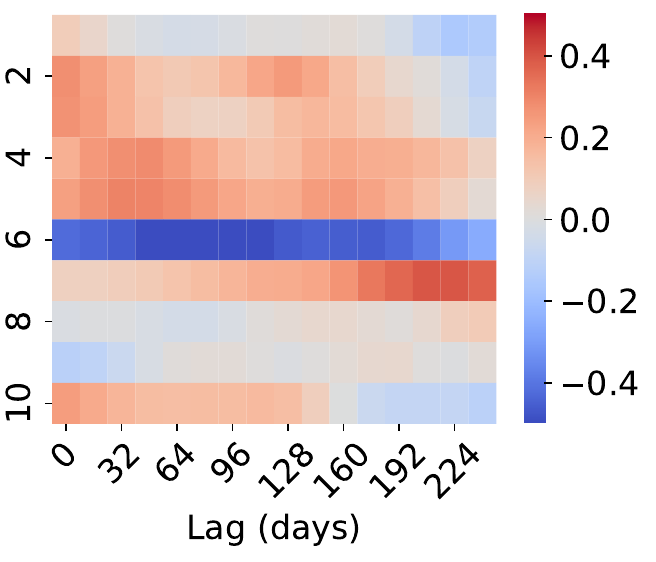}
    \vspace{-.75cm}
    \caption{Sep.--Feb.}
  \end{subfigure}
  \vspace{-.15cm}
  \caption{Correlation matrices of principal components from land pixels with respect to lagged ENSO (ENSO leads) by using (a) all available data, and (b-c) season-based data sorted according to iMSM regime posteriors.}
  \label{fig:iMSM_ndvi_enso_correlation}
\end{figure}

\paragraph{ENSO matrices.} Here we detail how Tables \ref{tab:imsm_season_table} and \ref{tab:isds_season_table} are computed. We use all available data and evaluate lagged correlations with ENSO on a $0$-$240$ day grid in $16$-day steps (ENSO leads). 
\begin{itemize}[leftmargin=1.5em]
  \setlength{\itemsep}{5pt}
  \setlength{\parskip}{0pt}
  \setlength{\parsep}{0pt}
    \item \textbf{iMSM:} For each principal component we form a correlation matrix with size $10\times16$ (components, lags) for ``all times'' and seasonal subsets obtained from iMSM regime posteriors. We visualise the iMSM correlation matrices in Figure~\ref{fig:iMSM_ndvi_enso_correlation}, and Table~\ref{tab:imsm_season_table} reports, for each slice, the peak absolute correlation and its lag (days).
    \item \textbf{iSDS:} For each region, we sample $10$ patches and sort the latent variables according to the iSDS regime posteriors. Then, for each season we compute similar per-patch correlation matrices of size $20\times16$ (latent variables, lags), and aggregate patches using Fisher $z$-averaging \citep{fisher1915frequency}. Table~\ref{tab:isds_season_table} reports peak correlations and lag (days) for each region and season, and we visualise the regional correlation matrices in Figure~\ref{fig:iSDS_ndvi_enso_correlation}.
\end{itemize}

\begin{figure}
    \centering
    \begin{subfigure}{0.2425\textwidth}
    \centering
    \includegraphics[width=\linewidth]{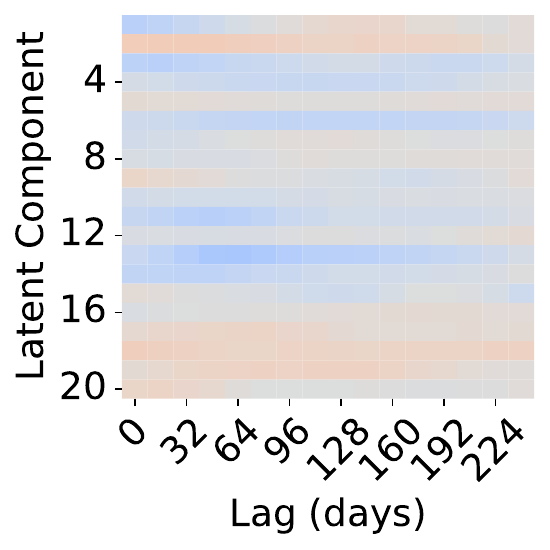}
    \vspace{-.75cm}
    \caption{Sahara, all times.}
  \end{subfigure}
  \begin{subfigure}{0.2175\textwidth}
    \includegraphics[width=\linewidth]{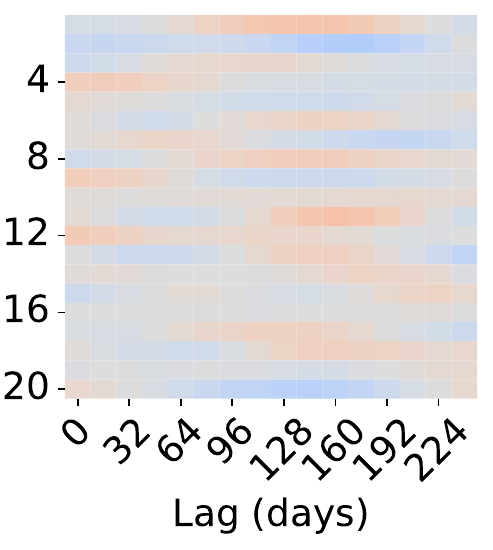}
    \vspace{-.75cm}
    \caption{Congo, all times.}
  \end{subfigure}
  \begin{subfigure}{0.2175\textwidth}
    \centering
    \includegraphics[width=\linewidth]{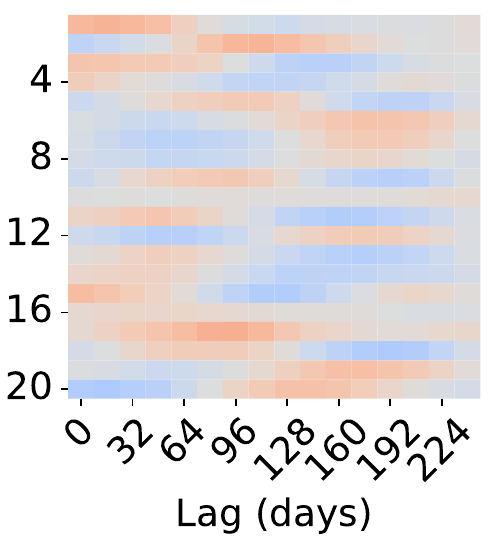}
    \vspace{-.75cm}
    \caption{Sahel, Mar--Aug.}
  \end{subfigure}
  \begin{subfigure}{0.2825\textwidth}
    \centering
    \includegraphics[width=\linewidth]{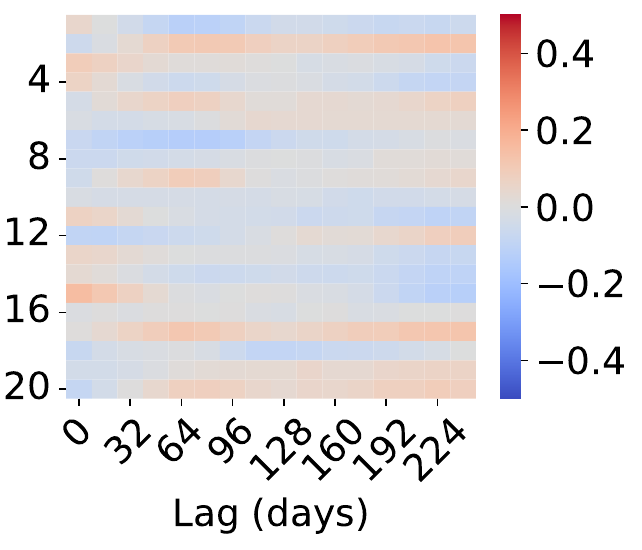}
    \vspace{-.75cm}
    \caption{Sahel, Sep--Feb.}
  \end{subfigure}

\vspace{.15cm}

  \begin{subfigure}{0.25\textwidth}
    \includegraphics[width=\linewidth]{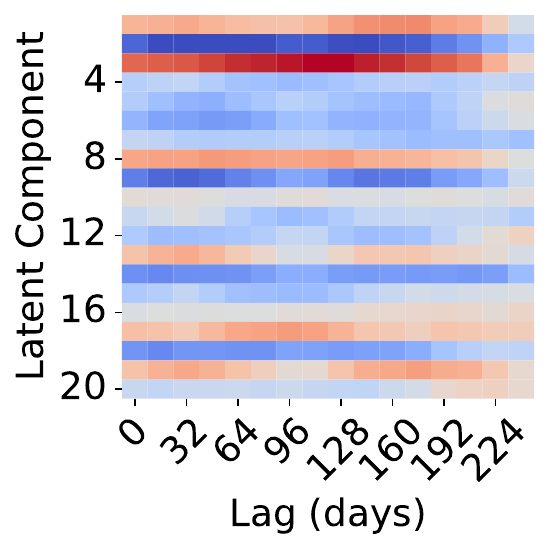}
    \vspace{-.75cm}
    \caption{Ethiopia, Jan--Mar.}
  \end{subfigure}
  \begin{subfigure}{0.225\textwidth}
    \centering
    \includegraphics[width=\linewidth]{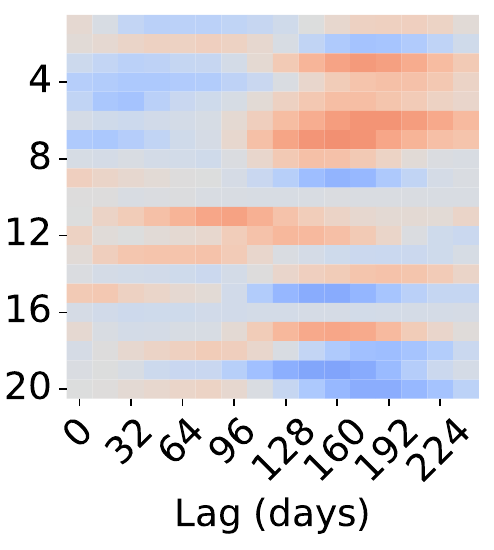}
    \vspace{-.75cm}
    \caption{Ethiopia, Apr--Aug.}
  \end{subfigure}
  \begin{subfigure}{0.29\textwidth}
    \centering
    \includegraphics[width=\linewidth]{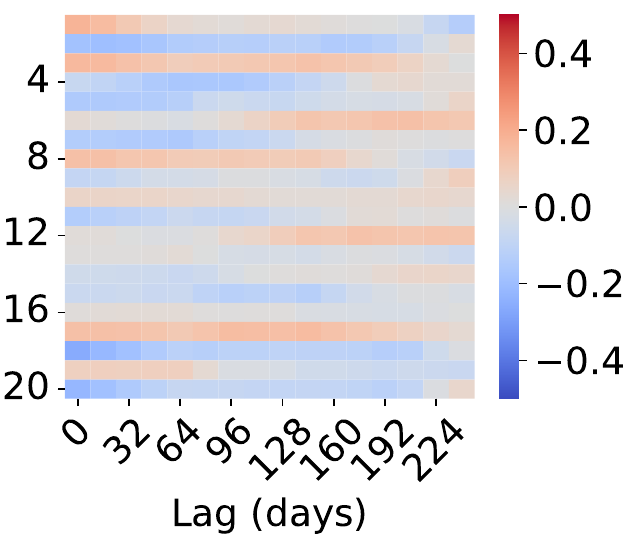}
    \vspace{-.75cm}
    \caption{Ethiopia, Aug--Dec.}
  \end{subfigure}

\vspace{.15cm}

  \begin{subfigure}{0.25\textwidth}
    \includegraphics[width=\linewidth]{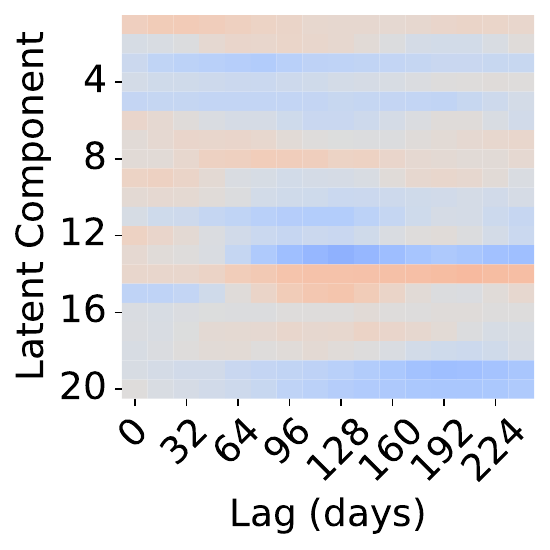}
    \vspace{-.75cm}
    \caption{S. Africa, Jun--Sep.}
  \end{subfigure}
  \begin{subfigure}{0.225\textwidth}
    \centering
    \includegraphics[width=\linewidth]{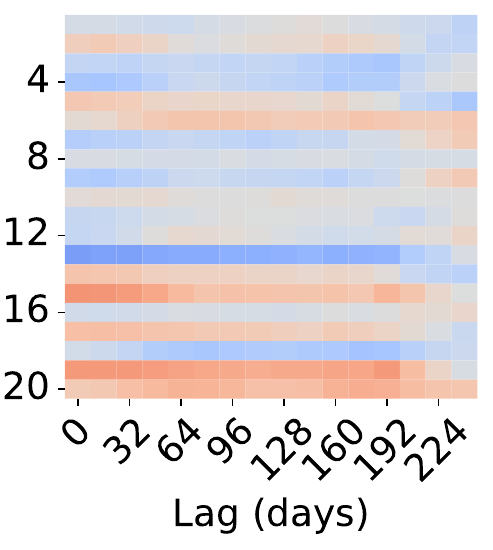}
    \vspace{-.75cm}
    \caption{S. Africa, Dec--Feb.}
  \end{subfigure}
  \begin{subfigure}{0.35\textwidth}
    \includegraphics[width=0.8285\linewidth]{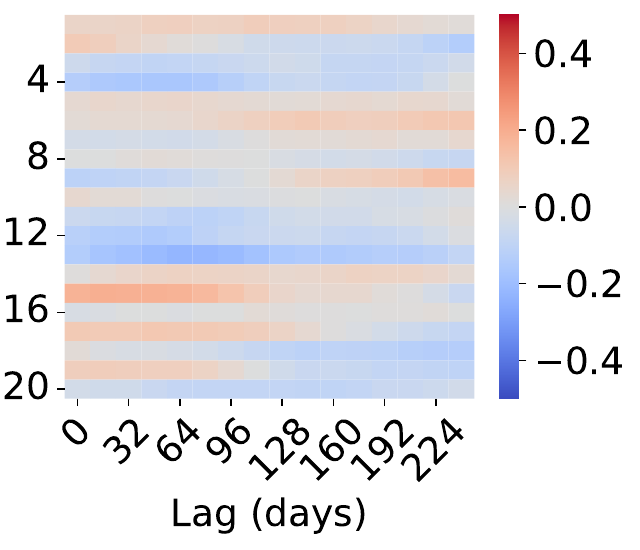}
    \vspace{-.25cm}
    \caption{S. Africa, Mar--Jun \& Sep-Nov.}
  \end{subfigure}
  \vspace{-.15cm}
  \caption{Correlation matrices of iSDS latent variables from different locations  with respect to lagged ENSO (ENSO leads) on (a) Sahara, (b) Congo, (c-d) Sahel, (e-g) Ethiopia, and (h-j) South Africa. For region-based correlation matrices, we sort the latent variables according to the corresponding regime posterior.}
  \label{fig:iSDS_ndvi_enso_correlation}
\end{figure}

\end{document}